\documentclass[10pt]{article} 
\usepackage[preprint]{tmlr}

\usepackage[T1]{fontenc}
\usepackage[normalem]{ulem}
\usepackage{tipa}
\usepackage{csquotes}
\usepackage[hidelinks]{hyperref}
\usepackage{url}
\hypersetup{    
	colorlinks = true,
	linkcolor  = black,
	urlcolor   = gray,
	citecolor  = blue
}
\usepackage[]{natbib}
\let\cite\citep
\usepackage{titletoc} 
\usepackage{tocloft}
\usepackage[acronym]{glossaries}
\usepackage{wasysym} 
\usepackage{graphicx}
\usepackage{subcaption}
\usepackage{wrapfig}
\usepackage{booktabs}
\usepackage{arydshln}
\usepackage{multirow}
\usepackage{amssymb}
\usepackage{esint}
\usepackage{amsthm}

\theoremstyle{definition}
\newtheorem{definition}{Definition}[section]
\newtheorem{assumption}[definition]{Assumption}
\newtheorem*{assumptionNoNb}{Assumption}

\theoremstyle{plain}
\newtheorem{theorem}[definition]{Theorem}
\newtheorem*{theoremNoNb}{Theorem}
\newtheorem{lemma}[definition]{Lemma}
\newtheorem*{lemmaNoNb}{Lemma}

\theoremstyle{remark}
\newtheorem{remark}[definition]{Remark}

\theoremstyle{plain}

\allowdisplaybreaks
\usepackage{bbold}
\renewcommand{\iff}{\Longleftrightarrow}
\usepackage{algorithm}
\usepackage{algpseudocode}

\newglossaryentry{formula}{name=formula,
                           description={A mathematical expression}}

\newacronym{wrt}{w.r.t.}{with respect to}
\newacronym{EU}{EU}{European Union}
\newacronym{iid}{iid}{independent and identically distributed}
\newacronym{iff}{iff}{if and only if}
\newacronym{wlog}{w.l.o.g.}{without loss of generality}
\newacronym{Wlog}{W.l.o.g.}{Without loss of generality}
\newacronym{SLE}{SLE}{system of linear equations}
\newacronym{SNLE}{SNLE}{system of non-linear equations}
% ...
\newacronym{CDF}{CDF}{cumulative distribution function}
% ...
\newacronym{OP}{OP}{optimization problem}
\newacronym{AI}{AI}{artificial intelligence}
\newacronym{XAI}{XAI}{explainable artificial intelligence}
\newacronym{ML}{ML}{machine learning}
\newacronym{DL}{DL}{deep learning}
\newacronym{GDL}{GDL}{geometric deep learning}
\newacronym{SOTA}{SOTA}{state of the art}
\newacronym{TP}{TP}{true positives}
\newacronym{FP}{FP}{false positives}
\newacronym{FN}{FN}{false negatives}
\newacronym{TN}{TN}{true negatives}
\newacronym{ACC}{ACC}{accuracy}                     
\newacronym{TPR}{TPR}{true positive rate}           
\newacronym{FPR}{FPR}{false positive rate}          
\newacronym{FNR}{FNR}{false negative rate}          
\newacronym{TNR}{TNR}{true negative rate}           
\newacronym{PPV}{PPV}{positive predictive value}    
\newacronym{FDR}{FDR}{false discorvery rate}        
\newacronym{FOR}{FOR}{false omission rate}          
\newacronym{NPV}{NPV}{negative predictive value}    
\newacronym{ROC}{ROC-curve}{receiver operating characteristic curve}
\newacronym{AUC}{AUC}{area under the (ROC) curve}
\newacronym{IG}{IG}{information gain}
\newacronym{MSE}{MSE}{mean squared error}
\newacronym{AE}{AE}{absolute error}
\newacronym{MAE}{MAE}{mean absolute error}
\newacronym{MRAE}{MRAE}{mean relative absolute error}
\newacronym{AC}{AC}{absolute cost}
\newacronym{MAC}{MAC}{mean absolute cost}
\newacronym{DI}{DI}{disparate impact}
\newacronym{DP}{DP}{demographic parity}
\newacronym{EOs}{EOs}{equalized odds}
\newacronym{EO}{EO}{equal opportunity}
\newacronym{SVM}{SVM}{support vector machine}
\newacronym{MLP}{MLP}{multi layer perceptron}
\newacronym{LLM}{LLM}{large language model}
\newacronym{NN}{NN}{neural network}
\newacronym{PINN}{PINN}{physics-informed neural network}
\newacronym{GNN}{GNN}{graph neural network}
\newacronym{MPNN}{MPNN}{message passing neural network}
\newacronym{GCN}{GCN}{graph convolutional network}
\newacronym{PI-GNN}{PI-GNN}{physics-informed graph neural network}
\newacronym{PI-GCN}{PI-GCN}{physics-informed graph convolutional network}
\newacronym{ADR-GNN}{ADR-GNN}{advection-diffusion-reaction graph neural network}
\newacronym{QAMP}{QAMP}{quantum advection message passing}
\newacronym{MeGA-MP}{MeGA-MP}{metric graph advection message passing}
% ...
\newacronym{WL}{WL}{Weisfeiler-Lehmann}
\newacronym{WDN}{WDN}{water distribution network}
\newacronym{WDS}{WDS}{water distribution system}
\newacronym{ERC}{ERC}{European Research Council}
\newacronym{PI-Algo}{PI-algorithm for hydraulic state reconstruction}{physics-informed algorithm for hydraulic state reconstruction}
\newacronym{ODE}{ODE}{ordinary differential equation}
\newacronym{PDE}{PDE}{partial differential equation}
\newacronym{ADR}{ADR}{advection-diffusion-reaction}
\newacronym{ADR-PDE}{ADR-PDE}{advection-diffusion-reaction partial differential equation}
\newacronym{MoC}{MoC}{Method of Characteristics}

\newcommand{\n}[1]{{n_{\text{#1}}}}                     
\newcommand{\kn}[1]{{k_{\text{#1}}}}                    
\newcommand{\tn}[1]{{t_{\text{#1}}}}                    
\newcommand{\N}{\mathbb{N}}                             
\newcommand{\R}{\mathbb{R}}                             
\newcommand{\An}[1]{{A_{\text{#1}}}}                    
\newcommand{\Tn}[1]{{T_{\text{#1}}}}                    
% ...
\newcommand{\Xm}{\mathbf{X}}                            
\newcommand{\xm}{\mathbf{x}}                            
\newcommand{\ym}{\mathbf{y}}                            
\newcommand{\ymhat}{\hat{\ym}}                          
% ...
\newcommand{\vhat}{\hat{v}}                             
\newcommand{\uhat}{\hat{u}}                             
\newcommand{\G}{\mathcal{G}}                            
\newcommand{\T}{\mathcal{T}}                            
\newcommand{\That}{\hat{\T}}                            
\newcommand{\Nbh}{\mathcal{N}}                          
\newcommand{\Nbhin}{{\Nbh_{\text{in}}}}                 
\newcommand{\Nbhout}{{\Nbh_{\text{out}}}}               
\newcommand{\Children}{\mathcal{C}}                     
\newcommand{\ptilde}{{\tilde{p}}}                       
\newcommand{\phat}{{\hat{p}}}                           
\newcommand{\Pa}{\mathcal{P}}                           
\newcommand{\Pahat}{\hat{\Pa}}                          
\newcommand{\qoverline}{\overline{q}}                   
\newcommand{\nuoverline}{\overline{\nu}}                
\newcommand{\chat}{\hat{c}}                             
\newcommand{\ztilde}{\tilde{z}}                         
\newcommand{\zcurl}{\text{\textctz}}                    
\newcommand{\scurl}{\text{\textrtails}}                 
\newcommand{\that}{{\hat{t}}}                           
\newcommand{\dm}{\mathbf{d}}                            
\newcommand{\qm}{\mathbf{q}}                            
\newcommand{\num}{\boldsymbol{\nu}}                     
\newcommand{\cm}{\mathbf{c}}                            
\newcommand{\cmhat}{\hat{\cm}}                          
\newcommand{\zm}{\mathbf{z}}                            
\newcommand{\lm}{\mathbf{l}}                            
\newcommand{\Vr}{{V_{\text{r}}}}                        
\newcommand{\Vc}{{V_{\text{c}}}}                        
\newcommand{\Vb}{{V_{\text{b}}}}                        
% ...

% ...
\DeclareMathOperator{\sgn}{sgn}                         
\DeclareMathOperator{\im}{im}                           
\DeclareMathOperator{\divtext}{div}                     
% ...

% ...

% ...
\DeclareMathOperator{\relu}{ReLU}                       
% ...

\newacronym{FDM}{FDM}{finite difference method}
\newacronym{IVP}{IVP}{initial value problem}
\newacronym{BVP}{BVP}{boundary value problem}

\definecolor{pass-through}{HTML}{00346F}
\definecolor{inverse-pass-through}{HTML}{4F7CB0}
\definecolor{self-loop}{HTML}{C95100}
\definecolor{inverse-self-loop}{HTML}{FC9432}
\definecolor{new}{HTML}{254ADB}

\title{MeGA-MP: Metric Graph Advection Message Passing \\
    \normalsize
    A Physics-Informed Message Passing Operator for Advection-Dominated Metric Graphs
    }

\author{\name Janine Strotherm\thanks{Authors contributed equally.}
    \email jstrotherm@techfak.uni-bielefeld.de \\
    \addr Bielefeld University
    \AND
    \name Luca Hermes$^{*}$
    \email lhermes@techfak.uni-bielefeld.de \\
    \addr Bielefeld University
    \AND
    \name André Artelt 
    \email aartelt@techfak.uni-bielefeld.de \\
    \addr Bielefeld University
    \AND
    \name Barbara Hammer
    \email bhammer@techfak.uni-bielefeld.de \\
    \addr Bielefeld University
}

\def\month{MM}  
\def\year{YYYY} 
\def\openreview{\url{https://openreview.net/forum?id=XXXX}} 

\begin{document}

\maketitle

\begin{abstract}
    Many real-world systems are organized as networks where spatio-temporal dynamics unfold along connections and not discretely between nodes. 
    Examples include utility networks such as water distribution systems or gas networks, electrical grids, and traffic flow networks.
    Such systems are naturally modeled as \textit{metric graphs}, where edges correspond to one-dimensional Euclidean subspaces connected at vertices.
    Metric graphs are independent of an underlying global Euclidean space, limiting direct application of typical PINNs and operator-learning methods. 
    Especially transport dynamics like advection require a methodology able to capture antisymmetric and long-range dependencies on graphs, which is itself a challenge.
    We propose a novel physics-informed message passing operator that encodes linear advection on metric graphs as an inductive bias.
    In the purely advective setting, the operator provably recovers the exact dynamics up to a theoretically derived discretization error without any training.
    Combined with trainable components like MLPs, our message passing operator extends to realistic advection-reaction dynamics in water distribution systems, where we achieve superior performance compared to baselines and zero-shot generalization across different graph topologies.
\end{abstract}

\section{Introduction}
In the last years, physics-informed \glspl*{GNN} have shown impressive zero-shot generalization capabilities when addressing dynamical systems on conventional graphs \cite{Thangamuthu2022PI-GNNs_Methods_Benchmark}. 
However, many real-world systems do not exhibit discrete node-to-node dynamics. Instead, dynamics unfold along connections between nodes.
Examples include quantum dynamics, water distribution systems, gas networks, electrical grids, and traffic flow networks \cite{Böttcher2024PI-GNNs_Methods_DynmaicalProcessesOnMetricNetworks,Rossman2020WaterDomain_StateSimulations_Epanet2.2}.
While a common approach is to discretize the domain and recover a conventional graph, \textit{metric graphs} naturally capture the geometry of these systems. 
In a metric graph, each edge is associated with a one-dimensional domain -- an interval of edge-dependent length -- governed by edge dynamics, expressed as an \gls*{ODE} or, more commonly, a \gls*{PDE}. In combination with suitable coupling conditions at nodes, such as continuity or conservation of mass or flux, this defines a coupled system of \glspl*{PDE} \cite{Böttcher2024PI-GNNs_Methods_DynmaicalProcessesOnMetricNetworks,Blechschmidt2025MGNOs_NormalMGNOs_Solution-basedArchitechtures_DeepONetsForDriftDiffusionOnMetricGraphs}.

We identify \textit{advection}, i.e., the transport of a substance through the presence of a flow field, as the main driver for common edge dynamics on a variety of application domains, such as
utility networks like
water distribution systems \cite{Rossman2020WaterDomain_StateSimulations_Epanet2.2}
and gas networks \cite{Gugat2022MetricGraphs_NumericalApproaches_GasNetworks}, 
traffic flow networks \cite{Gugat2005MetricGraphs_NumericApproaches_TrafficFlowNetworks},
telecommunication networks,
and
blood flow through vascular systems \cite{Bressan2014MetricGraphs_NumericalArchitechtures_FlowNetworks}.
From the \gls*{ML} perspective, an intuitive approach is to model such dynamics with \glspl*{GNN}.
However, spatio-temporal advection-dominated dynamics on metric graphs pose well-known challenges in the context of \glspl*{GNN}. 
First, the directional nature of transport requires antisymmetric message passing. 
Second, spatial dependencies can become global over time, requiring long-range spatial integration. 
Furthermore, sharp gradients and shocks are preserved and propagated rather than smoothed out as in diffusion-dominated settings. 
These high-frequency signals cannot be captured by \glspl*{GNN} with inherent over-smoothing characteristics.

In this work, we make use of a suitable inductive bias for linear advection to overcome these challenges and model the dynamics accurately on node level. 
More precisely, we avoid computationally expensive sampling of function evaluations along each edge space by solving the linear advection \gls*{PDE} using the \gls*{MoC} and integrating the results in an iterative message passing architecture.
This enables an efficient and exact realization of the linear advection dynamics while maintaining end-to-end adaptability if observational data is available.

\paragraph{Contributions}
We propose \textit{\gls*{MeGA-MP}}, a \gls*{MPNN} framework that provably models linear advection on a metric graph iteratively without requiring learning.
More precisely, we rigorously derive \gls*{MeGA-MP} directly from the underlying physical priors and provide a theoretical bound on the discretization error (Section \ref{section_MethodMetricGraphAdvectionMessagePassing}).
We show how \gls*{MeGA-MP} can be used as a numerical solver for the spatio-temporal node forecasting task on a large-scale and realistic metric graph with advection edge dynamics and demonstrate its overall power by extending it with learning components to solve advection-reaction edge dynamics.
Moreover, we investigate the effect of both the purely physics-driven and the learnable components of the model in two ablation studies (Section \ref{section_Experiment1_Zero-ShotTransferLearningOfAdvection-ReactionDynamicsOnRealWorldMetricGraphs}).
To benchmark \gls*{MeGA-MP} against classical solvers, we use the 1-dimensional Euclidean domain as a special case of a metric graph (Section \ref{subsection_Experiments_AdvectionOnA1DEuclideanDomain}).
A browser-version of our model with live control over inflow conditions is available \href{https://anonymous.4open.science/w/tmp-preprint-BB4F/}{online}\footnote{https://anonymous.4open.science/w/tmp-preprint-BB4F/}.

\paragraph{Related Work}
\gls*{MeGA-MP} can serve as a standalone numerical solver for linear advection, expressed as a physics-informed \gls*{MPNN}-architecture. This formulation also allows a straight-forward extension with learnable components into a trainable \gls*{MPNN} to solve more complex, advection-dominated problems on metric graphs. This positions the work at the intersection of numerical \gls*{PDE} solvers for metric graphs, physics-informed \gls*{ML}, and \gls*{ML} for metric graphs. In this order, we will now discuss relevant existing approaches.

Numerical methods for metric graphs largely are specifically designed for quantum graphs, i.e., metric graphs with edge dynamics and coupling conditions associated with the self-adjoint Laplace operator \cite{Arioli2018MetricGraphs_NumericalArchitechtures_QuantumGraphs,Brio2022MetricGraphs_NumericalArchitechtures_QuantumGraphs,Böttcher2024PI-GNNs_Methods_DynmaicalProcessesOnMetricNetworks}. Such works benefit from the rich existing theory on the Laplace operator \cite{Berkolaiko2017MetricGraphs_Theory_QuantumGraphs,Berkolaiko2013MetricGraphTheory_QuantumGraphs,Kuchment2003MetricGraphs_Theory_QuantumGraphs,Kuchment2005MetricGraphs_Theory_QuantumGraphs,Kuchment2008MetricGraphs_Theory_QuantumGraphs}. In particular, the Laplace operator on a metric graph with continuity and Kirchhoff-Neumann coupling conditions has an orthogonal eigenbasis, which can be utilized in different ways, e.g., by eigenfunction-based spectral solutions \cite{Arioli2018MetricGraphs_NumericalArchitechtures_QuantumGraphs,Brio2022MetricGraphs_NumericalArchitechtures_QuantumGraphs,Böttcher2024PI-GNNs_Methods_DynmaicalProcessesOnMetricNetworks}. 
Besides quantum graphs, numeric approaches also exist for specific application domains, such as water distribution systems \cite{Rossman2020WaterDomain_StateSimulations_Epanet2.2,Shang2023WaterDomain_StateSimulations_EpanetMSX2.0}, gas networks \cite{Gugat2022MetricGraphs_NumericalApproaches_GasNetworks} or traffic networks \cite{Bressan2014MetricGraphs_NumericalArchitechtures_FlowNetworks,Gugat2005MetricGraphs_NumericApproaches_TrafficFlowNetworks}.
For a broader overview of application domains and related numeric approaches, we refer to \citet{Böttcher2024PI-GNNs_Methods_DynmaicalProcessesOnMetricNetworks}. 
However, numerical approaches typically require discretization of the continuous edge domains, resulting in large numbers of degrees of freedom and increasing computational and memory costs as the spatial resolution grows. In addition, numerical solvers can become computationally demanding when high spatial or temporal resolution is required. 
These challenges have motivated the development of \gls*{ML}-based approaches for solving dynamical systems.

To this end, \citet{Raissi2019PINNs_Methods_PINN} introduced \glspl*{PINN}, a framework for \glspl*{NN} that learns the unknown solution of a given \gls*{PDE} on a continuous time and space domain. 
Since \gls*{PDE} solutions depend on configurations like initial conditions, boundary conditions, and control signals, \glspl*{PINN} typically learn the solution to one such configuration.
Neural operators \cite{Kovachki2023PI-GNNs_Methods_NeuralOperator} address this limitation by learning a map from functions that configure a \gls*{PDE} to its solution. 
Prominent examples include DeepONet by \citet{Lu2019PI-GNNs_Methods_DeepONet,Lu2021PI-GNNs_Methods_DeepONet} and a variation of neural operators by \citet{Li2020PI-GNNs_Methods_FourierNeuralOperator, Li2020PI-GNNs_Methods_NeuralOperator, Li2020PI-GNNs_Methods_MultipoleGraphNeuralOperator,  Li2024PI-GNNs_Methods_PI-NeuralOperator}.
%
%
The majority of these methods are developed for Euclidean domains and do not directly transfer to metric graphs. 

This issue was recently addressed in different works that apply domain decomposition techniques. This transforms the learning problem into solving dynamics on multiple Euclidean domains -- the edge spaces -- while constraining correct coupling at nodes.
\citet{Blechschmidt2022MGNNs_MetricGraphNeuralODEandPDESolver_Solution-basedArchitechtures_PINNsForDriftDiffusionOnMetricGraph} use edge-specific \glspl*{PINN} in such a setup to solve drift-diffusion equations on metric graphs. Since \gls*{PDE} parametrizations are -- in line with typical \gls*{PINN} setups -- not a model input, the approach cannot generalize to changing conditions.
\citet{Laczko2025MGNNs_MetricGraphNeuralODEandPDESolver_Solution-basedArchitechtures_PINNsForSchrödingerOnQuantumGraphs} apply \glspl*{PINN} in a Richardson iterative scheme to solve the time-independent Schrödinger equation on metric graphs. This allows the graph structure to be excluded from model training, making the approach graph agnostic. However, generalization to changing \gls*{PDE} parameters is limited.
\citet{Blechschmidt2025MGNOs_NormalMGNOs_Solution-basedArchitechtures_DeepONetsForDriftDiffusionOnMetricGraphs} address this by utilizing DeepONets and building a \textit{lego}-like graph-agnostic architecture to solve drift-diffusion equations on metric graphs. However, this framework requires a node-level optimization procedure to determine interface boundary values for each new problem instance, introducing additional computational overhead at graph interfaces.
Moreover, all edge dynamics considered in these works are diffusion-dominated and thus accompanied by the aforementioned Kirchhoff-Neumann conditions -- a coupling condition fundamentally different to the ones required for advection-dominated edge dynamics.

An orthogonal line of work has shown that the discretization of \glspl*{PDE} in time and space yield powerful physics-motivated \glspl*{MPNN} suitable for conventional tasks on graphs, such as node classification \cite{Chamberlain2021PI-GNNs_Methods_GRAND,Eliasof2021PI-GNNs_Methods_PDE-GCN,Rusch2022PI-GNNs_Methods_GraphCON,Choi2023PI-GNNs_Methods_GREAD,Eliasof2024PI-GNNs_Methods_ADR-GNN}, 
but they do not directly address the continuous edge domains characteristic of metric graphs.
The reason is the difference in geometry between a conventional graph with discrete node-to-node relations vs. a metric graph with continuous edge spaces.
See Appendix \ref{section_Background_Appendix} for a detailed discussion, which we recommend reading after the next section.

The existing \gls*{ML}-based methods discussed above either learn solutions for individual \gls*{PDE} instances (PINNs), rely on interface optimization procedures (DeepONet-based approaches), or are designed for conventional graphs without continuous edge domains. In contrast, \gls*{MeGA-MP} directly embeds the numerical structure of linear advection on metric graphs into an \gls*{MPNN} architecture, providing a graph-native framework that can be extended with learnable components while preserving the underlying physical coupling mechanisms.

\section{Background}
\label{section_Background}

The common use case for graphs in dynamical system modeling is as a means of discretizing a continuous Euclidean domain $\Omega \subset \R^n$, with nodes representing spatial locations and edges encoding proximity in $\Omega$. 
For example, a grid is a graph that approximates a multidimensional Euclidean space, preserves the structure of the continuum, and is commonly used when solving PDEs with numerical schemes \cite{Brandtstetter2023GNOs_NormalGNOs_Derivative-basedArchitechtures_MP-PDESolver}.
The same concept has also been adopted to formulate physics-motivated \gls*{MPNN} architectures such as ADR-GNN by \citet{Eliasof2024PI-GNNs_Methods_ADR-GNN}.
Such a setup is not practical in cases where dynamical systems govern the evolution of signals in networks such as water distribution systems, gas networks, electrical grids or traffic flow networks. Instead, \textit{metric graphs} can be employed, as they naturally capture the topology of the original domain.
In particular, the edges $e_{uv} \in E$ of a metric graph $\G = (V,E)$ are parametrized by lengths $l_{uv} > 0$ and, potentially, capacities $\alpha_{uv} > 0$, and are associated with one-dimensional Euclidean subdomains $\Omega_{uv} = (0,l_{uv})$ aligned with the edge direction. 
Here, the graph $\G$ describes the system $\Omega := \sqcup_{e_{uv} \in E} \Omega_{uv}$ of such subdomains coupled at nodes $V$, which correspond to the boundaries $\partial \Omega_{uv} = \{0,l_{uv}\}$ of the edge spaces $\Omega_{uv}$.
We visualize the difference between metric graphs and graphs that discretize an underlying continuous space using the example of advection in Figure \ref{figure_MetricGraphVsNonMetricGraph} in Appendix \ref{section_Background_Appendix}.

\paragraph{Dynamical Systems on Metric Graphs}
In this work, we focus on metric graphs with associated dynamics that govern the spatio-temporal evolution of signals along the graph.
We model this system as a finite, connected, and symmetric directed graph $\G = (V,E)$ together with the following properties which link signals $c_v: \R \rightarrow \R_{\geq 0}, t \mapsto c_v(t)$ and  $c_{uv}: \R \times \Omega_{uv} \rightarrow \R_{\geq 0}, (t,z) \mapsto c_{uv}(t,z)$ over time $\R$, at each node $v \in V$ and along the edge space $\Omega_{uv} = (0,l_{uv})$ of each edge $e_{uv} \in E$, respectively.

\begin{enumerate}
    \item The \textit{edge dynamics} are usually governed by a \gls*{PDE} that describes the change of the signal $c_{uv}$ over time $\R$ and along the edge space $\Omega_{uv}$ of an edge $e_{uv} \in E$, that is,
    \begin{align}
    \label{align_EdgeDynamics}
        \frac{\partial c_{uv}(t,z)}{\partial t}
        =
        f(t,z,c_{uv},\nu_{uv})
    \end{align}
    \normalsize
    \noindent 
    for all $t \in \R$, $z \in \Omega_{uv} = (0,l_{uv})$ and known functions or coefficients $\nu_{uv}$ that parametrize the edge-wise \gls*{PDE} function $f$.
    
    \item The \textit{coupling conditions} define how signals propagate through nodes in between adjacent edge spaces. Specifically, a node $v \in V$ can be understood as the boundary $z = l_{uv}$ or $z = 0$ of multiple adjacent edge spaces $\Omega_{uv}$ or $\Omega_{vu}$ aligned with the direction of the corresponding incoming or outgoing edge $e_{uv} \in E$ or $e_{vu} \in E$, respectively.
\end{enumerate}

\paragraph{Advection on Metric Graphs}
We model the \textit{edge dynamics}, i.e., Equation \eqref{align_EdgeDynamics}, as the linear advection \gls*{PDE}
\begin{align}
\label{align_EdgeDynamics_Advection}
    \frac{
    \partial c_{uv}(t,z)
    }{
    \partial t
    }
    =
    - \nu_{uv}(t) ~
    \frac{
    \partial c_{uv}(t,z)
    }{
    \partial z
    }
\end{align}
\normalsize
which describes how the signal $c_{uv}$ is transported in the presence of an incompressible, i.e., divergence-free, flow field $\nu_{uv}: \R \rightarrow \R$ over time $\R$ and along the edge space $\Omega_{uv}$. The flow field additionally satisfies
\begin{align}
\label{align_ConservationOfFlows}
    \nu_{vu}(t) = -\nu_{uv}(t)
\end{align}
\normalsize
\noindent 
and we use the convention that the sign of a flow velocity $\nu_{uv}(t)$ is 
positive (negative) if the direction of the corresponding edge $e_{uv} \in E$ aligns (does not align) with the direction of the physical flow,
i.e., if the physical flow is from $u \in \Nbh(v)$ to $v \in V$ (from $v \in V$ to $u \in \Nbh(v)$).

The \textit{coupling conditions} for advection-dominated dynamics describe how signals propagate between different edge spaces. 
On the one hand, the \textit{mixing at nodes}
\begin{align}
\label{align_CouplingCondition_MixingAtNodes}
    c_v(t) 
    =&~
    \frac{
    \sum_{u \in \Nbh_-(v,t)}
    q_{uv}(t) \cdot c_{uv}(t,l_{uv})
    }{
    \sum_{u \in \Nbh_-(v,t)} 
    q_{uv}(t)
    }
\end{align}
\normalsize
\noindent
describes how signals propagate from edges into nodes.
On the other hand, the \textit{inflow continuity condition} 
\begin{align}
\label{align_CouplingCondition_InflowContinuity}
    c_{uv}(t,0) 
    =&~ 
    c_u(t)
    \text{ for all }
    u \in \Nbh_-(v,t)
\end{align}
\normalsize
\noindent
describes how signals propagate from nodes into edges and can be seen as the \textit{local inflow boundary condition} of the edge dynamics (Eq. \eqref{align_EdgeDynamics_Advection}).
In Equation \eqref{align_CouplingCondition_MixingAtNodes}, $q_{uv}(t)$ defines the \textit{flow rate} $q_{uv} = \alpha_{uv} \nu_{uv}: \R \rightarrow \R$ from (temporal) \textit{inflow neighbors} $u \in \Nbh_-(v,t) = \{u \in \Nbh(v) ~|~ \nu_{uv}(t) > 0\}$ of $v \in V$ at time $t \in \R$, that are neighbors $u \in \Nbh(v)$ where at time $t$, the physical flow field is directed from $u$ to $v$. For details, we refer to Appendix \ref{section_Background_Appendix}.

\begin{remark}[Global vs. local boundary conditions]
\label{remark_GlobalVsLocalBoundaryConditions}
    Note that Equation \eqref{align_CouplingCondition_MixingAtNodes} and \eqref{align_CouplingCondition_InflowContinuity} define \textit{local} boundary conditions of the edge dynamics, i.e., Equation \eqref{align_EdgeDynamics_Advection}.
    In contrast, \textit{global} boundary conditions of the overall dynamical system of the metric graph $\G$ correspond to Dirichlet \textit{boundary nodes} $v \in \Vb \subsetneq V$ for which the signal $c_v(t)$ is known at all times $t \in \R$, also see Definition \ref{definition_ProblemDefinition}.
\end{remark}

\section{Method: Metric Graph Advection Message Passing}
\label{section_MethodMetricGraphAdvectionMessagePassing}

In this section, we cast the advection edge dynamics (Eq. \eqref{align_EdgeDynamics_Advection}) and corresponding coupling conditions (Eq. \eqref{align_CouplingCondition_MixingAtNodes} and \eqref{align_CouplingCondition_InflowContinuity}) into a novel message-passing framework that only requires knowledge of the system states $\cm = (c_v)_{v \in V}$ at nodes and the flow field $\num = (\nu_e)_{e \in E}$ along edges.
This is based on the observation that Equation \eqref{align_CouplingCondition_MixingAtNodes} is a special case of message-passing (cf. Appendix \ref{section_DetailsOnTheMethod}) in the form of
\begin{align}
\label{align_CouplingCondition_MixingAtNodes_MessagePassingFormulation}
\begin{split}
    c_v(t) 
    =&~
    \sum_{u \in \Nbh(v)} w_{uv}(t) ~ \phi(t,e_{uv},\cm,\num)
    ~ \text{ with}
    \\
    w_{uv}(t) 
    =&~
    \frac{
    \relu(q_{uv}(t))
    }{
    \sum_{u \in \Nbh(v)} \relu(q_{uv}(t))
    }
    \quad \quad \text{ and } \quad \quad
    \phi(t,e_{uv},\cm,\num) 
    =~
    c_{uv}(t,l_{uv}),
    \end{split}
\end{align}
\normalsize
\noindent
where $w_{uv}(t)$ are the normalized weights that select inflows $q_{uv}(t) > 0$ to $v \in V$ only and scale the messages generated by the message function $\phi$.
In the current formulation, this requires knowledge of the signal $c_{uv}(t,l_{uv})$ at the outflow boundary of each edge space $\Omega_{uv} = (0, l_{uv})$. 
However, instead of explicitly solving the dynamics along the edge spaces $\Omega = \sqcup_{e_{uv} \in E} \Omega_{uv}$, we derive a representation of $\phi$ that only relies on signals $\cm$ at nodes $V$. 
This leads to the following general problem we are trying to solve.

\begin{definition}[Spatio-temporal node forecasting with global boundary conditions]
\label{definition_ProblemDefinition}
    Let $\Vb \subsetneq V$ be a subset of global boundary nodes (cf. Remark \ref{remark_GlobalVsLocalBoundaryConditions}) and
    $T = \{t_i \in \R ~|~ i \in I = \{1,...,n\} \}$ a set of discrete times.
    For a fixed time $t_i \in T$, we define 
    the \textit{history set}
    $\Tn{hi}(t_i) := \{t_1,...,t_i\}$ 
    with cardinality $\n{hi}(t_i) := |\Tn{hi}(t_i)| = i$
    and
    the \textit{future set}
    $\Tn{fu}(t_i) := \{t_{i+1},...,t_n\}$
    with cardinality $\n{fu}(t_i) = |\Tn{fu}(t_i)| = n-i$.
    For simplicity, we omit the dependency of $\n{hi}$, $\n{fu}$, $\Tn{hi}$ and $\Tn{fu}$ on $t_i$ if $t_i$ is clear from the context.
    The spatio-temporal node forecasting problem with global boundary conditions refers to the learning problem
    $\big( (c_v(t))_{v \in V \setminus \Vb, t \in \Tn{hi}}, (c_v(t))_{v \in \Vb, t \in \Tn{hi} \sqcup \Tn{fu}}) \big) \longmapsto (c_v(t))_{v \in V \setminus \Vb, t \in \Tn{fu}}$,
    that is the problem of inferring the signal $\cm$ at times in the future set from the signal $\cm$ at times in the history set and at boundary nodes $\Vb$.
\end{definition}

Throughout the rest of this section, 
we propose \gls*{MeGA-MP}, a \gls*{MPNN} that solves this problem for advection-dominated edge dynamics, where the signal $\cm$ usually refers to a \textbf{c}oncentration.
More specifically, 
in Section \ref{subsection_PhysicalDerivationOfMeGa-MP}, we derive advection message passing by relating the \textit{edge dynamics} to \textit{message generation} and the \textit{coupling conditions} to \textit{message aggregation and update}.
In Section \ref{subsection_FinalAlgorithm:MeGA-MP}, we present the overall \gls*{MeGA-MP} algorithm taking numerical aspects due to the discretization of time (Definition \ref{definition_ProblemDefinition}) into account.

\subsection{Physical Derivation of MeGA-MP}
\label{subsection_PhysicalDerivationOfMeGa-MP}

\paragraph{Characteristic Curves}
We model advection of the concentration $c_{uv}: \R \times \Omega_{uv} \rightarrow \R_{\geq 0}$ over time $\R$ and along the edge space $\Omega_{uv} = (0,l_{uv})$ in between two nodes $v \in V$ and $u \in \Nbh(v)$ as the transport process governed by Equation \eqref{align_EdgeDynamics_Advection} with incompressible flow field $\nu_{uv}: \R \rightarrow \R$. 
In this case, by the \acrfull*{MoC}, we can derive \textit{characteristic curves} along which the function $c_{uv}$ remains constant over time $\R$ and space $\Omega_{uv}$.

\begin{lemma}[Constant concentration along characteristic curves]
\label{lemma_ConstantConcentrationAlongCharacteristicCurves}
    Let $e_{uv} \in E$ hold.
    If the function $c_{uv}$ obeys Equation \eqref{align_EdgeDynamics_Advection},
    for each $t_0 \in \R$ and $z_0 \in \R$, the curve
    %
    %
    \begin{align*}
        \gamma_{t_0,z_0}: ~ T_{t_0,z_0} \longrightarrow \R_{\geq 0}, ~~
        t \longmapsto 
        c_{uv} \Big( t, z_0 + \textstyle \int_{t_0}^{t} \nu_{uv}(s) ~ds \Big)
    \end{align*}
    
    \noindent 
    is constant on
    $
    T_{t_0,z_0}
    =
    \{
    t \in \R 
    ~|~
    0 
    <
    z_0 + \int_{t_0}^{t^{'}} \nu_{uv}(s) ~ds
    <
    l_{uv}
    ~\forall t^{'} \in 
    \big[ \min\{t_0,t\} , \max\{t_0,t\} \big]
    \}.
    $
    %
\end{lemma}
\noindent An intuition on Lemma \ref{lemma_ConstantConcentrationAlongCharacteristicCurves} can be found in Appendix \ref{subsubsection_MessageGeneration_AdvectiveTransportAlongEdges}. 
Figure \ref{figure_CharacteristicCurves} (left) visualizes several characteristic curves for an exemplary flow field $\nu_{uv}: \R \rightarrow \R$ along an edge $e_{uv} \in E$ of length $l_{uv} = 1[m]$. 
Following Lemma \ref{lemma_ConstantConcentrationAlongCharacteristicCurves}, the concentration $c_{uv}$ is the same at all positions along such a curve. Consecutively, for the example given in Figure \ref{figure_CharacteristicCurves} (right), knowing the concentration $c_u(t - \delta t_{uv})$ at the node $u \in \Nbh(v)$ ($z = 0$) is sufficient to determine the concentration $c_v(t)$ at node $v \in V$ ($z = l_{uv}$). 
This can be observed for the heatmap in the background of the figure, which shows a high-concentration pulse traveling from one boundary at $z = 0$ to the other at $z = l_{uv}$ in $\delta t_{uv} \approx 5[s]$.

\begin{figure}[t]
    \centering
    \includegraphics[width=0.98\linewidth]{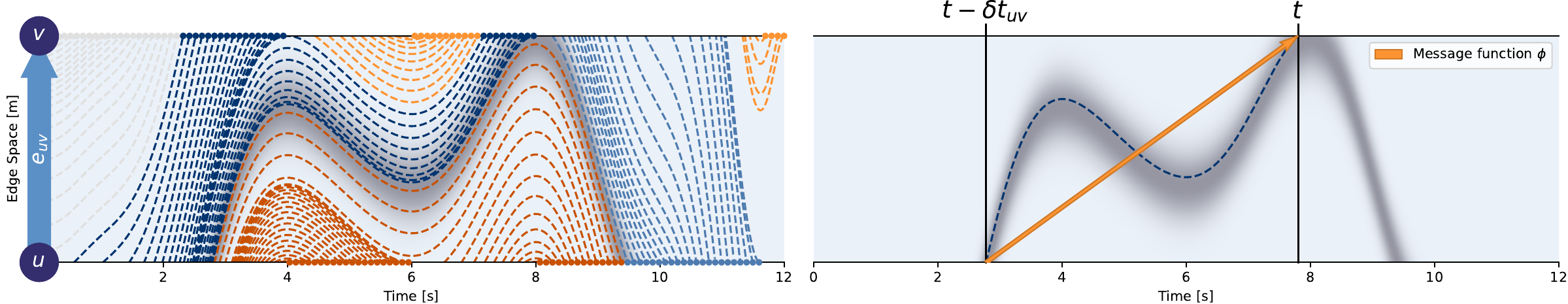}
    \caption{Characteristic curves along the edge $e_{uv} \in E$:
    If a curve originates at node $u$ at time $t - \delta t_{uv}$ and terminates at node $v$ at time $t$, 
    we call it a \textit{pass-through} (\textcolor{pass-through}{dark blue}). 
    If the direction is reversed, 
    we call it \textit{inverse pass-through} (\textcolor{inverse-pass-through}{light blue}). 
    If a curve starts and ends at the same node, 
    we call it a \textit{self-loop} (\textcolor{self-loop}{dark orange})
    if that node is $u$ and 
    \textit{inverse self-loop} (\textcolor{inverse-self-loop}{light orange}) 
    if that node is $v$. 
    The corresponding transport time $\delta t_{uv} > 0$ (right) denotes the duration between the curve’s start and end. The shaded area (\textcolor{gray}{gray}) in the background shows a high-concentration pulse entering the domain at node $u$, that is constant along characteristic curves.
    }
    \label{figure_CharacteristicCurves}
\end{figure}

\paragraph{Transport Times}
To formulate this relationship as a message function, we need a way to compute the start and end coordinates of characteristic curves that correspond to the sender and receiver node, respectively, as well as the time shift $\delta t_{uv}$, which we call \textit{transport time}. 
We also use specific terminology to differentiate the following situations, which are also described visually in Figure \ref{figure_CharacteristicCurves} (left). 
To get a deeper understanding of the different kinds of transport times, details can also be found in Appendix \ref{subsubsection_MessageGeneration_AdvectiveTransportAlongEdges}. 

\begin{definition}[Transport times]
\label{definition_TransportTimes_shortened}
    Let $e_{uv} \in E$ and $t \in \R$ hold.
    If there exists a $\delta t_{uv} := \delta t_{uv}(t) \in \R$, such that
    \begin{align*}
        \int_{t - \delta t_{uv}}^{t}
        \nu_{uv}(s) ~ds
        =&~
        {\color{pass-through}
        l_{uv}
        }
        &&
        (
        {\color{inverse-pass-through} 
        -l_{uv}
        }
        ,
        {\color{self-loop} 
        0
        }
        ,
        {\color{inverse-self-loop} 
        0
        }
        ~)
        \text{ and }
        \\
        \int_{t - \delta t_{uv}}^{t^{'}}
        \nu_{uv}(s) ~ds
        \in&~ 
        {\color{pass-through}
        (0,l_{uv})
        }
        &&
        (
        {\color{inverse-pass-through} 
        (-l_{uv},0)
        }
        ,
        {\color{self-loop} 
        (0,l_{uv})
        }
        ,
        {\color{inverse-self-loop} 
        (-l_{uv},0)
        }
        )
    \end{align*}
    \normalsize
    \noindent holds for all $t^{'} \in (\min\{t - \delta t_{uv},t\},\max\{t - \delta t_{uv},t\})$,
    we call $\delta t_{uv}$ a \textit{transport time}, or, more concretely, the 
    \textcolor{pass-through}{\textit{pass-through}}
    (\textcolor{inverse-pass-through}{\textit{inverse pass-through}},
    \textcolor{self-loop}{\textit{self-loop}},
    \textcolor{inverse-self-loop}{\textit{inverse self-loop}})
    \textit{time}
    \textit{\gls*{wrt} the edge $e_{uv} \in E$ and time $t \in \R$}. 
    Usually, $e_{uv} \in E$ and $t \in \R$ are clear from the context and we just say 
    \textit{pass-through}
    (\textit{inverse pass-through},
    \textit{self-loop},
    \textit{inverse self-loop})
    \textit{time},
    just as we omit the dependency of $\delta t_{uv}$ on $t$ for better readability.
\end{definition}

\paragraph{Message Function}
If the flow $\nu_{uv}$ is constant over time, the transport time $\delta t_{uv}$ is also constant, and concentrations are shifted between the nodes $u$ and $v$.
If the flow $\nu_{uv}$ varies, so does the transport time $\delta t_{uv}$, and instead of a clear shift, we observe more complex \textit{time warping}.
As this generalizes shift, we use the prior as our message function.
To compute a unique positive transport time for a given time $t \in \R$, we need to measure the distance traveled in an arbitrary time interval $[t - \delta t, t]$. For this purpose, we define the function
\begin{align}
\label{align_TransportTime_Functionscurlzcurl}
    \zcurl: \R \longrightarrow \R, ~
    \delta t \longmapsto \zcurl(\delta t) := \int_{t-\delta t}^{t} \nu_{uv}(s) ~ds,
\end{align}
\normalsize
\noindent 
where we omit the dependency of $\zcurl$ on $t \in \R$, $e_{uv} \in E$ and the overall flow field $\num$ for better readability.
Using $\zcurl$, we can compute a positive transport time as the solution to the following \gls*{OP}:
\begin{align}
\label{align_TransportTime_OP}
    \delta t_{uv} 
    := 
    \min
    \Big \{
    \delta t > 0 
    ~ \big |~
    \zcurl(\delta t)
    \in 
    \{-l_{uv},0,l_{uv}\}
    \Big \}.
\end{align}
\normalsize
\noindent 
In Appendix \ref{subsubsection_MessageGeneration_AdvectiveTransportAlongEdges}, we prove that if there exists a solution, this solution uniquely determines one of the transport times from Definition \ref{definition_TransportTimes_shortened} (cf. Theorem \ref{theorem_ExistenceAndUniquenessOfPositiveTransportTimes}).
This in turn allows to define our advection-message-generation functions as 
\begin{align}
\label{align_MessageFunction_Advection}
\begin{split}
    \phi(t,e_{uv},\cm,\num)
    :=
    \begin{cases}
        c_u(t - \delta t_{uv})
        & \text{ if } 
        \zcurl(\delta t_{uv}) = \pm l_{uv} 
        \\
        c_v(t - \delta t_{uv})
        & \text{ if } 
        \zcurl(\delta t_{uv}) = 0
    \end{cases}.
\end{split}
\end{align}
\normalsize
\noindent 
Especially to mention, using $\zcurl$ again, $\phi$ automatically identifies whether information is carried along a(n) (inverse) pass-through or a(n) (inverse) self-loop and outputs the information from the neighbor $u \in \Nbh(v)$ or the node $v$ itself accordingly.
A rigorous derivation of Equation \eqref{align_MessageFunction_Advection} under consideration of physical priors and the intuition behind this can be found in Appendix \ref{subsubsection_MessageGeneration_AdvectiveTransportAlongEdges}. 

\paragraph{Message Aggregation and Update}
We integrate Equation \eqref{align_MessageFunction_Advection} into a complete message passing scheme that determines the concentration $c_v(t)$ at a node $v \in V$ and time $t \in \R$ as
\begin{align}
\label{align_MessagePassing_Advection}
\begin{split}
    c_v(t)
    =&~
    \sum_{u \in \Nbh(v)}
    w_{uv}(t)
    \cdot
    \phi(t,e_{uv},\cm,\num)
\end{split}
\end{align}
\normalsize
\noindent 
with $w_{uv}(t)$ as defined in Equation \eqref{align_CouplingCondition_MixingAtNodes_MessagePassingFormulation} and $\phi(t,e_{uv},\cm,\num)$ as defined in Equation \eqref{align_MessageFunction_Advection}.

\paragraph{Theoretical Guarantees} 
As elaborated in Section \ref{section_Background},
advection on metric graphs is defined by multiple physical priors, 
the edge dynamics (Eq. \eqref{align_EdgeDynamics_Advection})
and
the coupling conditions (Eq. \eqref{align_CouplingCondition_MixingAtNodes}) and Eq. \eqref{align_CouplingCondition_InflowContinuity}).
While the former are defined on the open edge spaces $\Omega_{uv} = (0,l_{uv})$ only, 
the latter rely on the values $c_{uv}(\cdot,0)$ and $c_{uv}(\cdot,l_{uv})$ on the boundary $\partial \Omega_{uv} = \{0,l_{uv}\}$ of $\Omega_{uv}$. 
Therefore, suitable smoothness assumptions on $c_{uv}$ and $\nu_{uv}$ are required in order to relate these principles.

\begin{assumption}[Smoothness assumptions]
\label{assumption_SmoothnessAssumptions}
    For each $e_{uv} \in E$, let 
    $c_{uv}: \R \times (0,l_{uv}) \rightarrow \R_{\geq 0}$ be differentiable and
    $c_{uv}: \R \times [0,l_{uv}] \rightarrow \R_{\geq 0}$ and
    $\nu_{uv}: \R \rightarrow \R$ continuous.


\end{assumption}

Under these assumptions, in Appendix \ref{subsection_PhysicalDerivationOfMeGA-MP_Appendix}, we extend the findings from Lemma \ref{lemma_ConstantConcentrationAlongCharacteristicCurves} and investigate the different phenomena as introduced in Definition \ref{definition_TransportTimes_shortened} more thoroughly. 
Based on that, we provide a rigorous theory of why our advection message-passing scheme models the described physics \textit{exactly}, with the main results as follows:

\begin{theorem}[Physical alignment]
\label{theorem_PhysicalAlignment}
    Let $v \in V$ and $t \in \R$ hold.
    We assume that Assumption \ref{assumption_SmoothnessAssumptions} holds and that for each $u \in \Nbh(v)$, the \gls*{OP} \eqref{align_TransportTime_OP} has a solution.
    If the function $c_{v}$ obeys the mixing at nodes (Eq. \eqref{align_CouplingCondition_MixingAtNodes}) and 
    if for each $u \in \Nbh(v)$, the function $c_{uv}$ obeys the advection edge dynamics (Eq. \eqref{align_EdgeDynamics_Advection}) with inflow continuity condition (Eq. \eqref{align_CouplingCondition_InflowContinuity}), 
    then Equation \eqref{align_MessagePassing_Advection} with $\phi$ as defined in Equation \eqref{align_MessageFunction_Advection} holds.
\end{theorem}

\subsection{Final Algorithm: MeGA-MP}
\label{subsection_FinalAlgorithm:MeGA-MP}

So far, we have introduced the advection message-passing operator (Eq. \eqref{align_MessagePassing_Advection}). We will now explain how it is applied in practice. For this, two aspects remain to be addressed: 
First, Equation \eqref{align_MessagePassing_Advection} is defined for continuous time $t \in \mathbb{R}$, whereas in practice boundary values and flow fields are typically known as measurements at discrete points in time. We therefore adapt Equation \eqref{align_MessagePassing_Advection} to handle discrete temporal information, described in paragraph \textit{\nameref{par_Interpolated_MsgFunction}}. 
Second, the edge lengths $\lm = (l_e)_{e \in E}$ and the flow velocities $\num = (\nu_e)_{e \in E}$ determine how far information travels over the metric graph. Under fast flow conditions, information can traverse multiple edges of the metric graph within the time interval $[t_1,t_n]$ (cf. Definition \ref{definition_ProblemDefinition} and Appendix \ref{section_DetailsOnTheExperiment_Appendix}). This cannot be captured by a single application of Equation \eqref{align_MessagePassing_Advection}, as it depends only on the immediate 1-hop node neighborhood. We therefore adopt an iterative message-passing scheme, described in paragraph \textit{\nameref{sec:ModelIterations}}.
We define the resulting iterative scheme described by Equation \eqref{align_MessagePassing_AdvectionIterations} as \textbf{\acrfull{MeGA-MP}}.

\paragraph{Temporal Discretization}
\label{par_Interpolated_MsgFunction}
We model the concentration $c_v: \R \rightarrow \R_{\geq 0}$ at a node $v \in V$ as a discrete time series $(c_v(t))_{t \in T}$ over equidistant discrete times $t \in T = \{t_i \in \R ~|~ i \in I = \{1,...,n\} \} = \Tn{hi} \sqcup \Tn{fu}$ with sampling interval $dt = t_{i+1} - t_i$.
In accordance to Definition \ref{definition_ProblemDefinition}, we aim to estimate the signal $\ym_v := (c_v(t))_{t \in \Tn{fu}} \in \R^\n{fu}$ at all non-boundary nodes $v \in V \setminus \Vb$ and times $t \in \Tn{fu}$.
To accommodate known values at times $t \in \Tn{hi}$ and unknown values at times $t \in \Tn{fu}$, we pad the known history $\xm_v := (c_v(t))_{t \in \Tn{hi}} \in \R^\n{hi}$ with zeros, yielding $\cmhat_v^{(0)} := (\xm_v^T,\mathbf{0}^T) \in \R^{\n{hi}+\n{fu}}$. Here, the zero entries are placeholders for the predictions $\ymhat_v := (\chat_v(t))_{t \in \Tn{fu}} \in \R^\n{fu}$ of $\ym_v$ (cf. Fig. \ref{figure_AdvectionMessagePassing} and Eq. \eqref{align_MessagePassing_AdvectionIterations} below).
Since the transport times $\delta t_{uv} >0$ are positive by definition (cf. Eq. \eqref{align_TransportTime_OP}), applying Equation \eqref{align_MessagePassing_Advection} to $\cmhat_v^{(0)}$ for all $v \in V \setminus \Vb$ shifts/warps information forward in time, populating the padded entries with the solutions to the node forecasting problem.
However, as both edge lengths $l_{uv}$ and flows $\nu_{uv}$ are still continuous, time shifts $t - \delta t_{uv}$ of the message function $\phi$ (Eq. \eqref{align_MessageFunction_Advection}) are also continuous and not constrained to the discrete time points $T$. Therefore, time warping is approximated by interpolating between the two values at times $t_j \in T$ and $t_{j+1} \in T$ for which $t - \delta t_{uv} \in [t_j,t_{j+1})$ holds. We denote the discrete version of $\phi$ by $\varphi$.
In our experiments, we use linear interpolation as detailed in Appendix \ref{subsection_FinalAlgorithm:MeGA-MP_Appendix}, but depending on the data, different interpolation methods can be considered.

\begin{figure}[t]
    \centering
    \includegraphics[width=\linewidth]{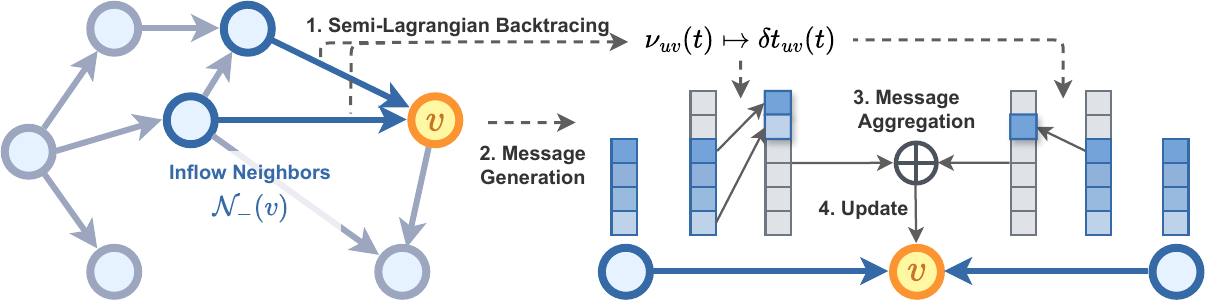}
    \caption{
    Initial iteration of \acrshort*{MeGA-MP} (Eq. \eqref{align_MessagePassing_AdvectionIterations}).
    Messages are passed from inflow neighbors $u \in \Nbh_{-}(v,t)$ to the target node $v$ (for all $v \in V$):
    First, flow velocities $\nu_{uv}$ are converted to the transport time $\delta t_{uv}$ (Eq. \eqref{align_TransportTime_OP}). Second, messages are generated via time-warping 
    (Eq. \eqref{align_MessageFunction_Advection}). Third, results are aggregated and assigned to the target node $v$ (Eq. \eqref{align_MessagePassing_Advection}). 
    The graph is directed according to the flow for visualization purposes.
    Grey entries correspond to padded entries and masked features according to Equation \eqref{align_MessagePassing_AdvectionIterations}.}    \label{figure_AdvectionMessagePassing}
\end{figure}

\paragraph{Model Iterations}
\label{sec:ModelIterations}
As mentioned above, applying Equation \eqref{align_MessagePassing_Advection} one time to $\cmhat_v^{(0)}$ for all $v \in V \setminus \Vb$ only fully solves the node forecasting problem if all evaluations of the kind $c_u(t - \delta t_{uv})$ in Equation \eqref{align_MessageFunction_Advection} index known concentration values. This is only the case if either $u \in \Vb$ is a boundary node, or if $t - \delta t_{uv} < t_i$ lies temporally before the upper bound of the known history $\Tn{hi}$. 
The latter corresponds to a constraint on the flow field $\num$.
However, generalizing this to arbitrary flow fields can be achieved by iterating Equation \eqref{align_MessagePassing_Advection} via
\begin{align}
\label{align_MessagePassing_AdvectionIterations}
\begin{split}
    \chat_v^{(0)}(t)
    :=&~
    \begin{cases}
        c_v(t) 
        &
        \text{if } v \in \Vb
        \text{ or } t \in \Tn{hi} = \{t_1,...,t_i\} 
        \\
        0 
        &
        \text{if } v \in V \setminus \Vb
        \text{ and } t \in  \Tn{fu} = \{t_{i+1},...,t_n\} 
    \end{cases},
    \\
    \chat_v^{(k)}(t)
    :=&~
    \begin{cases}
        0 
        & 
        \text{if } v \in \Vb
        \text{ or } t \in \Tn{hi}
        \\
        \sum_{u \in \Nbh_-(v,t)}
        w_{uv}(t)
        ~
        \varphi(t,e_{uv},\cmhat^{(k-1)},\num)
        &
        \text{if } v \in V \setminus \Vb 
        \text{ and } t \in \Tn{fu}
    \end{cases}
    \text{ for  all }  k \geq 1,
    \\
    \chat_v(t)
    :=&~
    \sum_{k=1}^{\infty} \chat_v^{(k)}(t)
    ~~
    \text{ for } v \in V \setminus \Vb 
    \text{ and } t \in  \Tn{fu},
\end{split}
\end{align}
\normalsize
\noindent 
where $\cmhat^{(k-1)} \in \R^{|V| \times (\n{hi}+\n{fu})}$ corresponds to the vectorized version of the function $\cm = (c_v)_{v \in V}$ that contains the values $\chat_v^{(k-1)}(t)$ for all $v \in V$ and $t \in T = \Tn{hi} \cup \Tn{fu}$. 
To build an intuition for how these iterations relax the constraint on $\num$, note that in a single iteration $k$, a node $v \in V \setminus \Vb$ receives information at time $t \in \Tn{fu}$ from its neighbors (or itself) at earlier times $t - \delta t_{uv} < t$. This defines a spatio-temporal inflow tree of depth one, along which information is passed. With each iteration, this spatio-temporal inflow tree extends further through the graph. As a result, the node $v$ receives information from increasingly distant nodes at progressively earlier times, eventually from either boundary nodes $\Vb$ or the known history $\Tn{hi}$. Importantly, only at these leaves, information is initially non-zero. Therefore, summing the values $\chat_v^{(k)}(t)$ over all $k \in \N$ correctly aggregates information from antisymmetrically distributed sending nodes along paths in the inflow tree.
A rigorous formalization of this intuition can be found in Appendix \ref{subsection_ErrorBoundsOnMeGA-MP}.
The initial iteration ($k = 0 \mapsto k = 1$) is visualized in Figure \ref{figure_AdvectionMessagePassing}. Subsequent iterations follow the same schema, but with node feature vectors that already contain information at future time steps $\Tn{fu}$.

\paragraph{Theoretical Guarantees} 
The following Theorem \ref{theorem_ConvergenceOfModelIterations} ensures that the prediction $\ymhat_v = (\chat_v(t))_{t \in \Tn{fu}}$ defined by the last row of Equation \eqref{align_MessagePassing_AdvectionIterations} is computable in finite time for all $v \in V \setminus \Vb$. 

\begin{theorem}[Convergence of model iterations]
\label{theorem_ConvergenceOfModelIterations}
    There exists a $\kappa \in \N$, such that $\chat_v^{(k)}(t)$ as defined in Equation \eqref{align_MessagePassing_AdvectionIterations} is equal to zero for all $k > \kappa$, $v \in V$ and $t \in T = \Tn{hi} \cup \Tn{fu}$.
\end{theorem}
\noindent
For a node $v \in \Vb$ or a time $t \in \Tn{hi}$, this statement is trivial ($\kappa(v,t)= 0$). In the proof of Theorem \ref{theorem_ConvergenceOfModelIterations}, we show that for the more interesting case of a node $v \in V \setminus \Vb$ and a time $t \in \Tn{fu}$, $\chat_v^{(k)}(t)$ is only non-zero for all iterations $k \in \N$ which correspond to the length $l_p$ of any path $p \in \Pahat(v,t)$ in the before-mentioned inflow tree $\That(v,t)$. This naturally defines $\kappa(v,t) = \max_{p \in \Pahat(v,t)} l_p$ as the maximum path length over all paths in the inflow tree; and $\kappa = \max_{v \in V} \max_{t \in t \in \Tn{hi} \sqcup \Tn{fu}} \kappa(v,t)$ as the maximum of the former over all nodes and time steps.
In practice, we compute the inflow tree and the node- and time-dependent $\kappa(v,t)$ on the fly by accumulating the transport times along paths of the inflow tree until either a boundary node $v \in \Vb$ is reached or the difference between the time $t \in \Tn{fu}$ and the accumulated transport time lies temporally before the upper bound of the known history $\Tn{hi}$. If this is the case, the end of a path in the inflow tree is reached. 

Finally, we show that \gls*{MeGA-MP} (Eq. \eqref{align_MessagePassing_AdvectionIterations}) reproduces the ground-truth (Eq. \eqref{align_MessagePassing_Advection}) if the interpolation error is zero.
An extended version of Theorem \ref{theorem_InterpolationError} with an error bound obtained through the temporal interpolation can be found in Appendix \ref{subsection_ErrorBoundsOnMeGA-MP}.

\begin{theorem}[Interpolation error]
\label{theorem_InterpolationError}
    Let $v \in V$ and $t \in T = \Tn{hi} \cup \Tn{fu}$ hold. 
    Let Assumption \ref{assumption_IdentifiablePaths} and \ref{assumption_EqualtTimeForJumpingIntoHistory} hold.
    If the temporal interpolation error is zero, i.e., if $\phi = \varphi$ holds, we obtain $\chat_v(t) = c_v(t)$.
\end{theorem}

\begin{remark}[Initial value problems and efficiency]
\label{remark_InitialValueProblemsAndEfficiency}
    For the sake of efficiency in both computational and memory complexity, we represent the solution of advection on a metric graph as time series at nodes only (cf. Definition \ref{definition_ProblemDefinition}). 
    The actual values along the 1-dimensional edge spaces are solved implicitly. 
    This enables our model to scale linearly with the number of nodes and, consequently, to handle even large graphs.
    For this advantage, our method cannot directly receive classical \glspl*{IVP}, where dense values $c_{uv}(t,z)$ along the edge spaces $\Omega_{uv}$ are known at $t = 0$, except when they are constant. 
    However, initial values along edges can be transformed into time series at nodes and vice versa. This is visualized by the gray characteristic curves in Figure \ref{figure_CharacteristicCurves}. 
    We will demonstrate this in Appendix \ref{subsection_Experiments_AdvectionOnA1DEuclideanDomain_Appendix} as an extension of Experiment 2 in Section \ref{subsection_Experiments_AdvectionOnA1DEuclideanDomain}. This can also be used to recover the explicit signal along the edge spaces, as used in the \href{https://anonymous.4open.science/w/tmp-preprint-BB4F/}{demo}.
\end{remark}

\paragraph{Runtime Complexity}  
\label{paragraph_runtime_complexity}
The runtime of \gls*{MeGA-MP} depends on edge lengths $l_{uv}$, flow velocities $\nu_{uv}$, number of history and predictions steps $\n{hi}$ and $\n{fu}$, respectively, and those of edges $|E|$ in the graph. We analyze runtime in two parts: Computation of transport times and the iterative message passing procedure (cf. Algorithms \ref{algorithm_FlowToTransportTime}, \ref{algorithm:MeGA-MP}).

The number of iterations required to solve the dynamics depends on the number of edges that information traverses within the future horizon $\Tn{fu}$ (cf. Appendix \ref{section_Background_Appendix}). 
An upper bound on the number of iterations can be defined as $\kn{up} = \frac{t_n - t_i}{t_{\min}}$, where $t_{\min} = \min_{e_{uv} \in E, t \in \Tn{fu}, \nu_{uv}(t) \neq 0} \frac{l_{uv}}{|\nu_{uv}(t)|}$.
In line with other message-passing algorithms, \gls*{MeGA-MP} also applies a message function per edge that processes vectors of length $n = \n{hi} + \n{fu}$, resulting in a worst-case runtime of $\mathcal{O}(\kn{up} n|E|)$.

The runtime for computing the transport times $\delta t_{uv}$ from the flow field $\nu_{uv}$ (cf. Algorithm \ref{algorithm_FlowToTransportTime}) depends on edge lengths $l_{uv}$ and velocities $\nu_{uv}$. The worst case number of iterations of the inner loop is $m = \min(n, \max_{e_{uv} \in E, t \in \Tn{fu}, \nu_{uv}(t) \neq 0} \frac{l_{uv}}{|\nu_{uv}(t)| dt})$, and since the outer loop iterates $n$ times for every edge in $E$, the resulting complexity is $\mathcal{O}(|E| m n)$ with an upper bound of $\mathcal{O}(|E|n^2)$.
Therefore, the overall worst-case runtime amounts to 
$\mathcal{O}(|E|(\kn{up} n + n^2)).$

\subsection{Extended Algorithm: MeGA-MP$^+$}
\label{subsection_ExtendendAlgorithm:MeGA-MP^+}

We defined \gls*{MeGA-MP} as the algorithm introduced in Section \ref{subsection_FinalAlgorithm:MeGA-MP}, characterized by 
the message aggregation scheme (Eq. \eqref{align_MessagePassing_Advection}),
induced from the mixing at nodes (Eq. \eqref{align_CouplingCondition_MixingAtNodes}), and 
the physics-informed advection message function $\phi$ (Eq. \eqref{align_MessageFunction_Advection}) within, 
derived from the advection edge dynamics (Eq. \eqref{align_EdgeDynamics_Advection}) 
with inflow continuity condition (Eq. \eqref{align_CouplingCondition_InflowContinuity}).
Formulating \gls*{MeGA-MP} as a message-passing operator allows straight-forward adaptations to more flexible and also trainable models. This can be achieved by extending the (discrete) message-function $\varphi$ (Eq. \eqref{align_MessagePassing_AdvectionIterations}/\eqref{align_MessageFunction_Advection_LinearInterpolation}). We refer to such extensions of $\varphi$ as $\varphi^+$, and denote the overall extended method as \gls*{MeGA-MP}$^+$. Note that this name stands for a concept rather than a specific method, since depending on the specific formulation of $\varphi^+$, multiple variations are possible.
On the one hand, added components can depend on the local metric graph neighborhood $\mathcal{N}(v)$, just as $\varphi$ does. The resulting \gls*{MeGA-MP}$^+$ operator can adapt to spatial dynamics other than linear advection. 
On the other hand, added components that do not depend on the neighborhood yield a \gls*{MeGA-MP}$^+$ operator that can adapt to additional non-spatial dynamics, while spatially still representing linear-advection. We showcase one such extension in Section \ref{section_Experiment1_Zero-ShotTransferLearningOfAdvection-ReactionDynamicsOnRealWorldMetricGraphs}.

\section{Experiment 1: Zero-Shot Transfer Learning of Advection-Reaction Dynamics on Real-World Metric Graphs}
\label{section_Experiment1_Zero-ShotTransferLearningOfAdvection-ReactionDynamicsOnRealWorldMetricGraphs}

In this section, we demonstrate the power of \gls*{MeGA-MP} using water quality data, a classical example of an advection-reaction metric graph.
To demonstrate that \gls*{MeGA-MP} can be used as a standalone solver for linear advection, but also that it can be extended by a learnable component to account for additional dynamics, we consider two physical configurations: a pure linear advection system, and a coupled advection-reaction system.
The objective of this experiment is to solve the node-level regression task, predicting the temporal evolution of signals at nodes, as outlined in Definition \ref{definition_ProblemDefinition}.

We benchmark both \gls*{MeGA-MP} for advection and an example of \gls*{MeGA-MP}$^+$ for advection-reaction dynamics against baseline models, and run an ablation study to show the effect of the different components of both models.
Since advection is the only spatial dynamic in both configurations,  \gls*{MeGA-MP} and \gls*{MeGA-MP}$^+$ should theoretically generalize to unseen topologies without retraining. We validate this zero-shot capability by testing both configurations on an additional metric graph structure not contained in the training set.

The datasets are described in detail in Section \ref{subsection_Domain_Dataset}, 
the comparison to baselines is in Section \ref{subsection_Experiment1.1_ComparisonToBaselines}, 
and the ablation study is in Section \ref{subsection_Experiment1.2_AblationStudies}.

\subsection{Domain and Dataset}
\label{subsection_Domain_Dataset}

\paragraph{Domain}
In this experiment, we consider \glspl*{WDS}, in which nodes, such as reservoirs or consumer junctions, are connected through pipes. As the pipes are associated with lengths, this domain can be modeled as a metric graph.
The \glspl*{WDS} Hanoi \cite{Vrachamis2018WaterDomain_WaterDistributionNetworks_Hanoi} and L-Town \cite{Vrachimis2020WaterDomain_WaterDistributionNetworks_L-Town} are shown in Figures \ref{subfigure_Hanoi} and \ref{subfigure_L-Town}, respectively.
Both constitute well-known benchmark networks, used by engineers in the water community. 

\begin{figure}[!h]
\centering
    \begin{minipage}{0.49\linewidth}
        \centering
        \includegraphics[width=0.9\linewidth]{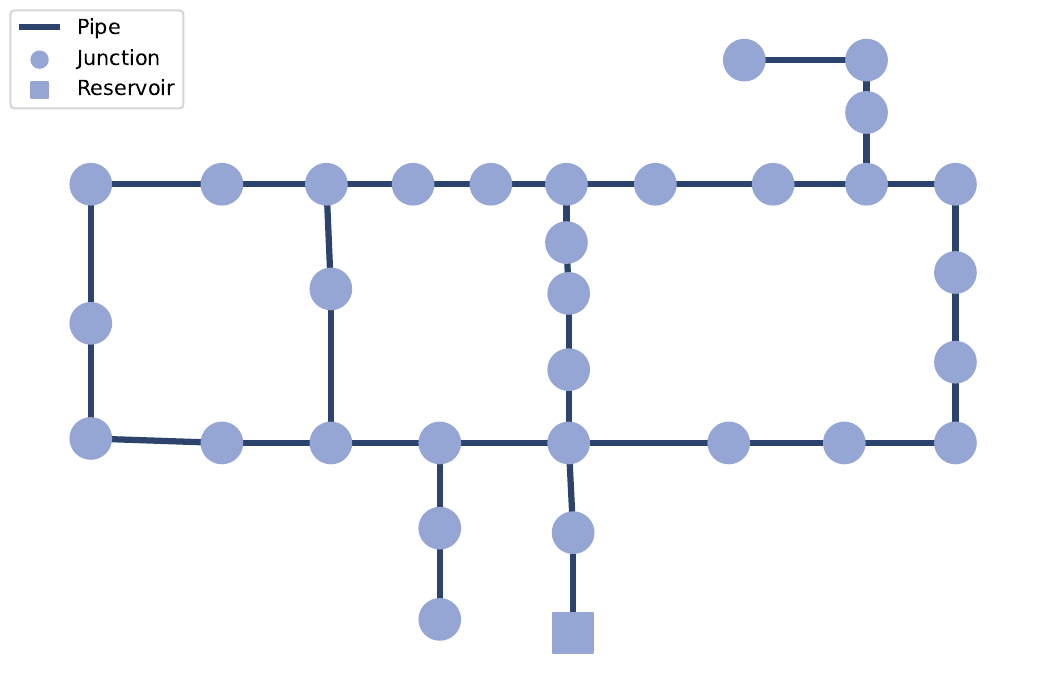}
        \caption{The \gls*{WDS} Hanoi with 32 nodes (31 consumer nodes $\Vc$ and one reservoir node $\Vr$) and 34 pipes used \textbf{for training and testing}.}
        \label{subfigure_Hanoi}
    \end{minipage}
    \hfill
    \begin{minipage}{0.49\linewidth}
        \centering
        \includegraphics[width=0.9\linewidth]{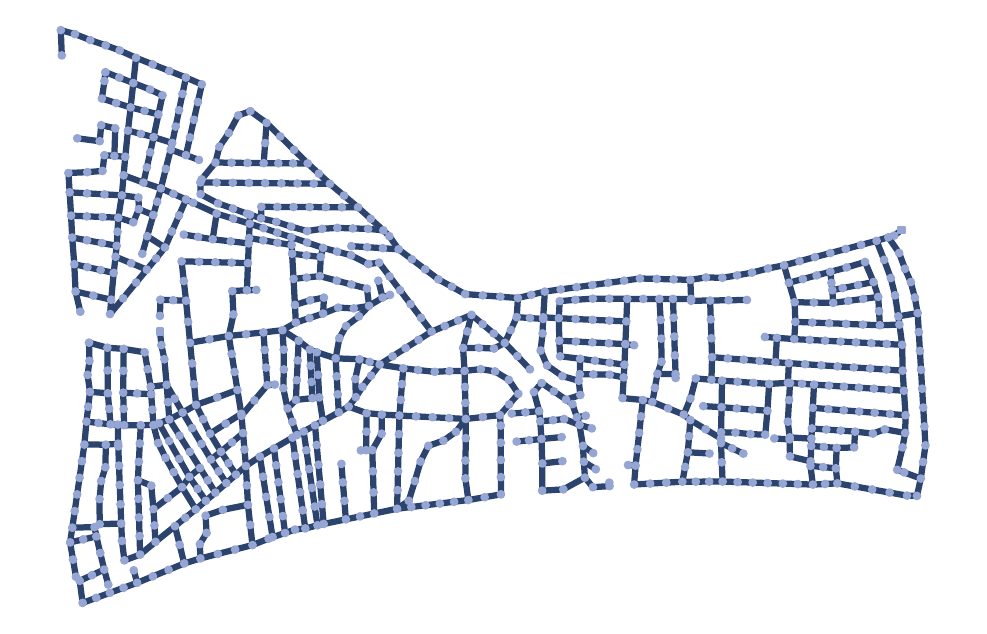}
        \caption{The \gls*{WDS} L-Town with 784 nodes (782 consumer nodes $\Vc$ and two reservoir nodes $\Vr$) and 908 pipes used \textbf{only for testing}.}
        \label{subfigure_L-Town}
    \end{minipage}
\end{figure}

A common task in \glspl*{WDS} is to estimate the water quality by predicting the distribution of a substance throughout the \gls*{WDS}. The two main drivers of such processes are advection and \textit{reactions} \cite{Rossman2020WaterDomain_StateSimulations_Epanet2.2}.
Reactions describe purely temporal changes in state variables such as chemical and physical transformations, interactions with the environment, and decay. Typically, such dynamics are expressed as \glspl*{ODE}.
When combined with advection, they extend Equation \eqref{align_EdgeDynamics_Advection} as follows:
\begin{align}
\label{align_EdgeDynamics_AdvectionReaction}
    \frac{\partial c_{uv}(t,z)}{\partial t}
    =&~
    - \nu_{uv}(t) ~
    \frac{\partial c_{uv}(t,z)}{\partial z}
    +
    f_r(t,z,e_{uv}),
\end{align}
\normalsize
\noindent
where $f_r$ is a substance-specific reaction term.
For this experiment, we simulate the injection of chlorine into the \gls*{WDS}. Chlorine decays over time due to chemical and physical interactions. Details on the specific formalization can be found in Appendix \ref{section_Experiment1_Zero-ShotTransferLearningOfAdvection-ReactionDynamicsOnRealWorldMetricGraphs_Appendix}.

\paragraph{Dataset} 
We generate two training, validation and test datasets by simulating water quality dynamics on the Hanoi \gls*{WDS} under the two previously mentioned configurations. Specifically, the first dataset models pure linear advection (Eq. \eqref{align_EdgeDynamics_Advection}), while the second dataset models coupled advection-reaction dynamics (Eq. \eqref{align_EdgeDynamics_AdvectionReaction}). To evaluate out-of-distribution topology generalization, we extend each test set by simulating the same dynamics on a different metric graph, the L-Town \gls*{WDS}. All ground truth simulations are generated using EPyt-Flow \cite{Artelt2024_WaterDomain_StateSimulation_EPyT-Flow} based on EPANET-MSX \cite{Shang2023WaterDomain_StateSimulations_EpanetMSX2.0}, a hydraulic and water quality simulator for \glspl*{WDS} that supports advection and user-defined reactions such as Equation \eqref{align_EdgeDynamics_Reaction_ChlorineDecay}. 
Details on the parametrization of these setups can be found in Appendix \ref{section_Experiment1_Zero-ShotTransferLearningOfAdvection-ReactionDynamicsOnRealWorldMetricGraphs_Appendix}.
Each sample of the datasets consists of a time series $T$ over $n = 701$ steps, where the first time step $t_1$ corresponds to the history set $\Tn{hi}(t_1) = \{t_1\}$ ($\n{hi} = 1)$ and the remaining $\n{fu} = 700$ steps correspond to the future set $\Tn{fu}(t_1)$ (cf. Definition \ref{definition_ProblemDefinition}).

\subsection{Experiment 1.1: Comparison to Baselines}
\label{subsection_Experiment1.1_ComparisonToBaselines}

\paragraph{Models}
We consider two variants of our method: \gls*{MeGA-MP}, the purely physics-informed message-passing operator derived in Section \ref{subsection_FinalAlgorithm:MeGA-MP}, and \gls*{MeGA-MP}$^+$, an extension with a non-spatial learnable component enabling adaptation to the advection-reaction dynamics (Eq. \eqref{align_EdgeDynamics_AdvectionReaction}) by learning the reaction part from data.
In the spirit of Section \ref{subsection_ExtendendAlgorithm:MeGA-MP^+}, we modify \gls*{MeGA-MP} by extending its message function $\varphi$ (Eq. \eqref{align_MessagePassing_AdvectionIterations}/\eqref{align_MessageFunction_Advection_LinearInterpolation}) with a \gls*{MLP}.  
The new message function $\varphi^+$ is designed to be suitable for first-order reaction kinetics, where the solution will be an exponential function. The resulting message function $\varphi^+$ is given by 
\begin{align}
\label{align_MessageFunction_AdvectionReaction}
    \varphi^+(t,e_{uv},\cmhat^{(k)},\num) 
    = 
    \underbrace{
    \varphi(t,e_{uv},\cmhat^{(k)},\num)
    }_{\text{Advection}} 
    \cdot 
    \underbrace{
    \exp(|\delta t_{uv}| \cdot \text{MLP}(\rho_{uv}^{-1})),
    }_{\text{Reaction}},
\end{align}
\normalsize
\noindent 
where $\rho_{uv} \in (0,1]$ denotes normalized pipe diameters. Details on the architecture and the training of the \gls*{MLP} can be found in Appendix \ref{section_Experiment1_Zero-ShotTransferLearningOfAdvection-ReactionDynamicsOnRealWorldMetricGraphs_Appendix}. 

To show the ability of our model for transfer learning, we train it on the \gls*{WDS} Hanoi (cf. Figure \ref{subfigure_Hanoi}) and test it on both Hanoi and the unseen, much larger \gls*{WDS} L-Town (cf. Figure \ref{subfigure_L-Town}).

\paragraph{Baselines}
As this is the first approach for modeling advection on metric graphs, we select a few representative baselines from the general graph domain.
We compare ourselves to the \gls*{GNN} by \citet{Gilmer2017GNNs_Methods_MPNN} (\enquote{NNConv}) as a representative of conventional \glspl*{GNN}, 
the \gls*{GNN} specifically designed for \glspl*{WDS} by \citet{Hermes2025WaterDomain_WaterQuality_Methods} (\enquote{\gls*{PDE}-GNN}), 
the transformer \gls*{GNN} by \citet{rampasek2022GPSConv} (\enquote{GPSConv}) as a representative for \glspl*{GNN} potentially able to handle global graph dependencies,
and 
the dynamic-system based \gls*{GNN} by \citet{gravina2023antisymmetric} (\enquote{A-DGN}) as a representative for \glspl*{GNN} potentially able to handle dynamical systems with long-range dependencies on graphs. 
For details regarding these baselines, we also refer to Appendix \ref{section_Experiment1_Zero-ShotTransferLearningOfAdvection-ReactionDynamicsOnRealWorldMetricGraphs_Appendix}. 
There, we also provide details on why the adaptation of the methods for metric graphs from \citet{Blechschmidt2022MGNNs_MetricGraphNeuralODEandPDESolver_Solution-basedArchitechtures_PINNsForDriftDiffusionOnMetricGraph,Blechschmidt2025MGNOs_NormalMGNOs_Solution-basedArchitechtures_DeepONetsForDriftDiffusionOnMetricGraphs} and \citet{Laczko2025MGNNs_MetricGraphNeuralODEandPDESolver_Solution-basedArchitechtures_PINNsForSchrödingerOnQuantumGraphs} were not applicable to our task.

\paragraph{Results and Analysis}

\begin{table}[!h]
\vspace{-2pt}
    \centering
    \caption{\acrshort*{MAE} between simulator ground truth (EPANET-MSX) and model prediction of the \textbf{advection} and \textbf{advection-reaction} dynamical systems averaged across samples, nodes and time steps of the test datasets. Hanoi is the graph seen in training, L-Town is the generalization network not seen in training. 
    The computation time for both networks and for each method is reported below (batch size = 1). The \textit{preprocessing} time is the time it takes to convert the flow field to transport times (Algorithm \ref{algorithm_FlowToTransportTime}). For a fair comparison to the simulator, this time has to be added to the computation time of all \gls{ML}-driven models.}
    \label{table_ComparisonBaselines}
    \begin{tabular}{cc||cc|cc}
       &    & \multicolumn{2}{c}{\textbf{Advection}} & \multicolumn{2}{c}{\textbf{Advection-Reaction}} \\
       &WDS & Hanoi & L-Town & Hanoi & L-Town
       \\
       &  &  (train/test graph) & (test graph) & (train/test graph) & (test graph)
       \\
       & & \scriptsize$|V| = 32, |E| = 34$ & \scriptsize$|V| = 784, |E| = 908$ & \scriptsize$|V| = 32, |E| = 34$ & \scriptsize$|V| = 784, |E| = 908$
       \\ \hline \hline 
       \multirow{6}{*}{\rotatebox[origin=c]{90}{MAE}}
        & NNConv & 0.1693{\scriptsize$\pm 0.1610$} & 0.2758{\scriptsize$\pm 0.2345$} & 0.1242{\scriptsize$\pm 0.1132$} & 0.2248{\scriptsize$\pm 0.2113$} \\ 
        & \gls*{PDE}-\gls*{GNN} & 0.0442{\scriptsize$\pm 0.0770$} & 0.1640{\scriptsize$\pm 0.1909$} & 0.0651{\scriptsize$\pm 0.0808$} & 0.4136{\scriptsize$\pm 0.4206$} \\ 
        & GPSConv & 0.1878{\scriptsize$\pm 0.2911$} & 0.7429{\scriptsize$\pm 2.7174$} & 0.1381{\scriptsize$\pm 0.1137$} & 0.7319{\scriptsize$\pm 1.7254$} \\ 
        & A-DGN & 0.0916{\scriptsize$\pm 0.0885$} & 0.2568{\scriptsize$\pm 0.2225$} & 0.1841{\scriptsize$\pm 0.1310$} & 0.2606{\scriptsize$\pm 0.1280$} \\ 
        & A-DGN$_{Dia}$ & 0.4260{\scriptsize$\pm 0.3247$} & 0.2502{\scriptsize$\pm 0.1612$} & 0.4441{\scriptsize$\pm 0.3340$} & 0.1614{\scriptsize$\pm 0.1060$} \\ 
        & \gls*{MeGA-MP}$\;\;\,$ & \textbf{0.0015}{\scriptsize$\pm 0.0076$} & \textbf{0.0017}{\scriptsize$\pm 0.0061$} & 0.1062{\scriptsize$\pm 0.1240$} & 0.2349{\scriptsize$\pm 0.2236$} \\ 
        & \gls*{MeGA-MP}$^+$ & \textbf{0.0015}{\scriptsize$\pm 0.0076$} & 0.0018{\scriptsize$\pm 0.0061$} & \textbf{0.0013}{\scriptsize$\pm 0.0037$} & \textbf{0.0008}{\scriptsize$\pm 0.0018$} \\ 
        \hline

       \multirow{8}{*}{\rotatebox[origin=c]{90}{Comp. Time (\textit{sec})}} 
       & NNConv & 0.0041{\scriptsize$\pm 0.0278$} & 0.0029{\scriptsize$\pm 0.0003$} & 0.0023{\scriptsize$\pm 0.0001$} & 0.0028{\scriptsize$\pm 0.0001$} \\ 
& \gls*{PDE}-\gls*{GNN} & 1.8648{\scriptsize$\pm 0.0168$} & 1.9251{\scriptsize$\pm 0.0162$} & 1.8926{\scriptsize$\pm 0.0231$} & 1.8925{\scriptsize$\pm 0.0170$} \\ 
& GPSConv & 0.0085{\scriptsize$\pm 0.0004$} & 0.0145{\scriptsize$\pm 0.0005$} & 0.0083{\scriptsize$\pm 0.0001$} & 0.0176{\scriptsize$\pm 0.0085$} \\ 
& A-DGN & 0.0266{\scriptsize$\pm 0.0004$} & 0.0271{\scriptsize$\pm 0.0004$} & 0.0260{\scriptsize$\pm 0.0004$} & 0.0267{\scriptsize$\pm 0.0001$} \\ 
& A-DGN$_{Dia}$ & 0.0074{\scriptsize$\pm 0.0001$} & 0.0420{\scriptsize$\pm 0.0003$} & 0.0073{\scriptsize$\pm 0.0001$} & 0.0418{\scriptsize$\pm 0.0004$} \\ 
& \gls*{MeGA-MP}$\;\;\,$ & 0.0447{\scriptsize$\pm 0.0055$} & 0.3964{\scriptsize$\pm 0.0189$} & 0.0449{\scriptsize$\pm 0.0049$} & 0.4063{\scriptsize$\pm 0.0202$} \\ 
& \gls*{MeGA-MP}$^+$ & 0.0450{\scriptsize$\pm 0.0050$} & 0.4068{\scriptsize$\pm 0.0194$} & 0.0450{\scriptsize$\pm 0.0050$} & 0.4036{\scriptsize$\pm 0.0173$} \\ \cdashline{2-6} 
& \textit{preprocessing} & 0.0920{\scriptsize$\pm 0.0295$} & 1.5156{\scriptsize$\pm 0.1350$} & 0.0992{\scriptsize$\pm 0.0299$} & 1.5184{\scriptsize$\pm 0.1270$} \\ \cdashline{2-6} 
& \textit{Simulator} & 0.6724{\scriptsize$\pm 0.1567$} & 31.6906{\scriptsize$\pm 2.6501$} & 0.7520{\scriptsize$\pm 0.1642$} & 34.8503{\scriptsize$\pm 2.9193$} \\
       \hline 
    \end{tabular}
\vspace{-5pt}
\end{table}

Table \ref{table_ComparisonBaselines} reports the average \gls*{MAE}\footnote{
    We clip the predictions of all baselines to the upper and lower bound of chlorine injections as especially PDE-GNN explodes for some samples in the L-Town evaluation dataset.}
for both configurations and for both topologies together with the standard deviation over all samples, nodes and time steps. 
Compared to all baselines, \gls*{MeGA-MP} and \gls*{MeGA-MP}$^+$ show superior performance in the dynamics they are designed for, while \gls*{MeGA-MP} can -- as expected -- not accurately model the additional reaction dynamics (cf. Section \ref{subsection_Experiment1.2_AblationStudies}).
Under zero-shot transfer from Hanoi to L-Town, they show no significant increase in error, whereas the performance of the baseline models degrades. An exception is the A-DGN$_{Dia}$ model, however, this baseline underfits in the training setting, indicating that the improvement in performance may be due to factors other than actual generalization. 
Table \ref{table_ComparisonBaselines} also reports the compute time required to run the models on individual samples (batch-size = 1). Since we use the transport times as input to each model, we report the latter without the time it takes to convert the flow field to transport times and instead disclose it as a separate row. Even when adding the preprocessing time  to the compute time, each of the models is faster on the large test network L-Town than the simulator that we use as the ground truth. The NNConv model is the fastest, as it is a relatively simple \gls*{MPNN}. 
The runtime complexity of \gls*{MeGA-MP} and \gls*{MeGA-MP}$^+$, on the other hand, additionally depend on the edge lengths and flow field, as detailed in Section \ref{subsection_FinalAlgorithm:MeGA-MP}. However, while \gls*{MeGA-MP} and \gls*{MeGA-MP}$^+$ are slower than most baselines, it is faster than PDE-GNN, which has second-best performance on the in-distribution setting.

We now compare the results qualitatively, first for the pure advection dataset, then for the advection-reaction dataset.

\textbf{Pure Advection Configuration:} 
In Figure \ref{figure_ComparisonBaselines_Advection_Hanoi_with_map}, we show the predictions of all models evaluated on the metric graph Hanoi. Note that in this case, \gls*{MeGA-MP} does not have to learn, while the other \gls*{GNN}-baselines do. 
The PDE-GNN models the time series the most accurate, the error mostly stems from misalignment, indicating a failure to properly model transport times. A-DGN can follow the dynamics but exhibits smoothing and fails to capture the high frequency components of the signal. NNConv and GPSConv clearly underfit the data. In contrast, \gls*{MeGA-MP} is able to accurately model the time series at all visualized nodes.
In Figure \ref{figure_ComparisonBaselines_Advection_L-Town_with_map}, we show the predictions of all models evaluated on L-Town. 
It can be observed that none of the baselines can precisely capture the signal at all nodes. PDE-GNN and A-DGN can follow the dynamics on some nodes, like nodes 1, 2, 3, and 8. These are the nodes closest to the injection/boundary nodes (stars), which suggests a failure to capture the long-range spatio-temporal dependencies between more distant nodes. \gls*{MeGA-MP} does not show a clear deviation from the ground truth.

\begin{figure}[!h]
    \centering
    \includegraphics[width=0.99\linewidth]{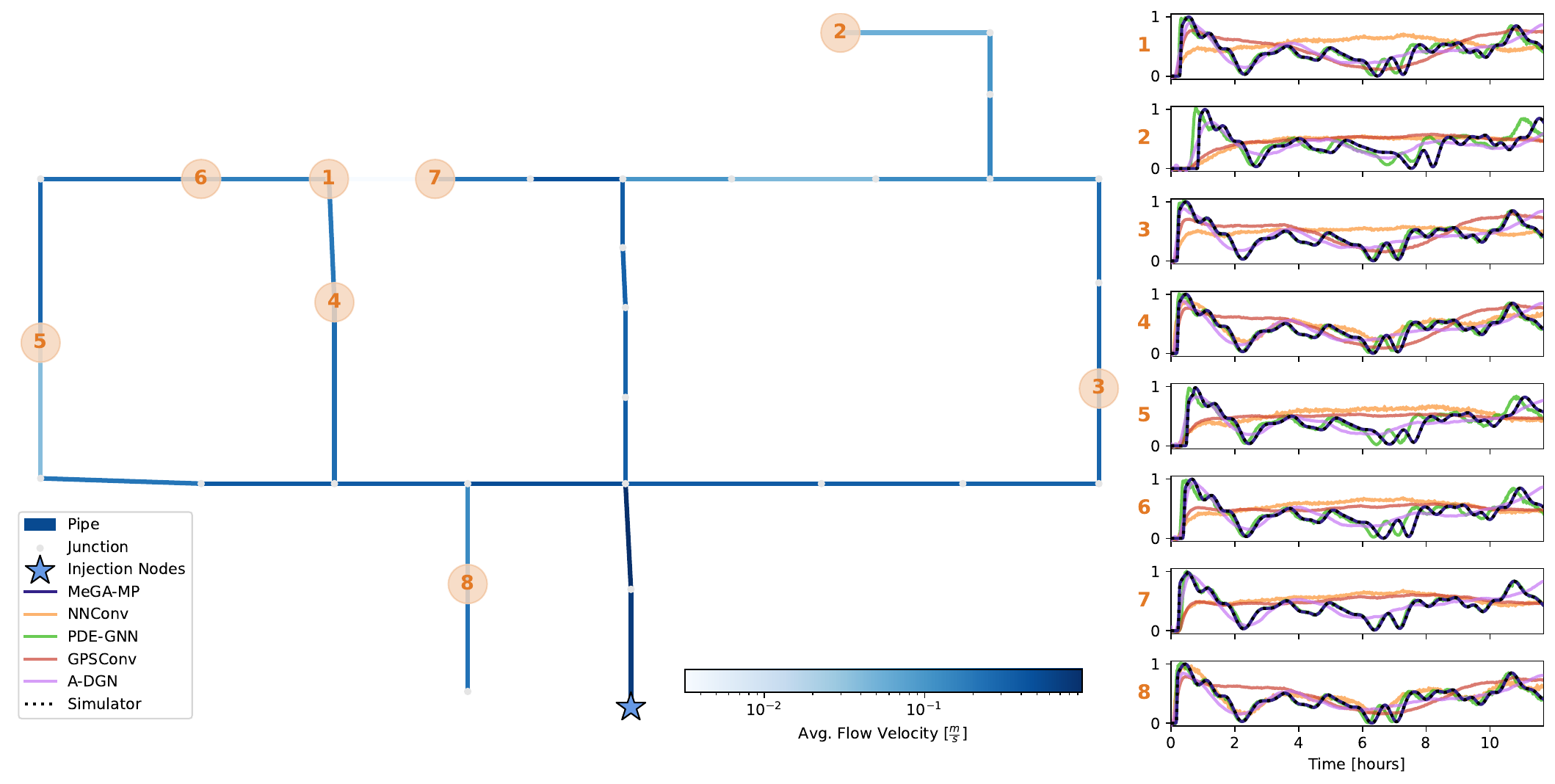}
    \caption{
    Predictions of baselines and \gls*{MeGA-MP} 
    for the \textbf{advection} dynamical system
    at a randomly chosen subset of nodes of the \acrshort*{WDS} \textbf{Hanoi}.
    Hanoi is the graph on which all models are trained.
    Edge color corresponds to the flow velocity averaged across all time steps of the sample.}
    \label{figure_ComparisonBaselines_Advection_Hanoi_with_map}
\end{figure}

\begin{figure}[!h]
    \centering
    \includegraphics[width=0.99\linewidth]{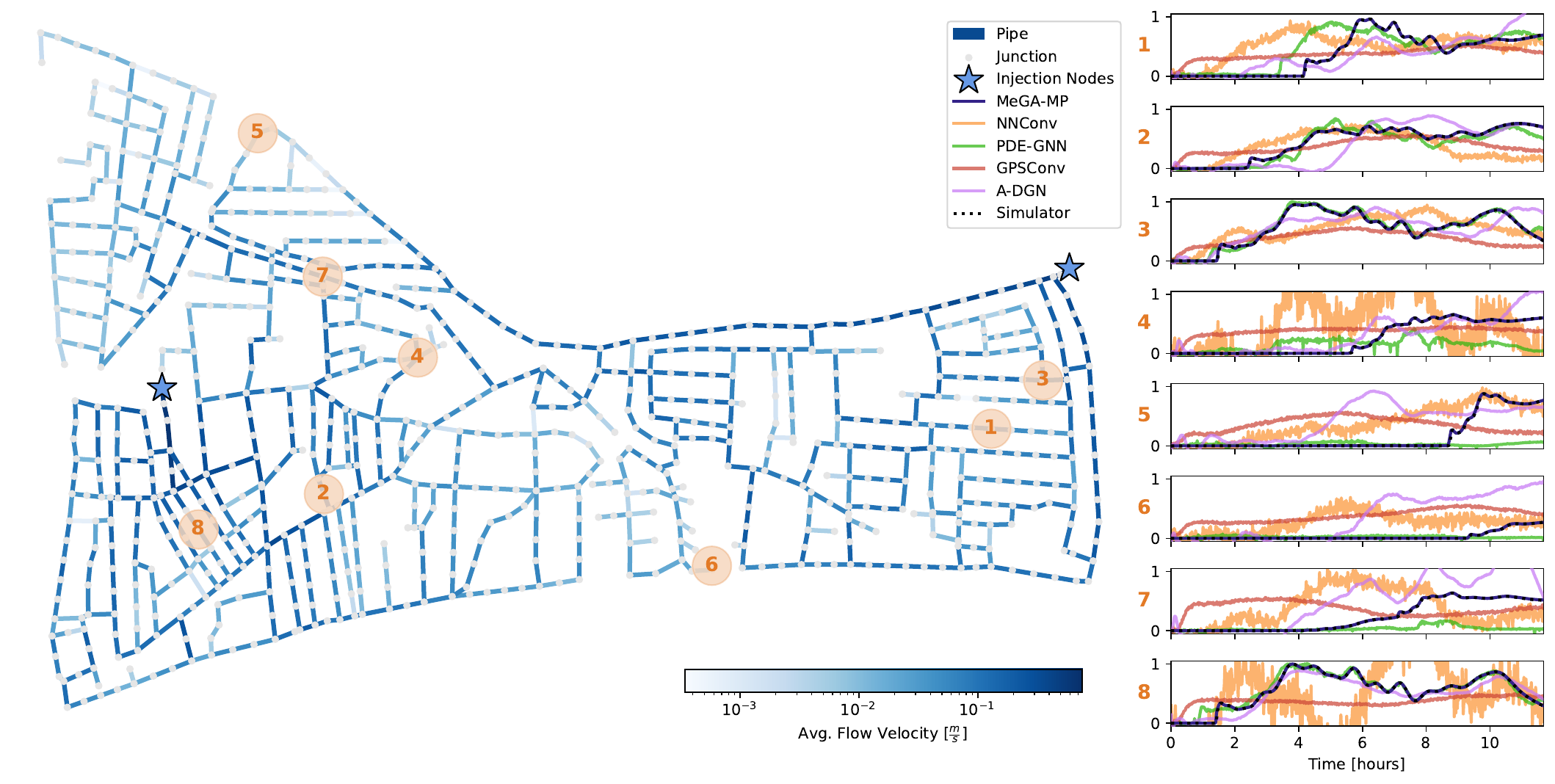}
    \caption{
    Predictions of baselines and \gls*{MeGA-MP} 
    for the \textbf{advection} dynamical system
    at a randomly chosen subset of nodes of the \acrshort*{WDS} \textbf{L-Town}.
    The L-Town metric graph is not part of the training set.
    Edge color corresponds to the flow velocity averaged across all time steps of the sample.}
    \label{figure_ComparisonBaselines_Advection_L-Town_with_map}
\end{figure}

\newpage
\textbf{Advection-Reaction Configuration:} 
In Figure \ref{figure_ComparisonBaselines_AdvectionReaction_Hanoi_with_map}, we show the predictions of all models evaluated on Hanoi. 
The observations are similar to the purely advective setting, the only significant difference is that A-DGN underfits here, and for most nodes resembles a highly-smoothed average signal, similar to the other two baselines NN-Conv and GPSConv. Again, \gls*{MeGA-MP}$^+$ accurately models the signal.
In Figure \ref{figure_ComparisonBaselines_AdvectionReaction_L-Town}, showing predictions of all models evaluated zero-shot (without any retraining) on L-Town, a similar degree of underfitting can be observed.  
Only \gls*{MeGA-MP}$^+$ can model the signal with high precision.

Overall, \gls*{MeGA-MP} and \gls*{MeGA-MP}$^+$ are the only models that perform consistently accurate over multiple nodes and generalize well from the training \gls*{WDS} Hanoi to the test \gls*{WDS} L-Town on the advection and the advection-reaction dynamics, respectively. The different generalization behaviors of the different baselines emphasize the challenges in modeling advection-dominated dynamics on metric graphs and why a suitable physical prior is necessary.

\begin{figure}[!h]
    \centering
    \includegraphics[width=\linewidth]{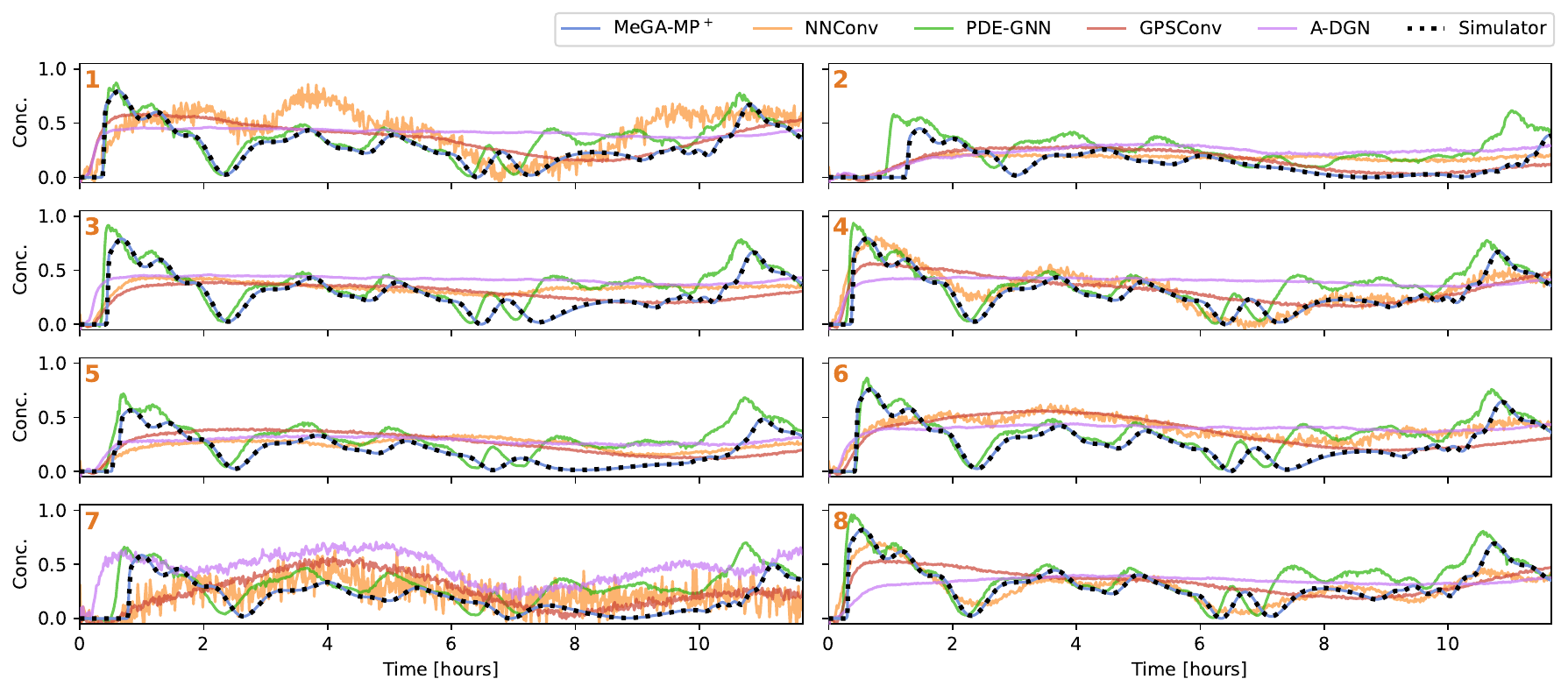}
    \caption{
    Predictions of baselines and \gls*{MeGA-MP}$^+$ 
    for the \textbf{advection-reaction} dynamical system 
    on the \acrshort*{WDS} \textbf{Hanoi}. 
    The selected nodes are the same as in Figure \ref{figure_ComparisonBaselines_Advection_Hanoi_with_map}.}
    \label{figure_ComparisonBaselines_AdvectionReaction_Hanoi_with_map}
\end{figure}

\begin{figure}[!h]
    \centering
    \includegraphics[width=\linewidth]{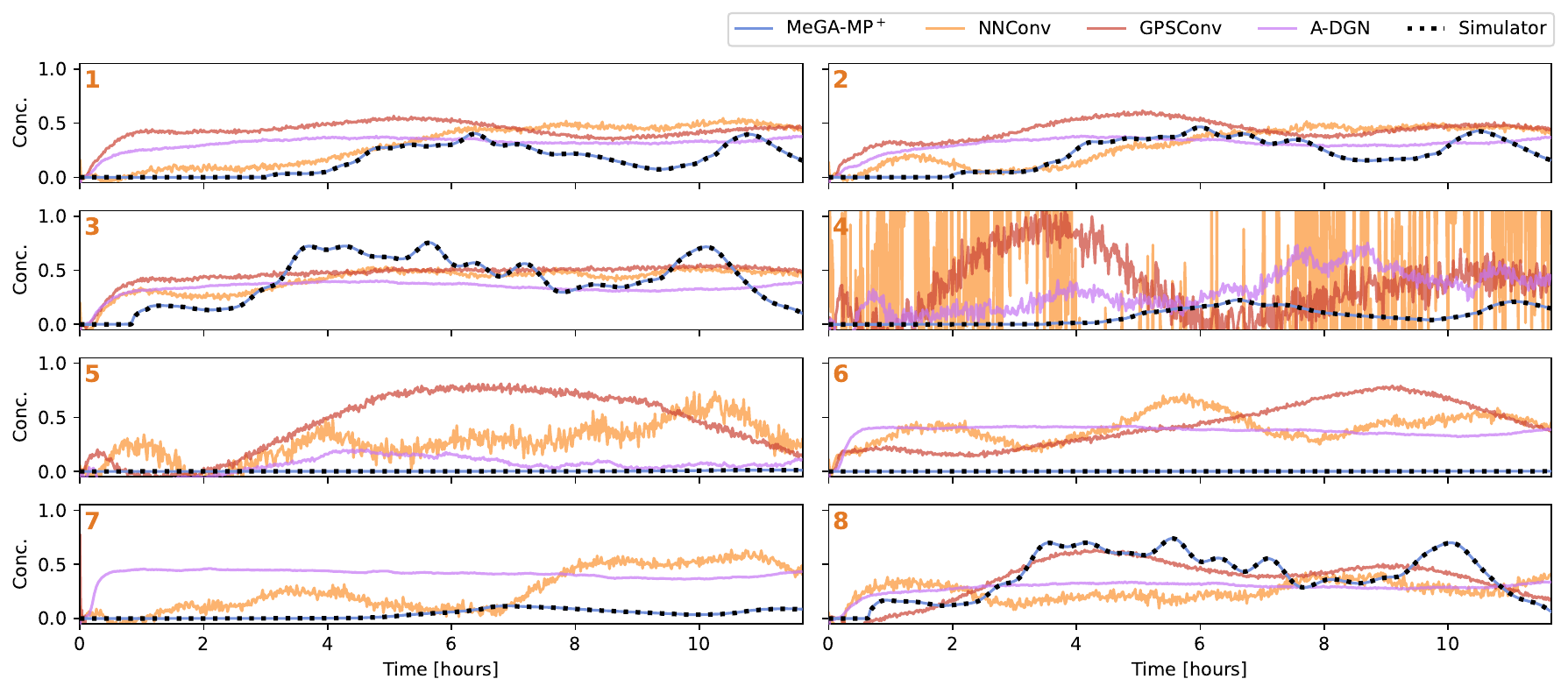}
    \caption{
    Predictions of baselines and \gls*{MeGA-MP}$^+$ 
    for the \textbf{advection-reaction} dynamical system
    on the \acrshort*{WDS} \textbf{L-Town}. 
    The selected nodes are the same as in Figure \ref{figure_ComparisonBaselines_Advection_L-Town_with_map}. 
    In the shown data sample, PDE-GNN explodes and has therefore been excluded.}
    \label{figure_ComparisonBaselines_AdvectionReaction_L-Town}
\end{figure}

\subsection{Experiment 1.2: Ablation Studies}
\label{subsection_Experiment1.2_AblationStudies}

In this section, we investigate how much the physics-informed components of \gls*{MeGA-MP} and \gls*{MeGA-MP}$^+$ contribute to the superior performance observed in Section \ref{subsection_Experiment1.1_ComparisonToBaselines}.

\paragraph{Models}
In our first ablation study, we investigate the contribution of the two components in the message function $\varphi^+$ of \gls*{MeGA-MP}$^+$, i.e., the non-learnable and physics-informed advection message function $\varphi$ (Eq. \eqref{align_MessagePassing_AdvectionIterations}/\eqref{align_MessageFunction_Advection_LinearInterpolation}), 
derived from the advection edge dynamics (Eq. \eqref{align_EdgeDynamics_Advection}) 
with inflow continuity condition (Eq. \eqref{align_CouplingCondition_InflowContinuity}),
and the learnable reaction \gls*{MLP} (Eq. \eqref{align_MessageFunction_AdvectionReaction}).
We do so by comparing the performance of \gls*{MeGA-MP} -- designed for advection dynamics -- and \gls*{MeGA-MP}$^+$ -- designed for advection-reaction dynamics -- on both the dataset for advection and advection-reaction dynamics.

In our second ablation study, we investigate the contribution of the message aggregation scheme (Eq. \eqref{align_MessagePassing_Advection}),
induced from the mixing at nodes (Eq. \eqref{align_CouplingCondition_MixingAtNodes}). 
We do so by training a model that has the same message aggregation structure as given in Equation \eqref{align_MessagePassing_Advection}, but where we replace the message function $\varphi$ not by $\varphi^+$, which yields \gls*{MeGA-MP}$^+$, but by a fully learnable \gls*{MLP} $\varphi^-$. Specifically, we design $\varphi^-$ as a 2-layer \gls*{MLP} with SeLU activation \cite{SELU2017} which has access to adjacent node information over the whole time horizon $T = \Tn{hi} \cup \Tn{fu}$ and the corresponding pre-processed transport times:
\begin{align*}
    \varphi^-(t,e_{uv},\cmhat^{(k-1)},\num)
    = 
    \text{MLP}\left( 
    \cmhat_v^{(k-1)}, 
    \cmhat_u^{(k-1)}, 
    \mathbf{\delta} \mathbf{t}_{uv}
    \right)_t,
\end{align*}

\noindent
where similar to $\cmhat^{(k-1)} \in \R^{|V| \times (\n{hi}+\n{fu})}$ (cf. Eq. \eqref{align_MessagePassing_AdvectionIterations}), $\cmhat_v^{(k-1)} \in \R^{\n{hi}+\n{fu}}$ corresponds to the vectorized version of the function $\cm = (c_v)_{v \in V}$ containing the values $\chat_v^{(k-1)}(t)$ for all $t \in T = \Tn{hi} \cup \Tn{fu}$,
and 
$\mathbf{\delta} \mathbf{t}_{uv}$ corresponds to the pre-processed vector $\mathbf{\delta t}_{uv} = (\delta t_{uv})_{t\in T_{hi}\cup T_{fu}} \in \mathbb{R}^{\n{hi} + \n{fu}}$ of transport times.
We denote this modification of \gls*{MeGA-MP} by \gls*{MeGA-MP}$^-$.

An overview of the different models in these ablations can be found in Table \ref{table_Ablations}.
The setup and datasets used in this ablation are equal to the setup of experiment 1 described in Section \ref{subsection_Experiment1.1_ComparisonToBaselines}.

\begin{table}[!ht]
\vspace{-2pt}
    \centering
    \caption{\acrshort*{MAE} between simulator ground truth (EPANET-MSX) and predictions of different modifications to \gls*{MeGA-MP} on the \textbf{advection} and \textbf{advection-reaction} dynamical systems averaged across samples, nodes and time steps of the test datasets. The differences between the models is highlighted by the first three columns.}
    \label{table_Ablations}
    \resizebox{\linewidth}{!}{
    \begin{tabular}{c||ccc||cc|cc}
       &
       \multirow{3}{*}{\rotatebox[origin=c]{0}{\begin{tabular}[t!]{@{}c@{}} \\ Msg. \\ update \\ via Eq. \eqref{align_MessagePassing_Advection} \end{tabular}}} & 
       \multirow{3}{*}{\rotatebox[origin=c]{0}{\begin{tabular}[t!]{@{}c@{}} \\ Msg. \\ function \end{tabular}}} &  
       \multirow{3}{*}{\rotatebox[origin=c]{0}{\begin{tabular}[t!]{@{}c@{}} \\ Learn. \\ comp. \end{tabular}}} 
        
       & \multicolumn{2}{c}{\textbf{Advection}} 
       & \multicolumn{2}{c}{\textbf{Advection-Reaction}} \\
       & & & 
       & Hanoi & L-Town & Hanoi & L-Town
       \\
       & & & &  (train/test graph) & (test graph) & (train/test graph) & (test graph)
       \\
       & & & & \scriptsize$|V| = 32, |E| = 34$ & \scriptsize$|V| = 784, |E| = 908$ & \scriptsize$|V| = 32, |E| = 34$ & \scriptsize$|V| = 784, |E| = 908$
       \\ \hline \hline 
        \gls*{MeGA-MP}~~ & $\checkmark$ & $\varphi$~~ & $\times$
        & \textbf{0.0015}{\scriptsize$\pm 0.0076$} & \textbf{0.0017}{\scriptsize$\pm 0.0061$} & 0.1062{\scriptsize$\pm 0.1240$} & 0.2349{\scriptsize$\pm 0.2236$} \\ 
        \gls*{MeGA-MP}$^+$ & $\checkmark$ & $\varphi^+$ & $\checkmark$
        & \textbf{0.0015}{\scriptsize$\pm 0.0076$} & 0.0018{\scriptsize$\pm 0.0061$} & \textbf{0.0013}{\scriptsize$\pm 0.0037$} & \textbf{0.0008}{\scriptsize$\pm 0.0018$} \\ 
        \gls*{MeGA-MP}$^-$ & $\checkmark$ & $\varphi^-$ & $\checkmark$
        & 0.0910{\scriptsize$\pm 0.0875$} & 0.2169{\scriptsize$\pm 0.1722$} & 0.0971{\scriptsize$\pm 0.0852$} & 0.1873{\scriptsize$\pm 0.3045$} \\ 
    \end{tabular}%
    }%
\end{table}


\paragraph{Results and Analysis}
Table \ref{table_Ablations} also shows the performance of the different modifications to \gls*{MeGA-MP}. 
As an extension of Table \ref{table_Ablations}, we provide corresponding figures in Appendix \ref{section_Experiment1_Zero-ShotTransferLearningOfAdvection-ReactionDynamicsOnRealWorldMetricGraphs_Appendix}.
Firstly, it is evident that \gls*{MeGA-MP}$^-$ performs significantly worse than the other two architectures in the purely advective setting.
This indicates that the full strength of our approach does not only come from the message aggregation scheme (Eq. \eqref{align_MessagePassing_Advection}),
induced from the mixing at nodes (Eq. \eqref{align_CouplingCondition_MixingAtNodes}, but also from the inductive bias in the advection message functions $\varphi$ (Eq. \eqref{align_MessagePassing_AdvectionIterations}/\eqref{align_MessageFunction_Advection_LinearInterpolation}) and thus, in $\varphi^+$ (Eq.  \eqref{align_MessageFunction_AdvectionReaction}),
derived from the advection edge dynamics (Eq. \eqref{align_EdgeDynamics_Advection}) 
with inflow continuity condition (Eq. \eqref{align_CouplingCondition_InflowContinuity}). Removing the physical priors from $\varphi$ and thus, $\varphi^+$ by replacing it with a fully learnable function $\varphi^-$ does not suffice to model the complex antisymmetric and long-range dependencies induced by the advection dynamics.
Secondly, the on-par-performance of \gls*{MeGA-MP} and \gls*{MeGA-MP}$^+$ in the advection setting indicates that the additional \gls*{MLP} is not required when advection is the only dynamic. 

In the case of advection-reaction, only \gls*{MeGA-MP}$^+$ is able to accurately capture the added complexity from the reaction term, emphasizing the utility of the simple modification in $\varphi^+$.
Due to the inability to learn dynamics, \gls*{MeGA-MP} fails to capture the reaction dynamic without the modification used in \gls*{MeGA-MP}$^+$. 
Furthermore, the fact that \gls*{MeGA-MP}$^-$ performs worse than \gls*{MeGA-MP}$^+$ indicates that the other \gls*{GNN}-based baselines from Section \ref{subsection_Experiment1.1_ComparisonToBaselines} do not perform worse because of a different message aggregation function.

In summary, our ablation studies show the utility of coding the advection characteristics directly into the model architecture. Without it, the model cannot reliably capture the challenging spatio-temporal dependencies. This  once again emphasizes the strength of physics-informed \gls*{ML} and carefully constructed inductive bias over purely data-driven approaches in advection-dominated settings.

\section{Experiment 2: Advection on a 1D Euclidean Domain}
\label{subsection_Experiments_AdvectionOnA1DEuclideanDomain}

In this section, we empirically validate \gls*{MeGA-MP} by using it as a numerical solver for advection in a domain where well-known versatile numerical solvers exist: On a 1-dimensional Euclidean domain. The advantage of this domain is that it also resembles a metric path graph. The goal of this experiment is to compare the characteristics of MeGA-MP to those of classical solvers. To this end, we compare our method against two classical baselines: A semi-Lagrangian solver, as well as a method-of-lines formulation integrating over time with a Runge-Kutta (RK4) solver using a WENO5 scheme \cite{WENO5_1996} to approximate the spatial derivatives. For details regarding the baselines, we refer to Appendix \ref{subsection_Experiments_AdvectionOnA1DEuclideanDomain_Appendix}.

\paragraph{Domain and Dataset}
The domain for this experiment is $\Omega = [0, l]$, describing a 1-dimensional space of length $l=100$.
We aim to predict the solution $c(t,z)$ of a synthetic 
\gls*{BVP} at equidistant grid points $t$ and $z$ of the time and space interval $[0,100]$ and $[0,100]$ for different temporal and spatial discretizations $dt$ and $dz$ from $\{\frac{100}{1000},\frac{100}{316},\frac{100}{100},\frac{100}{32},\frac{100}{10}\}$, respectively.
Conceptually, this models an injection at the boundary $z = 0$, advecting along $\Omega$ over time $t \in [0,100]$.
To enable a fair comparison, we choose the \gls*{BVP} such that an analytical solution exists.
The setup is described in detail in Appendix \ref{subsection_Experiments_AdvectionOnA1DEuclideanDomain_Appendix}.
In light of our discussion on \glspl*{IVP} and efficiency (Remark \ref{remark_InitialValueProblemsAndEfficiency}), we discuss a similar setup posed as an \gls*{IVP} with boundary values in Appendix \ref{subsection_Experiments_AdvectionOnA1DEuclideanDomain_Appendix}, too.

\begin{table}[b]
\centering
    \caption{\acrshort*{MAE} between analytical ground truth and different solver solutions over all temporal and spatial grid points. 
    The entry \enquote{--} indicates divergence of the method.
    Find the full version of this table and the runtime analysis in Appendix \ref{subsection_Experiments_AdvectionOnA1DEuclideanDomain_Appendix}.}
    \label{table_SyntheticBoundaryValueProblem}
    \begin{tabular}{cr||cccc} \toprule
        $dt$ & $dz$ & semi- Lagr. & RK4 & \gls{MeGA-MP}
        \\
\hline \multirow[t]{5}{*}{$0.1$} & \hfill$ 0.10$ & \textbf{0.0016}{\scriptsize$\pm 0.0050$} & -- & 0.0024{\scriptsize$\pm 0.0081$} \\ 
  & \hfill 1.00 & 0.0108{\scriptsize$\pm 0.0230$} & 0.0035{\scriptsize$\pm 0.0073$} & \textbf{0.0012}{\scriptsize$\pm 0.0068$} \\ 
  & \hfill10.00 & 0.0562{\scriptsize$\pm 0.0764$} & 0.0288{\scriptsize$\pm 0.0415$} & \textbf{0.0057}{\scriptsize$\pm 0.0182$} \\ 
\hline \multirow[t]{5}{*}{$1.0$} & \hfill$ 0.10$ & 0.0080{\scriptsize$\pm 0.0144$} & -- & 0.0157{\scriptsize$\pm 0.0298$} \\ 
  & \hfill 1.00 & 0.0097{\scriptsize$\pm 0.0137$} & \textbf{0.0088}{\scriptsize$\pm 0.0155$} & 0.0103{\scriptsize$\pm 0.0208$} \\ 
  & \hfill10.00 & 0.0545{\scriptsize$\pm 0.0720$} & 0.0291{\scriptsize$\pm 0.0395$} & \textbf{0.0115}{\scriptsize$\pm 0.0222$} \\ 
\hline \multirow[t]{5}{*}{$10.0$} & \hfill$ 0.10$ & \textbf{0.0763}{\scriptsize$\pm 0.0964$} & -- & 0.0939{\scriptsize$\pm 0.0981$} \\ 
  & \hfill 1.00 & \textbf{0.0728}{\scriptsize$\pm 0.0923$} & -- & 0.0915{\scriptsize$\pm 0.0952$} \\ 
  & \hfill10.00 & \textbf{0.0671}{\scriptsize$\pm 0.0842$} & 0.0728{\scriptsize$\pm 0.0936$} & 0.0761{\scriptsize$\pm 0.0763$} \\ 
        \hline
    \end{tabular}
\end{table}

\paragraph{Results and Analysis}
Table \ref{table_SyntheticBoundaryValueProblem} shows the \acrshort*{MAE} of the different solvers at different space-time resolutions. 
We find that \gls*{MeGA-MP} approximates the ground truth well and provides results competitive to other solvers. It can also be observed that \gls*{MeGA-MP} shows the best performance when the temporal resolution is high ($dt = 0.1$), while it is less affected by changes in spatial resolution. 
This is an expected result, as our method interpolates the signal at nodes temporally. Therefore, the interpolation error correlates with the temporal resolution, as derived in depth in Appendix \ref{subsection_ErrorBoundsOnMeGA-MP}.
This is in contrast to the baseline semi-Lagrangian solver, which interpolates spatially and thus benefits from higher spatial resolution. RK4 becomes unstable for certain combinations of spatial and temporal resolutions.

The \gls*{MeGA-MP} solutions for $dz = 1.0$ are visualized in Figure \ref{figure_SyntheticBoundaryValueProblem_MeGAMP_Dts}, where the numerical smoothing from temporal interpolation is visible, as the solution degrades with increasing $dt$. This is amplified because of the sharp peak and therefore high-frequency content of the signal. One can also observe that smoothing increases in $z$-direction, i.e. in the direction of the flow field, as this coincides with the direction of spatial integration and error accumulation (cf. Appendix \ref{subsection_ErrorBoundsOnMeGA-MP}).
More detailed visualizations of all combinations of space-time discretization, an extended version of Table \ref{table_SyntheticBoundaryValueProblem} and a runtime analysis can be found in Appendix \ref{subsection_Experiments_AdvectionOnA1DEuclideanDomain_Appendix}.

\begin{figure}[t]
    \centering
    \includegraphics[width=\linewidth]{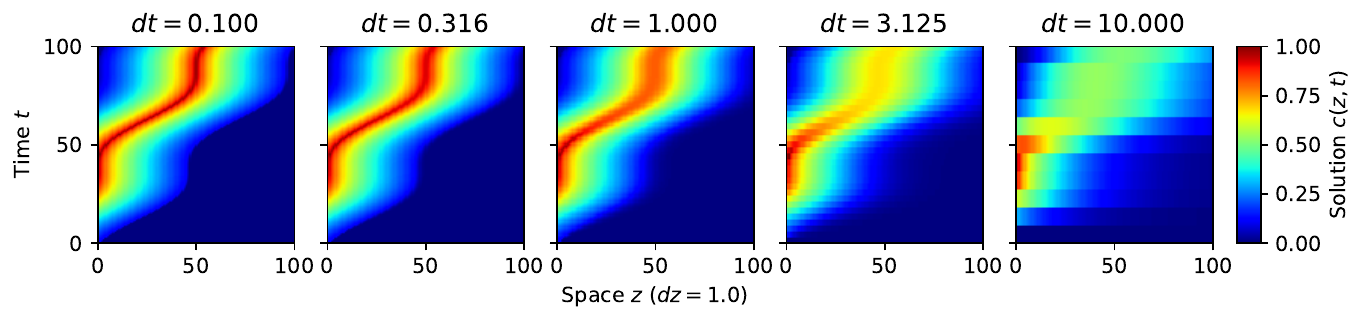}
    \caption{\gls*{MeGA-MP} solution of the \gls*{BVP} \eqref{align_SyntheticBoundaryValueProblem} for a fixed spatial discretization of $dz = 1.0$ for varying values of temporal discretizations $dt$. Each vertical pixel row corresponds to the time series at a node (i.e., a grid point), each horizontal row corresponds to one time step across all nodes. The sinusoidal flow field (Eq. \eqref{align_euclidean_flow_field}) is always positive, i.e. the injected signal at the left boundary ($z = 0$) is transported towards the right boundary ($z = 100$).}
\label{figure_SyntheticBoundaryValueProblem_MeGAMP_Dts}
\end{figure}


\section{Discussion \& Conclusion}
In this work, we have introduced \gls*{MeGA-MP}, a message-passing formulation solving linear advection dynamical systems on metric graphs. We have shown theoretically and empirically that it is a valid solver with a provable error bound and that it can be extended to solve more challenging, advection-dominated dynamics. In this work, we used an \gls*{MLP} as the extension. However, following the structural archetype of \glspl*{MPNN} in principle also allows more complex integrations with existing \glspl*{GNN}.

The comparison to other \gls*{PDE} solvers shows the saliency of our method, as our solution is comparable to that of established numerical solvers in scenarios where a comparison is possible.  
By experimenting on dynamics in \glspl*{WDS} we did not only show that our method can solve advection-dominated dynamics on metric graphs, but also that \gls*{MeGA-MP}($^+$) is directly applicable as a surrogate model for water quality, specifically the modeling of chlorine. It closely models the simulator EPANET-MSX, while being considerably faster and more flexible, as it integrates much better into the existing \gls*{ML} infrastructure.
Successful zero-shot transfer to unseen metric graphs further strengthens the positioning of \gls*{MeGA-MP}($^+$) as a viable surrogate model. 
We attribute this transferability to novel geometries mainly to the inductive bias in our message-passing scheme, allowing a separation of spatial integration of graph signals by \gls*{MeGA-MP} and data-driven adaptation to the additional non-spatial reaction dynamic using \gls*{MeGA-MP}$^+$. It is unclear to what extent the transferability persists when learnable spatial components are added. We will discuss this in the following sections, but argue that in the presence of advection, \gls*{MeGA-MP} always contributes a suitable inductive bias at low cost.

\paragraph{Limitations} 
\gls*{MeGA-MP}($^+$) is specifically designed for linear advection-dominated dynamics on metric graphs and is naturally limited to \glspl*{PDE} where linear advection is the dominant spatial dynamic. 
The approximation error is controlled by two factors: the temporal discretization and the number of message passing iterations required to propagate information across the graph. 
The latter depends on the edge lengths and flow velocities; short lengths and fast flows require more iterations, leading to greater accumulation of temporal interpolation errors. 
Both effects are consistent with our theoretical error bound (Theorem \ref{theorem_InterpolationError}) and empirically supported by Experiment 2 (Section \ref{subsection_Experiments_AdvectionOnA1DEuclideanDomain}). 
We therefore recommend choosing sufficiently fine temporal discretizations relative to the flow velocities in the system.
Beyond numerical accuracy, \gls*{MeGA-MP}($^+$) currently assumes exact knowledge of the flow field, which could be unavailable in some practical settings. To this end, \gls*{MeGA-MP}($^+$) does not provide explicit uncertainty quantification, but the low execution time allows feasible Monte-Carlo estimates of uncertainty. 

\paragraph{Future Work}
Several directions can build naturally on this work.
First, extending \gls*{MeGA-MP} to more complex spatial dynamics, such as semi- and non-linear advection or advection-diffusion, and to more complex temporal dynamics, such as multi-species reactions, via additional learnable operators is a promising avenue. 
The latter, in particular, would bring this method closer to a full \gls*{ML}-driven surrogate model for the simulator EPANET-MSX for \glspl*{WDS}.
Second, exploring inverse problems such as identifying flow fields or reaction coefficients from observational data would not only address the current limitation of requiring exact flow field knowledge, but would further emphasize the practical advantages of such surrogates over numerical simulators.
Thirdly, a systematic investigation of robustness to noisy or incomplete parameters such as flow fields is an important prerequisite for deployment of such systems in the real-world.
Finally, metric graphs governed by other \glspl*{PDE} beyond advection-dominated edge dynamics pose fundamentally different challenges. A systematic study of these different problem classes, including benchmark datasets and baseline evaluations, is an important direction for future work and the comparability of different methods across domains.

\section*{Impact Statement}
This work presents a physics-informed \gls*{ML} method for modeling spatio-temporal dynamics on advection-dominated metric graph domains. 
The data used in this work is self-generated and does not contain personal or sensitive information, ensuring no privacy concerns under GDPR or comparable regulations, or ethical concerns. 
However, real-world deployment of such methods may require sensor data or consumption records that can be linked to individual consumers. 
In such cases, data collection and processing must comply with applicable data protection regulations, and anonymization or aggregation techniques should be applied where possible.

The primary intended application of this work is modeling in scientific and engineering context, with the goal of improving the efficiency and accuracy of simulations. 
We also acknowledge the potential broader impacts of physics-informed \gls*{ML} including misuse and emphasize that responsible development and independent validation are essential. 
For example, in safety-critical settings, incorrect model predictions could have negative consequences for the provision of essential services.
Therefore, before deploying such systems in the real world, a thorough investigation of failure modes is essential, including error-proneness under distribution shift, sensitivity to misspecified physical parameters and the absence of uncertainty quantification in the current formulation.

\section*{LLM Statement}
This research was partially assisted by using \glspl*{LLM} such as ChatGPT, Gemini or Claude.
More precisely, \glspl*{LLM} were used for retrieval and discovery, such as finding related work. 
However, all related work cited in this paper was read and verified by the authors.
Moreover, \glspl*{LLM} were used to generate boilerplate code that has been checked for correctness by the authors.
Finally, \glspl*{LLM} were used to assist with editing and proofreading of selected parts of the manuscript with a focus on improving clarity and readability. 
The scientific content, the experimental design, the code implementing the novel methods presented in this work, the results, most parts of the text in this manuscript as well as the theory and equations were not generated by \glspl*{LLM}.

\bibliography{main}
\bibliographystyle{tmlr}


\newpage
\appendix
\onecolumn

\section*{Organization of the Supplementary Material}
\startcontents[sections]
\printcontents[sections]{l}{1}{\setcounter{tocdepth}{3}}

\newpage
\section{Details on the Background}
\label{section_Background_Appendix}

\paragraph{Advection on Metric Graphs}
Generally speaking, advection describes transport through the presence of a flow field $\nu_{uv}: \R \times \Omega_{uv} \rightarrow \R$.
For each time $t \in \R$, the flow $\nu_{uv}(t,\cdot): \Omega_{uv} \rightarrow \R$ defines a one-dimensional vector field along the edge space $\Omega_{uv} = (0,l_{uv})$. 
By definition of incompressibility, $\divtext(\nu_{uv}(t,\cdot)) = \frac{\partial \nu_{uv}(t,\cdot)}{\partial z} = 0$ holds,
or equivalently, 
the vector field $\nu_{uv}(t,\cdot)$ is constant on $\Omega_{uv} = (0,l_{uv})$, modeling the edge $e_{uv} \in E$.
Analogously, 
the vector field $\nu_{vu}(t,\cdot)$ is constant on $\Omega_{vu} = (0,l_{vu}) = (0,l_{uv})$, modeling the contrasting edge $e_{vu} \in E$.
Since both constant vector fields $\nu_{uv}(t,\cdot)$ and $\nu_{vu}(t,\cdot)$ model the same physical phenomenon, namely the constant, actually observable \textit{physical} flow field between the nodes $v \in V$ and $u \in \Nbh(v)$, but along contrasting directions (the beginning of the edge $e_{uv}$ is the end of the edge $e_{vu}$ and vice versa), the orientation of $\nu_{vu}(t,\cdot)$ must be the opposite of the one of $\nu_{uv}(t,\cdot)$. Since they both correspond to the same, constant physical flow field, $\nu_{vu}(t,\cdot) = -\nu_{uv}(t,\cdot)$ (Eq. \eqref{align_ConservationOfFlows}) must hold, and since they do not depend on space, we omit the spatial dependency.

Together with the convention of signs (cf. Section \ref{section_Background}), this allows to formally define the inflow neighbors from $v \in V$ at time $t \in \R$, i.e., neighbors where there is a flow from $u \in \Nbh(v)$ \textit{in}to $v \in V$ at time $t \in \R$, as
\begin{align}
\label{align_InflowNeighbors}
    \Nbh_{-}(v,t) 
    := 
    \{ u \in \Nbh(v) ~|~ \sgn(\nu_{vu}(t)) = -1 \} 
    = 
    \{ u \in \Nbh(v) ~|~ \sgn(\nu_{uv}(t)) = 1 \}.
\end{align}

\noindent 
While incompressibility of a single flow field $\nu_{uv}(t,\cdot)$ implies the independence of the spatial variable $z \in \Omega_{uv}$, the incompressibility of the overall flow field $\num = (\nu_e)_{e \in E}$ translates to the fact that the sum of flow rates $q_{uv}(t) = a_{uv} \nu_{uv}(t)$ into a node $v \in V$ is equal to the sum of the flow rates $q_{vu}(t) = a_{vu} \nu_{vu}(t)$ out of $v$:
\begin{align*}
    \sum_{u \in \mathcal{N}_-(v,t)} 
    q_{uv}(t)
    =
    \sum_{u \in \mathcal{N}_+(v,t)} 
    q_{vu}(t).
\end{align*}

\paragraph{The Challenges of Advection-dominated Dynamics on Metric Graphs}
Advection-dominated dynamical systems on metric graphs entail the challenge of complex, antisymmetric, and spatio-temporal dependency structures.
We want visualize this using the simple example of an advection metric path graph, that is, a sequence of nodes $(v_0,v_1,...,v_n)$ with edges $e_{j(j+1)}$ in-between two consecutive nodes $v_j$ and $v_{j+1}$.
We assume that the corresponding edge spaces $\Omega_{j(j+1)}$ are of equal size $l_{j(j+1)} = 1m$ and have a constant flow field of $\nu_{j(j+1)} = 2m/s$  for all $e_{j(j+1)} \in E$.
If we access information at each of the nodes $v_j$ at discrete time points $t_0,...,t_i,t_{i+1},...$ with temporal discretization $dt = t_{i+1} - t_i = 60s$, information travels $dt \cdot \nu_{j(j+1)} = 60s \cdot 2m/s = 120m$ per time step $dt$, which corresponds to $\frac{dt \cdot \nu_{j(j+1)}}{l_{j(j+1)}} = 120$ hops in the path graph. By this, the signal $c_{v}(t_i)$ at a node $v$ at time $t_i$ depends on the signal $c_{u}(t_j)$ at time $t_j$ and a node $u$ which lies 120 hops upstream of $v$. 
If such path does not exist, the path will be shorter and will originate at boundary nodes. 
The above example provides a sketch how the spatial dependency structure depends on the lengths, the flow field and the temporal discretization. 

\begin{figure}[!h]
\centering
    \includegraphics[width=0.7\linewidth]{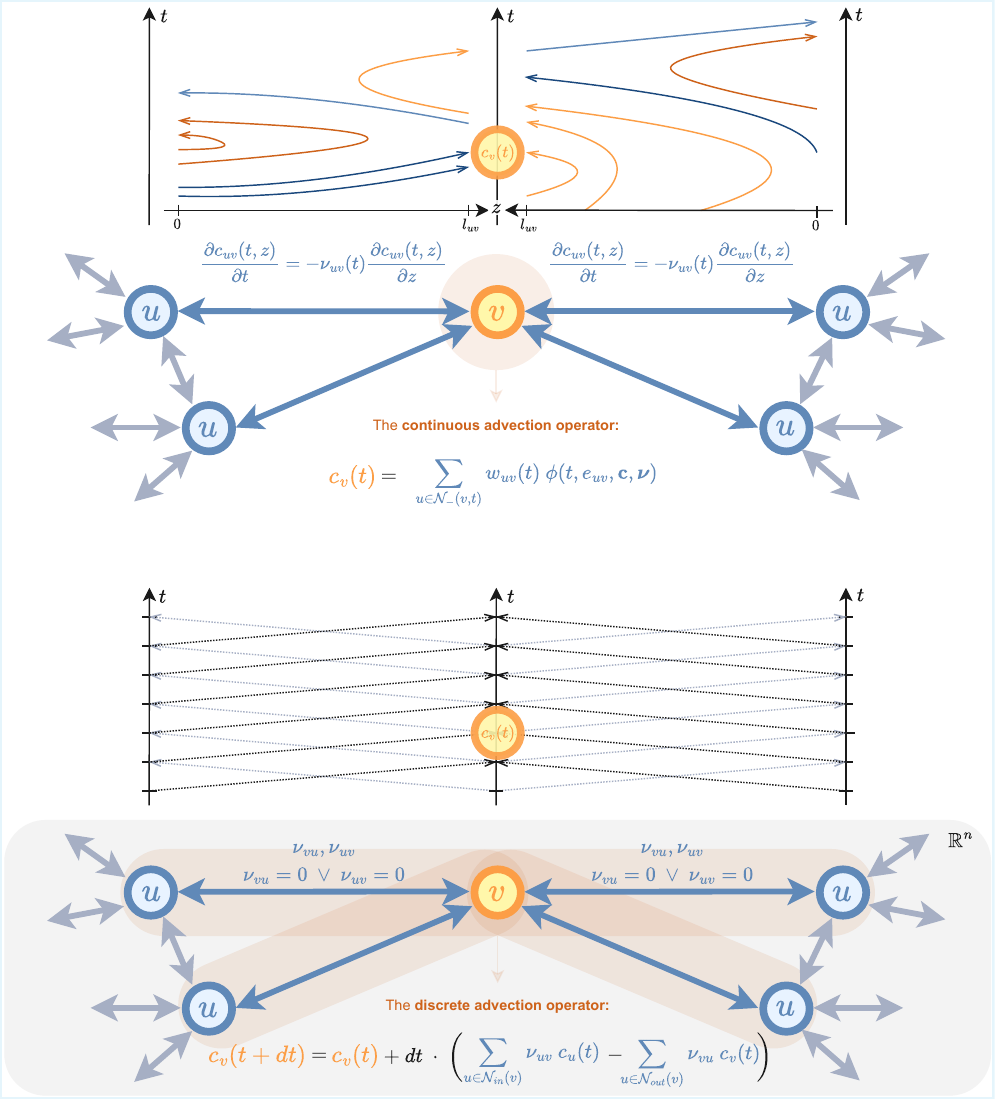}
    \caption{
    Advection on a metric graph (top) vs. on a non-metric graph (bottom):
    A metric graph is a system of one-dimensional subspaces $\Omega_{uv} = (0,l_{uv})$, corresponding to edges $e_{uv} \in E$, which are coupled at nodes $V$.
    Each edge $e_{uv}$ is associated with a one-dimensional \gls*{PDE} that defines the dynamics of the signal $c_{uv}(t,z)$ over time $t \in \R$ and along the edge space $z \in \Omega_{uv} = (0,l_{uv})$, parametrized by a known function $\nu_{uv}$.
    In case of advection, the known functions $\num = (\nu_{uv})_{e_{uv} \in E}$ correspond to the time-dependent flow field of the metric graph affecting the spatio-temporal evolution of the signals $c_{uv}$ (\textcolor{pass-through}{dark blue}, \textcolor{inverse-pass-through}{light blue}, \textcolor{self-loop}{dark orange} and \textcolor{inverse-self-loop}{light orange} lines).
    The coupling condition that defines the signal $c_v(t)$ at a node $v \in V$ and over time $t \in \R$ can be expressed as a weighted aggregation of features $\phi(t,e_{uv},\cm,\num) = c_{uv}(t,l_{uv})$ at the boundary $z = l_{uv}$ (\textcolor{inverse-self-loop}{orange-shaded ball} around $v$) of the edge spaces $\Omega_{uv}$ associated with incoming edges $e_{uv}$ (cf. Eq. \eqref{align_CouplingCondition_MixingAtNodes_MessagePassingFormulation}).
    Leveraging the edge dynamics, their value can be traced back to either features $c_u(t-\delta t_{uv})$ of neighboring nodes $u \in \Nbh_-(v,t)$ (\textcolor{pass-through}{dark blue lines}) or features $c_v(t-\delta t_{uv})$ of the node itself (\textcolor{inverse-self-loop}{light orange lines}) and at earlier, over the time and neighboring nodes varying times $t-\delta t_{uv}$ (cf. Eq. \eqref{align_MessageFunction_Advection}).
    \\
    In contrast, a (conventional, non-metric) graph can be considered as a discretization of the spatial Euclidean space $\Omega = \R^n$ (\textcolor{gray}{gray area}). 
    Here, edges $e_{uv} \in E$ preserve proximity in between nodes $V$, corresponding to coordinates in the Euclidean domain.
    The overall space $\Omega = \R^n$ is associated with a multi-dimensional \gls*{PDE} that defines the dynamics of the signal $c(t,v) := c_v(t)$ over time $t \in \R$ and space $v \in V \subset \Omega$, parametrized by a function $\num$. 
    In case of advection, the discretized advection-operator of the underlying multi-dimensional and continuous advection-\gls*{PDE} $\frac{\partial c_v(t)}{\partial t} = - \num(t) \cdot \nabla c_v(t,\zm)$ defines a node update $c_v(t+dt)$ that aggregates the net inflow from neighboring nodes $u \in \Nbhin(v)$ minus the outflow to neighboring nodes $u \in \Nbhout(v)$ and at a fixed time $t \in \R$ (black and \textcolor{gray}{gray} lines and \textcolor{inverse-self-loop}{orange-shaded area} aorund $v$). 
    In practice, the discretized version $\num = (\nu_{uv})_{e_{uv} \in E}$ of the continuous flow field is a learned parameter of a \gls*{GNN}.
    }
    \label{figure_MetricGraphVsNonMetricGraph}
\end{figure}

For more complex structured graphs and more complex configurations of \gls*{PDE}-parameters, the spatial dependency structure will be a tree that extends from the target (root) node $v$ at time $t$ to upstream nodes (leafs) as the node $u$ in the example above (also see Appendix \ref{subsection_ErrorBoundsOnMeGA-MP}). 
The dependency tree will extend further in directions with high flow rate and less in direction where the flow rate is lower.
This results in two major challenges for standard \glspl*{MPNN}: Firstly, the extent of the receptive field can vary drastically across the space of the metric graph and over time due to varying \gls*{PDE}-parameters, making it impossible to find a single suitable number of message passing layers. Secondly, the receptive field can be highly asymmetric. 
These two properties require highly selective and anisotropic message passing that depends on the structure of the metric graph, flow field and the temporal discretization. 

While there already exist work that discusses graphs and \glspl*{GNN} specifically taylored to advection problems \cite{Chapman2015Graphs_AdvectionOnGraphs,Eliasof2024PI-GNNs_Methods_ADR-GNN}, the underlying graph structure corresponds to a discretization of the Euclidean space; a paradigm different from the one on metric graphs (cf. Section \ref{section_Background}). 
We highlight this difference in Figure \ref{figure_MetricGraphVsNonMetricGraph}.


\newpage
\section{Details on the Method}
\label{section_DetailsOnTheMethod}

\paragraph{Message Passing}
has become a common principle in \glspl*{GNN} \cite{Gilmer2017GNNs_Methods_MPNN}. It describes how local information at the nodes of a graph is passed to neighboring nodes. 
Message passing as defined in \citet{Gilmer2017GNNs_Methods_MPNN} consists of three consecutive steps: Message generation, message aggregation, and node update. 
Message generation describes how a feature vector at one node is transformed as it is sent to a neighbor. 
Message aggregation and node update describe how multiple messages arriving at a node are accumulated into a single feature vector and update the features at the target node. 
A general formulation of this method is given by
\begin{align}
\label{align_MessagePassing}
    \xm_v(t+1) = \theta\left(\sum_{u \in \Nbh(v)} \phi \big(\xm_u(t), \xm_v(t), \qm_{uv}(t)\big) ,~ \xm_v(t)\right),
\end{align}

\noindent where $\phi$ is the message generation function, $\sum$ is a permutation-invariant aggregation function (e.g. sum or mean), $\theta$ is a function that updates the features $\xm_v$ of node $v \in V$, and $\qm_{uv}$ are optional edge features of the edge $e_{uv} \in E$ from node $u \in \Nbh(v)$ to node $v \in V$. Moreover, $t$ denotes the message-passing iteration step.

\subsection{Final Algorithm: MeGA-MP}
\label{subsection_FinalAlgorithm:MeGA-MP_Appendix}

\paragraph{Interpolation}\label{par_MoCInterpolationOperator}
Different interpolation schemes can be applied for time warping, depending on the use case. 
A version of linear interpolation of a value $c_u(t - \delta t_{uv})$ can be modeled as 
\begin{align}
\label{align_LinearInterpolation}
    c_u(t - \delta t_{uv}) 
    \approx~
    a(t_j)
    \cdot 
    c_u(t_j)
    +
    a(t_{j+1})
    \cdot 
    c_u(t_{j+1})
\end{align}

\noindent 
with $a(t_{j+1}) :=  \frac{t - \delta t_{uv} - t_j}{t_{j+1} - t_j}$ and $a(t_j) := 1 -  \frac{t - \delta t_{uv} - t_j}{t_{j+1} - t_j}$ 
for the $j \in I$ such that $t - \delta t_{uv} \in [t_j,t_{j+1})$ holds.

\paragraph{Model Iterations}
Based on Equation \eqref{align_LinearInterpolation}, we formally define $\varphi(t,e_{uv},\cmhat^{(k)},\num)$ from Equation \eqref{align_MessagePassing_AdvectionIterations} as

\begin{align}
\label{align_MessageFunction_Advection_LinearInterpolation}
\begin{split}
    \varphi(t,e_{uv},\cmhat^{(k)},\num) 
    :=
    \begin{cases}
        a(t_j) \cdot \chat_u^{(k)}(t_j)
        +
        a(t_{j+1}) \cdot \chat_u^{(k)}(t_{j+1})
        & \text{ if } 
        \zcurl(\delta t_{uv}) \in \{-l_{uv},l_{uv}\}
        \\
        a(t_j) \cdot \chat_v^{(k)}(t_j)
        +
        a(t_{j+1}) \cdot \chat_v^{(k)}(t_{j+1})
        & \text{ if } 
        \zcurl(\delta t_{uv}) = 0
    \end{cases}
\end{split}
\end{align}

\noindent 
for the $j \in I$ for which $t - \delta t_{uv} \in [t_j,t_{j+1})$ holds.
In practice, we precompute the transport times $\delta t_{uv}$ and implement $\varphi$ as an interpolation operator $\mathcal{I}[c_{w}^{(k)}](t - \delta t_{uv})$ at the correct node $w \in \{u,v\}$ that does not depend on the flow field $\num$ anymore (cf. Algorithm \ref{algorithm:MeGA-MP}).

Further details on the implementation can be found in algorithms \ref{algorithm_FlowToTransportTime} and \ref{algorithm:MeGA-MP}.

\begin{algorithm}
\caption{Preprocessing: Converting flows $\nu_{uv}(t)$ to transport times $\delta t_{uv}(t)$ (Definition \ref{definition_TransportTimes}). This algorithm is implemented in a parallelized/vectorized form for efficiency.}
\label{algorithm_FlowToTransportTime}
\begin{algorithmic}[1]  
\State \textbf{Inputs:} Flow velocities $\nu_{uv} \in \mathbb{R}^n$ for edge $e_{uv} \in E$ for $n$ time steps ($n = \n{hi} + \n{fu}$); Edge length $l_{uv} \in \mathbb{R}$ for edge $e_{uv} \in E$.
\State $C \gets \text{cumsum}(\nu_{uv}) \in \mathbb{R}^{n+1}$ \Comment{cumulative distance traveled over time, with $C_0 = 0$}
\State $\Delta T \in \mathbb{R}^{n}$ \text{ init. empty} \Comment{array to store transport times}
\State $M^{sl} \in \{0,1\}^{n}$ \text{ init. zero} \Comment{array to store if self-loop occurs at time $t$} \\
\Comment{Compute (inverse) pass-through transport times}
    \For{$t = 0 \textbf{ to } n-1$} \Comment{Trace characteristic curve starting at time $t$}
        \For{$x = t \textbf{ down to } 0$} \Comment{Trace backward along time steps}
            \State $d_{x,t} \gets C_t - C_x$ \Comment{distance traveled between $x$ and $t$}
            \If{$d_{x,t} \ge l_{uv}$} \Comment{true for pass-throughs}
                \State interpolate $\hat{x}$ between $C_x$ and $C_{x+1}$ if needed
                \State $\Delta T_t \gets t - \hat{x}$ \textbf{; break}
            \ElsIf{$d_{x,t} \le -l_{uv}$} \Comment{true for inverse pass-throughs}
                \State interpolate $\hat{x}$ between $C_x$ and $C_{x+1}$ if needed
                \State $\Delta T_t \gets \hat{x} - t$ \textbf{; break}
            \EndIf
        \EndFor
    \EndFor
\Comment{Compute self-loop transport times}
    \For{$t = 1 \textbf{ to } n-1$} \Comment{Trace characteristic curve starting at time $t$}
        \For{$x = t \textbf{ down to } 1$}
            \State $d_1 \gets C_{x-1} - C_t$; $d_2 \gets C_x - C_t$ \Comment{distance differences between consecutive steps}
            \If{$\text{sign}(d_1) \neq \text{sign}(d_2)$} \Comment{zero-crossing indicates self-loop}
                \State interpolate zero crossing: $\hat{x} \gets x - d_2 / (d_2 - d_1)$
                \State $\Delta T_t \gets \text{sign}(d_2) \cdot (t - \hat{x})$
                \State $M^{sl}_t \gets 1$  \textbf{; break}
            \EndIf
        \EndFor
    \EndFor
\State \textbf{Output:} $\Delta T \in \mathbb{R}^{n}$; $M^{sl} \in \{0,1\}^{n}$ \Comment{$\Delta T_t > 0 \Leftrightarrow$ transport in edge dir., $\Delta T_t < 0 \Leftrightarrow$ transport opposing edge dir.}
\end{algorithmic}
\end{algorithm}

\begin{algorithm}
\caption{MeGA-MP: Metric Graph Advection Message Passing}
\label{algorithm:MeGA-MP}
\begin{algorithmic}[1]  
    \State \textbf{Inputs:} Graph structure $\mathcal{G} = (V, E)$; Initial node states / known history $x \in \mathbb{R}^{|V| \times k_{hi}}$; Transport times $\delta t_{uv}(t)$ for each edge $e_{uv} \in E$ and for each time step $t_i \in T$, where $|T| = n = \n{hi} + \n{fu}$.
    \State $x^{(0)} \gets \text{pad}(x, \n{fu})$ \Comment{pad initial condition by prediction length $\n{fu}$}
    \State Apply Dirichlet boundary conditions to $x^{(0)}$ at boundary nodes
    \State $x_{\text{out}} \gets x^{(0)}$ \Comment{initialize output accumulator}
    \For{$i = 1$ \textbf{to} $\kappa$} \Comment{Iterate until convergence, c.f. Theorem \ref{theorem_ConvergenceOfModelIterations}}
        \For{each node $v \in V$ and time $t$} \Comment{Message-Passing}
            \State $\begin{aligned}\textstyle
                x^{(i)}_{v, t} \gets 
                \sum_{u \in \Nbh(v)} w_{uv}(t) ( 
                & \mathcal{I}[x^{(i-1)}_u](t - \delta t_{uv}(t)) \mathbb{1}_{M^{sl}_{uv}(t) = 0} \\
                + \quad
                & \mathcal{I}[x^{(i-1)}_v](t - \delta t_{uv}(t)) \mathbb{1}_{M^{sl}_{uv}(t) = 1})
            \end{aligned}$
            \Comment{cf. Eq. \ref{align_CouplingCondition_MixingAtNodes_MessagePassingFormulation} and Sec. \ref{par_Interpolated_MsgFunction} and \ref{par_MoCInterpolationOperator}}
            
        \EndFor
        \State $x^{(i)}_{v, t} \gets 0 \quad \forall v \in V_b, \, \forall t \in T$ \Comment{Set to zero at boundary nodes.}
        \State $x_{\text{out}} \gets x_{\text{out}} + x^{(i)}$
    \EndFor

    \If{boundary values exist} 
        \State Set $x_{\text{out}}$ at boundary nodes
    \EndIf

    \State \textbf{Output:} $x_{\text{out}}$ \Comment{final node states after advection message passing}
\end{algorithmic}
\end{algorithm}

\newpage
\section{Details on Theoretical Guarantees}
\label{section_DetailsOnTheoreticalGuarantees_Appendix}

\subsection{Physical Derivation of MeGA-MP}
\label{subsection_PhysicalDerivationOfMeGA-MP_Appendix}

While in Section \ref{subsection_PhysicalDerivationOfMeGa-MP}, we provide an intuition on how \gls*{MeGA-MP} incorporates the physical priors from Section \ref{section_Background},
in this section, we present a rigorous derivation of our metric graph advection message passing algorithm based on these priors. 
More precisely, we prove that 
Equation \eqref{align_MessagePassing_Advection} 
is equal to 
the mixing at nodes (Eq. \eqref{align_CouplingCondition_MixingAtNodes}) 
given 
the advection dynamics (Eq. \eqref{align_EdgeDynamics_Advection}) 
with inflow continuity condition (Eq. \eqref{align_CouplingCondition_InflowContinuity}) 
and 
suitable smoothness assumptions (Assumption \eqref{assumption_SmoothnessAssumptions}).
We repeat these priors for convenience while giving an overview of the rest of this section.

In Section \ref{subsubsection_MessageGeneration_AdvectiveTransportAlongEdges},
we derive the message-generation function $\phi$ as defined in Equation \eqref{align_MessageFunction_Advection}
based on the \textit{advection edge dynamics}.
For each $e_{uv} \in E$, 
this refers to 
the change of the concentration $c_{uv}: \R \times \Omega_{uv} \rightarrow \R_{\geq 0}$
over time $\R$ and along the edge space $\Omega_{uv} = (0,l_{uv})$,
determined by the incompressible \textit{flow field} $\nu_{uv}: \R \rightarrow \R$ over time $\R$. 
Formally, this is defined by the \gls*{PDE}
\begin{align*}
    \frac{
    \partial c_{uv}(t,z)
    }{
    \partial t
    }
    =
    - \nu_{uv}(t) ~
    \frac{
    \partial c_{uv}(t,z)
    }{
    \partial z
    }
    \quad
    \text{ (Equation \eqref{align_EdgeDynamics_Advection})}.
\end{align*}

\noindent
In Section \ref{subsection_MessageAggregationAndUpdate_MixingAtNodes}, 
we derive the message aggregation and node update scheme $c_v(t)$ as defined in Equation \eqref{align_MessagePassing_Advection} 
based on the \textit{coupling conditions} given by the \textit{mixing at nodes} and the \textit{inflow continuity condition}.
For each $v \in V$,
the former refers to 
the concentration $c_v: \R \rightarrow \R_{\geq 0}$ 
over time $\R$,
determined by a weighted aggregation of the concentrations $c_{uv}(\cdot,l_{uv})$ carried by the \textit{flow rates} $q_{uv} = \alpha_{uv} \nu_{uv}: \R \rightarrow \R$ from (temporal) inflow neighbors $u \in \Nbh_-(v,\cdot)$ over time $\R$.
Formally, this is defined as
\begin{align*}
    c_v(t) 
    =
    \frac{
    \sum_{u \in \Nbh_-(v,t)}
    q_{uv}(t) \cdot c_{uv}(t,l_{uv})
    }{
    \sum_{u \in \Nbh_-(v,t)} 
    q_{uv}(t)
    }
    \quad
    \text{ (Equation \eqref{align_CouplingCondition_MixingAtNodes})}.
\end{align*}

\noindent
For each edge $e_{uv} \in E$, 
the latter refers to
the concentration $c_{uv}: \R \times \Omega_{uv} \rightarrow \R_{\geq 0}$
over time $\R$ and at the temporal inflow boundary $z = 0$ of the edge space $\Omega_{uv} = (0,l_{uv})$,
determined by the concentration $c_u: \R \rightarrow \R$ of the temporal inflow neighbor $u \in \Nbh_-(v,\cdot)$ over time $\R$. 
Formally, this is defined by 
\begin{align*}
    c_{uv}(t,0) 
    =&~ 
    c_u(t)
    \text{ for all }
    u \in \Nbh_-(v,t)
    \quad
    \text{ (Equation \eqref{align_CouplingCondition_InflowContinuity})}.
\end{align*}

\noindent 
The mixing at nodes and the inflow continuity condition formalize the following intuition:
On the one hand, 
if there is a (physical) flow from $u \in \Nbh_-(v,t)$ to $v \in V$ entering the edge $e_{uv} \in E$ at time $t \in \R$, 
the concentration $c_{uv}(t,0)$ at the beginning $z = 0$ of that edge $e_{uv}$ 
is equal to the concentration $c_u(t)$ at the node $u$ out of which the flow flows \textit{into that edge} (\textit{local inflow boundary condition}). 
On the other hand, 
if there is a (physical) flow from $u \in \Nbh_-(v,t)$ to $v \in V$  leaving the edge $e_{uv} \in E$ at time $t \in \R$, 
the concentration $c_{uv}(t,l_{uv})$ at the end $z = l_{uv}$ of that edge $e_{uv}$ 
contributes to the weighted aggregation at the node $v$ into which the flow flows \textit{out of that edge} (\textit{local outflow boundary condition}).

While the edge dynamics (Eq. \eqref{align_EdgeDynamics}) are defined on the open edge spaces $\Omega_{uv} = (0,l_{uv})$, 
the coupling conditions (Eq. \eqref{align_CouplingCondition_MixingAtNodes} and Eq. \eqref{align_CouplingCondition_InflowContinuity}) rely on the values $c_{uv}(\cdot,0)$ and $c_{uv}(\cdot,l_{uv})$ on the boundary $\partial \Omega_{uv} = \{0,l_{uv}\}$ of $\Omega_{uv}$. 
Therefore, suitable smoothness assumptions on $c_{uv}$ and $\nu_{uv}$ are required to be able to transfer information in-between the edge spaces and its coupling boundaries.\footnote{
    Note that strictly speaking, we also need some part of the assumptions simply for Equation \eqref{align_EdgeDynamics_Advection} or more generally, Equation \eqref{align_EdgeDynamics}, to be well-defined.
} 

\begin{assumptionNoNb}[Assumption \ref{assumption_SmoothnessAssumptions}]
    For each $e_{uv} \in E$, let 
    \begin{itemize}
        \item $c_{uv}: \R \times (0,l_{uv}) \rightarrow \R_{\geq 0}$ be differentiable,

        \item $c_{uv}: \R \times [0,l_{uv}] \rightarrow \R_{\geq 0}$ be continuous and

        \item $\nu_{uv}: \R \rightarrow \R$ be continuous.
    \end{itemize}
\end{assumptionNoNb}

\begin{remark}[Symmetry of signals]
\label{remark_SymmetryOfSignals}   
    Note that next to the physical priors introduced above, simply because of the fact that we model a metric graph as symmetric directed graph $\G$ (cf. Section \ref{section_Background}), 
    \begin{align}
    \label{align_SymmetryOfSignals}
        c_{vu}(t,z) = c_{uv}(t,l_{vu}-z)
    \end{align}

    \noindent 
    holds for all $v \in V$, $u \in \Nbh(v)$, $t \in \R$ and $z \in \overline{\Omega}_{vu} = [0,l_{vu}]$. 
    That is, given that we consider a symmetric directed graph $\G = (V,E)$ and the contrasting edges $e_{vu}, e_{uv} \in E$ of length $l_{vu} = l_{uv}$, we expect that the signal $c_{vu}(t,z)$ along the edge $e_{vu}$ at time $t$ and position $z \in [0,l_{vu}]$ is equal to the signal along the edge $e_{uv}$ at time $t$ and position $l_{vu} - z = l_{uv} - z \in [0,l_{uv}] = \overline{\Omega}_{uv}$.
\end{remark}

\subsubsection{Message Generation: Advective Transport Along Edges}
\label{subsubsection_MessageGeneration_AdvectiveTransportAlongEdges}

Throughout the rest of this section, we assume that for each $e_{uv} \in E$ (that is, for each $v \in V$ and $u \in \Nbh(v)$), Equation \eqref{align_EdgeDynamics_Advection} with coupling conditions as defined by Equation \eqref{align_CouplingCondition_MixingAtNodes} and \eqref{align_CouplingCondition_InflowContinuity} and smothness assumptions as defined by Assumption \ref{assumption_SmoothnessAssumptions} holds.
According to Equation \eqref{align_CouplingCondition_MixingAtNodes_MessagePassingFormulation}, the message-generation function $\phi(t,e_{uv},\cm,\num) = c_{uv}(t,l_{uv})$ relies on 
the signal $c_{uv}(\cdot,l_{uv})$ over time $\R$ and at the temporal \textit{outflow} boundary $z = l_{uv}$ of the edge space $\Omega_{uv} = (0,l_{uv})$.
Since we are only aware of 
the signal $c_{uv}(\cdot,0)$ over time $\R$ and at the temporal \textit{inflow} boundary $z = 0$ of each edge space $\Omega_{uv} = (0,l_{uv})$ 
through the inflow boundary condition (Eq. \eqref{align_CouplingCondition_InflowContinuity}),
computing $\phi(t,e_{uv},\cm,\num)$ would normally require to explicitly solve the edge dynamics along the whole edge space $\overline{\Omega}_{uv} = [0,l_{uv}]$.
This is typically done by computationally expensive numerical solvers. 
However, our message-generation function $\phi$ builds upon the \gls*{MoC} to avoid the explicit computation of concentrations $c_{uv}(\cdot, z)$ in between the boundaries, i.e., for all $z \in (0, l_{uv})$.
More specifically, we can derive \textit{characteristic curves} along which the function $c_{uv}$ remains constant over time $\R$ and space $\Omega_{uv}$.

\begin{lemmaNoNb}[Constant concentration along characteristic curves (Lemma \ref{lemma_ConstantConcentrationAlongCharacteristicCurves})]
    Let $e_{uv} \in E$ hold.
    If the function $c_{uv}$ obeys Equation \eqref{align_EdgeDynamics_Advection},
    for each $t_0 \in \R$ and $z_0 \in \R$, the curve
    %
    %
    \begin{align*}
        \gamma_{t_0,z_0}: ~ T_{t_0,z_0} \longrightarrow \R_{\geq 0}, ~~
        t \longmapsto 
        c_{uv} \Big( t, z_0 + \textstyle \int_{t_0}^{t} \nu_{uv}(s) ~ds \Big)
    \end{align*}
    
    \noindent 
    is constant on
    $
    T_{t_0,z_0}
    =
    \Big \{
    t \in \R 
    ~|~
    0 
    <
    z_0 + \int_{t_0}^{t^{'}} \nu_{uv}(s) ~ds
    <
    l_{uv}
    ~\forall t^{'} \in 
    \big[ \min\{t_0,t\} , \max\{t_0,t\} \big]
    \Big \}.
    $
    %
\end{lemmaNoNb}

Intuitively, in Lemma \ref{lemma_ConstantConcentrationAlongCharacteristicCurves}, we fix an arbitrary time $t_0 \in \R$ and a position $z_0 \in \R$ and follow any particle that is at this position $z_0$ at this time $t_0$ along the path $z(t) := z_0 + \int_{t_0}^{t} \nu_{uv}(s) ~ds$ it travels during the time $t - t_0$ (or $t_0 - t$) for some time $t \in T_{t_0,z_0}$ with possibly varying flow velocities $\nu_{uv}(s)$ for all $s \in [\min\{t_0,t\} , \max\{t_0,t\}]$. 
If we assign a concentration $c_{uv}(t_0,z_0)$ to this particle, the concentration should stay constant along the path this particle travels.
The choice of $t \in T_{t_0,z_0}$ guarantees that the before-mentioned path remains within the edge $e_{uv}$, or in other words, in the open sub-domain $\R \times \Omega_{uv}$ of the function $c_{uv}$ on which Equation \eqref{align_EdgeDynamics_Advection} is defined on, as we will investigate in detail in the next Remark \ref{remark_OnT_{t_0,z_0}}.

\begin{remark}[On $T_{t_0,z_0}$]
\label{remark_OnT_{t_0,z_0}}
    Since for $e_{uv} \in E$, the velocity $\nu_{uv}:\R \rightarrow \R$ can change its sign,
    for each $t_0 \in \R$ and $z_0 \in \R$, the function
    \begin{align*}
        z: \R \longrightarrow \R, ~
        t \longmapsto z(t) := z_0 + \int_{t_0}^{t} \nu_{uv}(s) ~ds
    \end{align*}
    
    \noindent is differentiable (and thus, continuous), but \textit{not} monotone, meaning that even if $z(t_0) = z_0 \in \Omega_{uv}$ and $z(t) \in \Omega_{uv}$ holds for a $t \in \R$, this does \textit{not} automatically mean that $z(t^{'}) \in \Omega_{uv}$ holds for all $t^{'} \in [\min\{t_0,t\} , \max\{t_0,t\}]$, too.
    However, since Equation \eqref{align_EdgeDynamics_Advection} only holds on the open sub-domain $\R \times \Omega_{uv}$ of the function $c_{uv}: \R \times [0,l_{uv}] \rightarrow \R_{\geq 0}$, any path from $z(t_0) = z_0$ to $z(t) = z_0 + \int_{t_0}^{t} \nu_{uv}(s) ~ds$ parametrized by time needs to be fully included in the space interval $\Omega_{uv} = (0,l_{uv})$.
    In other words, $t \in T_{t_0,z_0}$ needs to hold.
\end{remark}

\noindent Since the set $T_{t_0,z_0}$ can also be empty, we are interested in when this is not the case and what properties it satisfies in this case. The following Lemma \ref{lemma_OnT_{t_0,z_0}} answers these questions.

\begin{lemma}[On $T_{t_0,z_0}$]
\label{lemma_OnT_{t_0,z_0}}
    Let $e_{uv} \in E$, $t_0 \in \R$, $z_0 \in \R$ hold and let $T_{t_0,z_0}$ be defined as in Lemma \ref{lemma_ConstantConcentrationAlongCharacteristicCurves}.

    \begin{enumerate}
        \item If $z_0 \not \in \Omega_{uv}$ holds, then $T_{t_0,z_0} = \emptyset$ holds.

        \item If $z_0 \in \Omega_{uv}$ holds, then $T_{t_0,z_0} \neq \emptyset$ holds. Even more, $t_0 \in T_{t_0,z_0}$ holds.


        
        \item If $z_0 \in \Omega_{uv}$ holds, 
        there exist $\tn{l} \in [-\infty,t_0)$ and $\tn{r} \in (t_0, \infty]$ 
        such that 

        \begin{enumerate}
            \item if $\tn{l} \neq - \infty$ holds, $z_0 + \int_{t_0}^{\tn{l}} \nu_{uv}(s) ~ds \in \{0,l_{uv}\}$ holds,
            
            \item if $\tn{r} \neq + \infty$ holds, $z_0 + \int_{t_0}^{\tn{r}} \nu_{uv}(s) ~ds \in \{0,l_{uv}\}$ holds and

            \item $T_{t_0,z_0} = (\tn{l},\tn{r}) = (t_0 - (t_0 - \tn{l}),t_0 + (\tn{r} - t_0))$ holds.
        \end{enumerate}
        
        \noindent Equivalently, there exist $\delta \tn{l}, \delta \tn{r} \in (0,\infty]$
        such that $T_{t_0,z_0} = (t_0 - \delta \tn{l},t_0 + \delta \tn{r})$ holds.
    \end{enumerate}
\end{lemma}

\noindent Since according to the previous Lemmas \ref{lemma_ConstantConcentrationAlongCharacteristicCurves} and \ref{lemma_OnT_{t_0,z_0}}, for any $t_0 \in \R$ and $z_0 \in \Omega_{uv}$, the concentration $c_{uv}$ is constant along characteristic curves $\gamma_{t_0,z_0}$ on an non-empty time interval $T_{t_0,z_0} = (\tn{l},\tn{r})$, it suffices to know the value of $c_{uv}$ at one position in this time interval to have knowledge of the concentration over all times in the interval:

\begin{theorem}[Advective transport along edges]
\label{theorem_AdvectiveTransportAlongEdges}
    Let $e_{uv} \in E$ hold.
    If the function $c_{uv}$ obeys Equation \eqref{align_EdgeDynamics_Advection}, for each $t_0 \in \R$ and $z_0 \in \Omega_{uv}$,
    \begin{align*}
        c_{uv} \left( t, z_0 + \int_{t_0}^{t} \nu_{uv}(s) ~ds \right) 
        = 
        c_{uv} \left( t, z_0 + \nuoverline_{uv}(t_0,t) ~ (t-t_0) \right)
        =
        c_{uv}(t_0,z_0)
    \end{align*}
    
    \noindent holds for all $t \in T_{t_0,z_0}$ with $T_{t_0,z_0}$ as defined in Lemma \ref{lemma_ConstantConcentrationAlongCharacteristicCurves} and for the mean velocity
    \begin{align*}
        \nuoverline_{uv}(t_0,t)
        :=
        \begin{cases}
            \fint_{t_0}^{t} \nu_{uv}(s) ~ds
            & \text{ if } t \neq t_0
            \\
            0
            & \text{ if } t = t_0
        \end{cases}
        =
        \begin{cases}
            \frac{1}{t - t_0}
            \int_{t_0}^{t} \nu_{uv}(s) ~ds
            & \text{ if } t \neq t_0
            \\
            0
            & \text{ if } t = t_0
        \end{cases}
        .
    \end{align*}
\end{theorem}

\noindent According to Lemma \ref{lemma_OnT_{t_0,z_0}}.3, for any $t_0 \in \R$ and $z_0 \in \Omega_{uv}$, the set $T_{t_0,z_0}$ usually corresponds to a time interval $T_{t_0,z_0} = (\tn{l},\tn{r})$ such that on its boundary $\partial (\tn{l},\tn{r}) = \{\tn{l},\tn{r}\}$, the information reaches exactly one of the ends of the edge $e_{uv}$, i.e., the node $v \in V$ or the node $u \in \Nbh(v)$. 
In view of the coupling conditions (Eq. \eqref{align_CouplingCondition_MixingAtNodes}), this motivates to leverage the advective transport along edges, i.e., Theorem \ref{theorem_AdvectiveTransportAlongEdges}, and combine it with the inflow boundary conditions defined through Equation \eqref{align_CouplingCondition_InflowContinuity}, in order to derive a 
representation of the advection coupling conditions (Eq. \eqref{align_CouplingCondition_MixingAtNodes}) that does not depend on the concentrations $c_{uv}(t,l_{uv})$ at the end $z = l_{uv}$ of the edges $e_{uv}$ for each $u \in \Nbh_-(v,t)$, 
but on the concentrations $c_u(t-\delta t_{uv})$ and $c_v(t-\delta t_{uv})$ at suitable times $t-\delta t_{uv} \in \R$.
This, in turn, will serve as the basis for our message-passing algorithm (cf. Eq. \eqref{align_CouplingCondition_MixingAtNodes_MessagePassingFormulation} and \eqref{align_MessageFunction_Advection}).

Intuitively, the time shift $\delta t_{uv} \in \R$ will correspond to the time that any information requires to travel along the edge $e_{uv}$ with the flow $\nu_{uv}: \R \rightarrow \R$ from one of the nodes $v \in V$ or $u \in \Nbh(v)$ until it reaches the node $v \in V$ again. 
From the sender perspective, two different scenarios can occur:
Either, the information leaves the node $u \in \Nbh(v)$ at time $t-\delta t_{uv} \in \R$ and enters the node $v \in V$ at time $t \in \R$. 
We will call this scenario a  \textit{pass-through}.
Or, the information leaves the node $v \in V$ at time $t-\delta t_{uv} \in \R$ and starts traveling to the node $u \in \Nbh(v)$, but before it reaches that node, the flow $\nu_{vu} = -\nu_{uv}: \R \rightarrow \R$ changes its sign and the information enters the node $v \in V$ at time $t \in \R$ again. 
We will call this scenario a \textit{self-loop}.
Of course, the sign of the flow can also change multiple times before the information enters the node $v \in V$.
As a first step, we now formalize the aforementioned pass-through and self-loop times formally and collect some helpful properties to get a deeper understanding of such. In practice, also corresponding inverse phenomena will appear. In summary, we call these times \textit{transport times}, a name which we will motivate later.

\begin{definition}[Transport times]
\label{definition_TransportTimes}
    Let $e_{uv} \in E$ and $t \in \R$ hold.
    If there exists a $\delta t_{uv} \in \R$, such that
    \begin{align}
    \label{align_PassThroughTime_Condition1}
        \int_{t - \delta t_{uv}}^{t}
        \nu_{uv}(s) ~ds
        =&~
        l_{uv}
        \text{ and }
        \\
    \label{align_PassThroughTime_Condition2}
        \int_{t - \delta t_{uv}}^{t^{'}}
        \nu_{uv}(s) ~ds
        \in&~ 
        (0,l_{uv})
    \end{align}
    
    \noindent hold for all $t^{'} \in (\min\{t - \delta t_{uv},t\},\max\{t - \delta t_{uv},t\})$,
    we call $\delta t_{uv}$ the 
    \textit{pass-through time} 
    \textit{\gls*{wrt} the edge $e_{uv} \in E$ and time $t \in \R$}. 
    If instead, there exists a $\delta t_{uv} \in \R$, such that
    \begin{align}
    \label{align_PassThroughTime_Condition1_Inverse}
        \int_{t - \delta t_{uv}}^{t}
        \nu_{uv}(s) ~ds
        =&~
        -l_{uv}
        \text{ and }
        \\
    \label{align_PassThroughTime_Condition2_Inverse}
        \int_{t - \delta t_{uv}}^{t^{'}}
        \nu_{uv}(s) ~ds
        \in&~ 
        (-l_{uv},0)
    \end{align}
    
    \noindent hold for all $t^{'} \in (\min\{t - \delta t_{uv},t\},\max\{t - \delta t_{uv},t\})$,
    we call $\delta t_{uv}$ the 
    \textit{inverse pass-through time} 
    \textit{\gls*{wrt} the edge $e_{uv} \in E$ and time $t \in \R$}.
    If instead, there exists a $\delta t_{uv} \in \R_{\neq 0}$, such that
    \begin{align}
    \label{align_SelfLoopTime_Condition1}
        \int_{t - \delta t_{uv}}^{t}
        \nu_{uv}(s) ~ds
        =&~
        0
        \text{ and }
        \\
    \label{align_SelfLoopTime_Condition2}
        \int_{t - \delta t_{uv}}^{t^{'}}
        \nu_{uv}(s) ~ds
        \in&~ 
        (0,l_{uv}),
    \end{align}
    
    \noindent hold for all $t^{'} \in (\min\{t - \delta t_{uv},t\},\max\{t - \delta t_{uv},t\})$,
    we call $\delta t_{uv}$ the 
    \textit{self-loop time} 
    \textit{\gls*{wrt} the edge $e_{uv} \in E$ and time $t \in \R$}.
    If instead, there exists a $\delta t_{uv} \in \R_{\neq 0}$, such that
    \begin{align}
    \label{align_SelfLoopTime_Condition1_Inverse}
        \int_{t - \delta t_{uv}}^{t}
        \nu_{uv}(s) ~ds
        =&~
        0
        \text{ and }
        \\
    \label{align_SelfLoopTime_Condition2_Inverse}
        \int_{t - \delta t_{uv}}^{t^{'}}
        \nu_{uv}(s) ~ds
        \in&~ 
        (-l_{uv},0),
    \end{align}
    
    \noindent hold for all $t^{'} \in (\min\{t - \delta t_{uv},t\},\max\{t - \delta t_{uv},t\})$,
    we call $\delta t_{uv}$ the 
    \textit{inverse self-loop time} 
    \textit{\gls*{wrt} the edge $e_{uv} \in E$ and time $t \in \R$}.
    
    Usually, $e_{uv} \in E$ and $t \in \R$ are clear from the context and we just say 
    \textit{pass-through time},
    \textit{inverse pass-through time},
    \textit{self-loop time} or
    \textit{inverse self-loop time}.
\end{definition}

\noindent The following lemmas help to get a deeper understanding of the different kinds of transport times.

\begin{lemma}[Pass-through time]
\label{lemma_PassThroughTime} 
    Let $e_{uv} \in E$ and $t \in \R$ hold.
    If there exists a pass-through time $\delta t_{uv} \in \R$ \gls*{wrt} $e_{uv} \in E$ and $t \in \R$, the following properties hold:

    \begin{enumerate}
        \item $\delta t_{uv} \in \R_{\neq 0}$ and thus, $t - \delta t_{uv} \neq t$ holds.

        \item There exists no other pass-through time of the same sign 
        (i.e., there exists either exactly one or no positive pass-through time,
        and there exists exactly one or no negative pass-through time).

        \item The mean velocity 
        \begin{align*}
            \nuoverline_{uv}(t) 
            := 
            \nuoverline_{uv}(t - \delta t_{uv},t) 
            = 
            \fint_{t - \delta t_{uv}}^{t} \nu_{uv}(s) ~ds 
            = 
            \frac{1}{\delta t_{uv}} \int_{t - \delta t_{uv}}^{t} \nu_{uv}(s) ~ds
        \end{align*}
     
        during the time interval $[\min\{t - \delta t_{uv},t\},\max\{t - \delta t_{uv},t\}]$ 
        as defined in Theorem \ref{theorem_AdvectiveTransportAlongEdges}
        is equal to
        $\nuoverline_{uv}(t) = \frac{l_{uv}}{\delta t_{uv}}$
        and thus satisfies
        \begin{align*}
            \nuoverline_{uv}(t) > 0 &\iff \delta t_{uv} > 0 \text{ and} \\
            \nuoverline_{uv}(t) < 0 &\iff \delta t_{uv} < 0
            .
        \end{align*}

        \item The flows $q_{uv}(t - \delta t_{uv})$ and $q_{uv}(t)$ linked to the velocity $\nu_{uv}(\cdot) = \frac{q_{vu(\cdot)}}{\alpha_{vu}}$ satisfy
        \begin{align*}
            \begin{cases}
                q_{uv}(t - \delta t_{uv}), q_{uv}(t) \geq 0
                & \text{ if } \delta t_{uv} > 0
                \\
                q_{uv}(t - \delta t_{uv}), q_{uv}(t) \leq 0
                & \text{ if } \delta t_{uv} < 0
            \end{cases}
            .
        \end{align*}
    \end{enumerate}
\end{lemma}

\noindent According to Lemma \ref{lemma_PassThroughTime},
if 
a pass-through time $\delta t_{uv} > 0$ is positive,
so does $\nuoverline_{uv}(t) > 0$ hold,
and 
there is an average flow from the node $u \in \Nbh(v)$ to the node $v \in V$
during the time $[t - \delta t_{uv},t]$.\footnote{
    More precisely, since in this case, 
    $q_{uv}(t - \delta t_{uv}), q_{uv}(t) \geq 0$ holds,
    we know that at time $t - \delta t_{uv} \in \R$ and $t \in \R$,
    there is a flow from $u \in \Nbh_-(v,t - \delta t_{uv}) \cap \Nbh_-(v,t)$ to $v \in V$.
    Note that we tacitly ignored the case of zero flows here due to Remark \ref{remark_PassThroughTime}.
}
If in contrast,
a pass-through time $\delta t_{uv} < 0$ is negative,
so does $\nuoverline_{uv}(t) < 0$ hold,
and 
there is an average flow from the node $v \in V$ to the node $u \in \Nbh(v)$
during the time $[t,t - \delta t_{uv}]$.\footnote{
    More precisely, since in this case, 
    $q_{uv}(t - \delta t_{uv}), q_{uv}(t) \leq 0$ holds,
    we know that at time $t \in \R$ and $t - \delta t_{uv} \in \R$,
    there is a flow from $v \in V$ to $u \in \Nbh_+(v,t) \cap \Nbh_+(v,t - \delta t_{uv})$.
    Note that we tacitly ignored the case of zero flows here due to Remark \ref{remark_PassThroughTime}.
}
In both cases, the time $\delta t_{uv} = \frac{l_{uv}}{\nuoverline_{uv}(t)}$ is exactly the time that any particle transported by the average flow $\qoverline_{uv}(t) = \alpha_{uv} ~\nuoverline_{uv}(t)$ needs to travel along the edge $e_{uv}$ of length $l_{uv}$. 
Therefore, we call $\delta t_{uv}$ the \textit{pass-through time}.
The cases are intuitively connected by the equality\footnote{
    The proof follows similar arguments as given in, e.g., the proof of Lemma \ref{lemma_InversePassThroughTime}.
}
\begin{align*}
    \int_{t - \delta t_{uv}}^{t}
    \nu_{uv}(s) ~ds
    =
    \int_{t}^{t - \delta t_{uv}}
    - \nu_{uv}(s) ~ds
    =
    \int_{t}^{t - \delta t_{uv}}
    \nu_{vu}(s) ~ds,
\end{align*}

\noindent where in dependence of the sign of the pass-through time $\delta t_{uv} \in \R_{\neq 0}$, we use the oriented of the two integrals on the left- or right-hand side for interpretation.
We can make similar observations for self-loop times.

\begin{lemma}[Self-loop time]
\label{lemma_SelfLoopTime}
    Let $e_{uv} \in E$ and $t \in \R$ hold.
    If there exists a self-loop time $\delta t_{uv} \in \R_{\neq 0}$ \gls*{wrt} $e_{uv} \in E$ and $t \in \R$, the following properties hold:

    \begin{enumerate}
        \item There exists no other self-loop time of the same sign 
        (i.e., there exists either exactly one or no positive self-loop time,
        and there exists exactly one or no negative self-loop time).

        \item The mean velocity 
        \begin{align*}
            \nuoverline_{uv}(t) 
            := 
            \nuoverline_{uv}(t - \delta t_{uv},t) 
            = 
            \fint_{t - \delta t_{uv}}^{t} \nu_{uv}(s) ~ds 
            = 
            \frac{1}{\delta t_{uv}} \int_{t - \delta t_{uv}}^{t} \nu_{uv}(s) ~ds
        \end{align*}

        \noindent during the time interval $[\min\{t - \delta t_{uv},t\},\max\{t - \delta t_{uv},t\}]$ 
        as defined in Theorem \ref{theorem_AdvectiveTransportAlongEdges}
        satisfies
        $\nuoverline_{uv}(t) = 0$.

        \item The flows $q_{uv}(t - \delta t_{uv})$ and $q_{uv}(t)$ linked to the velocity $\nu_{uv}(\cdot) = \frac{q_{vu(\cdot)}}{\alpha_{vu}}$ satisfy
        \begin{align*}
            \begin{cases}
                q_{uv}(t - \delta t_{uv}) \geq 0,~ q_{uv}(t) \leq 0
                & \text{ if } \delta t_{uv} > 0
                \\
                q_{uv}(t - \delta t_{uv}) \leq 0,~ q_{uv}(t) \geq 0
                & \text{ if } \delta t_{uv} < 0
            \end{cases}
            .
        \end{align*}
    \end{enumerate}
\end{lemma}

\noindent According to Lemma \ref{lemma_SelfLoopTime},
independent on whether 
a self-loop time $\delta t_{uv} \in \R_{\neq 0}$ is positive or negative,
$\nuoverline_{uv}(t) = 0$ holds.
If 
a self-loop time $\delta t_{uv} > 0$ is positive,
there is an average flow from the node $u \in \Nbh(v)$ to itself
during the time $[t - \delta t_{uv},t]$.\footnote{
    More precisely, since in this case, 
    $q_{uv}(t - \delta t_{uv}) \geq 0$ and $q_{uv}(t) \leq 0$ holds,
    we know that at time $t - \delta t_{uv} \in \R$,
    there is a flow from $u \in \Nbh_-(v,t - \delta t_{uv})$ to $v \in V$,
    while at time $t \in \R$,
    there is a flow from $v \in V$ to $u \in \Nbh_+(v,t)$.
    Note that we tacitly ignored the case of zero flows here due to Remark \ref{remark_SelfLoopTime}.
}
If in contrast,
a self-loop time $\delta t_{uv} < 0$ is negative,
there is an average flow from the node $u \in \Nbh(v)$ to itself
during the time $[t,t - \delta t_{uv}]$.\footnote{
    More precisely, since in this case, 
    $q_{uv}(t) \geq 0$ and $q_{uv}(t - \delta t_{uv}) \leq 0$ holds,
    we know that at time $t \in \R$,
    there is a flow from $u \in \Nbh_-(v,t)$ to $v \in V$,
    while at time $t - \delta t_{uv} \in \R$,
    there is a flow from $v \in V$ to $u \in \Nbh_+(v,t - \delta t_{uv})$.
    Note that we tacitly ignored the case of zero flows here due to Remark \ref{remark_SelfLoopTime}.
}
In both cases, the time $\delta t_{uv}$ is exactly the time that any particle transported by the average flow $\qoverline_{uv}(t) = \alpha_{uv} ~\nuoverline_{uv}(t) = 0$ needs to travel from $u$ along some part of the edge $e_{uv}$ until the flow changes its sign and the particle travels back to the node $u$ it came from. 
Therefore, we call $\delta t_{uv}$ the \textit{self-loop time}.

As another result that will turn out to be useful in practise, by the symmetry of the underlying graph $\G$ and the conservation of flows (Eq. \eqref{align_ConservationOfFlows}), 
an inverse pass-through time induces a pass-through time and vice versa, while
an inverse self-loop time induces a self-loop time and vice versa.

\begin{lemma}[Inverse pass-through time]
\label{lemma_InversePassThroughTime}
    Let $e_{uv} \in E$ and $t \in \R$ hold.
    $\delta t_{uv} \in \R$ is an inverse pass-through time \gls*{wrt} $e_{uv} \in E$ and $t \in \R$ 
    \gls*{iff}
    $\delta t_{vu} := \delta t_{uv} \in \R$ is a pass-through time \gls*{wrt} $e_{vu} \in E$ and $t \in \R$ 
\end{lemma}

\begin{lemma}[Inverse self-loop time]
\label{lemma_InverseSelfLoopTime}
    Let $e_{uv} \in E$ and $t \in \R$ hold.
    $\delta t_{uv} \in \R$ is an inverse self-loop time \gls*{wrt} $e_{uv} \in E$ and $t \in \R$ 
    \gls*{iff}
    $\delta t_{vu} := \delta t_{uv} \in \R$ is a self-loop time \gls*{wrt} $e_{vu} \in E$ and $t \in \R$ 
\end{lemma}

\noindent We now collected all relevant properties which help us to prove the main results of this section as motivated above:
Given a pass-through time $\delta t_{uv} \in \R$ \gls*{wrt} $e_{uv} \in E$ and $t \in \R$, 
we can re-write the concentrations $c_{uv}(t,l_{uv})$ and $c_{vu}(t - \delta t_{uv},l_{vu})$ at the end $z = l_{uv} = l_{vu}$ of the edges $e_{uv}$ and $e_{vu}$ using the concentrations $c_u: \R \rightarrow \R_{\geq 0}$ and $c_v: \R \rightarrow \R_{\geq 0}$, respectively: 

\begin{theorem}[Pass-through time]
\label{theorem_PassThroughTime}
    Let $e_{uv} \in E$ and $t \in \R$ hold.
    If the function $c_{uv}$ obeys Equation \eqref{align_EdgeDynamics_Advection} and there exists a pass-through time $\delta t_{uv} \in \R_{\neq 0}$ \gls*{wrt} $e_{uv} \in E$ and $t \in \R$,
    we obtain
    \begin{align*}
        c_{uv}(t,l_{uv})
        =
        c_{uv}(t - \delta t_{uv},0)
        =
        c_{uv} \left( t - \frac{l_{uv}}{\nuoverline_{uv}(t)},0 \right).
    \end{align*}

    \noindent Even more, the concentrations $c_{uv}(\cdot,l_{uv})$ and $c_{vu}(\cdot,l_{vu})$ at the end $z = l_{uv} = l_{vu}$ of the edges $e_{uv} \in E$ and $e_{vu} \in E$ are connected to the concentrations $c_u$ and $c_v$ of the nodes of that edge, respectively, by
    \begin{align*}
        \begin{cases}
            c_{uv}(t,l_{uv})
            =
            c_u(t - \delta t_{uv})
            & \text{ if } \delta t_{uv} > 0
            \\
            c_{vu}(t - \delta t_{uv},l_{vu})
            =
            c_v(t)
            & \text{ if } \delta t_{uv} < 0
        \end{cases}
        .
    \end{align*}
\end{theorem}


\noindent According to Theorem \ref{theorem_PassThroughTime},
if
a pass-through time $\delta t_{uv} > 0$ is positive,
the concentration $c_{uv}(t,l_{uv})$ at time $t \in \R$ and at the end of the edge $e_{uv} \in E$, 
along which 
there is an average flow from the node $u \in \Nbh(v)$ to the node $v \in V$
during the time $[t - \delta t_{uv},t]$,
is equal to the concentration $c_u(t - \delta t_{uv})$ at the node $u \in \Nbh(v)$ and $-\delta t_{uv} < 0$ units of time \textit{earlier} than $t$.
If in contrast,
a pass-through time $\delta t_{uv} < 0$ is negative,
the concentration $c_{vu}(t - \delta t_{uv},l_{vu})$ at time $t - \delta t_{uv} \in \R$ and at the end of the edge $e_{vu} \in E$, 
along which 
there is an average flow from the node $v \in V$ to the node $u \in \Nbh(v)$
during the time $[t,t - \delta t_{uv}]$,
is equal to the concentration $c_v(t)$ at the node $v \in V$ and $+ \delta t_{uv} < 0$ units of time \textit{earlier} than $t - \delta t_{uv}$.

Analogously, 
given a self-loop time $\delta t_{uv} \in \R$ \gls*{wrt} $e_{uv} \in E$ and $t \in \R$, 
we can re-write the concentration $c_{vu}(t,l_{vu})$ at the end $z = l_{vu} = l_{uv}$ of the edge $e_{vu}$ using the concentration $c_u:\R \rightarrow \R_{\geq 0}$: 

\begin{theorem}[Self-loop time]
\label{theorem_SelfLoopTime}
    Let $e_{uv} \in E$ and $t \in \R$ hold.
    If the function $c_{uv}$ obeys Equation \eqref{align_EdgeDynamics_Advection} and there exists a self-loop time $\delta t_{uv} \in \R_{\neq 0}$ \gls*{wrt} $e_{uv} \in E$ and $t \in \R$,
    we obtain
    \begin{align*}
        c_{uv}(t,0)
        =
        c_{uv}(t - \delta t_{uv},0).
    \end{align*}

    \noindent Even more, the concentration $c_{vu}(\cdot,l_{vu})$ at the end $z = l_{vu} = l_{uv}$ of the edge $e_{vu} \in E$ is connected to the concentration $c_u$ of the node of that edge by
    \begin{align*}
        \begin{cases}
            c_{vu}(t,l_{vu})
            =
            c_u(t - \delta t_{uv})
            & \text{ if } \delta t_{uv} > 0
            \\
            c_{vu}(t - \delta t_{uv},l_{vu})
            =
            c_u(t)
            & \text{ if } \delta t_{uv} < 0
        \end{cases}
        .
    \end{align*}
\end{theorem}


\noindent According to Theorem \ref{theorem_SelfLoopTime}, 
if 
a self-loop time $\delta t_{uv} > 0$ is positive,
the concentration $c_{vu}(t,l_{vu})$ at time $t \in \R$ and at the end of the edge $e_{vu} \in E$, 
along which 
there is an average flow from the node $u \in \Nbh(v)$ to itself
during the time $[t - \delta t_{uv},t]$,
is equal to the concentration $c_u(t - \delta t_{uv})$ at the node $u \in \Nbh(v)$ and $-\delta t_{uv} < 0$ units of time \textit{earlier} than $t$.
If in contrast,
a self-loop time $\delta t_{uv} < 0$ is negative,
the concentration $c_{vu}(t - \delta t_{uv},l_{vu})$ at time $t - \delta t_{uv} \in \R$ and at the end of the edge $e_{vu} \in E$, 
along which 
there is an average flow from the node $u \in \Nbh(v)$ to itself
during the time $[t,t - \delta t_{uv}]$,
is equal to the concentration $c_u(t)$ at the node $u \in \Nbh(v)$ and $+ \delta t_{uv} < 0$ units of time \textit{earlier} than $t - \delta t_{uv}$.
\\
\\
In summary, Theorem \ref{theorem_PassThroughTime} and Theorem \ref{theorem_SelfLoopTime} illustrate the time shift required to pass information from one node to another.
Therefore we call these shifts \textit{transport times},
and the underlying concepts serve as a basis for our advection-message-generation function $\phi$ (cf. Eq. \eqref{align_CouplingCondition_MixingAtNodes_MessagePassingFormulation}). To be able to properly define it, we need to make sure that for each edge $e_{uv} \in E$ and time $t \in \R$, such a transport time exists and is unique. For the sake of brevity, for such a pair, we define the function
\begin{align*}
    \zcurl: \R \longrightarrow \R, ~
    \delta t \longmapsto \zcurl(\delta t) := \int_{t-\delta t}^{t} \nu_{uv}(s) ~ds
    \quad
    \text{ (Equation \eqref{align_TransportTime_Functionscurlzcurl})},
\end{align*}

\noindent
where we omit the dependency of $\zcurl$ on $e_{uv} \in E$, $t \in \R$ and the overall flow field $\num$ for better readability.
Then we obtain the following existence and uniqueness result:

\begin{theorem}[Existence and uniqueness of positive transport times]
\label{theorem_ExistenceAndUniquenessOfPositiveTransportTimes}
    Let $e_{uv} \in E$ and $t \in \R$ hold.
    If $\nu_{uv}(t) \neq 0$ holds and the set
    \begin{align*}
        A_{uv}(t) :=& 
        \bigg \{
        \delta t \in (0,\infty) 
        ~ \big |~
        \int_{t - \delta t}^{t} \nu_{uv}(s) ~ds
        \in 
        \{-l_{uv},0,l_{uv}\}
        \bigg \}
        \\
        =&
        \bigg \{
        \delta t \in (0,\infty) 
        ~ \big |~
        \zcurl(\delta t)
        \in 
        \{-l_{uv},0,l_{uv}\}
        \bigg \}
    \end{align*}

    \noindent is not empty, the following properties hold:

    \begin{enumerate}
        \item $\delta t_{uv} := \min A_{uv}(t) = \min \{\delta t \in (0,\infty) ~ \big |~ \zcurl(\delta t) \in \{-l_{uv},0,l_{uv}\}\} \in (0,\infty)$ exists.

        \item $\delta t_{uv}$ is a positive pass-through time, inverse pass-through time, self-loop time or inverse self-loop time.
        More specifically, 
        $\delta t_{uv}$ is (always \gls*{wrt} $e_{uv}$ and $t$) ...
        \begin{align*}
            & \text{... an inverse pass-through time}
            && \iff
            \zcurl(\delta t_{uv}) = -l_{uv},
            \\
            & \text{... a self-loop time} 
            \text{ or }
            \text{an inverse self-loop time}
            && \iff
            \zcurl(\delta t_{uv}) = 0,
            \\
            & \text{... a pass-through time}
            && \iff
            \zcurl(\delta t_{uv}) = l_{uv}.
        \end{align*}

        \item There exists no other positive pass-through time, inverse pass-through time, self-loop time or inverse self-loop time than $\delta t_{uv}$.
    \end{enumerate}
\end{theorem}

\noindent Theorem \ref{theorem_ExistenceAndUniquenessOfPositiveTransportTimes} guarantees that the mapping
\begin{align}
\label{align_TransportTime_Functionscurl}
    \scurl: \R \times E  \times F(\R,\R^{\n{e}}) \longrightarrow (0,\infty), ~
    (t,e_{uv},\num) \longmapsto \scurl(t,e_{uv},\num) := \delta t_{uv} 
\end{align}

\noindent is well-defined.
Throughout most parts of this work, $t \in \R$, $e_{uv} \in E$ and $\num = (\nu_e)_{e \in E}$ are known from the context and we use the notation $\delta t_{uv}$ as a shortcut for $\scurl(t,e_{uv},\num)$ to ease notation whenever it is reasonable.
Finally, this allows to define the advection-message-generation function in accordance to the results of Theorems \ref{theorem_PassThroughTime} and \ref{theorem_SelfLoopTime} as
\begin{align*}
    \phi: \R \times E \times F(\R,\R^{\n{n}}) \times F(\R,\R^{\n{e}}) \longrightarrow&~ \R_{\geq 0},\\
    (t,e_{uv},\cm,\num) \longmapsto&~ \phi(t,e_{uv},\cm,\num) 
    :=~ 
    \begin{cases}
        c_u(t - \delta t_{uv})
        & \text{ if } 
        \zcurl(\delta t_{uv}) \in \{-l_{uv},l_{uv}\}
        \\
        c_v(t - \delta t_{uv})
        & \text{ if } 
        \zcurl(\delta t_{uv}) = 0
    \end{cases}
    \quad
    \text{ (Equation \eqref{align_MessageFunction_Advection})}.
\end{align*}

\begin{remark}[Double assignment phenomenon]
\label{remark_DoubleAssignmentPhenomenon}
    Note that through the relation\footnote{
    See, e.g., the proof of Lemma \ref{lemma_InversePassThroughTime}.
    }
    \begin{align*}
        \int_{t - \delta t_{uv}}^{t}
        \nu_{vu}(s) ~ds
        =
        -
        \int_{t - \delta t_{uv}}^{t}
        \nu_{uv}(s) ~ds,
    \end{align*}
    
    two neighboring nodes $v \in V$ and $u \in \Nbh(v)$ will be assigned to the same transport time $\delta t_{uv}$, corresponding to either pass-through and inverse pass-through times (cf. Lemma \ref{lemma_InversePassThroughTime}) or self-loop and inverse self-loop times (cf. Lemma \ref{lemma_InverseSelfLoopTime}). 
    We call this the \textit{double-assignment phenomenon}.
\end{remark}

In the following Section \ref{subsection_MessageAggregationAndUpdate_MixingAtNodes}, we integrate our advection-message-generation function $\phi$ into the coupling conditions (Eq. \eqref{align_CouplingCondition_MixingAtNodes}) and show that the resulting representation of such corresponds to our advection-message-passing algorithm as presented in Section \ref{section_MethodMetricGraphAdvectionMessagePassing} (Eq. \eqref{align_MessagePassing_Advection}) and which automatically handles the double-assignment phenomenon.

\subsubsection{Message Aggregation and Update: Mixing at Nodes}
\label{subsection_MessageAggregationAndUpdate_MixingAtNodes}
 
As investigated in the  previous Section \ref{subsubsection_MessageGeneration_AdvectiveTransportAlongEdges}, 
for a node $v \in V$ and a time $t \in \R$, we are interested in a 
representation of the advection coupling conditions (Eq. \eqref{align_CouplingCondition_MixingAtNodes}) that does not depend on the concentrations $c_{uv}(t,l_{uv})$ at the end $z = l_{uv}$ of the edges $e_{uv}$ for each $u \in \Nbh_-(v,t)$,
but on the concentrations $c_u(t-\delta t_{uv})$ and $c_v(t-\delta t_{uv})$ shifted by suitable \textit{transport times} $\delta t_{uv} \in \R$.
Since in practice, we can only assume to have knowledge of the past, the suitability refers to the requirement that $\delta t_{uv} > 0$ should hold.
Theorem \ref{theorem_ExistenceAndUniquenessOfPositiveTransportTimes} guarantees the existence and uniqueness of such.

In this subsection, we prove that the message-generation function as defined in Equation \eqref{align_MessageFunction_Advection} indeed assigns to each time $t \in \R$ and each edge $e_{uv} \in E$ a message $\phi(t,e_{uv},\cm,\num)$ that determines the correct representation of $c_{uv}(t,l_{uv})$ from the correct source under consideration of the concentrations $\cm$ and the flow field $\num$.
The following lemma will help reaching this goal.

\begin{lemma}[Transport times]
\label{lemma_TransportTimes}
    Let $e_{uv} \in E$ and $t \in \R$ hold.
    In the setting of Theorem \ref{theorem_ExistenceAndUniquenessOfPositiveTransportTimes},
    the following property holds:

    \begin{enumerate}
        \item $\delta t_{uv}$ is  (always \gls*{wrt} $e_{uv}$ and $t$) ...
        \begin{align*}
            & \text{... a pass-through time}
            && \iff
            \zcurl(\delta t_{uv}) \in \{-l_{uv},l_{uv}\}
            \text{ and }
            q_{uv}(t) > 0,
            \\
            & \text{... an inverse pass-through time}
            && \iff
            \zcurl(\delta t_{uv}) \in \{-l_{uv},l_{uv}\}
            \text{ and }
            q_{uv}(t) < 0,
            \\
            & \text{... a self-loop time} 
            && \iff
            \zcurl(\delta t_{uv}) = 0
            \text{ and }
            q_{uv}(t) < 0,
            \\
            &\text{... an inverse self-loop time} 
            && \iff
            \zcurl(\delta t_{uv}) = 0
            \text{ and }
            q_{uv}(t) > 0.
        \end{align*}
    \end{enumerate}
\end{lemma}

\noindent Lemma \ref{lemma_TransportTimes} justifies why in the definition of the advection-message-generation function $\phi$ (Eq. \eqref{align_MessageFunction_Advection}), we do not need to distinguish between pass-through and inverse pass-through or self-loop and inverse self-loop times and though automatically handle the double assignment phenomenon (cf. Remark \ref{remark_DoubleAssignmentPhenomenon}):
Taking into account the sign of the flow rate $q_{uv}(t)$ at time $t \in \R$, corresponding counterparts will vanish in the coupling condition due to the summation over inflow neighbors $u \in \Nbh_-(v,t)$ only (cf. Eq. \eqref{align_CouplingCondition_MixingAtNodes}) and in our advection-message-passing algorithm due to the $\relu$-function (cf. Eq. \eqref{align_CouplingCondition_MixingAtNodes_MessagePassingFormulation} and \eqref{align_MessagePassing_Advection}).

Moreover, the assumptions of Theorem \ref{theorem_ExistenceAndUniquenessOfPositiveTransportTimes} and Lemma \ref{lemma_TransportTimes} are naturally satisfied in the contexts where they are applied:
On the one hand, summands where the velocity $\nu_{uv}(t)$ is equal to zero do not contribute by construction of the coupling conditions \eqref{align_CouplingCondition_MixingAtNodes}.
On the other hand, whenever the set $A_{uv}(t)$ is empty (that is, no transport time exists), we assume to have knowledge of the values $c_v(t)$ for the respective times $t$: In cases of initial value problems, these values are known by definition (cf. Remark \ref{remark_InitialValueProblemsAndEfficiency}); for boundary value problems we assume an initial state of zero at non-boundary nodes, these values $c_v(t)$ for the respective times $t$ are then also zero by definition.

Consequently, we can combine the results from the advective transport along edges (Section \ref{subsubsection_MessageGeneration_AdvectiveTransportAlongEdges}) with the results of this Section to obtain a message-passing representation of the coupling conditions (Eq. \eqref{align_CouplingCondition_MixingAtNodes}) which indeed coincides with our final advection-message-passing scheme \eqref{align_MessagePassing_Advection}: 

\begin{theorem}[Physical alignment (full version)]
\label{theorem_PhysicalAlignment_fullVersion}
    Let $v \in V$ and $t \in \R$ hold.
    We assume that Assumption \ref{assumption_SmoothnessAssumptions} holds and that for each $u \in \Nbh(v)$, 
    $A_{uv}(t)$ is not empty 
    (that is, the \gls*{OP} \eqref{align_TransportTime_OP} has a solution).
    If the function $c_{v}$ obeys the mixing at nodes (Eq. \eqref{align_CouplingCondition_MixingAtNodes}) and 
    if for each $u \in \Nbh(v)$, the function $c_{uv}$ obeys the advection edge dynamics (Eq. \eqref{align_EdgeDynamics_Advection}) with inflow continuity condition (Eq. \eqref{align_CouplingCondition_InflowContinuity}), 
    then
    \begin{align*}
        c_v(t) 
        =~
        \sum_{ 
        \substack{
        u \in \Nbh(v) \\ 
        \zcurl(\delta t_{uv}) = \pm l_{uv}
        }}
        w_{uv}(t) \cdot c_u(t - \delta t_{uv})
        +
        \sum_{ 
        \substack{
        u \in \Nbh(v) \\ 
        \zcurl(\delta t_{uv}) = 0
        }}
        w_{uv}(t) \cdot c_v(t - \delta t_{uv})
        =~
        \sum_{u \in \Nbh(v)}
        w_{uv}(t) \cdot \varphi(t,e_{uv},\cm,\num)
    \end{align*}

    \noindent
    holds, where $\delta t_{uv} = \min A_{uv}(t)$ is the transport time as defined in Theorem \ref{theorem_ExistenceAndUniquenessOfPositiveTransportTimes}.
    In other words, Equation \eqref{align_MessagePassing_Advection} with $\phi$ as defined in Equation \eqref{align_MessageFunction_Advection} holds.

\end{theorem}

\subsection{Error Bounds on MeGA-MP}
\label{subsection_ErrorBoundsOnMeGA-MP}

In this section, 
we first provide some theory on a special graph structure, trees, and introduce further notation along the way (Section \ref{subsubsection_Preliminaries_TreesWithChildren-AggregationScheme}).
This does not only allow to formalize the full version of Theorem \ref{theorem_InterpolationError}, but also to prove Theorem \ref{theorem_ConvergenceOfModelIterations} and \ref{theorem_InterpolationError} (Section \ref{subsubsection_ConstructionOfInflowTrees}).
Nevertheless, the following sections can also help to get a deeper understanding of our iterative algorithm, Equation \eqref{align_MessagePassing_AdvectionIterations}, whose intuition was already describe in Section \ref{subsection_FinalAlgorithm:MeGA-MP}.

\subsubsection{Preliminaries: Trees with Children-Aggregation Scheme}
\label{subsubsection_Preliminaries_TreesWithChildren-AggregationScheme}

In this section, we first investigate on the intuition behind the message passing scheme \eqref{align_MessagePassing_Advection} (and analogously, Eq. \eqref{align_MessagePassing_AdvectionIterations}), which -- by the definition of the weights $w_{uv}$ (Eq. \eqref{align_CouplingCondition_MixingAtNodes_MessagePassingFormulation}) -- can also be written as
\begin{align*}
    c_v(t)
    =
    \sum_{u \in \Nbh_-(v,t)}
    w_{uv}(t)
    \cdot
    \phi(t,e_{uv},\cm,\num).
\end{align*}

\noindent 
For simplicity, let us assume that the message function $\phi$ outputs $\phi(t,e_{uv},\cm,\num) = c_u(t - \delta t_{uv})$ (cf. Eq. \eqref{align_MessageFunction_Advection}), that is, that at time $t$, we observe a pass-through along the edge $e_{uv}$ and not a an inverse self-loop (cf. Lemma \ref{lemma_TransportTimes}).
In this case, Equation \eqref{align_MessagePassing_Advection} can -- with
(1) a change in notation and
(2) using the more formal notation from Equation \eqref{align_TransportTime_Functionscurl} 
-- be written as
\begin{align*}
    c_v(t)
    \overset{\text{(1)}}{=}
    \sum_{v_1 \in \Nbh_-(v,t)}
    w_{v_1v}(t)
    \cdot
    c_{v_1}(t - \delta t_{{v_1}v})
    \overset{\text{(2)}}{=}
    \sum_{v_1 \in \Nbh_-(v,t)}
    w_{v_1v}(t)
    \cdot
    c_{v_1}(
    \underbrace{
    t - \scurl(t,e_{v_1v},\num)
    }_{
    =: t_1
    }
    )
    .
\end{align*}

For each inflow neighbor $v_1 \in \Nbh_-(v,t)$ of $v$ at time $t$, we can again apply Equation \eqref{align_MessagePassing_Advection}, yielding
\begin{align*}
    c_v(t)
    =
    \sum_{v_1 \in \Nbh_-(v,t)}
    \sum_{v_2 \in \Nbh_-(v_1,t_1)}
    w_{v_1v}(t)
    \cdot
    w_{v_2v_1}(t_1)
    \cdot
    c_{v_2}(
    \underbrace{
    t_1 - \scurl(t_1,e_{v_2v_1},\num)
    }_{
    =: t_2
    }
    ).
\end{align*}

\noindent Intuitively, we can repeat this process as often as we want until a neighborhood $\Nbh_-(v_l,t_l)$ is empty, which by definition is the case if $v_l \in \Vb$ is a source node. What we create along the way are paths $p = ((v_0,t_0) := (v,t),(v_1,t_1),(v_2,t_2),...)$ of both locations $v_l \in V$ and times $t_l \in \R$ for indices $l = 1,2,...$ along which we trace back information that contributes to the signal $c_v(t)$.\footnote{
    Note that the indices $l = 1,2,...$ do \textit{not} necessarily correspond to the discrete times indices $I$ in Section \ref{subsection_FinalAlgorithm:MeGA-MP}.
}
Knowledge of such paths, which we call \textit{inflow paths}, can be helpful in case where the signal $c_v(t)$ at time $t$ is unknown, expressing it through signals $c_{v_l}(t_l)$ at earlier times $t_l < t$.
While the graph $\G$ is a symmetric directed graph (cf. Section \ref{section_Background}), the set of all such inflow paths $\Pa(v,t)$ defines a tree structure\footnote{
    Note that by definition, the transport time is positive (Theorem \ref{theorem_ExistenceAndUniquenessOfPositiveTransportTimes}).
    Therefore, $t > t_1 > t_2 > ...$ holds.
}, 
which we construct formally and iteratively from the root node $(v_0,t_0) = (v,t)$ to leaf nodes in the next Section \ref{subsubsection_ConstructionOfInflowTrees}. 
Before, we provide a general result on trees with children-aggregation aggregation schemes:

\begin{figure}[!h]
    \centering
    \includegraphics[width=0.8\textwidth]{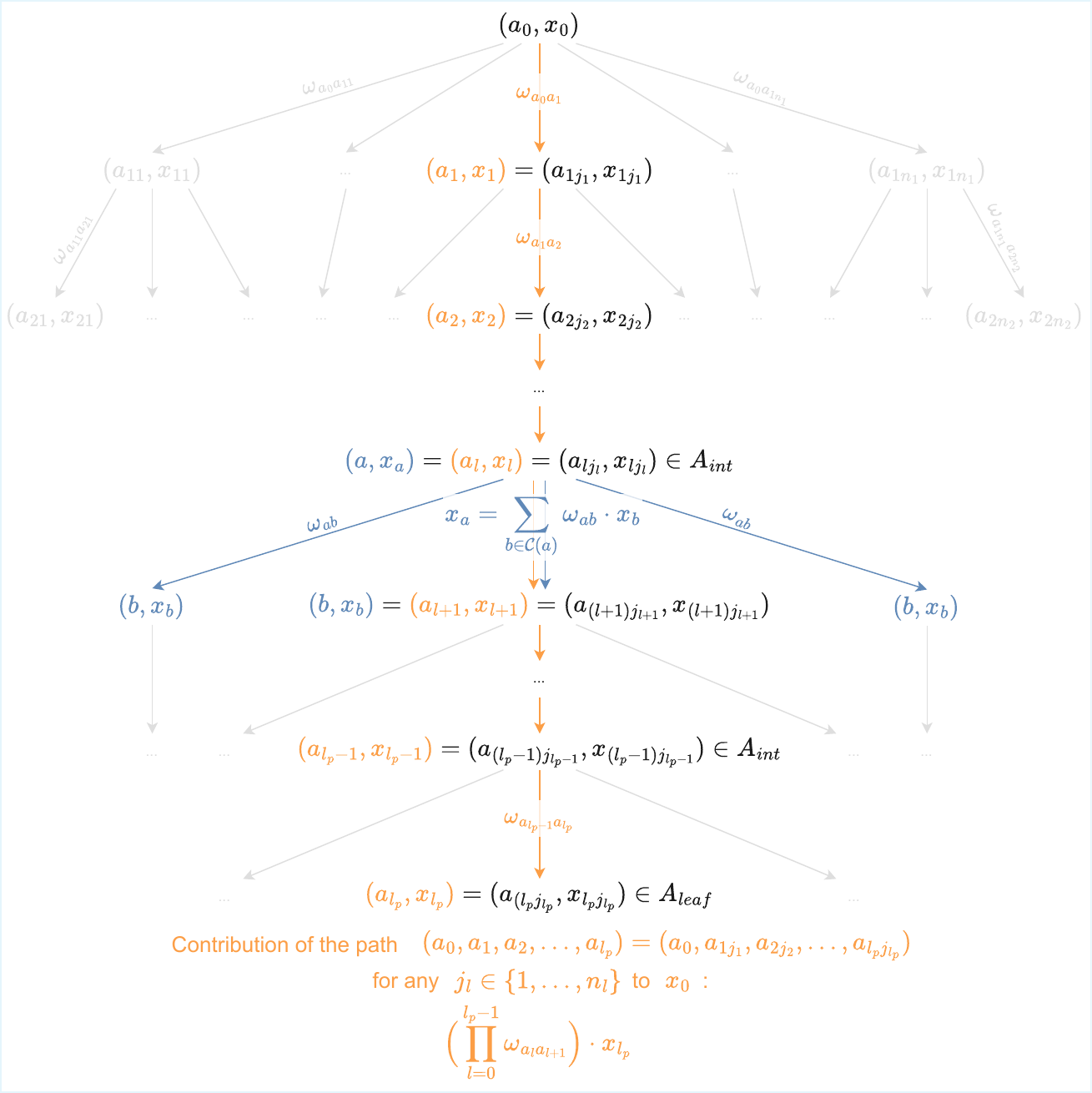}
    \caption{A visualization of the concepts from Theorem \ref{theorem_TreesWithChildrenAggregationScheme}.}
    \label{figure_TreesWithChildren-AggregationScheme}
\end{figure}

\begin{theorem}[Trees with children-aggregation scheme]
\label{theorem_TreesWithChildrenAggregationScheme}
    Let $\T = (A,E)$ be a finite tree with root node $a_0 \in A$, \textbf{int}ernal nodes $\An{int}$ and \textbf{leaf} nodes $\An{leaf}$ (that is, $A = \{a_0\} \sqcup \An{int} \sqcup \An{leaf}$ holds).
    If each node $a \in A$ is associated with a magnitude $x_a \in \R$ that satisfies
    \begin{align*}
            x_a
            =
            \sum_{b \in \Children(a)}
            \omega_{ab} \cdot x_b
            \text{ for all } 
            a \in \{a_0\} \sqcup \An{int},
    \end{align*}

    \noindent 
    where the set $\Children(a) \subset A$ consists of the children of $a \in A$ and $w_{ab}$ is a weight of the edge $e_{ab} \in E$ that depends on both the parent $a \in A$ and its child $b \in \Children(a)$, 
    then the magnitude $x_0 := x_{a_0}$ of the root node $a_0 \in A$ is determined by the magnitudes $x_{l_p} := x_{a_{l_p}}$ of the leaf nodes $a_{l_p} \in \An{leaf}$ by
    \begin{align*}
        x_0
        =
        \sum_{p = (a_0,...,a_{l_p}) \in \Pa(a_0,\An{leaf})}
        \omega(a_0,...,a_{l_p})
        \cdot
        x_{l_p},
    \end{align*}

    \noindent 
    where 
    $\Pa(a_0,\An{leaf})$ is the set of paths connecting the root node $a_0$ to one of its leaf nodes $\An{leaf}$ 
    and
    \begin{align*}
        \omega(a_0,...,a_{l_p})
        :=
        \prod_{l = 0}^{l_p-1}
        \omega_{a_la_{l+1}}
    \end{align*}

    \noindent 
    is the product of weights along a(n) (arbitrary) path $p = (a_0,...,a_{l_p}) \in \Pa$ with length $l_p \in \N$.
\end{theorem}

A visualization of such a tree can be found in Figure \ref{figure_TreesWithChildren-AggregationScheme}.

\subsubsection{Construction of Inflow Trees}
\label{subsubsection_ConstructionOfInflowTrees}

In this section, we first provide some intermediate results -- Theorem \ref{theorem_AdvectionMessagePassing_closedForm} and Theorem \ref{theorem_AdvectionMessagePassingIterations_closedForm} -- that help proving Theorem \ref{theorem_ConvergenceOfModelIterations} and \ref{theorem_InterpolationError}.
In order to prove the former, we construct suitable trees in the sense of Theorem \ref{theorem_TreesWithChildrenAggregationScheme} as already motivated in Section \ref{subsubsection_Preliminaries_TreesWithChildren-AggregationScheme}:

\begin{definition}[Inflow tree and interpolation inflow tree]
\label{definition_InflowTreeAndInterpolationInflowTree}
    Let $\G = (V,E)$ be an advection metric graph as defined in Section \ref{section_Background}.
    Let $v \in V$ and $t \in \R$ hold.

    We define the inflow tree $\T(v,t)$ of $v$ at time $t$ using the notation from Theorem \ref{theorem_TreesWithChildrenAggregationScheme} by iteratively defining the nodes and children per layer $l \in \N_0$ of the tree until a stopping criteria is fulfilled:

    \begin{itemize}
        \item $l = 0$-th layer:
        We choose the root node $a_0 = (v_0,t_0) := (v,t)$.

        \item $l+1$-th layer:
        Given a parent node $a_l = (u_l,v_l,t_l)$\footnote{
            Note that the definition of $a_{l+1}$ does not depend on $u_l$,
            which is why this step is well-defined for $l=0$ and $l+1=1$, too. 
        } in the $l$-th layer,
        we choose its children $a_{l+1} = (u_{l+1},v_{l+1},t_{l+1}) \in \Children(a_l)$ as each combination of
        
        \begin{itemize}
            \item an inflow neighbor $u_{l+1} \in \Nbh_-(v_l,t_l)$ of $v_l$ at time $t_l$ and

            \item the corresponding shifted time $t_{l+1} := t_l - \scurl(t_l,e_{u_{l+1}v_l},\num)$ obtained through the time shift applied to $t_l$ using the (positive) transport time $\scurl(t_l,e_{u_{l+1}v_l},\num) > 0$ (cf. Eq. \eqref{align_TransportTime_Functionscurl}) along the edge $e_{u_{l+1}v_l} \in E$ and

            \item the corresponding source node\footnote{
                In the former case,
                we observe a pass-through along the edge $e_{u_{l+1}v_l}$,
                that is, information flows from $v_{l+1} = u_{l+1}$ to $v_l$.
                In the latter case,
                we observe an inverse self-loop along the edge $e_{u_{l+1}v_l}$,
                that is, information flows from $v_{l+1} = v_l$ to $v_l$.
                For more details, see Section \ref{subsection_PhysicalDerivationOfMeGA-MP_Appendix}.
            }
            \begin{align}
            \label{align_(Interpolation)InflowTree_SourceNode}
                v_{l+1} 
                :=
                \begin{cases}
                    u_{l+1}
                    & \text{ if }
                    \zcurl(\scurl(t_l,e_{u_{l+1}v_l},\num)) \neq 0,
                    \\
                    v_l
                    & \text{ if }
                    \zcurl(\scurl(t_l,e_{u_{l+1}v_l},\num)) = 0.
                \end{cases}
            \end{align}
        \end{itemize}

        \item We stop the construction after an $a_{l+1} = (u_{l+1},v_{l+1},t_{l+1})$ if a stopping criteria is fulfilled, which for now can be any stopping criteria, or until naturally, the neighborhood $\Nbh_-(v_{l+1},t_{l+1})$ is empty, which is by definition the case if $v_{l+1} \in \Vb$ is a source node.

        \item We denote the set of paths connecting the root node $a_0 = (v_0,t_0) = (v,v,t)$ to one of its leaf nodes $\An{leaf}$ by $\Pa(v,t)$, which is an abbreviation for $\Pa(a_0,\An{leaf})$.
    \end{itemize}

    We similarly define the interpolation inflow tree $\That(v,t)$ of $v$ at time $t$ using the notation from Theorem \ref{theorem_TreesWithChildrenAggregationScheme} by iteratively defining the nodes and children per layer $l \in \N_0$ of the tree until a stopping criteria is fulfilled:

    \begin{itemize}
        \item $l = 0$-th layer:
         We choose the root node $a_0 = (v_0,t_0) := (v,t)$.

        \item $l+1$-th layer:
        Given a parent node $a_l = (u_l,v_l,t_l)$ in the $l$-th layer, 
        we choose its children $a_{l+1} = (u_{l+1},v_{l+1},t_{l+1}) \in \Children(a_l)$ as each combination of
        
        \begin{itemize}
            \item an inflow neighbor $u_{l+1} \in \Nbh_-(v_l,t_l)$ of $v_l$ at time $t_l$ and

            \item an corresponding interpolation time $t_{l+1} \in \{t_j,t_{j+1}\}$ for the $j \in I$ for which $t_l - \scurl(t_l,e_{u_{l+1}v_l},\num) \in [t_j,t_{j+1})$ holds (cf. Eq. \eqref{align_MessageFunction_Advection_LinearInterpolation})\footnote{
                 Similar to the construction to the inflow tree $\T(v,t)$,
                $t_l - \scurl(t_l,e_{u_{l+1}v_l},\num)$ is the -- possibly not in the observable, discrete subdomain of $\R$ lying -- shifted time obtained through the time shift applied to $t_l$ using the (positive) transport time $\scurl(t_l,e_{u_{l+1}v_l},\num) > 0$ (cf. Eq. \eqref{align_TransportTime_Functionscurl}) along the edge $e_{u_{l+1}v_l} \in E$.
                Note that for $l \in \N$, $t_l$ already corresponds to one of the discrete time points that are known, that is, there exists an $j \in I$ such that $t_l = t_j$ holds.
            } and

            \item the corresponding source node $v_{l+1}$ analogously to Equation \eqref{align_(Interpolation)InflowTree_SourceNode}.\footnote{
                Note that the definition of $v_{l+1}$ does not depend on $t_{l+1}$, which is why the corresponding source node $v_{l+1}$ of both $t_{l+1} = t_j$ and $t_{l+1} = t_{j+1}$ is the same.
            }
        \end{itemize}

    \item We stop the construction after an $a_{l+1} = (u_{l+1},v_{l+1},t_{l+1})$ for which $v_{l+1} \in \Vb$ or $t_{l+1} \in  \Tn{hi} = \{t_1,...,t_i\}$ holds.\footnote{
        Note that similar to the construction to the inflow tree $\T(v,t)$, we can also chose any other suitable stopping criteria. 
        We chose the one that we will need later. 
        Also note that with this stopping criterium, if $v \in \Vb$ or $t \in \Tn{hi}$ holds, the creation of the tree $\That(v,t)$ stops immediately and does only consist of the path $p = ((v_0,t_0) = (v,t))$ of length zero.
    }

    \item We denote the set of paths connecting the root node $a_0 = (v_0,t_0) = (v,t)$ to one of its leaf nodes $\An{leaf}$ by $\Pahat(v,t)$, which is an abbreviation for $\Pahat(a_0,\An{leaf})$.

    \end{itemize}
\end{definition}

\begin{remark}[Finite depths of interpolation inflow tree]
\label{remark_FiniteDepthOfInterpolationInflowTrees} 
    We want to explicitly highlight that for each $v \in V$ and $t \in \R$, the stopping criteria for the interpolation inflow tree $\That(v,t)$ is fulfilled at some point.
    In other words, the interpolation inflow tree $\That(v,t)$ has a finite depth and each path $p = ((v_0,t_0)=(v,t),(u_1,v_1,t_1),...,(u_{l_p},v_{l_p},t_{l_p}))\in \Pahat(v,t)$ has a finite length $l_p \in \N$.
    This is because by definition of the \gls*{OP} \eqref{align_TransportTime_OP}, we assume positive transport times (also see discussion above Theorem \ref{theorem_PhysicalAlignment_fullVersion}).
    Consequently, $t = t_0 > t_1 > t_2 > ...$ holds and for arbitrary $t \in \R$, at some point, $t_{l_p} \in \Tn{hi}$ must hold.
\end{remark}

We now leverage these trees together with Theorem \ref{theorem_TreesWithChildrenAggregationScheme} to prove the following two intermediate results, which allow to reconstruct the magnitudes $c_v(t)$ and $\chat_v^{(k)}(t)$ of interest along inflow paths:

\begin{theorem}[Closed form of advection message passing]
\label{theorem_AdvectionMessagePassing_closedForm}
    Let $v \in V$ and $t \in \R$ hold.
    Let $\T(v,t)$ be the inflow tree of $v$ at time $t$ as defined in Definition \ref{definition_InflowTreeAndInterpolationInflowTree}.
    If $c_v(t)$ is defined by Equation \eqref{align_MessagePassing_Advection}, then
    \begin{align}
    \label{align_MessagePassing_Advection_closedForm}
        c_v(t)
        =~
        \sum_{ p = ((v_0,t_0)=(v,t),(u_1,v_1,t_1),...,(u_{l_p},v_{l_p},t_{l_p})) \in \Pa(v,t) } ~
        \left(
        \prod_{l = 0}^{l_p-1}  
        w_{u_{l+1}v_l}(t_l)
        \right)
        \cdot
        c_{v_{l_p}}(t_{l_p})
    \end{align}

    \noindent 
    with weights $w_{u_{l+1}v_l}(t_l)$ according to Equation \eqref{align_CouplingCondition_MixingAtNodes_MessagePassingFormulation} holds.
\end{theorem}

\begin{theorem}[Closed form of advection message passing iterations]
\label{theorem_AdvectionMessagePassingIterations_closedForm}
    Let $v \in V$ and $t \in T = \Tn{hi} \cup \Tn{fu}$ hold.
    Let $\That(v,t)$ be the interpolation inflow tree of $v$ at time $t$ as defined in Definition \ref{definition_InflowTreeAndInterpolationInflowTree}.
    Let $k \geq \max_{p \in \Pahat(v,t)} l_p$ hold.
    If $\chat_v^{(k)}(t)$ is defined by Equation \eqref{align_MessagePassing_AdvectionIterations}, then
    \begin{align}
    \label{align_MessagePassing_AdvectionIterations_closedForm}
        \chat_v^{(k)}(t)
        =~
        \sum_{ p = ((v_0,t_0)=(v,t),(u_1,v_1,t_1),...,(u_{l_p},v_{l_p},t_{l_p})) \in \Pahat(v,t) } ~
        \left(
        \prod_{l = 0}^{l_p-1}  
        w_{u_{l+1}v_l}(t_l)
        \cdot
        a(t_{l+1})
        \right)
        \cdot
        \chat_{v_{l_p}}^{(k-l_p)}(t_{l_p})
    \end{align}

    \noindent 
    with weights $w_{u_{l+1}v_l}(t_l)$ and $a(t_{l+1})$ according to Equation \eqref{align_CouplingCondition_MixingAtNodes_MessagePassingFormulation} and \eqref{align_LinearInterpolation}, respectively, holds.
\end{theorem}

Theorem \ref{theorem_AdvectionMessagePassingIterations_closedForm} already allows a proof of Theorem \ref{theorem_ConvergenceOfModelIterations}, which is given together with the other proofs in Appendix \ref{section_Proofs}.

For the proof of Theorem \ref{theorem_InterpolationError}, we need to make some further assumptions:
While Equation \eqref{align_MessagePassing_Advection_closedForm} and \eqref{align_MessagePassing_AdvectionIterations_closedForm} are already of similar structure,
the number of paths $\phat \in \Pahat(v,t)$ grows exponentially by a factor of $2^l$, where $l$ is the depth of the tree $T(v,t)$ and $\That(v,t)$, as compared to the number of paths $p \in P(v,t)$.
This is due the fact that per time that we observe in the inflow tree $T(v,t)$, we observe two interpolation times in the interpolation inflow tree $\That(v,t)$ (cf. Definition \ref{definition_InflowTreeAndInterpolationInflowTree}).
To make each summand in Equation \eqref{align_MessagePassing_AdvectionIterations_closedForm} comparable to a summand in Equation \eqref{align_MessagePassing_Advection_closedForm}, we need to make the following assumptions:

\begin{assumption}[Identifiable paths]
\label{assumption_IdentifiablePaths}
    Let $v \in V$ and $t \in \R$ hold.
    Let $T(v,t)$ and $\That(v,t)$ be the inflow tree and the interpolation inflow tree from Definition \ref{definition_InflowTreeAndInterpolationInflowTree}, respectively.
    We assume that each path 
    $\phat = ((\vhat_0,\that_0)=(v,t),(\uhat_1,\vhat_1,\that_1)...,(\uhat_{l_p},\vhat_{l_p},\that_{l_p})) \in \Pahat(v,t)$ 
    is uniquely identifiable with a path 
    $p = ((v_0,t_0)=(v,t),(u_1,v_1,t_1),...,(u_{l_p},v_{l_p},t_{l_p})) \in \Pa(v,t)$ 
    in the sense that 
    $\uhat_l = u_l$ and $\vhat_l = v_l$
    holds for all $l = 1,...,l_p$.
\end{assumption}

\begin{assumption}[Equal time for jumping into history]
\label{assumption_EqualtTimeForJumpingIntoHistory}
    Let $v \in V$ and $t \in \R$ hold.
    Let $\That(v,t)$ be the interpolation inflow tree from Definition \ref{definition_InflowTreeAndInterpolationInflowTree}.
    In the definition of each $t_{l+1} \in \{t_j,t_{j+1}\}$, we assume that $t_j, t_{j+1} \in \Tn{fu}$ or $t_j,t_{j+1} \in \Tn{hi}$ holds.
\end{assumption}

Assumption \ref{assumption_IdentifiablePaths} is satisfied as long as the inflow neighborhoods $\Nbh_-(v_l,t_l)$, from which we choose an inflow neighbor $u_{l+1}$ when creating a path in $\T(v,t)$, are equal to the corresponding neighborhoods $\Nbh_-(\vhat_l,\that_l)$, from which we choose an inflow neighbor $\uhat_{l+1}$ when creating a path in $\That(v,t)$. This is guaranteed if the underlying flow field $\num$ that defines these neighborhoods does not change its orientation during the time intervals $[t_j,t_{j+1}]$ for which $t_l \in [t_j,t_{j+1})$ holds.
If the assumption holds, it allows the following definition.

\begin{definition}[Joint paths]
\label{definition_JointPaths}
    Let $v \in V$ and $t \in \R$ hold.
    Let $T(v,t)$ and $\That(v,t)$ be the inflow tree and the interpolation inflow tree from Definition \ref{definition_InflowTreeAndInterpolationInflowTree}, respectively.
    Let assumption \ref{assumption_IdentifiablePaths} hold. 
    If the path 
    $\phat = ((\vhat_0,\that_0)=(v,t),(\uhat_1,\vhat_1,\that_1)...,(\uhat_{l_p},\vhat_{l_p},\that_{l_p})) \in \Pahat(v,t)$ 
    is uniquely identifiable with the path
    $p = ((v_0,t_0)=(v,t),(u_1,v_1,t_1),...,(u_{l_p},v_{l_p},t_{l_p})) \in \Pa(v,t)$,
    we define the joint path of $\phat$ by
    \begin{align}
    \label{align_JointPaths}
        \ptilde = ((v_0,t_0)=(v,t),(u_1,v_1,t_1,\that_1),...,(u_{l_p},v_{l_p},t_{l_p},\that_{l_p})).
    \end{align}

    \noindent 
    We can still consider $\ptilde$ as an element of $\That(v,t)$, since we only added information that are unused, in, e.g., Theorem \ref{theorem_AdvectionMessagePassingIterations_closedForm}. 
\end{definition}

\noindent 
Moreover, Assumption \ref{assumption_EqualtTimeForJumpingIntoHistory} guarantees that there exist two such paths that are exactly equal up to the last $\that_{l_p} \in \{t_j,t_{j+1}\}$ for a $j \in I$ in the sense of Definition \ref{definition_InflowTreeAndInterpolationInflowTree}.
Finally, using the property that per such two paths ending in such two left and right interpolation bounds, the factors $a(t_j)$ and $a(t_{j+1})$ add up to one, we can canonically extend the sum in Equation \eqref{align_MessagePassing_Advection_closedForm} to the sum in Equation \eqref{align_MessagePassing_AdvectionIterations_closedForm} using the joint paths.
This allows to prove Theorem \ref{theorem_InterpolationError}, which in its full version is phrased as follows:

\begin{theorem}[Interpolation error (full version)]
\label{theorem_InterpolationError_fullVersion}
    Let $v \in V$ and $t \in T = \Tn{hi} \cup \Tn{fu}$ hold.
    Let $\T(v,t)$ and $\That(v,t)$ be the inflow tree and the interpolation inflow tree of $v$ at time $t$ as defined in Definition \ref{definition_InflowTreeAndInterpolationInflowTree}, respectively.
    Let Assumption \ref{assumption_IdentifiablePaths} and \ref{assumption_EqualtTimeForJumpingIntoHistory} hold.
    Then the prediction error in-between $c_v(t)$ and $\chat_v(t)$ is given by
    \begin{align*}
        &~~~
        |c_v(t) - \chat_v(t)| 
        \\
        \leq&~
        \sum_{ \ptilde \in \Pahat(v,t) } ~
        a(\that_1,...,\that_{l_p})
        \cdot
        \left(
        \underbrace{
        |c_{v_{l_p}}(t_{l_p}) - c_{v_{l_p}}(\that_{l_p})|
        }_{
        \text{Temporal interpolation error}
        }
        +
        |c_{v_{l_p}}(t_{l_p})|
        \cdot
        \underbrace{
        |w(t_0,...,t_{l_p-1}) - w(\that_0,...,\that_{l_p-1})|
        }_{
        \text{Accumulated weighting error}
        }
        \right)
    \end{align*}

    \noindent
    with joint paths $\ptilde = ((v_0,t_0)=(v,t),(u_1,v_1,t_1,\that_1),...,(u_{l_p},v_{l_p},t_{l_p},\that_{l_p})) \in \Pahat(v,t)$ according to Definition \ref{definition_JointPaths},
    \begin{align*}
        a(\that_1,...,\that_{l_p})
        :=
        \prod_{l = 0}^{l_p-1}
        a(\that_{l+1})
        \quad \text{ and } \quad
        w(\that_0,...,\that_{l_p-1})
        :=
        \prod_{l = 0}^{l_p-1}  
        w_{u_{l+1}v_l}(\that_l)
    \end{align*}

    \noindent 
    with weights $w_{u_{l+1}v_l}(t_l)$ and $a(t_{l+1})$ according to Equation \eqref{align_CouplingCondition_MixingAtNodes_MessagePassingFormulation} and \eqref{align_LinearInterpolation}, respectively.
    If the temporal interpolation error is zero, we obtain $|c_v(t) - \chat_v(t)| = 0$.
\end{theorem}

Theorem \ref{theorem_InterpolationError_fullVersion} states that the error between the ground truth $c_v(t)$ and the prediction $\chat_v(t)$ solely stems from the time interpolation:
On the one hand, the temporal interpolation error is the error from observing a value $c_{v_{l_p}}(\that_{l_p})$ at a node $v_{l_p} \in V$ and at a discrete  time $t_{l_p} \in \Tn{hi}$, when actually, we needed to observe the value $c_{v_{l_p}}(t_{l_p})$ at an eventually non-discrete time $t_{l_p} \in \R$.
On the other hand, for exactly the same reason, the weights $w_{u_{l+1}v_l}(\that_l)$ that we observe at discrete times $\that_l \in \Tn{hi} \sqcup \Tn{fu}$ might differ from the weights $w_{u_{l+1}v_l}(t_l)$ that we needed to observe at non-discrete times $t_l \in \R$ for all $l = 1,...,l_p-1$, yielding the accumulated weighted error. 
Both errors decrease with decreasing content of high frequencies in concentrations $\cm$ and flow fields $\num$.

\section{Details on the Experiments}
\label{section_DetailsOnTheExperiment_Appendix}

\subsection{Experiment 1: Zero-Shot Transfer Learning of Advection-Reaction Dynamics on Real-World Metric Graphs}
\label{section_Experiment1_Zero-ShotTransferLearningOfAdvection-ReactionDynamicsOnRealWorldMetricGraphs_Appendix}

\paragraph{Domain}
In both networks Hanoi and L-Town, we model chlorine decay, which is part of the edge dynamics in Equation \eqref{align_EdgeDynamics_AdvectionReaction}, as a first-order reaction over time and also include interactions with the pipe wall, which depend on the pipe diameter $\rho_{uv} \in \R_{\geq 0}$. This is formally defined as
\begin{align}
\label{align_EdgeDynamics_Reaction_ChlorineDecay}
    f_r(t, z, e_{uv}) 
    = 
    -\left(k_b + \frac{4 k_w}{\rho_{uv}} \right) 
    c_{uv}(t, z),
\end{align}

\noindent where the parameters $k_b \in \R_{\geq 0}$ and $k_w \in \R_{\geq 0}$ control the intensity of chlorine decay over time and due to wall reactions, respectively.

\paragraph{Dataset} 
We generate datasets of 1000 training, 200 validation, and 200 testing samples for the Hanoi \gls*{WDS}, and 200 testing samples for the L-Town \gls*{WDS}.
Each sample consists of $n = 701$ time steps with a time difference of $dt = 60$ (in s) and corresponds to a time series $T$ that describes the distribution of chlorine $\cm$ over discrete time points $t \in T$ and through the \gls*{WDS} $\G$, determined by advection or advection-reaction dynamics as described in section \ref{subsection_Domain_Dataset}.
The chlorine is injected over time $t \in T$ at boundary nodes $\Vb$. 
The samples differ in graph topology, in flow fields and injection pattern (boundary condition).
We now describe this data generation process in more detail.
As stated in Section \ref{section_Experiment1_Zero-ShotTransferLearningOfAdvection-ReactionDynamicsOnRealWorldMetricGraphs}, we use EPyT-Flow \cite{Artelt2024_WaterDomain_StateSimulation_EPyT-Flow} based on EPANET-MSX \cite{Shang2023WaterDomain_StateSimulations_EpanetMSX2.0} as the simulator.

\begin{itemize}
    \item \textbf{Fixed parameters:} Boundary nodes $\Vb$ at which chlorine is injected correspond to reservoir nodes $\Vr$. 
    This corresponds to one and two boundary nodes in the case on Hanoi and L-Town, respectively.
    For the advection-reaction datasets, the values of the reaction coefficients in Equation \eqref{align_EdgeDynamics_Reaction_ChlorineDecay} that model the decay of the concentration are set to $k_b = 0.04$ and $k_w = 0.034$.

    \item \textbf{Varying graph topology:} In order to vary graph topology, we sample 
    pipe lengths
    $l_{uv} \sim U(0.1, 80)$ (in m) and 
    pipe diameters
    $\rho_{uv} \sim U(0.06, 0.025)$ (in m)
    uniformly per sample at random. To maintain the typical \gls*{WDS} topology of larger pipes at reservoirs and successively smaller pipes towards leaf nodes, diameters are ordered from low to high, as in the original \glspl*{WDS} Hanoi and L-Town.

    \item \textbf{Varying flow fields:} In order to generate scenarios with different flow fields, we vary nodal demands $\dm(t) = (d_v(t))_{v \in V}$ that directly determine the flow field $\num(t) = (\nu_e(t))_{e \in E}$ through hydraulic equations \cite{Rossman2020WaterDomain_StateSimulations_Epanet2.2}. 
    Empirically, independent random demands yield spatially homogeneous flow fields.  
    To introduce spatial variations of the flow fields, we apply spectral clustering $s : V \rightarrow \{0, ..., \n{c}\}$ on the binary network adjacency matrix to partition the network into $\n{c}=4$ clusters for Hanoi and $\n{c}=7$ clusters for L-Town. Then the demand pattern for a node $v \in V$ is sampled from a cluster-shared Gaussian process plus independent noise:
    \begin{align*}
        d_v(t) 
        = 
        f_{s(v)}(t) + \epsilon_v, 
        \quad 
        \epsilon_{v} \sim \mathcal{N}(0, 0.1^2), 
        \quad 
        f_{s(v)} \sim \mathcal{GP}(0, K), 
        \quad 
        K(i, j) = \exp\left(-\frac{(i - j)^2}{2 \cdot 0.1^2 \cdot (n - 1)^2} \right).
    \end{align*}
    
    \noindent
    The resulting hydraulics show flow velocities $\num$ that vary in space and time and can also change direction. Assigning demands $d_v(t)$ may result in physically implausible states that generate negative pressures, however, these states can be easily made plausible by offsetting the heads accordingly, which would have no effect on the flows. Therefore, the flows are physically plausible and since the quality simulation only depends on flows and not on pressures, the resulting evolution of chlorine in the system is still valid.
    
    \item \textbf{Varying injection patterns:} 
    The injection patterns $c_v(t)$ (Dirichlet boundary condition) in which chlorine is injected at boundary nodes $v \in \Vb$ are sampled from the Fourier function space (using the first $M$ modes) as follows:
    \begin{align*}
        \overline{c}_v(t) 
        &= 
        \sum_{k = 1}^{M} \frac{1}{k^\alpha}
        \left(
        a_k \cos\left(2k\pi \frac{t}{60 \cdot 700}\right) + b_k \sin\left(2k\pi \frac{t}{60 \cdot 700}\right)
        \right), \\
        c_v(t) 
        &= 
        \frac{
        \overline{c}_v(t) 
        - 
        \min_{t \in T} \hat{c}_v(t)
        }{
        \max_{t \in T} \overline{c}_v(t) 
        - 
        \min_{t \in T}\overline{c}_v(t)
        }
        \in [0,1],
    \end{align*}
    
    \noindent
    with $M \sim \mathcal{U}_{int}(3, 30)$, $\alpha \sim \mathcal{U}(0.5, 1.5)$, and $a_k, b_k \sim \mathcal{N}(0, 1)$, for all $v \in \Vb$, $t \in T$. Here $\mathcal{U}_{int}$ corresponds to an integer-valued uniform distribution.
    Three samples from this function space are visualized in Figure \ref{fig_InjectionPattern}.  
    
    \begin{figure}[!h]
        \centering
        \includegraphics[width=\linewidth]{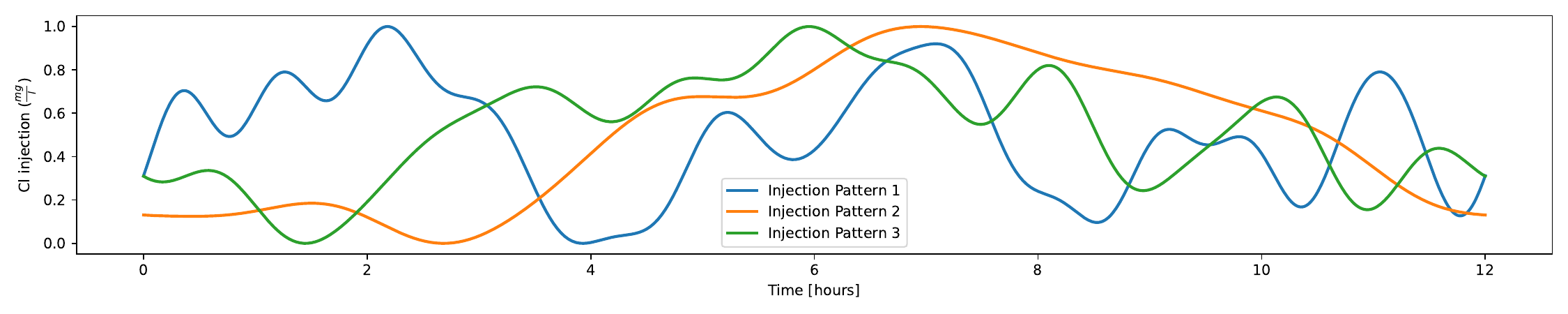}
        \caption{Example patterns used for the chlorine injection at boundary nodes $\Vb$.}
        \label{fig_InjectionPattern}
    \end{figure}
\end{itemize}

\paragraph{Model}
The \gls*{MLP} that is intended to learn the reactions consists of two linear layers with $8$ neurons each, and a SELU activation \cite{SELU2017} in between. We train the model with the Adam optimizer \cite{Kingma2014AdamAM} over $250$ epochs with a batch size of $256$ using a learning rate of $0.001$ that is reduced by a factor of $0.2$ if the training loss did not decrease for the last 3 epochs. The loss function is the \gls*{MAE}. Boundary nodes are masked.

\paragraph{Baselines}
The baselines we compare ourselves to are:

\begin{itemize} 
    \item \textbf{NNConv:} As a simple baseline, we use the kernel-based \gls*{MPNN} from \citet{Gilmer2017GNNs_Methods_MPNN}. The node feature matrix $\Xm \in \mathbb{R}^{|V| \times n}$ with $n = \n{hi} + \n{fu}$ is populated with the complete time series at boundary nodes $\Vb$, as well as the first time step as the known history $\Tn{hi}$ at all other nodes $V \setminus \Vb$. The edge feature matrix $\mathbf{E} \in \mathbb{R}^{|E| \times 3 n}$ contains the time series of scaled edge flows, edge capacities, and transport times.
    Node and edge features are first mapped into a latent space via two separate linear layers. After that, GG-NN convolutional layers \cite{Gilmer2017GNNs_Methods_MPNN} (NNConv implementation from torch-geometric) are applied.
    The message function is implemented as a two-layer \gls*{MLP} with a ReLU activation between the two layers. 
    The messages are aggregated via sum aggregation and decoded by a single linear layer calibrated manually.
    
    \item \textbf{PDE-GNN:} \citet{Hermes2025WaterDomain_WaterQuality_Methods} is a \gls*{PDE}-style \gls*{GNN} which solves initial and boundary value problems and has been proposed for chlorine estimation in \glspl*{WDS}. The model learns a PDE function that is integrated by Euler integration. Calibrated manually, we apply the same scheduler with $0.2$ learning rate decay upon training loss plateau over 3 epochs and we set the initial learning rate to 0.001. The time step used for Euler integration is set to $dt = 0.25$. 
    Some hyperparameters are tuned, other hyperparameters are left as specified in \citet{Hermes2025WaterDomain_WaterQuality_Methods}.

    \item \textbf{GPSConv:}
    \citet{rampasek2022GPSConv} is a graph transformer that we include as a representative of an architecture capable of capturing global dependencies. As positional encodings, we utilize the graph Laplacian eigenvectors, since positional encodings based on random walks would degenerate to zero if the initial conditions are zero. The Laplacian positional encoding is truncated to 30 components, allowing the transfer-learning experiment from Hanoi (32 nodes) to the L-Town (784 nodes), and embedded to 8 dimensions via a linear layer. We tune the message-passing module, which corresponds to the parameter $\Phi$ in Table \ref{table_hyperparameter}, candidates are NNConv, GINE, and GatedGCN \cite{Gilmer2017GNNs_Methods_MPNN, hu2020gine, bresson2017gatedGCN}, which are architectures that can accommodate edge features.

    \item \textbf{A-DGN:}
    \citet{gravina2023antisymmetric} is a dynamical-systems based \gls*{GNN}, where the message-passing operation acts as a differential operator, similar to an \gls*{ODE}. As such, it computes updates $\frac{\partial \mathbf{X}}{\partial l}$ of the node features $\mathbf{X}$ that are integrated over several layers $l$ via Euler integration with a step size of $\epsilon$. This model allows deep architectures that are stable and non-dissipative, mediated through antisymmetric weight matrices applied to the output of any message passing layer $\Phi$. We use this model in the weight-sharing style and select the same options for $\Phi$ as for GPSConv. 
    As generalization is a focus here, we evaluate two variations A-DGN and A-DGN$_{Dia}$, to account for the difference in size of the training and evaluation graph: The first is a hyperparameter tuned baseline, while for the second variant, we change the number of iterations of the trained model to the graph diameter. We have found that the latter works much better when the step size $\epsilon$ is adapted as well. Specifically, we set $\epsilon'=\epsilon \frac{num\_iters}{diameter(\mathcal{G})}$, and $num\_iters' = diameter(\mathcal{G})$.

     \item \textbf{\gls*{PINN}-based approaches} on metric graphs \cite{Blechschmidt2022MGNNs_MetricGraphNeuralODEandPDESolver_Solution-basedArchitechtures_PINNsForDriftDiffusionOnMetricGraph, Laczko2025MGNNs_MetricGraphNeuralODEandPDESolver_Solution-basedArchitechtures_PINNsForSchrödingerOnQuantumGraphs} are effective for learning solutions of specific \gls*{PDE} settings on metric graphs and are suited for inverse problems. However, \glspl*{PINN} are not operator-learning frameworks: A change in PDE parameters or boundary conditions typically requires retraining. Since our focus is on flexible models that generalize across varying flow conditions and topologies, \gls*{PINN}-based approaches are not an appropriate baseline in our setting. 

    \item \textbf{PI DEEPONET} is a \textit{LEGO}-like neural operator for metric graphs \cite{Blechschmidt2025MGNOs_NormalMGNOs_Solution-basedArchitechtures_DeepONetsForDriftDiffusionOnMetricGraphs} originally developed for drift–diffusion edge dynamics. While this approach is highly effective for diffusive dynamics, its direct application to purely hyperbolic problems such as linear advection is not straightforward. First, the method assumes flow-directed edges, effectively imposing a fixed transport direction that is incompatible with more general advection fields that can change direction. Second, at inner nodes the coupling is enforced through Kirchhoff–Neumann conditions, which conserve diffusive gradient fluxes $\partial_zc$. For hyperbolic dynamics, the flux is defined as $\nu(t, z) c(t, z)$ which would require a different coupling law such as mixing. For these reasons we omit this work as a baseline for now.

\end{itemize}

All models are trained using the Adam optimizer over a maximum of 1000 epochs. During training we validate the model every 10 epochs and training is stopped early if the validation loss did not improve over the last 50 epochs.

For a fair comparison of \gls*{MeGA-MP} and \gls*{MeGA-MP}$^+$ to the other methods, we run a small-scale hyperparameter tuning for every baseline method using Optuna \cite{optuna_2019}. The search space is specified in Table \ref{table_hyperparameter}. Every configuration is trained for 60 epochs. Trials are terminated at an earlier epoch if the trial's result is worse than the median of results of previous trials at the same epoch (median pruning). New configurations are sampled based on the tree-structured Parzen Estimator (TPE) \cite{watanabe2025treestructuredparzenestimatorunderstanding} implemented in Optuna.

\begin{table}[!h]
\centering
\renewcommand{\arraystretch}{1.3}
\caption{Optuna Search Space for Different Model Types}
\label{table_hyperparameter}
\resizebox{\linewidth}{!}{
\begin{tabular}{l|c|c|c|c}
\hline
\textbf{Parameter} & \textbf{NNConv} & \textbf{PDE-GNN} & \textbf{GPSConv} & \textbf{A-DGN} \\
\hline

\textit{hidden\_dim }
& \multicolumn{4}{c}{\{64, 96, 128, 160\}} \\
\hline

\textit{num\_layers} 
& \multicolumn{4}{c}{\{2, 3, 4, 5, 6\}} \\
\hline

\textit{learning\_rate }
& \multicolumn{4}{c}{$[10^{-4}, 3\times 10^{-3}]$ (log)} \\
\hline

\textit{weight\_decay }
& \multicolumn{4}{c}{$[10^{-6}, 10^{-1}]$ (log)} \\
\hline

\textit{batch\_size }
& \multicolumn{4}{c}{\{16, 32, 64\} (tests with larger batch-sizes were not successful)} \\
\hline

\textit{clip\_grad\_norm }
& --
& $\{10^{-4}, 0.5, 1.0, 2.0, -\}$ 
& --
& --
 \\
\hline

$\n{fu}$ (training) 
& 700 
& $\{32,100,150\}$ 
& 700 
& 700 \\
\hline

\textit{dropout }
& -- 
& [0.0, 0.3]
& [0.0, 0.3] 
& --
\\
\hline

\textit{num\_heads }
& -- 
& -- 
& \{2,4,8\}
& --\\
\hline

\textit{attn\_dropout }
& --
& --  
& [0.0, 0.3]
& -- 
\\
\hline

\hline

\textit{step\_size} / $\epsilon$ 
 & -- & \{1/4, 1/3, 1/2, 1\} & --  & \{ 1.0, 0.1, 0.01, 0.001 \}\\
\hline

$\gamma$ 
& -- & --  & -- & \{ 1.0, 0.1, 0.01, 0.001 \}\\
\hline

$\Phi$ 
 & -- & --  & \{ \small NNConv, GINE, GatedGCN \} & \{ \small NNConv, GINE, GatedGCN \}\\
\hline

\textit{num\_iters }
& -- & -- & -- &  \{ 1, 2, 5, 10, 20, 50 \}\\
\hline

& \multicolumn{4}{c}{\textbf{Fixed Hyperparameters}} \\
\hline



\multirow{1}{*}{$\n{hi}$} 
& \multicolumn{4}{c}{1} \\
\hline

$\n{fu}$ (evaluation)
& \multicolumn{4}{c}{700} \\
\hline

\multirow{1}{*}{\textit{loss\_fn}} 
& \multicolumn{4}{c}{L1} \\
\hline

\multirow{1}{*}{\textit{epochs}} 
& \multicolumn{4}{c}{60 (hyperparameter tuning) / 1000 (training, with early stopping)} \\
\hline

\end{tabular}
}
\end{table}

\newpage
\paragraph{Results}
As an extension of our ablation studies in Section \ref{subsection_Experiment1.2_AblationStudies},
in Figures \ref{figure_AblationStudy_Advection_Hanoi} to \ref{figure_AblationStudy_AdvectionReaction_L-Town},
we compare the performance of \gls*{MeGA-MP}, \gls*{MeGA-MP}$^+$ and \gls*{MeGA-MP}$^-$ on both the advection and the advection-reaction dynamical systems as well as on Hanoi and L-Town qualitatively.

\begin{figure}[!h]
    \centering
    \includegraphics[width=0.99\linewidth]{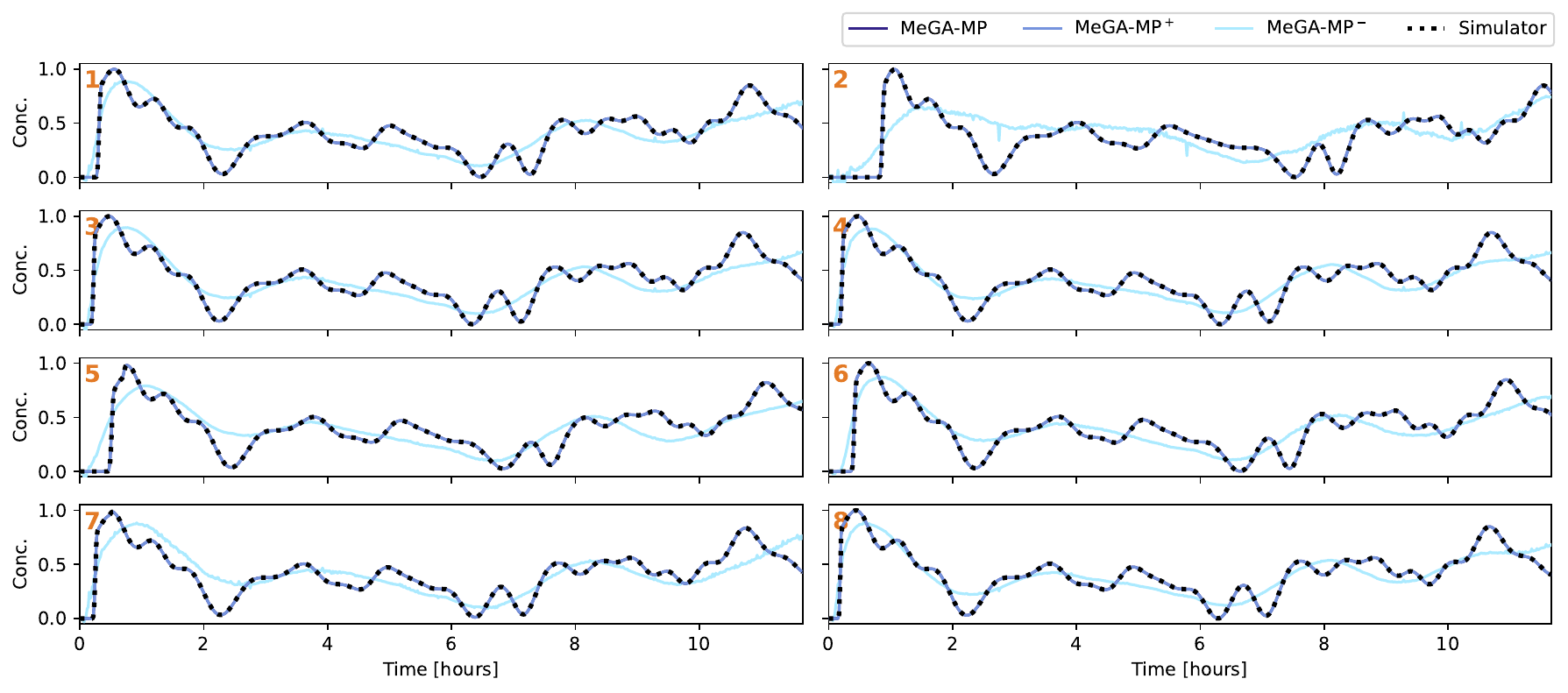}
    \caption{
    Predictions of \gls*{MeGA-MP}, \gls*{MeGA-MP}$^+$ and \gls*{MeGA-MP}$^-$
    for the \textbf{advection} dynamical system 
    on the \gls*{WDS} \textbf{Hanoi}. 
    The selected nodes are the same as in Figure \ref{figure_ComparisonBaselines_Advection_Hanoi_with_map}.}
\label{figure_AblationStudy_Advection_Hanoi}
\end{figure}

\begin{figure}[!h]
    \centering
    \includegraphics[width=0.99\linewidth]{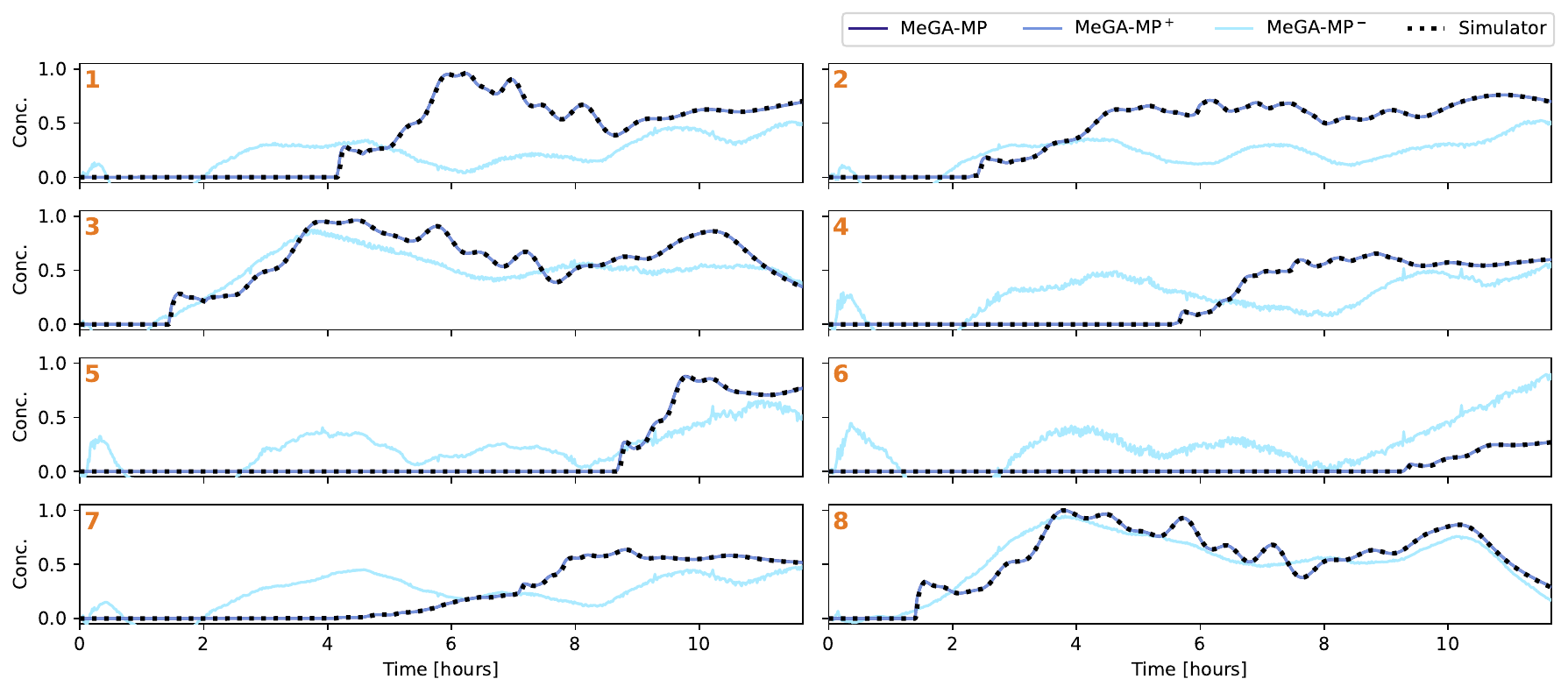}
    \caption{
    Predictions of \gls*{MeGA-MP}, \gls*{MeGA-MP}$^+$ and \gls*{MeGA-MP}$^-$
    for the \textbf{advection} dynamical system 
    on the \gls*{WDS} \textbf{L-Town}. 
    The selected nodes are the same as in Figure \ref{figure_ComparisonBaselines_Advection_L-Town_with_map}.}
    \label{figure_AblationStudy_Advection_L-Town}
\end{figure}

\begin{figure}[!h]
    \centering
    \includegraphics[width=0.99\linewidth]{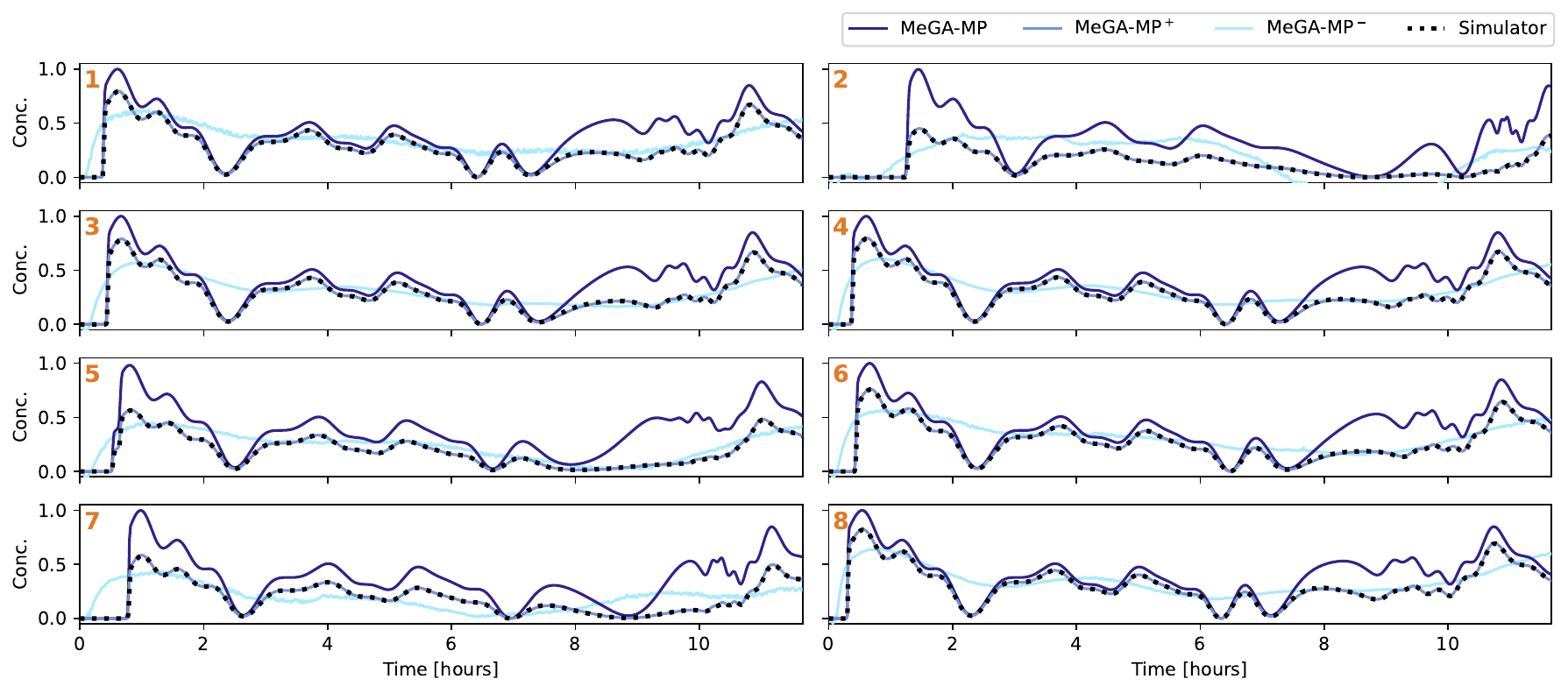}
    \caption{
    Predictions of \gls*{MeGA-MP}, \gls*{MeGA-MP}$^+$ and \gls*{MeGA-MP}$^-$
    for the \textbf{advection-reaction} dynamical system 
    on the \gls*{WDS} \textbf{Hanoi}. 
    The selected nodes are the same as in Figure \ref{figure_ComparisonBaselines_Advection_Hanoi_with_map}.}
\label{figure_AblationStudy_AdvectionReaction_Hanoi}
\end{figure}

\begin{figure}[!h]
    \centering
    \includegraphics[width=0.99\linewidth]{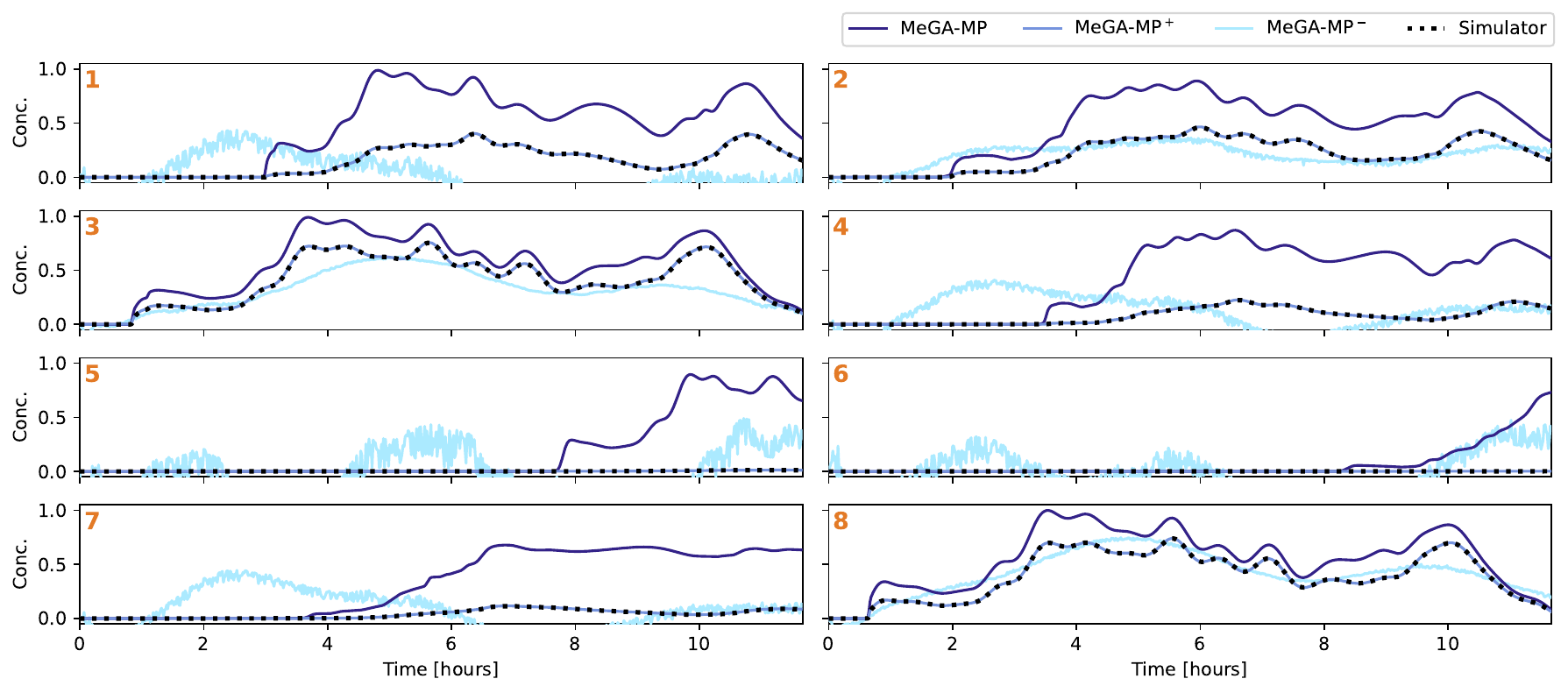}
    \caption{
    Predictions of \gls*{MeGA-MP}, \gls*{MeGA-MP}$^+$ and \gls*{MeGA-MP}$^-$
    for the \textbf{advection-reaction} dynamical system 
    on the \gls*{WDS} \textbf{L-Town}. 
    The selected nodes are the same as in Figure \ref{figure_ComparisonBaselines_Advection_L-Town_with_map}.}
    \label{figure_AblationStudy_AdvectionReaction_L-Town}
\end{figure}

\newpage
\subsection{Experiment 2: Advection on a 1D Euclidean Domain}
\label{subsection_Experiments_AdvectionOnA1DEuclideanDomain_Appendix}

\paragraph{Domain and Dataset}
The boundary value problem which we analyse in Section  \ref{subsection_Experiments_AdvectionOnA1DEuclideanDomain} is given by 
\begin{align}
\label{align_SyntheticBoundaryValueProblem}
\begin{split}
    \frac{\partial c(t,z)}{\partial t}
    &=
    - \nu(t) ~
    \frac{\partial c(t,z)}{\partial z}
    \quad
    \text{with}
    \quad 
    \nu(t)
    =
    0.3 \cdot \sin \left( \frac{2 \pi t}{100} \right) + 0.3
    \text{ and}
    \\
    c(t,0)
    &=
    \exp \left(\frac{1}{2} \left(\frac{t-\mu_t}{\sigma_t}\right)^2 \right)
\end{split}
\end{align}

\noindent for all $t \in [0,100]$ and $z \in [0,100]$ and with $\mu_t = 38$ and $\sigma_t = 18$.
Its analytical solution is computed analogously to the solution of the initial value problem with boundary condition, as discussed next.

\noindent
Additionally, in this Appendix, we analyze 
the initial value problem with boundary condition given by 
\begin{align}
\label{align_SyntheticInitialValueProblemWithBoundaryCondition}
\begin{split}
    \frac{\partial c(t,z)}{\partial t}
    &=
    - \nu(t) ~
    \frac{\partial c(t,z)}{\partial z}
    \quad
    \text{with}
    \quad 
    \nu(t)
    =
    0.3 \cdot \sin \left( \frac{2 \pi t}{100} \right) + 0.3
    \text{ and}
    \\
    c(0,z)
    &=
    \exp \left(\frac{1}{2} \left(\frac{z-\mu_z}{\sigma_z}\right)^2 \right),
    \quad
    c(t,0)
    =
    \exp \left(\frac{1}{2} \left(\frac{t-\mu_t}{\sigma_t}\right)^2 \right)
\end{split}
\end{align}

\noindent for all $t \in [0,100]$ and $z \in [0,100]$ and with $\mu_z = 8$, $\sigma_z = 2.5$, $\mu_t = 38$ and $\sigma_t = 18$.

To compute an analytical solution, we can solve the initial value problem $c_{ivp}(t, z)$ and the boundary value problem $c_{bvp}(t, z)$ separately and sum the individual results. This decomposition is possible since for a positive and incompressible flow $\nu(t) > 0$, characteristic lines do not cross and each point of the solution corresponds to a value of either the initial value function or the boundary value function.

The initial value problem is solved by integrating the velocity field to compute the spatial shift of the Gaussian $c(0, z)$ for each $t \in [0,100]$, via $t \longmapsto \zcurl_0(t) := \int_0^t \nu(s) ~ds$. The antiderivative of $\nu(t)$ is given by 

\begin{align}
\label{align_euclidean_flow_field}
\bar{\nu}(t) = 0.3 \left(t - \frac{\cos(4 \pi t)}{4 \pi} + \frac{1}{4 \pi}\right),
\end{align}

\noindent and the solution of the initial value problem is $c_{ivp}(t, z) = c(0, z + \zcurl_0(t))$.

To solve the boundary value problem, we find the transport time $\delta t(z)$ from the boundary at $z = 0$ to every other point $z$ in the domain. We leverage the same procedure as for our \gls*{MeGA-MP} (cf. Eq. \ref{align_TransportTime_OP}) and compute $\delta t(z)$ as the unique solution of $\int_0^{\delta t(z)} \nu(s) ds = z$. Since $\bar{\nu}(t)$ is strictly increasing the solution is unique and well-defined.
According to Theorem \ref{theorem_PassThroughTime}, the coordinate $(t + \delta t(z), z)$ lies on the characteristic curve that intersects the boundary at $(t, 0)$. We can shift the boundary value Gaussian accordingly and obtain $c_{bvp}(t, z) = c(t + \delta t(z), 0)$.
The solution to the initial value problem with boundary condition is then 
$$
c(t, z) = c_{ivp}(t, z) + c_{bvp}(t, z)
$$
In practice, we compute the respective values numerically at each discrete spatial and temporal grid point.

Next to the analytical solution, the numerical and analytical solvers we compare ourselves to are:

\begin{itemize}
    \item \textbf{Semi-Lagrangian scheme}: We use a custom semi-Lagrangian solver. For every grid point $z \in \Omega$ and iteration step $n+1$, the method computes the departure point $z'$ of a particle arriving at $z$ by tracing backward along the velocity field $\nu(t, z)$. Since the departure point $z'$ is a continuous coordinate, we interpolate spatially to compute the new solution $c(t_{n+1}, z)$. If we assume a spatially constant and temporally piecewise constant flow field, the departure point is simply $z' = z - dt \cdot \nu(t_n, z)$. Otherwise, it is approximated by integrating the characteristic ODE $\frac{dz}{dt} = -\nu(t, z)$ via the RK4 method. The result is then $c(t_{n+1}, z) = \mathcal{I}[c(t_{n}, \cdot)](z')$, where $\mathcal{I}$ is the spatial interpolation operator like the linear interpolation in Eq. \eqref{align_LinearInterpolation}.
    
    \item \textbf{Runge-Kutta-4} is the classical order-four Runge-Kutta (RK4) numerical method for solving \glspl*{PDE} (and \glspl*{ODE}) of the kind
    \begin{align*}
        \frac{\partial c(t,z)}{\partial t}
        = 
        f(t,z,c)
    \end{align*}

    \noindent 
    with the update rule 
    \begin{align*}
        c(t_{n+1},z) 
        =
        c(t_n,z)
        + 
        \frac{dt}{6} \left(k_1 + 2k_2 + 2k_3 + k_4\right)
    \end{align*}        

    \noindent 
    where $n \in \N$ is the iteration step, 
    $dt$ is the step-size and
    $k_1 = f(t_n,z,c)$, 
    $k_2 = f(t_n + \frac{dt}{2}, c(t_n,z) + \frac{dt}{2} k_1)$, 
    $k_3 = f(t_n + \frac{dt}{2}, c(t_n,z) + \frac{dt}{2} k_2)$ and 
    $k_4 = f(t_n + dt, c(t_n,z) + dt k_3)$. We apply the RK4 within a method-of-lines setup. The right-hand side function $f(t,z,c)$ approximates the negative spatial derivative of the flux via a fifth-order Weighted Essentially Non-Oscillatory (WENO5) scheme \cite{WENO5_1996} using Lax-Friedrichs flux-splitting for numerical stability. Specifically, left- and right-moving flux components are reconstructed separately at cell interfaces, resulting in the spatial derivative approximation at cell grid-points $i$:
    $
        f(t,z_i,c) \approx -\frac{\hat{F}_{i+1/2} - \hat{F}_{i-1/2}}{dx}
    $
    where $\hat{F}_{i \pm 1/2}$ represent the combined numerical interface fluxes.
\end{itemize}

\paragraph{Baselines and Results}
In Figure \ref{figure_SyntheticBoundaryValueProblem_AnalyticSolutionVsMeGAMP}, the analytical solution of the boundary value problem, which serves as ground truth, and the \gls*{MeGA-MP} solution are visualized at different discretizations and on the whole resulting grid.
Moreover,
in Figure \ref{figure_SyntheticInitialValueProblemWithBoundaryCondition_AnalyticSolutionVsMeGAMP}, the analytical solution of the initial value problem with boundary values, which serves as ground truth, and the \gls*{MeGA-MP} solution are visualized at different discretizations and on the whole resulting grid.
Additionally, in Figure \ref{figure_SyntheticBoundaryValueProblem}, we visualize the solutions of the initial value problem with boundary conditions \eqref{align_SyntheticInitialValueProblemWithBoundaryCondition} for different solvers and different temporal and spatial discretizations at the last step $t=100$.

As an extension of Table \ref{table_SyntheticBoundaryValueProblem}, 
in Table \ref{table_SyntheticBoundaryValueProblem_extended}, we compare our method to the numerical solvers at different discretizations.
We also report the runtime of each method in Table \ref{table_SyntheticBoundaryValueProblemComplexity}. 
It can be observed that a larger number of edges (smaller $dx$) results in longer runtime, which is expected. 
Note that these values are only meant to convey a rough idea of time complexity since there may be several ways to improve the performance of the baselines. We ran all methods on the CPU.

\begin{figure}[!h]
    \centering
    \begin{subfigure}{0.49\linewidth}
        \centering
        \includegraphics[width=\linewidth]{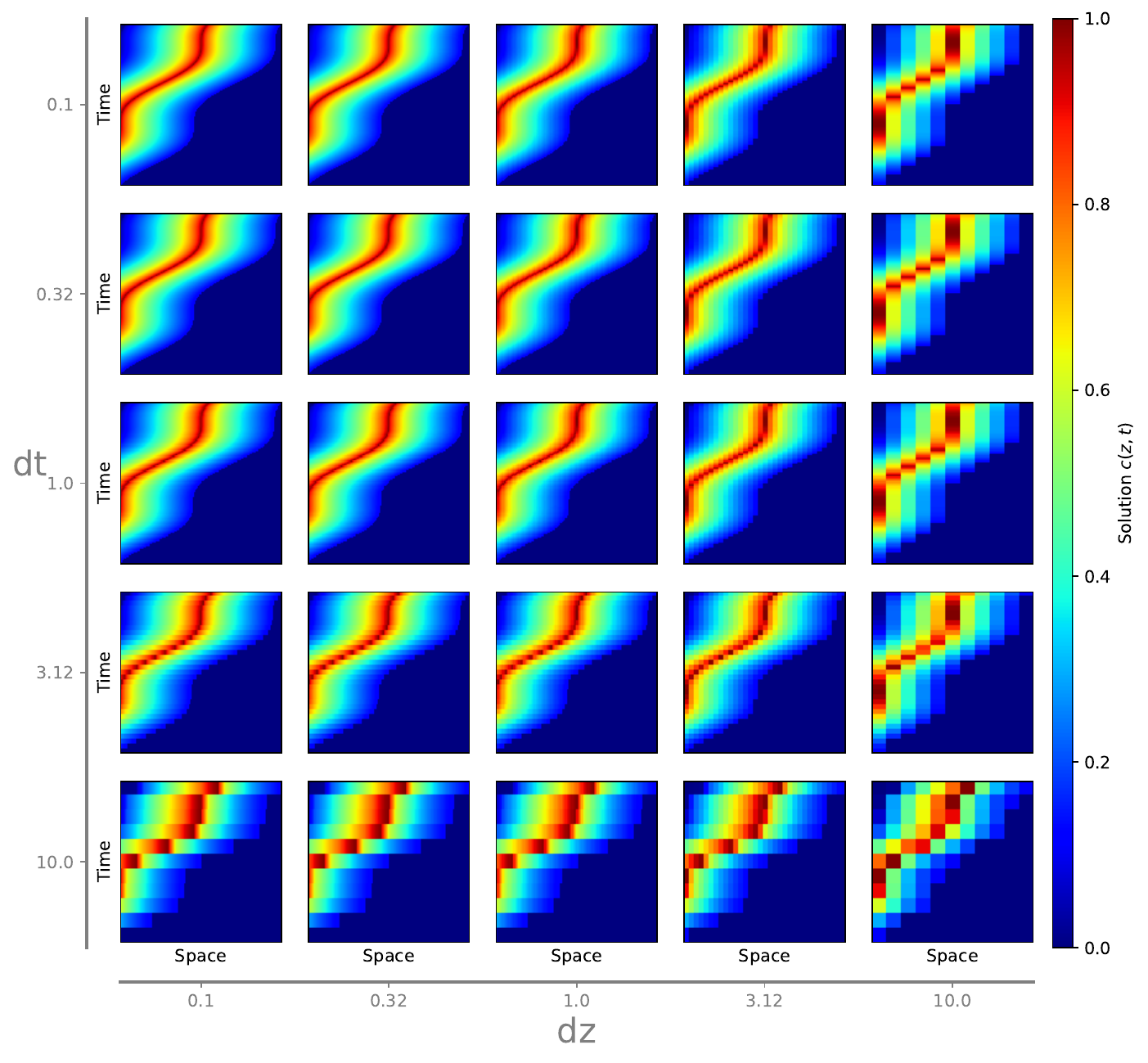}
        \caption{Analytical solution.}
        \label{figure_SyntheticBoundaryValueProblem_AnalyticalSolution}
    \end{subfigure}
    \hfill
    \begin{subfigure}{0.49\linewidth}
        \centering
        \includegraphics[width=\linewidth]{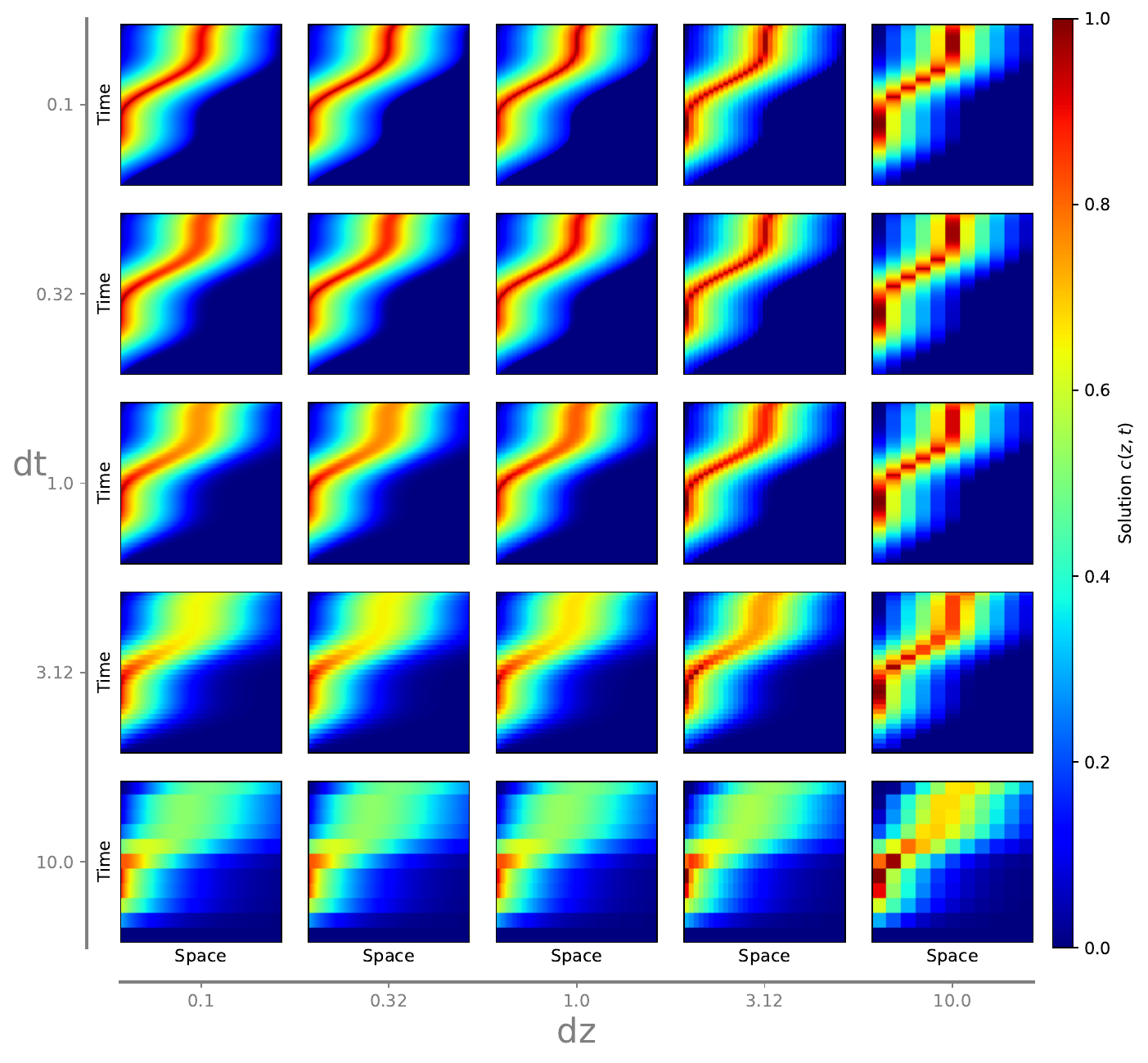}
        \caption{\gls*{MeGA-MP} solution.}
        \label{figure_SyntheticBoundaryValueProblem_MeGAMPSolution}
    \end{subfigure}
    \caption{Solutions of the \textbf{boundary value problem} \eqref{align_SyntheticBoundaryValueProblem} for different temporal and spatial discretizations and on the whole resulting grid }
    \label{figure_SyntheticBoundaryValueProblem_AnalyticSolutionVsMeGAMP}
\end{figure}

\begin{figure}[!h]
    \centering
    \begin{subfigure}{0.49\linewidth}
        \centering
        \includegraphics[width=\linewidth]{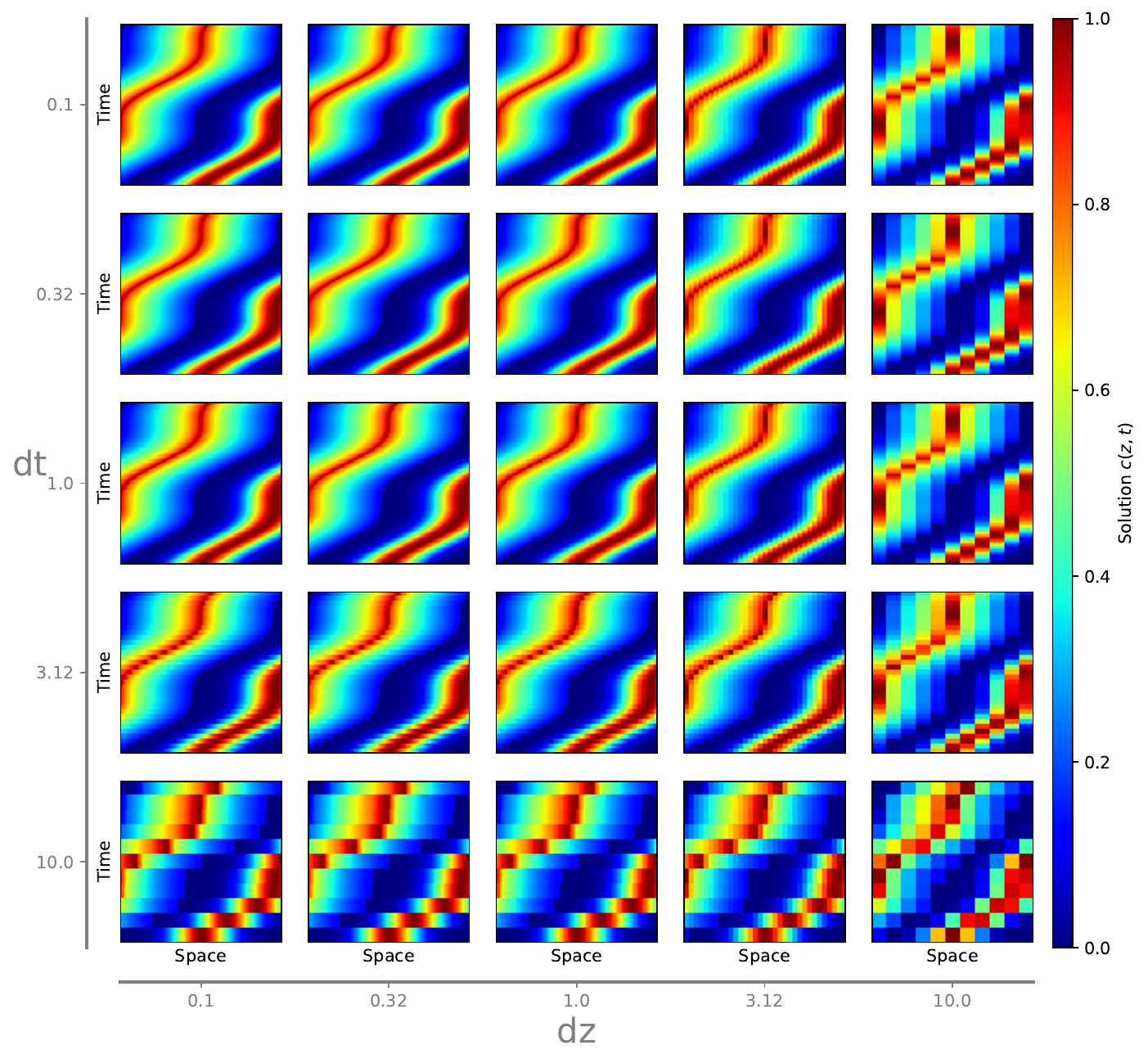}
        \caption{Analytical solution.}
        \label{figure_SyntheticInitialValueProblemWithBoundaryCondition_AnalyticalSolution}
    \end{subfigure}
    \hfill
    \begin{subfigure}{0.49\linewidth}
        \centering
        \includegraphics[width=\linewidth]{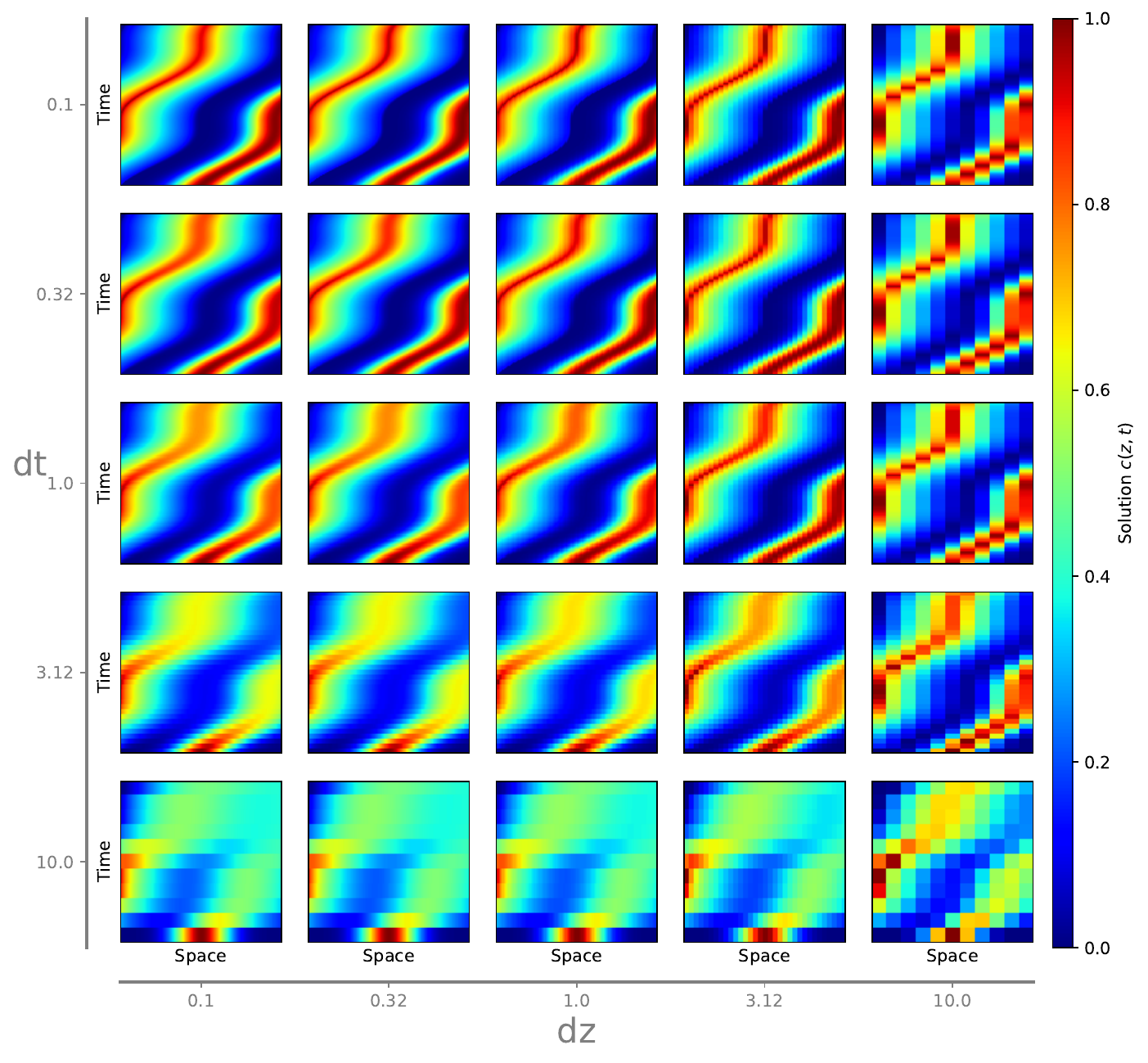}
        \caption{\gls*{MeGA-MP} solution.}
        \label{figure_SyntheticInitialValueProblemWithBoundaryCondition_MeGAMPSolution}
    \end{subfigure}
    \caption{Solutions of \textbf{initial value problem} with boundary condition 
    \eqref{align_SyntheticInitialValueProblemWithBoundaryCondition} for different temporal and spatial discretizations and on the whole resulting grid.}
    \label{figure_SyntheticInitialValueProblemWithBoundaryCondition_AnalyticSolutionVsMeGAMP}
\end{figure}

\begin{figure}[!h]
    \centering
    \includegraphics[width=\linewidth]{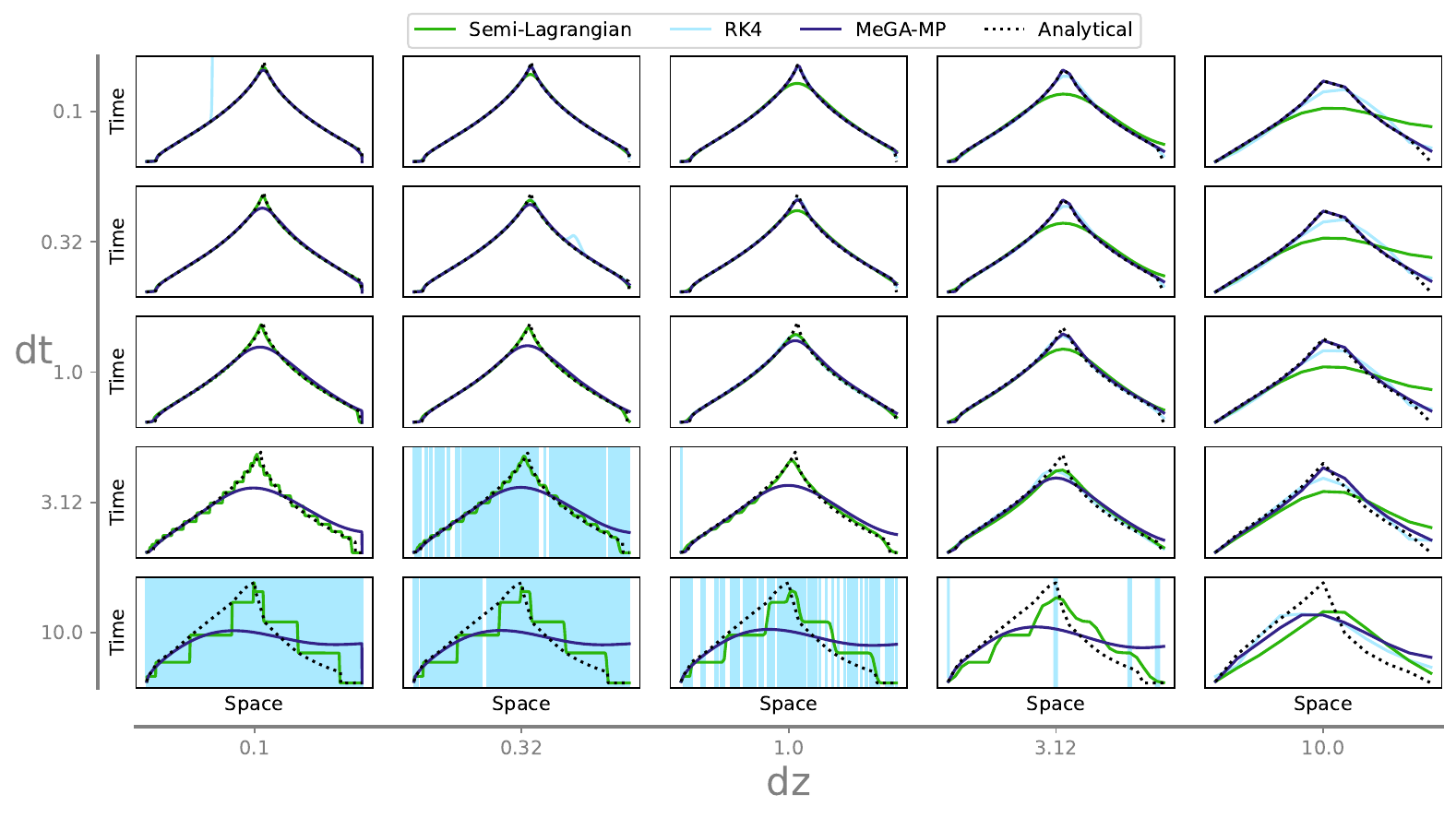}
    \caption{Solutions of the initial value problem with boundary condition \eqref{align_SyntheticBoundaryValueProblem} for different solvers and different temporal and spatial discretizations at the last step $t=100$.}
    \label{figure_SyntheticBoundaryValueProblem}
\end{figure}

\begin{table}[!h]
    \centering
    \caption{\acrshort*{MAE} between analytical ground truth and different solver solutions over all temporal and spatial grid points. 
    The entry \enquote{--} indicates divergence of the method.}
    \label{table_SyntheticBoundaryValueProblem_extended}
    \begin{tabular}{cr|cccc} \toprule
        $dt$ & $dx$ & semi- Lagr. & RK4 & \gls{MeGA-MP}
        \\
\hline \multirow[t]{5}{*}{$0.1$} & \hfill$ 0.10$ & \textbf{0.0016}{\scriptsize$\pm 0.0050$} & -- & 0.0024{\scriptsize$\pm 0.0081$} \\ 
  & \hfill 0.32 & 0.0046{\scriptsize$\pm 0.0114$} & \textbf{0.0013}{\scriptsize$\pm 0.0043$} & \textbf{0.0013}{\scriptsize$\pm 0.0058$} \\ 
  & \hfill 1.00 & 0.0108{\scriptsize$\pm 0.0230$} & 0.0035{\scriptsize$\pm 0.0073$} & \textbf{0.0012}{\scriptsize$\pm 0.0068$} \\ 
  & \hfill 3.12 & 0.0239{\scriptsize$\pm 0.0422$} & 0.0105{\scriptsize$\pm 0.0152$} & \textbf{0.0023}{\scriptsize$\pm 0.0113$} \\ 
  & \hfill10.00 & 0.0562{\scriptsize$\pm 0.0764$} & 0.0288{\scriptsize$\pm 0.0415$} & \textbf{0.0057}{\scriptsize$\pm 0.0182$} \\ 
\hline \multirow[t]{5}{*}{$0.32$} & \hfill$ 0.10$ & 0.0027{\scriptsize$\pm 0.0060$} & -- & 0.0071{\scriptsize$\pm 0.0162$} \\ 
  & \hfill 0.32 & \textbf{0.0038}{\scriptsize$\pm 0.0075$} & 0.0108{\scriptsize$\pm 0.0325$} & 0.0049{\scriptsize$\pm 0.0123$} \\ 
  & \hfill 1.00 & 0.0095{\scriptsize$\pm 0.0184$} & 0.0042{\scriptsize$\pm 0.0082$} & \textbf{0.0031}{\scriptsize$\pm 0.0092$} \\ 
  & \hfill 3.12 & 0.0227{\scriptsize$\pm 0.0400$} & 0.0117{\scriptsize$\pm 0.0160$} & \textbf{0.0036}{\scriptsize$\pm 0.0126$} \\ 
  & \hfill10.00 & 0.0557{\scriptsize$\pm 0.0754$} & 0.0304{\scriptsize$\pm 0.0418$} & \textbf{0.0071}{\scriptsize$\pm 0.0189$} \\ 
\hline \multirow[t]{5}{*}{$1.0$} & \hfill$ 0.10$ & 0.0080{\scriptsize$\pm 0.0144$} & -- & 0.0157{\scriptsize$\pm 0.0298$} \\ 
  & \hfill 0.32 & 0.0080{\scriptsize$\pm 0.0129$} & -- & 0.0143{\scriptsize$\pm 0.0277$} \\ 
  & \hfill 1.00 & 0.0097{\scriptsize$\pm 0.0137$} & \textbf{0.0088}{\scriptsize$\pm 0.0155$} & 0.0103{\scriptsize$\pm 0.0208$} \\ 
  & \hfill 3.12 & 0.0206{\scriptsize$\pm 0.0324$} & 0.0100{\scriptsize$\pm 0.0155$} & \textbf{0.0081}{\scriptsize$\pm 0.0165$} \\ 
  & \hfill10.00 & 0.0545{\scriptsize$\pm 0.0720$} & 0.0291{\scriptsize$\pm 0.0395$} & \textbf{0.0115}{\scriptsize$\pm 0.0222$} \\ 
\hline \multirow[t]{5}{*}{$3.12$} & \hfill$ 0.10$ & 0.0250{\scriptsize$\pm 0.0364$} & -- & 0.0384{\scriptsize$\pm 0.0535$} \\ 
  & \hfill 0.32 & \textbf{0.0243}{\scriptsize$\pm 0.0350$} & -- & 0.0375{\scriptsize$\pm 0.0524$} \\ 
  & \hfill 1.00 & \textbf{0.0231}{\scriptsize$\pm 0.0315$} & -- & 0.0345{\scriptsize$\pm 0.0491$} \\ 
  & \hfill 3.12 & 0.0254{\scriptsize$\pm 0.0321$} & \textbf{0.0231}{\scriptsize$\pm 0.0333$} & 0.0242{\scriptsize$\pm 0.0370$} \\ 
  & \hfill10.00 & 0.0504{\scriptsize$\pm 0.0639$} & 0.0299{\scriptsize$\pm 0.0406$} & \textbf{0.0238}{\scriptsize$\pm 0.0314$} \\ 
\hline \multirow[t]{5}{*}{$10.0$} & \hfill$ 0.10$ & \textbf{0.0763}{\scriptsize$\pm 0.0964$} & -- & 0.0939{\scriptsize$\pm 0.0981$} \\ 
  & \hfill 0.32 & \textbf{0.0752}{\scriptsize$\pm 0.0954$} & -- & 0.0933{\scriptsize$\pm 0.0977$} \\ 
  & \hfill 1.00 & \textbf{0.0728}{\scriptsize$\pm 0.0923$} & -- & 0.0915{\scriptsize$\pm 0.0952$} \\ 
  & \hfill 3.12 & \textbf{0.0675}{\scriptsize$\pm 0.0852$} & -- & 0.0843{\scriptsize$\pm 0.0904$} \\ 
  & \hfill10.00 & \textbf{0.0671}{\scriptsize$\pm 0.0842$} & 0.0728{\scriptsize$\pm 0.0936$} & 0.0761{\scriptsize$\pm 0.0763$} \\ 
        \hline
    \end{tabular}
\end{table}

\begin{table}[!h]
    \centering
    \caption{Time in seconds that it takes each method to solve the \gls{BVP}. The time it takes for the preprocessing to convert flows into transport times (s. Algorithm \ref{algorithm_FlowToTransportTime}) is included in these values. 
    The entry \enquote{--} indicates divergence of the method.}
    \label{table_SyntheticBoundaryValueProblemComplexity}
    \begin{tabular}{cr|ccc|c} \toprule
        $dt$ & $dx$ & semi- Lagr. & RK4 & \gls{MeGA-MP} & Analytical
        \\
       \hline \multirow[t]{5}{*}{$0.1$} & \hfill$ 0.10$ & \textbf{58.1295} & -- & 63.1235 & 0.2867 \\ 
  & \hfill 0.32 & 16.9899 & \textbf{1.0589} & \textbf{5.9495} & 0.0792 \\ 
  & \hfill 1.00 & 5.5392 & \textbf{0.8999} & \textbf{1.3102} & 0.0416 \\ 
  & \hfill 3.12 & 1.7895 & 0.8881 & \textbf{0.5884} & 0.0097 \\ 
  & \hfill10.00 & 0.6871 & 0.8863 & \textbf{0.1444} & 0.0039 \\ 
\hline \multirow[t]{5}{*}{$0.32$} & \hfill$ 0.10$ & 18.0842 & -- & \textbf{13.9027} & 0.0398 \\ 
  & \hfill 0.32 & 5.4553 & \textbf{0.3194} & 1.7582 & 0.0130 \\ 
  & \hfill 1.00 & 1.6860 & \textbf{0.2795} & 0.5185 & 0.0049 \\ 
  & \hfill 3.12 & 0.6071 & 0.2794 & \textbf{0.2051} & 0.0019 \\ 
  & \hfill10.00 & \textbf{0.2222} & 0.2766 & 0.3668 & 0.0013 \\ 
\hline \multirow[t]{5}{*}{$1.0$} & \hfill$ 0.10$ & 6.0633 & -- & \textbf{6.0126} & 0.0109 \\ 
  & \hfill 0.32 & 1.7784 & -- & \textbf{1.1763} & 0.0029 \\ 
  & \hfill 1.00 & 0.5320 & \textbf{0.0893} & 0.2651 & 0.0012 \\ 
  & \hfill 3.12 & 0.1810 & 0.0883 & \textbf{0.0628} & 0.0007 \\ 
  & \hfill10.00 & 0.0699 & 0.0871 & \textbf{0.0218} & 0.0005 \\ 
\hline \multirow[t]{5}{*}{$3.12$} & \hfill$ 0.10$ & \textbf{1.7991} & -- & 3.6538 & 0.0024 \\ 
  & \hfill 0.32 & \textbf{0.5470} & -- & 0.7568 & 0.0012 \\ 
  & \hfill 1.00 & 0.1691 & -- & \textbf{0.1539} & 0.0007 \\ 
  & \hfill 3.12 & 0.0582 & \textbf{0.0283} & 0.0470 & 0.0006 \\ 
  & \hfill10.00 & 0.0223 & 0.0281 & \textbf{0.0191} & 0.0004 \\ 
\hline \multirow[t]{5}{*}{$10.0$} & \hfill$ 0.10$ & \textbf{0.5620} & -- & 2.3683 & 0.0012 \\ 
  & \hfill 0.32 & \textbf{0.1711} & -- & 0.4693 & 0.0008 \\ 
  & \hfill 1.00 & \textbf{0.0570} & -- & 0.1159 & 0.0005 \\ 
  & \hfill 3.12 & \textbf{0.0189} & -- & 0.0400 & 0.0005 \\ 
  & \hfill10.00 & \textbf{0.0075} & 0.0088 & 0.0172 & 0.0005 \\ 
        \hline
    \end{tabular}
\end{table}

\newpage
\section{Proofs}
\label{section_Proofs}

\subsection{Proof of Lemma \ref{lemma_ConstantConcentrationAlongCharacteristicCurves}}

\begin{lemmaNoNb}[Lemma \ref{lemma_ConstantConcentrationAlongCharacteristicCurves}]
    Let $e_{uv} \in E$ hold.
    If the function $c_{uv}$ obeys Equation \eqref{align_EdgeDynamics_Advection}, 
    for each $t_0 \in \R$ and $z_0 \in \R$, the curve
    %
    %
    \begin{align*}
        \gamma_{t_0,z_0}: ~ T_{t_0,z_0} \longrightarrow \R_{\geq 0}, ~~
        t \longmapsto 
        c_{uv} \Big( t, z_0 + \textstyle \int_{t_0}^{t} \nu_{uv}(s) ~ds \Big)
    \end{align*}

    \noindent 
    is constant on
    $
    T_{t_0,z_0}
    =
    \Big \{
    t \in \R 
    ~|~
    0 
    <
    z_0 + \int_{t_0}^{t^{'}} \nu_{uv}(s) ~ds
    <
    l_{uv}
    ~\forall t^{'} \in 
    \big[ \min\{t_0,t\} , \max\{t_0,t\} \big]
    \Big \}.
    $
    %
\end{lemmaNoNb}

\begin{proof}
    Let $c_{uv}: \R \times \Omega_{uv} \rightarrow \R_{\geq 0}$ obey Equation \eqref{align_EdgeDynamics_Advection} (note that Equation \eqref{align_EdgeDynamics_Advection} is only defined for $z \in \Omega_{uv} = (0,l_{uv})$).
    We can consider its space coordinate $z$ as a function $z: \R \rightarrow \Omega_{uv}$ parametrized by time. 
    Consequently, we can use the method of characteristics (cf. \citet{Habermann2013PartialDifferentialEquations}) to derive the following system of \glspl*{ODE} from Equation \eqref{align_EdgeDynamics_Advection} for the composed, now only time-dependent function $c_{uv}: \R \rightarrow \R_{\geq 0}, c_{uv}(t) := c_{uv}(t, z(t))$:
    \begin{align*}
        \frac{dc_{uv}(t)}{dt} = 0, \quad
        \frac{dz(t)}{dt} = \nu_{uv}(t), \quad
        \frac{dt}{dt} = 1.
    \end{align*}
    
    \noindent By the fundamental theorem of calculus, we observe that
    for each $t_0 \in \R$ and $z_0 \in \R$, the function
    \begin{align*}
        z: \R \longrightarrow \R, ~
        t \longmapsto z(t) := z_0 + \int_{t_0}^{t} \nu_{uv}(s) ~ds
    \end{align*}

    \noindent satisfies the second \gls*{ODE}.
    Therefore, by the first \gls*{ODE},  
    we can conclude that the function $c_{uv}$ is constant along characteristic curves 
    $(t, z(t)) = (t, z_0 + \int_{t_0}^{t} \nu_{uv}(s) ~ds)$ 
    along the time $\R$ 
    as well as
    along the space interval $\Omega_{uv}$, i.e., \textit{within} the edge $e_{uv}$ of length $l_{uv} > 0$.
    Therefore, taking into account the domain and co-domain of $c_{uv}$, $\nu_{uv}$ and $z$ as a function, 
    $t_0, t \in \R$
    and
    $z(t^{'}) \in \Omega_{uv} = (0,l_{uv}$)
    need to hold for all $t^{'} \in [\min\{t_0,t\} , \max\{t_0,t\}]$.\footnote{
        Especially to mention, the requirement $z_0 + \int_{t_0}^{t} \nu_{uv}(s) ~ds \in \Omega_{uv}$ does not suffices, cf. Remark \ref{remark_OnT_{t_0,z_0}}.
    }
    
    In summary, 
    the curve 
    $\gamma_{z_0}: T_{t_0,z_0} \rightarrow \R_{\geq 0}, t \mapsto c_{uv}(t, z_0 + \int_{t_0}^{t} \nu_{uv}(s) ~ds)$ 
    is constant on
    \begin{align*}
        T_{t_0,z_0}
        =
        \Big \{
        t \in \R 
        ~|~
        0 
        <
        z_0 + \int_{t_0}^{t^{'}} \nu_{uv}(s) ~ds
        < 
        l_{uv}
        ~\forall t^{'} \in 
        \big[ \min\{t_0,t\} , \max\{t_0,t\} \big]
        \Big \}.
    \end{align*}
\end{proof}

\subsection{Proof of Lemma \ref{lemma_OnT_{t_0,z_0}}}

\begin{lemmaNoNb}[Lemma \ref{lemma_OnT_{t_0,z_0}}]
    Let $e_{uv} \in E$, $t_0 \in \R$, $z_0 \in \R$ hold and let $T_{t_0,z_0}$ be defined as in Lemma \ref{lemma_ConstantConcentrationAlongCharacteristicCurves}.

    \begin{enumerate}
        \item If $z_0 \not \in \Omega_{uv}$ holds, then $T_{t_0,z_0} = \emptyset$ holds.

        \item If $z_0 \in \Omega_{uv}$ holds, then $T_{t_0,z_0} \neq \emptyset$ holds. Even more, $t_0 \in T_{t_0,z_0}$ holds.


        
        \item If $z_0 \in \Omega_{uv}$ holds, 
        there exist $\tn{l} \in [-\infty,t_0)$ and $\tn{r} \in (t_0, \infty]$ 
        such that 

        \begin{enumerate}
            \item if $\tn{l} \neq - \infty$ holds, $z_0 + \int_{t_0}^{\tn{l}} \nu_{uv}(s) ~ds \in \{0,l_{uv}\}$ holds,
            
            \item if $\tn{r} \neq + \infty$ holds, $z_0 + \int_{t_0}^{\tn{r}} \nu_{uv}(s) ~ds \in \{0,l_{uv}\}$ holds and

            \item $T_{t_0,z_0} = (\tn{l},\tn{r}) = (t_0 - (t_0 - \tn{l}),t_0 + (\tn{r} - t_0))$ holds.
        \end{enumerate}
        
        \noindent Equivalently, there exist $\delta \tn{l}, \delta \tn{r} \in (0,\infty]$
        such that $T_{t_0,z_0} = (t_0 - \delta \tn{l},t_0 + \delta \tn{r})$ holds.
    \end{enumerate}
\end{lemmaNoNb}

\begin{proof}
    Similarly to Remark \ref{remark_OnT_{t_0,z_0}}, we define the function
    \begin{align*}
        z: \R \longrightarrow \R, ~
        t \longmapsto z(t) := z_0 + \int_{t_0}^{t} \nu_{uv}(s) ~ds.
    \end{align*}

    \noindent By the fundamental theorem of calculus, the function $z$ is differentiable and therefore, continuous. 
    Moreover, using the function $z$, we can re-write the set $T_{t_0,z_0}$ as
    \begin{align*}
        T_{t_0,z_0}
        =
        \Big \{
        t \in \R 
        ~|~
        z(t^{'})
        \in
        \Omega_{uv} = (0,l_{uv})
        ~\forall t^{'} \in 
        \big[ \min\{t_0,t\} , \max\{t_0,t\} \big]
        \Big \}.
    \end{align*}
    
    \noindent 
    \textit{1.:}
    If $z_0 \not \in \Omega_{uv}$ holds, 
    let us assume that $T_{t_0,z_0} \neq \emptyset$ holds. 
    By definition of  $T_{t_0,z_0}$, 
    there exists a $t \in \R$, such that 
    $z(t^{'}) \in \Omega_{uv}$ holds for all 
    $t^{'} \in [\min\{t_0,t\} , \max\{t_0,t\}]$.
    However, for $t^{'} = t_0 \in [\min\{t_0,t\} , \max\{t_0,t\}]$, we obtain 
    $z(t^{'}) = z(t_0) = z_0 + \int_{t_0}^{t_0} \nu_{uv}(s) ~ds = z_0 \not \in \Omega_{uv}$ 
    -- a contradiction.
    \\
    \\\textit{2.:}
    If $z_0 \in \Omega_{uv}$ holds, 
    we obtain 
    $z(t_0) = z_0 + \int_{t_0}^{t_0} \nu_{uv}(s) ~ds = z_0 \in \Omega_{uv}$ 
    and therefore, trivially, $t_0 \in T_{t_0,z_0}$ and $T_{t_0,z_0} \neq \emptyset$ holds. 
    \\
    \\\textit{3.:}
    We define the sets
    \begin{align*}
        A_{\text{l}}
        :=&
        \{ t^{'} \in (-\infty,t_0] ~|~ z(t^{'}) \not \in \Omega_{uv} \}
        \text{ and }
        \\
        A_{\text{r}}
        :=&
        \{ t^{'} \in [t_0,+\infty) ~|~ z(t^{'}) \not \in \Omega_{uv} \}.
    \end{align*}
    
    \noindent Based on these sets, we choose
    \begin{align*}
        \tn{l} :=
        \begin{cases}
            \max A_{\text{l}} & \text{ if } A_{\text{l}} \neq \emptyset \\
            -\infty           & \text{ if } A_{\text{l}} = \emptyset
        \end{cases}
        \quad \text{ and } \quad
        \tn{r} :=
        \begin{cases}
            \min A_{\text{r}} & \text{ if } A_{\text{r}} \neq \emptyset \\
            +\infty           & \text{ if } A_{\text{r}} = \emptyset
        \end{cases}
        .
    \end{align*}

    \noindent If $z_0 \in \Omega_{uv}$ holds, 
    as seen above, $z(t_0) = z_0 \in \Omega_{uv}$ holds. Therefore, we observe that $t_0 \not \in A_{\text{l}}$ and  $t_0 \not \in A_{\text{r}}$, or equivalently, $A_{\text{l}} \subset (-\infty,t_0)$ and $A_{\text{r}} \subset (t_0,+\infty)$ hold. 
    Consequently, our choices of $\tn{l}$ and $\tn{r}$ indeed satisfy  $\tn{l} \in [-\infty,t_0)$ and $\tn{r} \in (t_0, \infty]$ and even more, $t_0 \in (\tn{l},\tn{r})$ holds.

    We now want to show that for these choices, indeed, \textit{5(a) - 5(c)} hold.
    \\
    \\\textit{3(a), 3(b):}
    \textit{
    If $\tn{l} \neq - \infty$ holds, $z_0 + \int_{t_0}^{\tn{l}} \nu_{uv}(s) ~ds \in \{0,l_{uv}\}$ holds and
    if $\tn{r} \neq + \infty$ holds, $z_0 + \int_{t_0}^{\tn{r}} \nu_{uv}(s) ~ds \in \{0,l_{uv}\}$ holds.
    }

    The claims are equivalent to show that 
    if $A_{\text{l}} \neq \emptyset$ holds, $z(\tn{l}) \in \{0,l_{uv}\}$ holds and
    if $A_{\text{r}} \neq \emptyset$ holds, $z(\tn{r}) \in \{0,l_{uv}\}$ holds.
    
    If $A_{\text{l}} \neq \emptyset$ holds, 
    let us assume that $z(\tn{l}) \not \in \{0,l_{uv}\}$ holds. 
    Since $\tn{l} = \max A_{\text{l}} \in A_{\text{l}}$, $z(\tn{l}) > l_{uv}$ or $z(\tn{l}) < 0$ must hold. 
    \gls*{Wlog}, let us assume that $z(\tn{l}) < 0$ holds.
    Since $z$ is continuous and $z(t_0) = z_0 \in \Omega_{uv}$, i.e., $0 < z(t_0)$ holds, 
    according to the intermediate value theorem, 
    there exists a $t^{'} \in (\tn{l},t_0)$ such that $z(t^{'}) = 0$ holds.
    Therefore, we can conclude that $\max A_{\text{l}} = \tn{l} < t^{'} \in A_{\text{l}}$ holds
    -- a contradiction to the choice of $\tn{l} = \max A_{\text{l}}$.

    If $A_{\text{r}} \neq \emptyset$ holds, 
    let us assume $z(\tn{r}) \not \in \{0,l_{uv}\}$ holds.
    Analogous arguments lead to a contradiction. 
    \\
    \\\textit{3(c), step 1:}
    \textit{$T_{t_0,z_0} \subset (\tn{l},\tn{r})$ holds.}

    Let $t \in T_{t_0,z_0} \subset \R$ and let us assume that $t \not \in (\tn{l},\tn{r})$, or equivalently, $t \leq \tn{l} < t_0$ or $t \geq \tn{r} > t_0$ holds.
    \\\textit{Case 1:}
    If $t \leq t_0$ and $A_{\text{l}} \neq \emptyset$ holds, we observe that $t \leq \tn{l} = \max A_{\text{l}} < t_0$ must hold.
    By definition of $T_{t_0,z_0}$,
    $z(t^{'}) \in \Omega_{uv}$ holds
    for all $t^{'} \in [\min\{t_0,t\} , \max\{t_0,t\}] = [t,t_0]$.
    However, by definition of $A_{\text{l}}$,
    for $t^{'} = \tn{l} \in [t,t_0) \subset  [t,t_0]$, 
    $z(t^{'}) = z(\tn{l}) \not \in \Omega_{uv}$ holds
    -- a contradiction to the first conclusion.
    \\\textit{Case 2:}
    If $t \leq t_0$ and $A_{\text{l}} = \emptyset$ holds, we observe that $t \leq \tn{l} = - \infty$ must hold -- a contradiction.
    \\\textit{Case 3:}
    If $t \geq t_0$ and $A_{\text{l}} \neq \emptyset$ holds, we observe that $t \geq \tn{r} = \min A_{\text{r}} > t_0$ must hold.
    By definition of $T_{t_0,z_0}$,
    $z(t^{'}) \in \Omega_{uv}$ holds
    for all $t^{'} \in [\min\{t_0,t\} , \max\{t_0,t\}] = [t_0,t]$.
    However, by definition of $A_{\text{r}}$,
    for $t^{'} = \tn{r} \in (t_0,t] \subset  [t,t_0]$, 
    $z(t^{'}) = z(\tn{r}) \not \in \Omega_{uv}$ holds
    -- a contradiction to the first conclusion.
    \\\textit{Case 4:}
    If $t \geq t_0$ and $A_{\text{l}} = \emptyset$ holds, we observe that $t \geq \tn{r} = + \infty$ must hold -- a contradiction.
    \\
    \\\textit{3(c), step 2:}
    \textit{$(\tn{l},\tn{r}) \subset T_{t_0,z_0}$ holds.}

    Let $t \in (\tn{l},\tn{r}) \subset \R$ and let us assume that $t \not \in  T_{t_0,z_0}$ holds.
    \\\textit{Case 1:}
    If $t \leq t_0$ and $A_{\text{l}} \neq \emptyset$ holds,
    by definition of $T_{t_0,z_0}$,
    there exists a $t^{'} \in [\min\{t_0,t\} , \max\{t_0,t\}] = [t,t_0] \subset (\tn{l},t_0] = (\max A_{\text{l}},t_0]$ such that
    $z(t^{'}) \not \in \Omega_{uv}$ holds.
    Therefore, we can conclude that $\max A_{\text{l}} = \tn{l} < t^{'} \in A_{\text{l}}$ holds
    -- a contradiction to the choice of $\tn{l} = \max A_{\text{l}}$.
    \\\textit{Case 2:}
    If $t \leq t_0$ and $A_{\text{l}} \neq \emptyset$ holds,
    by definition of $T_{t_0,z_0}$,
    there exists a $t^{'} \in [\min\{t_0,t\} , \max\{t_0,t\}] = [t,t_0] \subset (\tn{l},t_0] = (-\infty,t_0]$ such that
    $z(t^{'}) \not \in \Omega_{uv}$ holds.
    Therefore, we can conclude that $t^{'} \in A_{\text{l}}$ holds
    -- a contradiction to $A_{\text{l}} \neq \emptyset$.
    \\\textit{Case 3:}
    If $t \geq t_0$ and $A_{\text{r}} \neq \emptyset$ holds,
    by definition of $T_{t_0,z_0}$,
    there exists a $t^{'} \in [\min\{t_0,t\} , \max\{t_0,t\}] = [t_0,t] \subset [t_0,\tn{r}) = [t_0,\min A_{\text{r}})$ such that
    $z(t^{'}) \not \in \Omega_{uv}$ holds.
    Therefore, we can conclude that $\min A_{\text{r}} = \tn{r} > t^{'} \in A_{\text{r}}$ holds
    -- a contradiction to the choice of $\tn{r} = \min A_{\text{l}}$.
    \\\textit{Case 4:}
    If $t \geq t_0$ and $A_{\text{r}} \neq \emptyset$ holds,
    by definition of $T_{t_0,z_0}$,
    there exists a $t^{'} \in [\min\{t_0,t\} , \max\{t_0,t\}] = [t_0,t] \subset [t_0,\tn{r}) = [t_0,+\infty)$ such that
    $z(t^{'}) \not \in \Omega_{uv}$ holds.
    Therefore, we can conclude that $t^{'} \in A_{\text{r}}$ holds
    -- a contradiction to $A_{\text{r}} \neq \emptyset$.
\end{proof}

\subsection{Proof of Theorem \ref{theorem_AdvectiveTransportAlongEdges}}

\begin{theoremNoNb}[Theorem \ref{theorem_AdvectiveTransportAlongEdges}]
    Let $e_{uv} \in E$ hold.
    If the function $c_{uv}$ obeys Equation \eqref{align_EdgeDynamics_Advection}, for each $t_0 \in \R$ and $z_0 \in \Omega_{uv}$,
    \begin{align*}
        c_{uv} \left( t, z_0 + \int_{t_0}^{t} \nu_{uv}(s) ~ds \right) 
        = 
        c_{uv} \left( t, z_0 + \nuoverline_{uv}(t_0,t) ~ (t-t_0) \right)
        =
        c_{uv}(t_0,z_0)
    \end{align*}
    
    \noindent holds for all $t \in T_{t_0,z_0}$ with $T_{t_0,z_0}$ as defined in Lemma \ref{lemma_ConstantConcentrationAlongCharacteristicCurves} and for the mean velocity
    \begin{align*}
        \nuoverline_{uv}(t_0,t)
        :=
        \begin{cases}
            \fint_{t_0}^{t} \nu_{uv}(s) ~ds
            & \text{ if } t \neq t_0
            \\
            0
            & \text{ if } t = t_0
        \end{cases}
        =
        \begin{cases}
            \frac{1}{t - t_0}
            \int_{t_0}^{t} \nu_{uv}(s) ~ds
            & \text{ if } t \neq t_0
            \\
            0
            & \text{ if } t = t_0
        \end{cases}
        .
    \end{align*}
\end{theoremNoNb}

\begin{proof}
    For each $t_0 \in \R$ and $z_0 \in \Omega_{uv} \subset \R$, by Lemma \ref{lemma_ConstantConcentrationAlongCharacteristicCurves}, $\gamma_{t_0,z_0}(t) = c_{uv}(t, z_0 + \int_{t_0}^{t} \nu_{uv}(s) ~ds) = c$ holds for all $t \in T_{t_0,z_0}$ and a constant $c \in \R_{\geq 0}$. 
    Thus, it suffices to show that $c = c_{uv}(t_0,z_0)$ holds.

    Let us assume that $c \neq c_{uv}(t_0,z_0)$ holds. 
    Since $z_0 \in \Omega_{uv}$, by Lemma \ref{lemma_OnT_{t_0,z_0}}, $t_0 \in T_{t_0,z_0}$ holds. 
    However, for $t = t_0 \in T_{t_0,z_0}$, we obtain $c = \gamma_{t_0,z_0}(t) = c_{uv}(t, z_0 + \int_{t_0}^{t_0} \nu_{uv}(s) ~ds) = c_{uv}(t_0,z_0)$ -- a contradiction.
    Therefore, $c_{uv}(t, z_0 + \int_{t_0}^{t} \nu_{uv}(s) ~ds) = c_{uv}(t_0,z_0)$ holds.
    Finally, by definition of $\nuoverline_{uv}(t_0,t)$, for $t \neq t_0$, we obtain
    \begin{align*}
        \nuoverline_{uv}(t_0,t) \cdot (t-t_0)
        =
        \frac{t - t_0}{t - t_0}
        \int_{t_0}^{t} \nu_{uv}(s) ~ds
        =
        \int_{t_0}^{t} \nu_{uv}(s) ~ds,
    \end{align*}

    \noindent allowing to conclude that $c_{uv}(t, z_0 + \int_{t_0}^{t} \nu_{uv}(s) ~ds) = c_{uv}(t, z_0 + \nuoverline_{uv}(t_0,t) ~ (t-t_0))$ holds.
\end{proof}

\subsection{Proof of Lemma \ref{lemma_PassThroughTime}}

\begin{lemmaNoNb}[Lemma \ref{lemma_PassThroughTime}]
    Let $e_{uv} \in E$ and $t \in \R$ hold.
    If there exists a pass-through time $\delta t_{uv} \in \R$ \gls*{wrt} $e_{uv} \in E$ and $t \in \R$, the following properties hold:

    \begin{enumerate}
        \item $\delta t_{uv} \in \R_{\neq 0}$ and thus, $t - \delta t_{uv} \neq t$ holds.

        \item There exists no other pass-through time of the same sign 
        (i.e., there exists either exactly one or no positive pass-through time,
        and there exists exactly one or no negative pass-through time).

        \item The mean velocity 
        \begin{align*}
            \nuoverline_{uv}(t) 
            := 
            \nuoverline_{uv}(t - \delta t_{uv},t) 
            = 
            \fint_{t - \delta t_{uv}}^{t} \nu_{uv}(s) ~ds 
            = 
            \frac{1}{\delta t_{uv}} \int_{t - \delta t_{uv}}^{t} \nu_{uv}(s) ~ds
        \end{align*}
     
        during the time interval $[\min\{t - \delta t_{uv},t\},\max\{t - \delta t_{uv},t\}]$ 
        as defined in Theorem \ref{theorem_AdvectiveTransportAlongEdges}
        is equal to
        $\nuoverline_{uv}(t) = \frac{l_{uv}}{\delta t_{uv}}$
        and thus satisfies
        \begin{align*}
            \nuoverline_{uv}(t) > 0 &\iff \delta t_{uv} > 0 \text{ and} \\
            \nuoverline_{uv}(t) < 0 &\iff \delta t_{uv} < 0
            .
        \end{align*}

        \item The flows $q_{uv}(t - \delta t_{uv})$ and $q_{uv}(t)$ linked to the velocity $\nu_{uv}(\cdot) = \frac{q_{vu(\cdot)}}{\alpha_{vu}}$ satisfy
        \begin{align*}
            \begin{cases}
                q_{uv}(t - \delta t_{uv}), q_{uv}(t) \geq 0
                & \text{ if } \delta t_{uv} > 0
                \\
                q_{uv}(t - \delta t_{uv}), q_{uv}(t) \leq 0
                & \text{ if } \delta t_{uv} < 0
            \end{cases}
            .
        \end{align*}
    \end{enumerate}
\end{lemmaNoNb}

\begin{proof}
    \textit{1.:}
    Let us assume that $\delta t_{uv} = 0$ holds.
    Consequently,
    (1) by basic results from analysis,
    (2) condition \eqref{align_PassThroughTime_Condition1} and
    (3) definition of the pipe length $l_{uv} > 0$,
    \begin{align*}
        0
        \overset{\text{(1)}}{=}
        \int_{t - 0}^{t} \nu_{uv}(s) ~ds
        \overset{\text{(1)}}{=}
        \int_{t - \delta t_{uv}}^{t} \nu_{uv}(s) ~ds
        \overset{\text{(2)}}{=}
        l_{uv}
        \overset{\text{(3)}}{>}
        0
    \end{align*}

    \noindent holds -- a contradiction.
    \\
    \\\textit{2.:}
    Let us assume that there exist two different pass-through times $\delta t_{1uv}, \delta t_{2uv} \in \R_{\neq 0}$ with the same sign.
    If $\delta t_{1uv}, \delta t_{2uv} > 0$ holds, \gls*{wlog}, let $0 < \delta t_{1uv} < \delta t_{2uv}$ hold.
    Consequently, 
    for $t^{'} = t - \delta t_{1uv} \in (t - \delta t_{2uv},t) = (\min\{t - \delta t_{2uv},t\},\max\{t - \delta t_{2uv},t\})$,
    by
    (1) by basic results from analysis and
    (2) condition \eqref{align_PassThroughTime_Condition1},
    we obtain
    \begin{align*}
        \int_{t - \delta t_{2uv}}^{t^{'}} \nu_{uv}(s) ~ds
        \overset{\text{(1)}}{=}
        \int_{t - \delta t_{2uv}}^{t} \nu_{uv}(s) ~ds
        -
        \int_{t - \delta t_{1uv}}^{t} \nu_{uv}(s) ~ds
        \overset{\text{(2)}}{=}
        l_{uv}
        -
        l_{uv}
        \overset{\text{(1)}}{=}
        0
    \end{align*}

    \noindent -- a contradiction to condition \eqref{align_PassThroughTime_Condition2}.

    If $\delta t_{1uv}, \delta t_{2uv} < 0$ holds, \gls*{wlog}, let $\delta t_{2uv} < \delta t_{1uv} < 0$ hold.
    Consequently, 
    for $t^{'} = t - \delta t_{1uv} \in (t,t - \delta t_{2uv}) = (\min\{t - \delta t_{2uv},t\},\max\{t - \delta t_{2uv},t\})$,
    we can use the same calculation to again obtain a contradiction.
    \\
    \\\textit{3.:}
    By
    (1) definition of the mean velocity $\nuoverline_{uv}(t - \delta t_{uv},t)$ as defined in Theorem \ref{theorem_AdvectiveTransportAlongEdges} and
    (2) condition \eqref{align_PassThroughTime_Condition1},
    we indeed obtain
    \begin{align*}
        \nuoverline_{uv}(t) 
        := 
        \nuoverline_{uv}(t - \delta t_{uv},t) 
        \overset{\text{(1)}}{=}
        \fint_{t - \delta t_{uv}}^{t} \nu_{uv}(s) ~ds 
        \overset{\text{(1)}}{=} 
        \frac{1}{\delta t_{uv}} \int_{t - \delta t_{uv}}^{t} \nu_{uv}(s) ~ds
        \overset{\text{(2)}}{=}
        \frac{l_{uv}}{\delta t_{uv}}.
    \end{align*}

    \noindent The second claim then follows by the fact that by definition of the pipe length, $l_{uv} > 0$ holds.
    \\
    \\\textit{4.:}
    If $\delta t_{uv} > 0$ holds,
    let us assume that $q_{uv}(t - \delta t_{uv}) < 0$, and consequently, $\nu_{uv}(t - \delta t_{uv}) < 0$ holds.
    Since $\nu_{uv} \in C^0(\R,\R)$ is continuous, 
    there exists a $t^{'} \in (t - \delta t_{uv},t) = (\min\{t - \delta t_{uv},t\},\max\{t - \delta t_{uv},t\}) $ 
    such that $\nu_{uv}(s) < 0$ holds 
    for all $s \in (t - \delta t_{uv},t^{'})$.
    Consequently, by
    (1)  basic results from analysis,
    we obtain
    \begin{align*}
        \int_{t - \delta t_{uv}}^{t^{'}}
        \underbrace{
        \nu_{uv}(s) 
        }_{
        < 0
        }
        ~ds
        \overset{\text{(1)}}{<}
        0
    \end{align*}

    \noindent -- a contradiction to condition \eqref{align_PassThroughTime_Condition2}.

    Similarly, 
    if $\delta t_{uv} > 0$ holds,
    let us assume that $q_{uv}(t) < 0$, and consequently, $\nu_{uv}(t) < 0$ holds.
    Since $\nu_{uv} \in C^0(\R,\R)$ is continuous, 
    there exists a $t^{'} \in (t - \delta t_{uv},t) = (\min\{t - \delta t_{uv},t\},\max\{t - \delta t_{uv},t\})$ 
    such that $\nu_{uv}(s) < 0$ holds 
    for all $s \in (t^{'},t)$.
    Consequently, by
    (1) basic results from analysis and
    (2) condition \eqref{align_PassThroughTime_Condition1},
    we obtain
    \begin{align*}
        \int_{t - \delta t_{uv}}^{t^{'}}
        \nu_{uv}(s) ~ds
        \overset{\text{(1)}}{=}
        \underbrace{
        \int_{t - \delta t_{uv}}^{t}
        \nu_{uv}(s) ~ds
        }_{
        \overset{\text{(2)}}{=} l_{uv}
        }
        \underbrace{
        -
        \underbrace{
        \int_{t^{'}}^{t}
        \underbrace{
        \nu_{uv}(s)
        }_{
        < 0
        }
        ~ds
        }_{
        \overset{\text{(1)}}{<} 0
        }
        }_{
        > 0
        }
        >
        l_{uv}
    \end{align*}

    \noindent -- again a contradiction to condition \eqref{align_PassThroughTime_Condition2}.

    In contrast, 
    if $\delta t_{uv} < 0$ holds,
    let us assume that $q_{uv}(t - \delta t_{uv}) > 0$, and consequently, $\nu_{uv}(t - \delta t_{uv}) > 0$ holds.
    Since $\nu_{uv} \in C^0(\R,\R)$ is continuous, 
    there exists a $t^{'} \in (t,t - \delta t_{uv}) = (\min\{t - \delta t_{uv},t\},\max\{t - \delta t_{uv},t\})$ 
    such that $\nu_{uv}(s) > 0$ holds 
    for all $s \in (t^{'},t - \delta t_{uv})$.
    Consequently, by
    (1) basic results from analysis,
    we obtain
    \begin{align*}
        \int_{t - \delta t_{uv}}^{t^{'}}
        \nu_{uv}(s) ~ds
        \overset{\text{(1)}}{=}
        -
        \underbrace{
        \int_{t^{'}}^{t - \delta t_{uv}}
        \underbrace{
        \nu_{uv}(s)
        }_{
        > 0
        }
        ~ds
        }_{
        \overset{\text{(1)}}{>} 0
        }
        <
        0
    \end{align*}

    \noindent -- again a contradiction to condition \eqref{align_PassThroughTime_Condition2}.

    Similarly,
    if $\delta t_{uv} < 0$ holds,
    let us assume that $q_{uv}(t) > 0$, and consequently, $\nu_{uv}(t) > 0$ holds.
    Since $\nu_{uv} \in C^0(\R,\R)$ is continuous, 
    there exists a $t^{'} \in (t,t - \delta t_{uv}) = (\min\{t - \delta t_{uv},t\},\max\{t - \delta t_{uv},t\})$ 
    such that $\nu_{uv}(s) > 0$ holds 
    for all $s \in (t,t^{'})$.
    Consequently, by
    (1) basic results from analysis and
    (2) condition \eqref{align_PassThroughTime_Condition1},
    we obtain
    \begin{align*}
        \int_{t - \delta t_{uv}}^{t^{'}}
        \nu_{uv}(s) ~ds
        \overset{\text{(1)}}{=}&~
        \int_{t - \delta t_{uv}}^{t}
        \nu_{uv}(s) ~ds
        -
        \int_{t^{'}}^{t}
        \nu_{uv}(s) ~ds
        \\
        \overset{\text{(1)}}{=}&~
        \underbrace{
        \int_{t - \delta t_{uv}}^{t}
        \nu_{uv}(s) ~ds
        }_{
        \overset{\text{(2)}}{=} l_{uv}
        }
        \underbrace{
        +
        \underbrace{
        \int_{t}^{t^{'}}
        \underbrace{
        \nu_{uv}(s)
        }_{
        > 0
        }
        ~ds
        }_{
        \overset{\text{(1)}}{>} 0
        }
        }_{
        > 0
        }
        >
        l_{uv}
    \end{align*}

    \noindent -- again a contradiction to condition \eqref{align_PassThroughTime_Condition2}.   

    \begin{remark}
    \label{remark_PassThroughTime}
        Note that if one of the flow rates discussed above are zero, that they can only be zero in this point of time, but not on a whole sub-interval of the time interval $[\min\{t - \delta t_{uv},t\},\max\{t - \delta t_{uv},t\}]$.

        More precisely, 
        if $\delta t_{uv} > 0$ holds,
        let us assume that 
        there exists an $\epsilon > 0$ 
        such that 
        for all $s \in [t - \delta t_{uv},t - \delta t_{uv}+\epsilon)$, 
        $q_{uv}(s) = 0$, and consequently, $\nu_{uv}(s) = 0$, holds.
        Thus, for  
        $t^{'}  = \min\{t - \delta t_{uv}+\epsilon,t\} \in (t - \delta t_{uv},t) = (\min\{t - \delta t_{uv},t\},\max\{t - \delta t_{uv},t\})$,
        by
        (1)  basic results from analysis,
        we obtain
        \begin{align*}
            \int_{t - \delta t_{uv}}^{t^{'}}
            \underbrace{
            \nu_{uv}(s) 
            }_{
            = 0
            }
            ~ds
            \overset{\text{(1)}}{=}
            0
        \end{align*}
    
        \noindent -- a contradiction to condition \eqref{align_PassThroughTime_Condition2}.
        In combination with the findings above, therefore, 
        $q_{uv}(t - \delta t_{uv}) \geq 0$ needs to hold, 
        and if the (unlikely) case $q_{uv}(t - \delta t_{uv}) = 0$ indeed applies, we observe that
        for all $\epsilon > 0$
        there exists an $s \in [t - \delta t_{uv},t - \delta t_{uv}+\epsilon)$ 
        such that
        $q_{uv}(s) \neq 0$, and consequently, $\nu_{uv}(s) \neq 0$, holds.
        Along the same lines, one can also show that indeed, 
        $q_{uv}(s) > 0$, and consequently, $\nu_{uv}(s) > 0$, holds
        (otherwise, we would again obtain a contradiction to condition \eqref{align_SelfLoopTime_Condition2}).

        The other cases are analogous adaptations of this discussion and the findings above.
    \end{remark}
\end{proof}

\subsection{Proof of Lemma \ref{lemma_SelfLoopTime}}

\begin{lemmaNoNb}[Lemma \ref{lemma_SelfLoopTime}]
    Let $e_{uv} \in E$ and $t \in \R$ hold.
    If there exists a self-loop time $\delta t_{uv} \in \R_{\neq 0}$ \gls*{wrt} $e_{uv} \in E$ and $t \in \R$, the following properties hold:

    \begin{enumerate}
        \item There exists no other self-loop time of the same sign 
        (i.e., there exists either exactly one or no positive self-loop time,
        and there exists exactly one or no negative self-loop time).

        \item The mean velocity 
        \begin{align*}
            \nuoverline_{uv}(t) 
            := 
            \nuoverline_{uv}(t - \delta t_{uv},t) 
            = 
            \fint_{t - \delta t_{uv}}^{t} \nu_{uv}(s) ~ds 
            = 
            \frac{1}{\delta t_{uv}} \int_{t - \delta t_{uv}}^{t} \nu_{uv}(s) ~ds
        \end{align*}

        \noindent during the time interval $[\min\{t - \delta t_{uv},t\},\max\{t - \delta t_{uv},t\}]$ 
        as defined in Theorem \ref{theorem_AdvectiveTransportAlongEdges}
        satisfies
        $\nuoverline_{uv}(t) = 0$.

        \item The flows $q_{uv}(t - \delta t_{uv})$ and $q_{uv}(t)$ linked to the velocity $\nu_{uv}(\cdot) = \frac{q_{vu(\cdot)}}{\alpha_{vu}}$ satisfy
        \begin{align*}
            \begin{cases}
                q_{uv}(t - \delta t_{uv}) \geq 0,~ q_{uv}(t) \leq 0
                & \text{ if } \delta t_{uv} > 0
                \\
                q_{uv}(t - \delta t_{uv}) \leq 0,~ q_{uv}(t) \geq 0
                & \text{ if } \delta t_{uv} < 0
            \end{cases}
            .
        \end{align*}
    \end{enumerate}
\end{lemmaNoNb}

\noindent The proof of Lemma \ref{lemma_SelfLoopTime} is very similar to the proof of Lemma \ref{lemma_PassThroughTime} and only requires a few changes. 
The interested reader can compare both proofs and investigate why the changes are required, and how they induce the different results in Lemma \ref{lemma_SelfLoopTime}.3 as compared to Lemma \ref{lemma_PassThroughTime}.4.

\begin{proof}
    \textit{1.:}
    Let us assume that there exist two different self-loop times $\delta t_{1uv}, \delta t_{2uv} \in \R_{\neq 0}$ with the same sign.
    If $\delta t_{1uv}, \delta t_{2uv} > 0$ holds, \gls*{wlog}, let $0 < \delta t_{1uv} < \delta t_{2uv}$ hold.
    Consequently, 
    for $t^{'} = t - \delta t_{1uv} \in (t - \delta t_{2uv},t) = (\min\{t - \delta t_{2uv},t\},\max\{t - \delta t_{2uv},t\})$,
    by
    (1) by basic results from analysis and
    (2) condition \eqref{align_SelfLoopTime_Condition1},
    we obtain
    \begin{align*}
        \int_{t - \delta t_{2uv}}^{t^{'}} \nu_{uv}(s) ~ds
        \overset{\text{(1)}}{=}
        \int_{t - \delta t_{2uv}}^{t} \nu_{uv}(s) ~ds
        -
        \int_{t - \delta t_{1uv}}^{t} \nu_{uv}(s) ~ds
        \overset{\text{(2)}}{=}
        0
        -
        0
        \overset{\text{(1)}}{=}
        0
    \end{align*}

    \noindent -- a contradiction to condition \eqref{align_PassThroughTime_Condition2}.

    If $\delta t_{1uv}, \delta t_{2uv} < 0$ holds, \gls*{wlog}, let $\delta t_{2uv} < \delta t_{1uv} < 0$ hold.
    Consequently, 
    for $t^{'} = t - \delta t_{1uv} \in (t,t - \delta t_{2uv}) = (\min\{t - \delta t_{2uv},t\},\max\{t - \delta t_{2uv},t\})$,
    we can use the same calculation to again obtain a contradiction.
    \\
    \\\textit{2.:}
    By
    (1) definition of the mean velocity $\nuoverline_{uv}(t - \delta t_{uv},t)$ as defined in Theorem \ref{theorem_AdvectiveTransportAlongEdges} and
    (2) condition \eqref{align_SelfLoopTime_Condition1},
    we indeed obtain
    \begin{align*}
        \nuoverline_{uv}(t) 
        := 
        \nuoverline_{uv}(t - \delta t_{uv},t) 
        \overset{\text{(1)}}{=}
        \fint_{t - \delta t_{uv}}^{t} \nu_{uv}(s) ~ds 
        \overset{\text{(1)}}{=} 
        \frac{1}{\delta t_{uv}} \int_{t - \delta t_{uv}}^{t} \nu_{uv}(s) ~ds
        \overset{\text{(2)}}{=}
        0.
    \end{align*}

    \noindent 
    \textit{3.:}
    If $\delta t_{uv} > 0$ holds,
    let us assume that $q_{uv}(t - \delta t_{uv}) < 0$, and consequently, $\nu_{uv}(t - \delta t_{uv}) < 0$ holds.
    Since $\nu_{uv} \in C^0(\R,\R)$ is continuous, 
    there exists a $t^{'} \in (t - \delta t_{uv},t) = (\min\{t - \delta t_{uv},t\},\max\{t - \delta t_{uv},t\})$ 
    such that $\nu_{uv}(s) < 0$ holds 
    for all $s \in (t - \delta t_{uv},t^{'})$.
    Consequently, by
    (1)  basic results from analysis,
    we obtain
    \begin{align*}
        \int_{t - \delta t_{uv}}^{t^{'}}
        \underbrace{
        \nu_{uv}(s) 
        }_{
        < 0
        }
        ~ds
        \overset{\text{(1)}}{<}
        0
    \end{align*}

    \noindent -- a contradiction to condition \eqref{align_SelfLoopTime_Condition2}.

    Similarly, 
    if $\delta t_{uv} > 0$ holds,
    let us assume that $q_{uv}(t) > 0$, and consequently, $\nu_{uv}(t) > 0$ holds.
    Since $\nu_{uv} \in C^0(\R,\R)$ is continuous, 
    there exists a $t^{'} \in (t - \delta t_{uv},t) = (\min\{t - \delta t_{uv},t\},\max\{t - \delta t_{uv},t\})$ 
    such that $\nu_{uv}(s) > 0$ holds 
    for all $s \in (t^{'},t)$.
    Consequently, by
    (1) basic results from analysis and
    (2) condition \eqref{align_SelfLoopTime_Condition1},
    we obtain
    \begin{align*}
        \int_{t - \delta t_{uv}}^{t^{'}}
        \nu_{uv}(s) ~ds
        \overset{\text{(1)}}{=}
        \underbrace{
        \int_{t - \delta t_{uv}}^{t}
        \nu_{uv}(s) ~ds
        }_{
        \overset{\text{(2)}}{=} 0
        }
        \underbrace{
        -
        \underbrace{
        \int_{t^{'}}^{t}
        \underbrace{
        \nu_{uv}(s)
        }_{
        > 0
        }
        ~ds
        }_{
        \overset{\text{(1)}}{>} 0
        }
        }_{
        < 0
        }
        <
        0
    \end{align*}

    \noindent -- again a contradiction to condition \eqref{align_SelfLoopTime_Condition2}.

    In contrast, 
    if $\delta t_{uv} < 0$ holds,
    let us assume that $q_{uv}(t - \delta t_{uv}) > 0$, and consequently, $\nu_{uv}(t - \delta t_{uv}) > 0$ holds.
    Since $\nu_{uv} \in C^0(\R,\R)$ is continuous, 
    there exists a $t^{'} \in (t,t - \delta t_{uv}) = (\min\{t - \delta t_{uv},t\},\max\{t - \delta t_{uv},t\})$ 
    such that $\nu_{uv}(s) > 0$ holds 
    for all $s \in (t^{'},t - \delta t_{uv})$.
    Consequently, by
    (1) basic results from analysis,
    we obtain
    \begin{align*}
        \int_{t - \delta t_{uv}}^{t^{'}}
        \nu_{uv}(s) ~ds
        \overset{\text{(1)}}{=}
        -
        \underbrace{
        \int_{t^{'}}^{t - \delta t_{uv}}
        \underbrace{
        \nu_{uv}(s)
        }_{
        > 0
        }
        ~ds
        }_{
        \overset{\text{(1)}}{>} 0
        }
        <
        0
    \end{align*}

    \noindent -- again a contradiction to condition \eqref{align_SelfLoopTime_Condition2}.

    Similarly,
    if $\delta t_{uv} < 0$ holds,
    let us assume that $q_{uv}(t) < 0$, and consequently, $\nu_{uv}(t) < 0$ holds.
    Since $\nu_{uv} \in C^0(\R,\R)$ is continuous, 
    there exists a $t^{'} \in (t,t - \delta t_{uv}) = (\min\{t - \delta t_{uv},t\},\max\{t - \delta t_{uv},t\})$ 
    such that $\nu_{uv}(s) < 0$ holds 
    for all $s \in (t,t^{'})$.
    Consequently, by
    (1) basic results from analysis and
    (2) condition \eqref{align_SelfLoopTime_Condition1},
    we obtain
    \begin{align*}
        \int_{t - \delta t_{uv}}^{t^{'}}
        \nu_{uv}(s) ~ds
        \overset{\text{(1)}}{=}&~
        \int_{t - \delta t_{uv}}^{t}
        \nu_{uv}(s) ~ds
        -
        \int_{t^{'}}^{t}
        \nu_{uv}(s) ~ds
        \\
        \overset{\text{(1)}}{=}&~
        \underbrace{
        \int_{t - \delta t_{uv}}^{t}
        \nu_{uv}(s) ~ds
        }_{
        \overset{\text{(2)}}{=} 0
        }
        \underbrace{
        +
        \underbrace{
        \int_{t}^{t^{'}}
        \underbrace{
        \nu_{uv}(s)
        }_{
        < 0
        }
        ~ds
        }_{
        \overset{\text{(1)}}{<} 0
        }
        }_{
        < 0
        }
        <
        0
    \end{align*}

    \noindent -- again a contradiction to condition \eqref{align_SelfLoopTime_Condition2}. 

    \begin{remark}
    \label{remark_SelfLoopTime}
        Note that if one of the flow rates discussed above are zero, that they can only be zero in this point of time, but not on a whole sub-interval of the time interval $[\min\{t - \delta t_{uv},t\},\max\{t - \delta t_{uv},t\}]$.

        More precisely, 
        if $\delta t_{uv} > 0$ holds,
        let us assume that 
        there exists an $\epsilon > 0$ 
        such that 
        for all $s \in [t - \delta t_{uv},t - \delta t_{uv}+\epsilon)$, 
        $q_{uv}(s) = 0$, and consequently, $\nu_{uv}(s) = 0$, holds.
        Thus, for  
        $t^{'}  = \min\{t - \delta t_{uv}+\epsilon,t\} \in (t - \delta t_{uv},t) = (\min\{t - \delta t_{uv},t\},\max\{t - \delta t_{uv},t\})$,
        by
        (1)  basic results from analysis,
        we obtain
        \begin{align*}
            \int_{t - \delta t_{uv}}^{t^{'}}
            \underbrace{
            \nu_{uv}(s) 
            }_{
            = 0
            }
            ~ds
            \overset{\text{(1)}}{=}
            0
        \end{align*}
    
        \noindent -- a contradiction to condition \eqref{align_SelfLoopTime_Condition2}.
        In combination with the findings above, therefore, 
        $q_{uv}(t - \delta t_{uv}) \geq 0$ needs to hold, 
        and if the (unlikely) case $q_{uv}(t - \delta t_{uv}) = 0$ indeed applies, we observe that
        for all $\epsilon > 0$
        there exists an $s \in [t - \delta t_{uv},t - \delta t_{uv}+\epsilon)$ 
        such that
        $q_{uv}(s) \neq 0$, and consequently, $\nu_{uv}(s) \neq 0$, holds.
        Along the same lines, one can also show that indeed, 
        $q_{uv}(s) > 0$, and consequently, $\nu_{uv}(s) > 0$, holds
        (otherwise, we would again obtain a contradiction to condition \eqref{align_SelfLoopTime_Condition2}).

        The other cases are analogous adaptations of this discussion and the findings above.
    \end{remark}
\end{proof}

\subsection{Proof of Lemma \ref{lemma_InversePassThroughTime}}

\begin{lemmaNoNb}[Lemma \ref{lemma_InversePassThroughTime}]
    Let $e_{uv} \in E$ and $t \in \R$ hold.
    $\delta t_{uv} \in \R$ is an inverse pass-through time \gls*{wrt} $e_{uv} \in E$ and $t \in \R$ 
    \gls*{iff}
    $\delta t_{vu} := \delta t_{uv} \in \R$ is a pass-through time \gls*{wrt} $e_{vu} \in E$ and $t \in \R$ 
\end{lemmaNoNb}

\begin{proof}
    \enquote{$\Rightarrow$}
    If $\delta t_{uv} \in \R$ is an inverse pass-through time \gls*{wrt} $e_{uv} \in E$ and $t \in \R$,
    by 
    (1) the fact that the geometric pipe features are symmetric in the sense that $l_{vu} = l_{uv}$ and $\alpha_{vu} = \alpha_{uv}$ holds for all $v \in V$ and $u \in \Nbh(v)$,
    (2) definition of the flow velocity $\nu_{vu}(\cdot) = \frac{q_{vu(\cdot)}}{\alpha_{vu}}$, 
    (3) the conservation of flows (cf. Eq. \eqref{align_ConservationOfFlows}) and
    (4) condition \eqref{align_PassThroughTime_Condition1_Inverse},
    we obtain
    \begin{align*}
        \int_{t - \delta t_{uv}}^{t}
        \nu_{vu}(s) ~ds
        \overset{\text{(1,2,3)}}{=}~
        - \int_{t - \delta t_{uv}}^{t}
        \nu_{uv}(s) ~ds
        \overset{\text{(4)}}{=}~
        - (-l_{uv})
        \overset{\text{(1)}}{=}~
        l_{vu},
    \end{align*}

    \noindent that is, condition \eqref{align_PassThroughTime_Condition1} is satisfied for $\delta t_{vu} := \delta t_{uv}$ and \gls*{wrt} $e_{vu} \in E$ and $t \in \R$.

    Similarly, by additionally using 
    (5) condition \eqref{align_PassThroughTime_Condition2_Inverse},
    we obtain
    \begin{align*}
        \int_{t - \delta t_{uv}}^{t^{'}}
        \nu_{vu}(s) ~ds
        \overset{\text{(1,2,3)}}{=}~
        - \int_{t - \delta t_{uv}}^{t}
        \nu_{uv}(s) ~ds
        \overset{\text{(5)}}{\in}~
        (0,l_{uv})
        \overset{\text{(1)}}{=}~
        (0,l_{vu}),
    \end{align*}

    \noindent that is, condition \eqref{align_PassThroughTime_Condition2} is satisfied for $\delta t_{vu} := \delta t_{uv}$ and \gls*{wrt} $e_{vu} \in E$ and $t \in \R$.
    Therefore, $\delta t_{vu} := \delta t_{uv} \in \R$ is a pass-through time \gls*{wrt} $e_{vu} \in E$ and $t \in \R$.
    \\
    \\
    \enquote{$\Leftarrow$}
    If $\delta t_{vu} := \delta t_{uv} \in \R$ is a pass-through time \gls*{wrt} $e_{vu} \in E$ and $t \in \R$,
    by
    (1) the fact that the geometric pipe features are symmetric in the sense that $l_{vu} = l_{uv}$ and $\alpha_{vu} = \alpha_{uv}$ holds for all $v \in V$ and $u \in \Nbh(v)$,
    (2) definition of the flow velocity $\nu_{uv}(\cdot) = \frac{q_{uv(\cdot)}}{\alpha_{uv}}$, 
    (3) the conservation of flows (cf. Eq. \eqref{align_ConservationOfFlows}) and
    (4) condition \eqref{align_PassThroughTime_Condition1},
    we obtain
    \begin{align*}
        \int_{t - \delta t_{uv}}^{t}
        \nu_{uv}(s) ~ds
        \overset{\text{(1,2,3)}}{=}~
        - \int_{t - \delta t_{uv}}^{t}
        \nu_{vu}(s) ~ds
        \overset{\text{(4)}}{=}~
        - (l_{vu})
        \overset{\text{(1)}}{=}~
        - l_{uv},
    \end{align*}

    \noindent that is, condition \eqref{align_PassThroughTime_Condition1_Inverse} is satisfied for $\delta t_{uv} := \delta t_{vu}$ and \gls*{wrt} $e_{uv} \in E$ and $t \in \R$.

    Similarly, by additionally using 
    (5) condition \eqref{align_PassThroughTime_Condition2},
    we obtain
    \begin{align*}
        \int_{t - \delta t_{uv}}^{t^{'}}
        \nu_{uv}(s) ~ds
        \overset{\text{(1,2,3)}}{=}~
        - \int_{t - \delta t_{uv}}^{t}
        \nu_{vu}(s) ~ds
        \overset{\text{(5)}}{\in}~
        (-l_{vu},0)
        \overset{\text{(1)}}{=}~
        (-l_{uv},0),
    \end{align*}

    \noindent that is, condition \eqref{align_PassThroughTime_Condition2_Inverse} is satisfied for $\delta t_{uv} := \delta t_{vu}$ and \gls*{wrt} $e_{uv} \in E$ and $t \in \R$.
    Therefore, $\delta t_{uv} := \delta t_{vu} \in \R$ is an inverse pass-through time \gls*{wrt} $e_{uv} \in E$ and $t \in \R$.
\end{proof}

\subsection{Proof of Lemma \ref{lemma_InverseSelfLoopTime}}

\begin{lemmaNoNb}[Lemma \ref{lemma_InverseSelfLoopTime}]
    Let $e_{uv} \in E$ and $t \in \R$ hold.
    $\delta t_{uv} \in \R$ is an inverse self-loop time \gls*{wrt} $e_{uv} \in E$ and $t \in \R$ 
    \gls*{iff}
    $\delta t_{vu} := \delta t_{uv} \in \R$ is a self-loop time \gls*{wrt} $e_{vu} \in E$ and $t \in \R$ 
\end{lemmaNoNb}

\noindent The proof of Lemma \ref{lemma_InverseSelfLoopTime} is very similar to the proof of Lemma \ref{lemma_InversePassThroughTime} and only requires a few changes. 
The interested reader can compare both proofs and investigate why the changes are required, and how they induce the different results in Lemma \ref{lemma_InverseSelfLoopTime} as compared to Lemma \ref{lemma_InverseSelfLoopTime}.

\begin{proof}
    \enquote{$\Rightarrow$}
    If $\delta t_{uv} \in \R$ is an inverse self-loop time \gls*{wrt} $e_{uv} \in E$ and $t \in \R$,
    by 
    (1) the fact that the geometric pipe features are symmetric in the sense that $l_{vu} = l_{uv}$ and $\alpha_{vu} = \alpha_{uv}$ holds for all $v \in V$ and $u \in \Nbh(v)$,
    (2) definition of the flow velocity $\nu_{vu}(\cdot) = \frac{q_{vu(\cdot)}}{\alpha_{vu}}$, 
    (3) the conservation of flows (cf. Eq. \eqref{align_ConservationOfFlows}) and
    (4) condition \eqref{align_SelfLoopTime_Condition1_Inverse},
    we obtain
    \begin{align*}
        \int_{t - \delta t_{uv}}^{t}
        \nu_{vu}(s) ~ds
        \overset{\text{(1,2,3)}}{=}~
        - \int_{t - \delta t_{uv}}^{t}
        \nu_{uv}(s) ~ds
        \overset{\text{(4)}}{=}~
        - 0
        =~
        0,
    \end{align*}

    \noindent that is, condition \eqref{align_SelfLoopTime_Condition1} is satisfied for $\delta t_{vu} := \delta t_{uv}$ and \gls*{wrt} $e_{vu} \in E$ and $t \in \R$.

    Similarly, by additionally using 
    (5) condition \eqref{align_SelfLoopTime_Condition2_Inverse},
    we obtain
    \begin{align*}
        \int_{t - \delta t_{uv}}^{t^{'}}
        \nu_{vu}(s) ~ds
        \overset{\text{(1,2,3)}}{=}~
        - \int_{t - \delta t_{uv}}^{t}
        \nu_{uv}(s) ~ds
        \overset{\text{(5)}}{\in}~
        (0,l_{uv})
        \overset{\text{(1)}}{=}~
        (0,l_{vu}),
    \end{align*}

    \noindent that is, condition \eqref{align_SelfLoopTime_Condition2} is satisfied for $\delta t_{vu} := \delta t_{uv}$ and \gls*{wrt} $e_{vu} \in E$ and $t \in \R$.
    Therefore, $\delta t_{vu} := \delta t_{uv} \in \R$ is a self-loop time \gls*{wrt} $e_{vu} \in E$ and $t \in \R$.
    \\
    \\
    \enquote{$\Leftarrow$}
    If $\delta t_{vu} := \delta t_{uv} \in \R$ is a self-loop time \gls*{wrt} $e_{vu} \in E$ and $t \in \R$,
    by
    (1) the fact that the geometric pipe features are symmetric in the sense that $l_{vu} = l_{uv}$ and $\alpha_{vu} = \alpha_{uv}$ holds for all $v \in V$ and $u \in \Nbh(v)$,
    (2) definition of the flow velocity $\nu_{uv}(\cdot) = \frac{q_{uv(\cdot)}}{\alpha_{uv}}$, 
    (3) the conservation of flows (cf. Eq. \eqref{align_ConservationOfFlows}) and
    (4) condition \eqref{align_SelfLoopTime_Condition1},
    we obtain
    \begin{align*}
        \int_{t - \delta t_{uv}}^{t}
        \nu_{uv}(s) ~ds
        \overset{\text{(1,2,3)}}{=}~
        - \int_{t - \delta t_{uv}}^{t}
        \nu_{vu}(s) ~ds
        \overset{\text{(4)}}{=}~
        - (0)
        =~
        0,
    \end{align*}

    \noindent that is, condition \eqref{align_SelfLoopTime_Condition1_Inverse} is satisfied for $\delta t_{uv} := \delta t_{vu}$ and \gls*{wrt} $e_{uv} \in E$ and $t \in \R$.

    Similarly, by additionally using 
    (5) condition \eqref{align_SelfLoopTime_Condition2},
    we obtain
    \begin{align*}
        \int_{t - \delta t_{uv}}^{t^{'}}
        \nu_{uv}(s) ~ds
        \overset{\text{(1,2,3)}}{=}~
        - \int_{t - \delta t_{uv}}^{t}
        \nu_{vu}(s) ~ds
        \overset{\text{(5)}}{\in}~
        (-l_{vu},0)
        \overset{\text{(1)}}{=}~
        (-l_{uv},0),
    \end{align*}

    \noindent that is, condition \eqref{align_SelfLoopTime_Condition2_Inverse} is satisfied for $\delta t_{uv} := \delta t_{vu}$ and \gls*{wrt} $e_{uv} \in E$ and $t \in \R$.
    Therefore, $\delta t_{uv} := \delta t_{vu} \in \R$ is an inverse self-loop time \gls*{wrt} $e_{uv} \in E$ and $t \in \R$.
\end{proof}

\subsection{Proof of Theorem \ref{theorem_PassThroughTime}}

\begin{theoremNoNb}[Theorem \ref{theorem_PassThroughTime}]
    Let $e_{uv} \in E$ and $t \in \R$ hold.
    If the function $c_{uv}$ obeys Equation \eqref{align_EdgeDynamics_Advection} and there exists a pass-through time $\delta t_{uv} \in \R_{\neq 0}$ \gls*{wrt} $e_{uv} \in E$ and $t \in \R$,
    we obtain
    \begin{align*}
        c_{uv}(t,l_{uv})
        =
        c_{uv}(t - \delta t_{uv},0)
        =
        c_{uv} \left( t - \frac{l_{uv}}{\nuoverline_{uv}(t)},0 \right).
    \end{align*}

    \noindent Even more, the concentrations $c_{uv}(\cdot,l_{uv})$ and $c_{vu}(\cdot,l_{vu})$ at the end $z = l_{uv} = l_{vu}$ of the edges $e_{uv} \in E$ and $e_{vu} \in E$ are connected to the concentrations $c_u$ and $c_v$ of the nodes of that edge, respectively, by
    \begin{align*}
        \begin{cases}
            c_{uv}(t,l_{uv})
            =
            c_u(t - \delta t_{uv})
            & \text{ if } \delta t_{uv} > 0
            \\
            c_{vu}(t - \delta t_{uv},l_{vu})
            =
            c_v(t)
            & \text{ if } \delta t_{uv} < 0
        \end{cases}
        .
    \end{align*}
\end{theoremNoNb}

\begin{proof}
    The main idea of this proof is to use the continuity of the function $c_{uv} \in C^0(\R \times [0,l_{uv}],\R_{\geq 0})$ to transfer the results from the advective transport along edges, i.e., Theorem \ref{theorem_AdvectiveTransportAlongEdges}, to the ends of the edge $e_{uv} \in E$, i.e., to the functions $c_u, c_v \in C^1(\R,\R_{\geq 0})$.

    The crucial part is that we would like to apply Theorem \ref{theorem_AdvectiveTransportAlongEdges} for $t_0 = t - \delta t_{uv}$, $z_0 = 0$ and $t \in \R$ as given. 
    However, the first problem is that $z_0 = 0 \not \in \Omega_{uv} = (0,l_{uv})$ does not satisfy the condition of Theorem \ref{theorem_AdvectiveTransportAlongEdges}.
    Moreover, since for this choice,
    \begin{align*}
        z_0 + \int_{t_0}^{t_0} \nu_{uv}(s) ~ds
        =
        z_0 + 0
        =
        0
        \quad \text{ and } \quad
        z_0 + \int_{t_0}^{t} \nu_{uv}(s) ~ds
        =
        z_0 + l_{uv}
        =
        l_{uv}
    \end{align*}

    \noindent hold, $t \not \in T_{t_0,z_0}$ does also not satisfy the condition of Theorem \ref{theorem_AdvectiveTransportAlongEdges}.
    Intuitively, we first have to transfer the information from the boundary of the edge $e_{uv}$ to the inside of the edge using the continuity of the function $c_{uv} \in C^1(\R \times [0,l_{uv}],\R_{\geq 0})$. 
    Afterwards, we can use Theorem \ref{theorem_AdvectiveTransportAlongEdges} for slightly modified $z_0$ and $t$.
    \\
    \\
    To do so, we define the two functions
    \begin{align*}
        z: &(\min\{0,\delta t_{uv}\},\max\{0,\delta t_{uv}\}) \longrightarrow \R, ~
        \epsilon \longmapsto z(\epsilon) := \int_{t - \delta t_{uv}}^{t - \epsilon} \nu_{uv}(s) ~ds
        \text{ and}
        \\
        \ztilde: &(\min\{0,\delta t_{uv}\},\max\{0,\delta t_{uv}\}) \longrightarrow \R, ~
        \epsilon \longmapsto \ztilde(e) := \frac{l_{uv} - z(\epsilon)}{2}.
    \end{align*}

    \noindent We first have to prove some technical details regarding these functions (\textit{Step 1}) in order to be able to apply Theorem \ref{theorem_AdvectiveTransportAlongEdges} as illustrated above (\textit{Step 2}).
    Afterwards, we can use the results from Lemma \ref{lemma_PassThroughTime} about the sign of the flow at time $t - \delta t_{uv}$ in dependence of the sign of the pass-through time $\delta t_{uv}$ in combination with the local inflow boundary conditions defined through Equation \eqref{align_CouplingCondition_InflowContinuity} to be able to transfer the results from the advective transport along edges to the end of the pipes (\textit{Step 3}).
    \\
    \\
    \textit{Step 1:}
    \textit{The functions $z$ and $\ztilde$ satisfy the following properties:}
    
    \begin{enumerate}
        \item $z$, $\ztilde$ and $z+\ztilde$ are continuous.
        
        \item $\im(z) \subset (0,l_{uv})$, 
        $\im(\ztilde) \subset (0,\frac{1}{2} l_{uv}) \subset (0,l_{uv})$ and
        $\im(z+\ztilde) \subset (\frac{1}{2},l_{uv}) \subset (0,l_{uv})$.

        \item $\lim_{\epsilon \rightarrow 0} z(\epsilon) = l_{uv}$,
        $\lim_{\epsilon \rightarrow 0} \ztilde(\epsilon) = 0$ and
        $\lim_{\epsilon \rightarrow 0} z(\epsilon)+\ztilde(\epsilon) = l_{uv}$.
    \end{enumerate}

    \noindent Note that the limit is to be understood for $\epsilon \in (\min\{0,\delta t_{uv}\},\max\{0,\delta t_{uv}\})$ in the domain of $z$ and $\ztilde$, i.e.,
    if $\delta t_{uv} > 0$, 
    $(\min\{0,\delta t_{uv}\},\max\{0,\delta t_{uv}\}) = (0,\delta t_{uv})$ holds, 
    and the limit $\epsilon \rightarrow 0$ is to be understood as the limit 
    $\epsilon \searrow 0, \epsilon < \delta t_{uv}$.
    In contrast,
    if $\delta t_{uv} < 0$, 
    $(\min\{0,\delta t_{uv}\},\max\{0,\delta t_{uv}\}) = (\delta t_{uv},0)$ holds, 
    and the limit $\epsilon \rightarrow 0$ is to be understood as the limit 
    $\epsilon \nearrow 0, \epsilon > \delta t_{uv}$.
    \\
    \\\textit{1.:}
    By the fundamental theorem of calculus, the function $t \mapsto \int_{t - \delta t_{uv}}^{t} \nu_{uv}(s) ~ds$ is differentiable and therefore, continuous.
    Therefore, as a composition of continuous functions, the functions $z$, $\ztilde$ and $z+\ztilde$ are continuous.
    \\
    \\\textit{2.:}
    By
    (1) basic transformations,
    (2) definition of $z$, $\ztilde$ and $z+\ztilde$, respectively, and 
    (3) condition \eqref{align_PassThroughTime_Condition2},
    we obtain
    \begin{align*}
        \im(z)
        \overset{\text{(1)}}{=}&~
        \{ 
        z(\epsilon) 
        ~|~ 
        \epsilon \in (\min\{0,\delta t_{uv}\},\max\{0,\delta t_{uv}\})
        \}
        \\
        \overset{\text{(2)}}{=}&~
        \bigg \{ 
        \int_{t - \delta t_{uv}}^{t - \epsilon} \nu_{uv}(s) ~ds
        ~\Big|~ 
        \epsilon \in (\min\{0,\delta t_{uv}\},\max\{0,\delta t_{uv}\})
        \bigg \}
        \\
        \overset{\text{(1)}}{=}&~
        \bigg \{ 
        \int_{t - \delta t_{uv}}^{t - \epsilon} \nu_{uv}(s) ~ds
        ~\Big|~ 
        t - \epsilon \in (\min\{t,t-\delta t_{uv}\},\max\{t,t-\delta t_{uv}\})
        \bigg \}
        \\
        \overset{\text{(3)}}{\subset}&~
        (0,l_{uv}).
    \end{align*}

    \noindent By additionally using that therefore,
    (4) $z(\epsilon) \in (0,l_{uv})$ holds for all $\epsilon \in (\min\{0,\delta t_{uv}\},$ $\max\{0,\delta t_{uv}\})$,
    we obtain
    \begin{align*}
        \im(\ztilde)
        \overset{\text{(1)}}{=}&~
        \{ 
        \ztilde(\epsilon) 
        ~|~ 
        \epsilon \in (\min\{0,\delta t_{uv}\},\max\{0,\delta t_{uv}\})
        \\
        \overset{\text{(2)}}{=}&~
        \bigg \{ 
        \frac{l_{uv} - z(\epsilon)}{2}
        ~\Big|~ 
        \epsilon \in (\min\{0,\delta t_{uv}\},\max\{0,\delta t_{uv}\})
        \bigg \}
        \\
        \overset{\text{(4)}}{\subset}&~
        (0,\textstyle \frac{1}{2} l_{uv})
        \quad \text{ and}
        \\
        \im(z+\ztilde)
        \overset{\text{(1)}}{=}&~
        \{ 
        z(\epsilon) + \ztilde(\epsilon) 
        ~|~ 
        \epsilon \in (\min\{0,\delta t_{uv}\},\max\{0,\delta t_{uv}\})
        \}
        \\
        \overset{\text{(2)}}{=}&~
        \bigg \{ 
        z(\epsilon) + \frac{l_{uv} - z(\epsilon)}{2}
        ~\Big|~ 
        \epsilon \in (\min\{0,\delta t_{uv}\},\max\{0,\delta t_{uv}\})
        \bigg \}
        \\
        \overset{\text{(1)}}{=}&~
        \bigg \{ 
        \frac{l_{uv} + z(\epsilon)}{2}
        ~\Big|~ 
        \epsilon \in (\min\{0,\delta t_{uv}\},\max\{0,\delta t_{uv}\})
        \bigg \}
        \\
        \overset{\text{(4)}}{\subset}&~
        (\textstyle \frac{1}{2} l_{uv},l_{uv}).
    \end{align*}

    \noindent
    \textit{3.:}
    By
    (1) continuity of $z$, $\ztilde$ and $z+\ztilde$, respectively,
    (2) definition of $z$, $\ztilde$ and $z+\ztilde$, respectively and
    (3) condition \eqref{align_PassThroughTime_Condition1},
    we obtain
    \begin{align*}
        \lim_{\epsilon \rightarrow 0} 
        z(\epsilon) 
        \overset{\text{(1)}}{=}&~
        z(0) 
        \overset{\text{(2)}}{=}~
        \int_{t - \delta t_{uv}}^{t} \nu_{uv}(s) ~ds
        \overset{\text{(3)}}{=}~
        l_{uv},
        \\
        \lim_{\epsilon \rightarrow 0} 
        \ztilde(\epsilon) 
        \overset{\text{(1)}}{=}&~
        \ztilde(0) 
        \overset{\text{(2)}}{=}~
        \frac{l_{uv} - z(0)}{2}
        \overset{\text{(2,3)}}{=}~
        0,
        \\
        \lim_{\epsilon \rightarrow 0} 
        z(\epsilon) + \ztilde(\epsilon) 
        \overset{\text{(1)}}{=}&~
        z(0) + \ztilde(0) 
        \overset{\text{(2,3)}}{=}~
        l_{uv}.
    \end{align*}

    \noindent
    \textit{Step 2:}
    \textit{$c_{uv}(t,l_{uv}) = c_{uv}(t - \delta t_{uv},0) =  c_{uv}(t - \frac{l_{uv}}{\nuoverline_{uv}(t)},0)$ holds.}

    For each $\epsilon \in (\min\{0,\delta t_{uv}\},\max\{0,\delta t_{uv}\})$, we can now choose $t_0 = t - \delta t_{uv} \in \R$ and -- as shown in \textit{Step 1} -- $z_0 = \ztilde(\epsilon) \in (0,l_{uv})$.
    Consequently, for each $\epsilon \in (\min\{0,\delta t_{uv}\},\max\{0,\delta t_{uv}\})$, by
    (1) the choices of $t_0 = t - \delta t_{uv} \in \R$ and $z_0 = \ztilde(\epsilon) \in (0,l_{uv})$,
    (2) the definition of $z$ and $\ztilde$ from \textit{Step 1} and
    (3) the results from \textit{Step 1},
    we observe that
    \begin{align*}
        z_0 + \int_{t_0}^{t-\epsilon} \nu_{uv}(s) ~ds
        \overset{\text{(1)}}{=}~
        \ztilde(\epsilon) + \int_{t - \delta t_{uv}}^{t-\epsilon} \nu_{uv}(s) ~ds
        \overset{\text{(2)}}{=}~
        \ztilde(\epsilon) + z(\epsilon)
        \overset{\text{(3)}}{\in}~
        (0,l_{uv}),
    \end{align*}

    \noindent and therefore, $t-\epsilon \in T_{t_0,z_0} = T_{t - \delta t_{uv},\ztilde(\epsilon)}$ holds.\footnote{
        Note that all the technical work of constructing $z$ and $\ztilde$ in \textit{Step 1} is because we would actually like to choose $z_0 = 0$ and $t \in \R$ as given, however, this $z_0$ and this $t$ do not satisfy the requirements of Theorem \ref{theorem_AdvectiveTransportAlongEdges}.
    }
    
    Therefore, by
    (1) the results from \textit{Step 1},
    (2) the continuity of the function $c_{uv} \in C^1(\R \times [0,l_{uv}],\R_{\geq 0})$,
    (3) definition of $z$ and
    (4) Theorem \ref{theorem_AdvectiveTransportAlongEdges} applied to $t_0 = t - \delta t_{uv}$, $z_0 = \ztilde(\epsilon) \in (0,l_{uv})$ and $t - \epsilon \in T_{t_0,z_0} = T_{t - \delta t_{uv},\ztilde(\epsilon)}$,
    we obtain
    \begin{align*}
        c_{uv}(t,l_{uv})
        \overset{\text{(1)}}{=}&~
        c_{uv} \left(
        \lim_{\epsilon \rightarrow 0}
        t - \epsilon
        ,
        \lim_{\epsilon \rightarrow 0}
        z(\epsilon) + \ztilde(\epsilon)
        \right)
        \\
        \overset{\text{(2)}}{=}&~
        \lim_{\epsilon \rightarrow 0}
        c_{uv} \left(
        t - \epsilon
        ,
        z(\epsilon) + \ztilde(\epsilon)
        \right)
        \\
        \overset{\text{(3)}}{=}&~
        \lim_{\epsilon \rightarrow 0}
        c_{uv} \bigg(
        t - \epsilon
        ~,~
        \underbrace{
        \underbrace{
        \ztilde(\epsilon)
        }_{
        \overset{\text{(1)}}{\in} (0,l_{uv})
        }
        +
        \underbrace{
        \int_{t - \delta t_{uv}}^{t - \epsilon} \nu_{uv}(s) ~ds
        }_{
        = z(\epsilon) \overset{\text{(1)}}{\in} (0,l_{uv})
        }
        }_{
        = z(\epsilon) + \ztilde(\epsilon) \overset{\text{(1)}}{\in} (0,l_{uv})
        }
        \bigg)
        \\
        \overset{\text{(4)}}{=}&~
        \lim_{\epsilon \rightarrow 0}
        c_{uv} \left(
        t - \delta t_{uv}
        ,
        \ztilde(\epsilon)
        \right)
        \\
        \overset{\text{(2)}}{=}&~
        c_{uv} \left(
        t - \delta t_{uv}
        ,
        \lim_{\epsilon \rightarrow 0}
        \ztilde(\epsilon)
        \right)
        \\
        \overset{\text{(1)}}{=}&~
        c_{uv}(t - \delta t_{uv},0).
    \end{align*}

    \noindent Finally, as seen in Lemma \ref{lemma_PassThroughTime}.3, $\delta t_{uv} =  \frac{l_{uv}}{\nuoverline_{uv}(t)}$ holds, and we can conclude 
    \begin{align*}
        c_{uv}(t - \delta t_{uv},0) 
        = 
        c_{uv} \left(t - \frac{l_{uv}}{\nuoverline_{uv}(t)} \right).
    \end{align*}

    \noindent 
    \textit{Step 3:}
    \textit{
    If $\delta t_{uv} > 0$ holds, 
    $c_{uv}(t,l_{uv}) = c_u(t - \delta t_{uv})$ holds.
    If $\delta t_{uv} < 0$ holds,
    $c_{vu}(t - \delta t_{uv},l_{vu}) = c_v(t)$ holds.
    }

    
    If $\delta t_{uv} > 0$ holds, by 
    (1) \textit{Step 2} and
    (2) Equation \eqref{align_CouplingCondition_InflowContinuity} together with the fact that by Lemma \ref{lemma_PassThroughTime}.4, $q_{uv}(t - \delta t_{uv}) \geq 0$ holds,
    we obtain
    \begin{align*}
        c_{uv}(t,l_{uv}) 
        \overset{\text{(1)}}{=}
        c_{uv}(t - \delta t_{uv},0)
        \overset{\text{(2)}}{=}
        c_u(t - \delta t_{uv}).
    \end{align*}

    \noindent If instead, $\delta t_{uv} < 0$ holds,
    by 
    (1) Equation \eqref{align_SymmetryOfSignals} for $z = l_{vu}$,
    (2) \textit{Step 2} and
    (3) Equation \eqref{align_CouplingCondition_InflowContinuity} together with the fact that by the conservation of flows (Eq. \eqref{align_ConservationOfFlows}) and Lemma \ref{lemma_PassThroughTime}.4, $q_{vu}(t) = -q_{uv}(t) \geq 0$ holds,
    we obtain
    \begin{align*}
        c_{vu}(t - \delta t_{uv},l_{vu})
        \overset{\text{(1)}}{=}
        c_{uv}(t - \delta t_{uv},0)
        \overset{\text{(2)}}{=}
        c_{uv}(t,l_{uv}) 
        \overset{\text{(1)}}{=}
        c_{vu}(t,0) 
        \overset{\text{(3)}}{=}
        c_v(t).
    \end{align*}
    %
\end{proof}

\subsection{Proof of Theorem \ref{theorem_SelfLoopTime}}

\begin{theoremNoNb}[Theorem \ref{theorem_SelfLoopTime}]
    Let $e_{uv} \in E$ and $t \in \R$ hold.
    If the function $c_{uv}$ obeys Equation \eqref{align_EdgeDynamics_Advection} and there exists a self-loop time $\delta t_{uv} \in \R_{\neq 0}$ \gls*{wrt} $e_{uv} \in E$ and $t \in \R$,
    we obtain
    \begin{align*}
        c_{uv}(t,0)
        =
        c_{uv}(t - \delta t_{uv},0).
    \end{align*}

    \noindent Even more, the concentration $c_{vu}(\cdot,l_{vu})$ at the end $z = l_{vu} = l_{uv}$ of the edge $e_{vu} \in E$ is connected to the concentration $c_u$ of the node of that edge by
    \begin{align*}
        \begin{cases}
            c_{vu}(t,l_{vu})
            =
            c_u(t - \delta t_{uv})
            & \text{ if } \delta t_{uv} > 0
            \\
            c_{vu}(t - \delta t_{uv},l_{vu})
            =
            c_u(t)
            & \text{ if } \delta t_{uv} < 0
        \end{cases}
        .
    \end{align*}
\end{theoremNoNb}

\noindent The proof of Theorem \ref{theorem_SelfLoopTime} is very similar to the proof of Theorem \ref{theorem_PassThroughTime} and only requires a few changes. 
The interested reader can compare both proofs and investigate why the changes are required, and how they induce the different results in Theorem \ref{theorem_SelfLoopTime} as compared to Theorem \ref{theorem_PassThroughTime}.

\begin{proof}
    The main idea of this proof is to use the continuity of the function $c_{uv} \in C^1(\R \times [0,l_{uv}],\R_{\geq 0})$ to transfer the results from the advective transport along edges, i.e., Theorem \ref{theorem_AdvectiveTransportAlongEdges}, to the end of the edge $e_{uv} \in E$, i.e., to the function $c_u \in C^1(\R,\R_{\geq 0})$.

    The crucial part is that we would like to apply Theorem \ref{theorem_AdvectiveTransportAlongEdges} for $t_0 = t - \delta t_{uv}$, $z_0 = 0$ and $t \in \R$ as given. 
    However, the first problem is that $z_0 = 0 \not \in \Omega_{uv} = (0,l_{uv})$ does not satisfy the condition of Theorem \ref{theorem_AdvectiveTransportAlongEdges}.
    Moreover, since for this choice,
    \begin{align*}
        z_0 + \int_{t_0}^{t_0} \nu_{uv}(s) ~ds
        =
        z_0 + 0
        =
        0
        \quad \text{ and } \quad
        z_0 + \int_{t_0}^{t} \nu_{uv}(s) ~ds
        =
        z_0 + 0
        =
        0
    \end{align*}

    \noindent hold, $t \not \in T_{t_0,z_0}$ does also not satisfy the condition of Theorem \ref{theorem_AdvectiveTransportAlongEdges}.
    Intuitively, we first have to transfer the information from the boundary of the edge $e_{uv}$ to the inside of the edge using the continuity of the function $c_{uv} \in C^1(\R \times [0,l_{uv}],\R_{\geq 0})$. 
    Afterwards, we can use Theorem \ref{theorem_AdvectiveTransportAlongEdges} for slightly modified $z_0$ and $t$.
    \\
    \\
    To do so, we define the two functions
    \begin{align*}
        z: &(\min\{0,\delta t_{uv}\},\max\{0,\delta t_{uv}\}) \longrightarrow \R, ~
        \epsilon \longmapsto z(\epsilon) := \int_{t - \delta t_{uv}}^{t - \epsilon} \nu_{uv}(s) ~ds
        \text{ and}
        \\
        \ztilde: &(\min\{0,\delta t_{uv}\},\max\{0,\delta t_{uv}\}) \longrightarrow \R, ~
        \epsilon \longmapsto \ztilde(e) := \frac{z(\epsilon)}{2}.
    \end{align*}

    \noindent We first have to prove some technical details regarding these functions (\textit{Step 1}) in order to be able to apply Theorem \ref{theorem_AdvectiveTransportAlongEdges} as illustrated above (\textit{Step 2}).
    Afterwards, we can use the results from Lemma \ref{lemma_SelfLoopTime} about the sign of the flow at time $t - \delta t_{uv}$ in dependence of the sign of the self-loop time $\delta t_{uv}$ in combination with the local inflow boundary conditions defined through Equation \eqref{align_CouplingCondition_InflowContinuity} to be able to transfer the results from the advective transport along edges to the end of the pipe (\textit{Step 3}).
    \\
    \\
    \textit{Step 1:}
    \textit{The functions $z$ and $\ztilde$ satisfy the following properties:}
    
    \begin{enumerate}
        \item $z$, $\ztilde$ and $z+\ztilde$ are continuous.
        
        \item $\im(z) \subset (0,l_{uv})$, 
        $\im(\ztilde) \subset (0,\frac{1}{2} l_{uv}) \subset (0,l_{uv})$ and
        $\im(z+\ztilde) \subset (0,\frac{3}{2} l_{uv}) \subset (0,l_{uv})$.

        \item $\lim_{\epsilon \rightarrow 0} z(\epsilon) = 0$,
        $\lim_{\epsilon \rightarrow 0} \ztilde(\epsilon) = 0$ and
        $\lim_{\epsilon \rightarrow 0} z(\epsilon)+\ztilde(\epsilon) = 0$.
    \end{enumerate}

    \noindent Note that the limit is to be understood for $\epsilon \in (\min\{0,\delta t_{uv}\},\max\{0,\delta t_{uv}\})$ in the domain of $z$ and $\ztilde$, i.e.,
    if $\delta t_{uv} > 0$, 
    $(\min\{0,\delta t_{uv}\},\max\{0,\delta t_{uv}\}) = (0,\delta t_{uv})$ holds, 
    and the limit $\epsilon \rightarrow 0$ is to be understood as the limit 
    $\epsilon \searrow 0, \epsilon < \delta t_{uv}$.
    In contrast,
    if $\delta t_{uv} < 0$, 
    $(\min\{0,\delta t_{uv}\},\max\{0,\delta t_{uv}\}) = (\delta t_{uv},0)$ holds, 
    and the limit $\epsilon \rightarrow 0$ is to be understood as the limit 
    $\epsilon \nearrow 0, \epsilon > \delta t_{uv}$.
    \\
    \\\textit{1.:}
    By the fundamental theorem of calculus, the function $t \mapsto \int_{t - \delta t_{uv}}^{t} \nu_{uv}(s) ~ds$ differentiable and therefore, continuous.
    Therefore, as compositions of continuous functions, the functions $z$, $\ztilde$ and $z+\ztilde$ are continuous.
    \\
    \\\textit{2.:}
    By
    (1) basic transformations,
    (2) definition of $z$, $\ztilde$ and $z+\ztilde$, respectively, and 
    (3) condition \eqref{align_SelfLoopTime_Condition2},
    we obtain
    \begin{align*}
        \im(z)
        \overset{\text{(1)}}{=}&~
        \{ 
        z(\epsilon) 
        ~|~ 
        \epsilon \in (\min\{0,\delta t_{uv}\},\max\{0,\delta t_{uv}\})
        \}
        \\
        \overset{\text{(2)}}{=}&~
        \bigg \{ 
        \int_{t - \delta t_{uv}}^{t - \epsilon} \nu_{uv}(s) ~ds
        ~\Big|~ 
        \epsilon \in (\min\{0,\delta t_{uv}\},\max\{0,\delta t_{uv}\})
        \bigg \}
        \\
        \overset{\text{(1)}}{=}&~
        \bigg \{ 
        \int_{t - \delta t_{uv}}^{t - \epsilon} \nu_{uv}(s) ~ds
        ~\Big|~ 
        t - \epsilon \in (\min\{t,t-\delta t_{uv}\},\max\{t,t-\delta t_{uv}\})
        \bigg \}
        \\
        \overset{\text{(3)}}{\subset}&~
        (0,l_{uv}).
    \end{align*}

    \noindent By additionally using that therefore,
    (4) $z(\epsilon) \in (0,l_{uv})$ holds for all $\epsilon \in (\min\{0,\delta t_{uv}\},$ $\max\{0,\delta t_{uv}\})$,
    we obtain
    \begin{align*}
        \im(\ztilde)
        \overset{\text{(1)}}{=}&~
        \{ 
        \ztilde(\epsilon) 
        ~|~ 
        \epsilon \in (\min\{0,\delta t_{uv}\},\max\{0,\delta t_{uv}\})
        \\
        \overset{\text{(2)}}{=}&~
        \bigg \{ 
        \frac{z(\epsilon)}{2}
        ~\Big|~ 
        \epsilon \in (\min\{0,\delta t_{uv}\},\max\{0,\delta t_{uv}\})
        \bigg \}
        \\
        \overset{\text{(4)}}{\subset}&~
        (0,\textstyle \frac{1}{2} l_{uv})
        \quad \text{ and}
        \\
        \im(z+\ztilde)
        \overset{\text{(1)}}{=}&~
        \{ 
        z(\epsilon) + \ztilde(\epsilon) 
        ~|~ 
        \epsilon \in (\min\{0,\delta t_{uv}\},\max\{0,\delta t_{uv}\})
        \}
        \\
        \overset{\text{(2)}}{=}&~
        \bigg \{ 
        z(\epsilon) + \frac{z(\epsilon)}{2}
        ~\Big|~ 
        \epsilon \in (\min\{0,\delta t_{uv}\},\max\{0,\delta t_{uv}\})
        \bigg \}
        \\
        \overset{\text{(1)}}{=}&~
        \bigg \{ 
        \frac{3}{2} z(\epsilon)
        ~\Big|~ 
        \epsilon \in (\min\{0,\delta t_{uv}\},\max\{0,\delta t_{uv}\})
        \bigg \}
        \\
        \overset{\text{(4)}}{\subset}&~
        (0,\textstyle \frac{3}{2} l_{uv}).
    \end{align*}

    \noindent
    \textit{3.:}
    By
    (1) continuity of $z$, $\ztilde$ and $z+\ztilde$, respectively,
    (2) definition of $z$, $\ztilde$ and $z+\ztilde$, respectively and
    (3) condition \eqref{align_SelfLoopTime_Condition1},
    we obtain
    \begin{align*}
        \lim_{\epsilon \rightarrow 0} 
        z(\epsilon) 
        \overset{\text{(1)}}{=}&~
        z(0) 
        \overset{\text{(2)}}{=}~
        \int_{t - \delta t_{uv}}^{t} \nu_{uv}(s) ~ds
        \overset{\text{(3)}}{=}~
        0,
        \\
        \lim_{\epsilon \rightarrow 0} 
        \ztilde(\epsilon) 
        \overset{\text{(1)}}{=}&~
        \ztilde(0) 
        \overset{\text{(2)}}{=}~
        \frac{z(0)}{2}
        \overset{\text{(2,3)}}{=}~
        0,
        \\
        \lim_{\epsilon \rightarrow 0} 
        z(\epsilon) + \ztilde(\epsilon) 
        \overset{\text{(1)}}{=}&~
        z(0) + \ztilde(0) 
        \overset{\text{(2,3)}}{=}~
        0.
    \end{align*}

    \noindent
    \textit{Step 2:}
    \textit{$c_{uv}(t,0) = c_{uv}(t - \delta t_{uv},0)$ holds.}

    For each $\epsilon \in (\min\{0,\delta t_{uv}\},\max\{0,\delta t_{uv}\})$, we can now choose $t_0 = t - \delta t_{uv} \in \R$ and -- as shown in \textit{Step 1} -- $z_0 = \ztilde(\epsilon) \in (0,l_{uv})$.
    Consequently, for each $\epsilon \in (\min\{0,\delta t_{uv}\},\max\{0,\delta t_{uv}\})$, by
    (1) the choices of $t_0 = t - \delta t_{uv} \in \R$ and $z_0 = \ztilde(\epsilon) \in (0,l_{uv})$,
    (2) the definition of $z$ and $\ztilde$ from \textit{Step 1} and
    (3) the results from \textit{Step 1},
    we observe that
    \begin{align*}
        z_0 + \int_{t_0}^{t-\epsilon} \nu_{uv}(s) ~ds
        \overset{\text{(1)}}{=}~
        \ztilde(\epsilon) + \int_{t - \delta t_{uv}}^{t-\epsilon} \nu_{uv}(s) ~ds
        \overset{\text{(2)}}{=}~
        \ztilde(\epsilon) + z(\epsilon)
        \overset{\text{(3)}}{\in}~
        (0,l_{uv}),
    \end{align*}

    \noindent and therefore, $t-\epsilon \in T_{t_0,z_0} = T_{t - \delta t_{uv},\ztilde(\epsilon)}$ holds.\footnote{
        Note that all the technical work of constructing $z$ and $\ztilde$ in \textit{Step 1} is because we would actually like to choose $z_0 = 0$ and $t \in \R$ as given; however, this $z_0$ and this $t$ do not satisfy the requirements of Theorem \ref{theorem_AdvectiveTransportAlongEdges}.
    }
    
    Therefore, by
    (1) the results from \textit{Step 1},
    (2) the continuity of the function $c_{uv} \in C^1(\R \times [0,l_{uv}],\R_{\geq 0})$,
    (3) definition of $z$ and
    (4) Theorem \ref{theorem_AdvectiveTransportAlongEdges} applied to $t_0 = t - \delta t_{uv}$, $z_0 = \ztilde(\epsilon) \in (0,l_{uv})$ and $t - \epsilon \in T_{t_0,z_0} = T_{t - \delta t_{uv},\ztilde(\epsilon)}$,
    we obtain
    \begin{align*}
        c_{uv}(t,0)
        \overset{\text{(1)}}{=}&~
        c_{uv} \left(
        \lim_{\epsilon \rightarrow 0}
        t - \epsilon
        ,
        \lim_{\epsilon \rightarrow 0}
        z(\epsilon) + \ztilde(\epsilon)
        \right)
        \\
        \overset{\text{(2)}}{=}&~
        \lim_{\epsilon \rightarrow 0}
        c_{uv} \left(
        t - \epsilon
        ,
        z(\epsilon) + \ztilde(\epsilon)
        \right)
        \\
        \overset{\text{(3)}}{=}&~
        \lim_{\epsilon \rightarrow 0}
        c_{uv} \bigg(
        t - \epsilon
        ~,~
        \underbrace{
        \underbrace{
        \ztilde(\epsilon)
        }_{
        \overset{\text{(1)}}{\in} (0,l_{uv})
        }
        +
        \underbrace{
        \int_{t - \delta t_{uv}}^{t - \epsilon} \nu_{uv}(s) ~ds
        }_{
        = z(\epsilon) \overset{\text{(1)}}{\in} (0,l_{uv})
        }
        }_{
        = z(\epsilon) + \ztilde(\epsilon) \overset{\text{(1)}}{\in} (0,l_{uv})
        }
        \bigg)
        \\
        \overset{\text{(4)}}{=}&~
        \lim_{\epsilon \rightarrow 0}
        c_{uv} \left(
        t - \delta t_{uv}
        ,
        \ztilde(\epsilon)
        \right)
        \\
        \overset{\text{(2)}}{=}&~
        c_{uv} \left(
        t - \delta t_{uv}
        ,
        \lim_{\epsilon \rightarrow 0}
        \ztilde(\epsilon)
        \right)
        \\
        \overset{\text{(1)}}{=}&~
        c_{uv}(t - \delta t_{uv},0).
    \end{align*}

    \noindent 
    \textit{Step 3:}
    \textit{
    If $\delta t_{uv} > 0$ holds, 
    $c_{vu}(t,l_{vu}) = c_u(t - \delta t_{uv})$ holds.
    If $\delta t_{uv} < 0$ holds,
    $c_{vu}(t - \delta t_{uv},l_{vu}) = c_u(t)$ holds.
    }

    
    If $\delta t_{uv} > 0$ holds, by 
    (1) Equation \eqref{align_SymmetryOfSignals} for $z = l_{vu}$,
    (2) \textit{Step 2} and
    (3) Equation \eqref{align_CouplingCondition_InflowContinuity} together with the fact that by Lemma \ref{lemma_SelfLoopTime}.4, $q_{uv}(t - \delta t_{uv}) \geq 0$ holds,
    we obtain
    \begin{align*}
        c_{vu}(t,l_{vu})
        \overset{\text{(1)}}{=}
        c_{uv}(t,0) 
        \overset{\text{(2)}}{=}
        c_{uv}(t - \delta t_{uv},0)
        \overset{\text{(3)}}{=}
        c_u(t - \delta t_{uv}).
    \end{align*}

    \noindent If instead, $\delta t_{uv} < 0$ holds,
    by 
    (1) Equation \eqref{align_SymmetryOfSignals} for $z = l_{vu}$,
    (2) \textit{Step 2} and
    (3) Equation \eqref{align_CouplingCondition_InflowContinuity} together with the fact that by Lemma \ref{lemma_SelfLoopTime}.4, $q_{uv}(t) \geq 0$ holds,
    we obtain
    \begin{align*}
        c_{vu}(t - \delta t_{uv},l_{vu})
        \overset{\text{(1)}}{=}
        c_{uv}(t - \delta t_{uv},0)
        \overset{\text{(2)}}{=}
        c_{uv}(t,0) 
        \overset{\text{(3)}}{=}
        c_u(t).
    \end{align*}
    %
    %
\end{proof}

\subsection{Proof of Theorem \ref{theorem_ExistenceAndUniquenessOfPositiveTransportTimes}}

\begin{theoremNoNb}[Theorem \ref{theorem_ExistenceAndUniquenessOfPositiveTransportTimes}]
    Let $e_{uv} \in E$ and $t \in \R$ hold.
    If $\nu_{uv}(t) \neq 0$ holds and the set
    \begin{align*}
        A_{uv}(t) :=& 
        \bigg \{
        \delta t \in (0,\infty) 
        ~ \big |~
        \int_{t - \delta t}^{t} \nu_{uv}(s) ~ds
        \in 
        \{-l_{uv},0,l_{uv}\}
        \bigg \}
        \\
        =&
        \bigg \{
        \delta t \in (0,\infty) 
        ~ \big |~
        \zcurl(\delta t)
        \in 
        \{-l_{uv},0,l_{uv}\}
        \bigg \}
    \end{align*}

    \noindent is not empty, the following properties hold:

    \begin{enumerate}
        \item $\delta t_{uv} 
        := 
        \min A_{uv}(t) 
        = 
        \min \{\delta t \in (0,\infty) ~ \big |~ \zcurl(\delta t) \in \{-l_{uv},0,l_{uv}\}\} \in (0,\infty)$ exists.

        \item $\delta t_{uv}$ is a positive pass-through time, inverse pass-through time, self-loop time or inverse self-loop time.
        More specifically, 
        $\delta t_{uv}$ is (always \gls*{wrt} $e_{uv}$ and $t$) ...
        \begin{align*}
            & \text{... an inverse pass-through time}
            && \iff
            \zcurl(\delta t_{uv}) = -l_{uv},
            \\
            & \text{... a self-loop time} 
            \text{ or }
            \text{an inverse self-loop time}
            && \iff
            \zcurl(\delta t_{uv}) = 0,
            \\
            & \text{... a pass-through time}
            && \iff
            \zcurl(\delta t_{uv}) = l_{uv}.
        \end{align*}

        \item There exists no other positive pass-through time, inverse pass-through time, self-loop time or inverse self-loop time than $\delta t_{uv}$.
    \end{enumerate}
\end{theoremNoNb}

\begin{proof}
    By the fundamental theorem of calculus, the function $t^{'} \mapsto \int_{t}^{t^{'}} \nu_{uv}(s) ~ds$ is differentiable and therefore, continuous.
    Therefore, as a composition of continuous functions, the function $\zcurl$ is continuous.
    \\
    \\
    \noindent
    \textit{1:}
    By assumption, $A_{uv}(t) \neq \emptyset$ holds. 
    Moreover, by definition, the value $0 \in \R$ is a lower bound of $A_{uv}(t) \subset (0,\infty)$. 
    Therefore, it suffices to show that $A_{uv}(t)$ is closed in $\R$; 
    in this case, the minimum of $A_{uv}(t)$ exists.

    To do so, let 
    $(\delta t_n)_{n \in \N}$ be a sequence in $A_{uv}(t)$ 
    such that 
    $\delta t_n \rightarrow \delta t_0$ holds
    for $n \rightarrow \infty$. 
    We need to show that $\delta t_0 \in A_{uv}(t)$ holds.
    \\
    \\\textit{Step 1.1:}
    \textit{$\zcurl(\delta t_0) \in \{-l_{uv},0,l_{uv}\}$ holds.}

    On the one hand,
    by 
    the continuity of $\zcurl$, 
    we obtain that 
    $\zcurl(\delta t_n) \rightarrow \zcurl(\delta t_0)$ holds
    for $n \rightarrow \infty$.
    One the other hand, 
    since
    $(\delta t_n)_{n \in \N}$ is a sequence in $A_{uv}(t)$,
    $(\zcurl(\delta t_n))_{n \in \N}$ is a sequence in $\{-l_{uv},0,l_{uv}\}$.
    Consequently, since 
    $\{-l_{uv},0,l_{uv}\}$ is closed in $\R$,
    the limit $\zcurl(\delta t_0)$ is also an element of $\{-l_{uv},0,l_{uv}\}$.
    \\
    \\\textit{Step 1.2:}
    \textit{$\delta t_0 > 0$ holds.}

    Since
    $(\delta t_n)_{n \in \N}$ is a sequence in $A_{uv}(t)$,
    $\delta t_n > 0$ holds for a $n \in \N$.
    Consequently, the limit
    $\delta t_0 = \lim_{n \rightarrow \infty} \delta t_n$ 
    satisfies
    $\delta t_0 \geq 0$,
    and it suffices to show that 
    $\delta t_0 \neq 0$
    holds.

    Let us assume that $\delta t_0 = 0$ holds.
    By 
    the continuity of $\nu_{uv}$
    and
    the assumption that $\nu_{uv}(t) \neq 0$ holds,
    there exists an $\epsilon > 0$ 
    such that
    $\nu_{uv}(s) > 0$ or $\nu_{uv}(s) < 0$ holds
    for all $s \in (t-\epsilon,t+\epsilon)$.
    \gls*{Wlog}, we can assume that $\epsilon < l_{uv}$ holds 
    (else, we choose a smaller $\epsilon$ that naturally satisfies this condition).
    Consequently,
    (1) by basic results from analysis,
    \begin{align}
    \label{align_InProof_zcurlNotZero}
        \zcurl(\delta t)
        =
        \int_{t - \delta t}^{t} 
        \underbrace{
        \nu_{uv}(s) 
        }_{
        > 0 \text{ or } <0
        }
        ~ds
        \overset{\text{(1)}}{\neq}
        0
    \end{align}

    \noindent holds for all $\delta t \in (0,\epsilon)$.
    Specifically, 
    since 
    $\delta t_n \rightarrow \delta t_0 = 0$ holds
    for $n \rightarrow \infty$,
    for the $\epsilon > 0$ from above,
    there exists an $n_0 \in \N$
    such that for all $n \geq n_0$,
    $|\delta t_n| < \epsilon$ holds;
    and since $(\delta t_n)_{n \in \N}$ is a sequence in $A_{uv}(t)$,
    even more,
    $\delta t_n \in (0,\epsilon)$ holds
    for all $n \geq n_0$.
    Consequently,
    for all $n \geq n_0$,
    Equation \eqref{align_InProof_zcurlNotZero} applies to $\delta t_n$;
    and since again $(\delta t_n)_{n \in \N}$ is a sequence in $A_{uv}(t)$, 
    even more,
    \begin{align}
    \label{align_InProof_zcurlIsPMLength}
        \zcurl(\delta t_n) \in \{-l_{uv},l_{uv}\}
    \end{align}

    \noindent 
    holds for all $n \geq n_0$.

    At the same time, 
    by the continuity of $\zcurl$, 
    we obtain that 
    for the $\epsilon > 0$ from above,
    there exists an $n_1 \in \N$
    such 
    \begin{align}
    \label{align_InProof_zcurlIsContinuous}
        |\zcurl(\delta t_n) - \zcurl(\delta t_0)|
        <
        \epsilon
    \end{align}

    \noindent
    holds for all $n \geq n_1$.
    
    Bringing these findings together,
    by
    (1) basic transformations,
    (2) Equation \eqref{align_InProof_zcurlIsPMLength}
    (3) the definition of $\zcurl$ (cf. Equation \eqref{align_TransportTime_Functionscurlzcurl}),
    (4) the assumption that $\delta t_0 = 0$ holds,
    (5) basic results from analysis,
    (6) Equation \eqref{align_InProof_zcurlIsContinuous} and
    (7) the choice of $\epsilon$,
    we obtain
    \begin{align*}
        l_{uv}
        \overset{\text{(1)}}{=}~
        |\pm l_{uv} - 0|
        \overset{\text{(2,3,4,5)}}{=}~
        |\zcurl(\delta t_n) - \zcurl(\delta t_0)|
        \overset{\text{(6)}}{<}~
        \epsilon
        \overset{\text{(7)}}{<}~
        l_{uv}
    \end{align*}
    
    for all $n \geq \max\{n_0,n_1\}$
    -- a contradiction.
    \\
    \\\textit{2:}
    For $\delta t_{uv} := \min A_{uv}(t)$, we show an even stronger claim (namely Lemma \ref{lemma_TransportTimes}.1) in several steps.
    \\
    \\\textit{Step 2.1:}
    \textit{$\im \zcurl_{|(0,\delta t_{uv})} \subset (-l_{uv},0) \sqcup (0,l_{uv})$ holds.}

    Let us assume that there exists a
    $\delta t \in (0,\delta t_{uv})$
    such that
    $\zcurl(\delta t) \in (-\infty,-l_{uv}] \sqcup \{0\} \sqcup [l_{uv},\infty)$
    holds.
    \\\textit{Case 1:}
    If $\zcurl(\delta t) \in \{-l_{uv},0,l_{uv}\}$ holds,
    we obtain $\delta t \in A_{uv}(t)$ with $\delta t < \delta t_{uv}$
    -- a contradiction to the choice of $\delta t_{uv} = \min A_{uv}(t)$ as the minimum of $A_{uv}(t)$.
    \\\textit{Case 2:}
    If $\zcurl(\delta t) \in (-\infty,-l_{uv})$ holds,
    by 
    the fact that $\zcurl(0) = 0$ holds and
    the intermediate value theorem,
    there exists a
    $\delta t^{'} \in (0,\delta t)$,
    such that
    $\zcurl(\delta t^{'}) = -l_{uv}$ holds.
    Consequently, 
    we obtain $\delta t^{'} \in A_{uv}(t)$  with $\delta t^{'} < \delta t_{uv}$
    -- a contradiction to the choice of $\delta t_{uv} = \min A_{uv}(t)$ as the minimum of $A_{uv}(t)$.
    \\\textit{Case 3:}
    If $\zcurl(\delta t) \in (l_{uv},\infty)$ holds,
    by 
    the fact that $\zcurl(0) = 0$ holds and
    the intermediate value theorem,
    there exists a 
    $\delta t^{'} \in (0,\delta t)$,
    such that
    $\zcurl(\delta t^{'}) = l_{uv}$ holds.
    Consequently, 
    we obtain $\delta t^{'} \in A_{uv}(t)$  with $\delta t^{'} < \delta t_{uv}$
    -- a contradiction to the choice of $\delta t_{uv} = \min A_{uv}(t)$ as the minimum of $A_{uv}(t)$.
    \\
    \\\textit{Step 2.2:}
    \textit{Either $\im z_{|(0,\delta t_{uv})} \subset (-l_{uv},0)$ or $\im z_{|(0,\delta t_{uv})} \subset (0,l_{uv})$ holds.}

    By step \textit{Step 2.1}, 
    $\im z_{|(0,\delta t_{uv})} \subset (-l_{uv},0) \sqcup (0,l_{uv})$ holds.
    Let us assume that there exist
    $\delta t_1, \delta t_2 \in (0,\delta t_{uv})$
    such that
    $\zcurl(\delta t_1) \in (-l_{uv},0)$ and $\zcurl(\delta t_2) \in (0,l_{uv})$
    holds.
    Again by the mean value theorem,
    there exists a
    $\delta t \in (\delta t_1,\delta_2)$,
    such that
    $\zcurl(\delta t) = 0$ holds.
    Consequently, 
    we obtain $\delta t \in A_{uv}(t)$  with $\delta t < \delta t_2 < \delta t_{uv}$
    -- a contradiction to the choice of $\delta t_{uv} = \min A_{uv}(t)$ as the minimum of $A_{uv}(t)$.
    \\
    \\\textit{Step 2.3:}
    \textit{
    If additionally, $\zcurl(\delta t_{uv}) = -l_{uv}$ holds,
    then $\im z_{|(0,\delta t_{uv})} \subset (-l_{uv},0)$ holds.
    If in contrast, additionally, $\zcurl(\delta t_{uv}) = l_{uv}$ holds,
    then $\im z_{|(0,\delta t_{uv})} \subset (0,l_{uv})$ holds.
    }

    By step \textit{Step 2.2},
    either $\im z_{|(0,\delta t_{uv})} \subset (-l_{uv},0)$ or $\im z_{|(0,\delta t_{uv})} \subset (0,l_{uv})$ holds.
    If additionally, 
    $\zcurl(\delta t_{uv}) = -l_{uv}$ holds,
    let us assume that there exists a
    $\delta t \in (0,\delta t_{uv})$
    such that
    $\zcurl(\delta t)  \in (0,l_{uv})$ 
    holds.
    Again by the intermediate value theorem,
    there exists a
    $\delta t^{'} \in (\delta t,\delta t_{uv})$,
    such that
    $\zcurl(\delta t^{'}) = 0$ holds.
    Consequently, 
    we obtain $\delta t^{'} \in A_{uv}(t)$  with $\delta t^{'} < \delta t_{uv}$
    -- a contradiction to the choice of $\delta t_{uv} = \min A_{uv}(t)$ as the minimum of $A_{uv}(t)$.

    If in contrast,
    $\zcurl(\delta t_{uv}) = l_{uv}$ holds,
    let us assume that there exists a
    $\delta t \in (0,\delta t_{uv})$
    such that
    $\zcurl(\delta t)  \in (-l_{uv},0)$ 
    holds.
    Again by the intermediate value theorem,
    there exists a
    $\delta t^{'} \in (\delta t,\delta t_{uv})$,
    such that
    $\zcurl(\delta t^{'}) = 0$ holds.
    Consequently, 
    we obtain $\delta t^{'} \in A_{uv}(t)$  with $\delta t^{'} < \delta t_{uv}$
    -- a contradiction to the choice of $\delta t_{uv} = \min A_{uv}(t)$ as the minimum of $A_{uv}(t)$.
    \\
    \\\textit{Step 2.4:}
    \textit{For $\delta t_{uv} := \min A_{uv}(t)$, we obtain}
    \begin{align*}
        \begin{cases}
            \im z_{|(0,\delta t_{uv})} \subset (-l_{uv},0) 
            & \text{ if }
            \zcurl(\delta t_{uv}) = -l_{uv},
            \\
            \im z_{|(0,\delta t_{uv})} \subset (-l_{uv},0) 
            \text{ or }
            \im z_{|(0,\delta t_{uv})} \subset (0,l_{uv})
            & \text{ if }
            \zcurl(\delta t_{uv}) = 0,
            \\
            \im z_{|(0,\delta t_{uv})} \subset (0,l_{uv})
            & \text{ if }
            \zcurl(\delta t_{uv}) = l_{uv}. 
        \end{cases}
    \end{align*}

    \noindent 
    The claim follows immediately from \textit{Step 2.2} and \textit{Step 2.3}.
    \\
    \\\textit{Step 2.5:}
    \textit{$\delta t_{uv}$ is (always \gls*{wrt} $e_{uv}$ and $t$) ...}
        \begin{align*}
            & \text{... an inverse pass-through time}
            && \iff
            \zcurl(\delta t_{uv}) = -l_{uv},
            \\
            & \text{... a self-loop time} 
            \text{ or }
            \text{an inverse self-loop time}
            && \iff
            \zcurl(\delta t_{uv}) = 0,
            \\
            & \text{... a pass-through time}
            && \iff
            \zcurl(\delta t_{uv}) = l_{uv}. 
        \end{align*}

    \noindent 
    \enquote{$\Rightarrow$}
    If $\delta t_{uv}$ is an inverse pass-through time, by 
    (1) definition of $\zcurl$ and
    (2) condition \eqref{align_PassThroughTime_Condition1_Inverse}, 
    we obtain 
    \begin{align*}
        \zcurl(\delta t_{uv})
        \overset{\text{(1)}}{=}
        \int_{t - \delta t_{uv}}^{t}
        \nu_{uv}(s) ~ds
        \overset{\text{(2)}}{=}
        - l_{uv}.
    \end{align*}
    
    \noindent Analogously, 
    if $\delta t_{uv}$ is a self-loop time or an inverse self-loop time,
    we use condition \eqref{align_SelfLoopTime_Condition1} or \eqref{align_SelfLoopTime_Condition1_Inverse}
    to conclude that $\zcurl(\delta t_{uv}) = 0$ holds.
    If $\delta t_{uv}$ is a pass-through time, 
    we use condition \eqref{align_PassThroughTime_Condition1} 
    to conclude that $\zcurl(\delta t_{uv}) = l_{uv}$ holds.
    \\
    \\
    \enquote{$\Leftarrow$}
    If $\zcurl(\delta t_{uv}) = -l_{uv}$ holds, by 
    (1) definition of $\zcurl$ and
    (2) this assumption,
    we obtain
    \begin{align*}
        \int_{t - \delta t_{uv}}^{t}
        \nu_{uv}(s) ~ds
        \overset{\text{(1)}}{=}
        \zcurl(\delta t_{uv})
        \overset{\text{(2)}}{=}
        - l_{uv}.
    \end{align*}

    \noindent that is, condition \eqref{align_PassThroughTime_Condition1_Inverse} is satisfied.
    
    Consequently, for any  $t^{'} \in (\min\{t - \delta t_{uv},t\},\max\{t - \delta t_{uv},t\}) = (t - \delta t_{uv},t)$,
    or equivalently, for any $t^{'} = t - \delta t$ with $\delta t \in (0,\delta t_{uv})$,
    by 
    (1) by basic results from analysis,
    (2) condition \eqref{align_PassThroughTime_Condition1_Inverse},
    (3) definition of $\zcurl$ and
    (4) \textit{Step 2.4},
    we obtain
    
    \begin{align*}
        \int_{t - \delta t_{uv}}^{t^{'}}
        \nu_{uv}(s) ~ds
        \overset{\text{(1)}}{=}
        \underbrace{
        \int_{t - \delta t_{uv}}^{t}
        \nu_{uv}(s) ~ds
        }_{
        \overset{\text{(2)}}{=}
        -l_{uv}
        }
        -
        \underbrace{
        \int_{t - \delta t}^{t}
        \nu_{uv}(s) ~ds
        }_{
        \overset{\text{(3)}}{=}
        \zcurl(\delta t)
        \overset{\text{(4)}}{\in}
        (-l_{uv},0)
        }
        \overset{\text{(1)}}{\in}
        (-l_{uv},0),
    \end{align*}

    \noindent that is, condition \eqref{align_PassThroughTime_Condition2_Inverse} is satisfied, too, and thus, $\delta t_{uv}$ is an inverse pass-through time.

    Analogously,
    if $\zcurl(\delta t_{uv}) = 0$ holds,
    we use this assumption to show that
    condition \eqref{align_SelfLoopTime_Condition1}, which is equal to condition \eqref{align_SelfLoopTime_Condition1_Inverse}, is satisfied.
    If $\zcurl(\delta t_{uv}) = l_{uv}$ holds,
    we use this assumption to show that
    condition \eqref{align_PassThroughTime_Condition1} is satisfied.
    Consequently, we use these conditions and \textit{Step 2.4} to show that
    condition \eqref{align_SelfLoopTime_Condition2}, \eqref{align_SelfLoopTime_Condition2_Inverse} and \eqref{align_PassThroughTime_Condition2} are satisfied in alignment with the conditions in \textit{Step 2.4}.
    We leave the details as an exercise to the reader, however, a summary of the cases to be considered is given as follows:
    \begin{align*}
        \underbrace{
        \int_{t - \delta t_{uv}}^{t}
        \nu_{uv}(s) ~ds
        }_{
        \scriptsize
        \begin{array}{l}
            = ~ -l_{uv}
            \\
            = ~ 0 
            \\
            = ~ 0 
            \\
            = ~ l_{uv}
        \end{array}
        }
        -
        \underbrace{
        \int_{t - \delta t}^{t}
        \nu_{uv}(s) ~ds
        }_{
        \scriptsize
        \begin{array}{l}
            \in ~ (-l_{uv},0)
            \\
            \in ~ (-l_{uv},0)
            \\
            \in ~ (0,l_{uv})
            \\
            \in ~ (0,l_{uv})
        \end{array}
        }
        \overset{\text{(1)}}{=}
        \underbrace{
        \int_{t - \delta t_{uv}}^{t^{'} = t - \delta t}
        \nu_{uv}(s) ~ds
        }_{
        \scriptsize
        \begin{array}{l}
            \in ~ (-l_{uv},0) 
            \\
            \in ~ (0,l_{uv})
            \\
            \in ~ (-l_{uv},0)
            \\
            \in ~ (0,l_{uv})
        \end{array}
        }
        {\color{white}
        \underbrace{
        \int_{t - \delta t_{uv}}^{t^{'} = t - \delta t}
        \nu_{uv}(s) ~ds
        }_{
        \scriptsize
        {\color{black}
        \begin{array}{l}
            \text{ if } 
            \zcurl(\delta t_{uv}) = -l_{uv}
            \\
            \text{ if } 
            \zcurl(\delta t_{uv}) = 0
            \text{ and }
            \im z_{|(0,\delta t_{uv})} \subset (-l_{uv},0)
            \\
            \text{ if } 
            \zcurl(\delta t_{uv}) = 0
            \text{ and }
            \im z_{|(0,\delta t_{uv})} \subset (0,l_{uv})
            \\
            \text{ if } 
            \zcurl(\delta t_{uv}) = l_{uv}
        \end{array}
        }
        }
        }
    \end{align*}

    \noindent
    \textit{3.:}
    Similar to the proof of Lemma \ref{lemma_PassThroughTime}.2 and Lemma \ref{lemma_SelfLoopTime}.1, let us assume that there exist two different and positive times $\delta t_{1uv}, \delta t_{2uv} \in \R_{\neq 0}$ which are each either a pass-through time, an inverse pass-through time, a self-loop time or an inverse self-loop time.
    \gls*{Wlog}, let $0 < \delta t_{1uv} < \delta t_{2uv}$ hold.
    Consequently, 
    for $t^{'} = t - \delta t_{1uv} \in (t - \delta t_{2uv},t) = (\min\{t - \delta t_{2uv},t\},\max\{t - \delta t_{2uv},t\})$,
    by
    (1) by basic results from analysis and
    (2) one or two of the conditions \eqref{align_PassThroughTime_Condition1}, \eqref{align_PassThroughTime_Condition1_Inverse}, \eqref{align_SelfLoopTime_Condition1} or \eqref{align_SelfLoopTime_Condition1_Inverse}\footnote{
        Note that we have $3 \cdot 3 = 9$ possible combinations:
        If $\delta t_{2uv}$ is a pass-through time, we work with condition \eqref{align_PassThroughTime_Condition1}.
        If $\delta t_{2uv}$ is an inverse pass-through time, we work with condition \eqref{align_PassThroughTime_Condition1_Inverse}.
        If $\delta t_{2uv}$ is a(n) (inverse) self-loop time, we work with condition \eqref{align_SelfLoopTime_Condition1}, which is equal to condition \eqref{align_SelfLoopTime_Condition1_Inverse}.
        Each of these three options need to be combined with the three analogous options for the second time $\delta t_{1uv}$.
    },
    we obtain
    \begin{align*}
        \int_{t - \delta t_{2uv}}^{t^{'}} \nu_{uv}(s) ~ds
        \overset{\text{(1)}}{=}
        \int_{t - \delta t_{2uv}}^{t} \nu_{uv}(s) ~ds
        -
        \int_{t - \delta t_{1uv}}^{t} \nu_{uv}(s) ~ds
        \overset{\text{(2)}}{\in}
        \{-2 l_{uv},-l_{uv},0,l_{uv},2 l_{uv}\}
    \end{align*}

    \noindent -- a contradiction to one of the conditions \eqref{align_PassThroughTime_Condition2}, \eqref{align_PassThroughTime_Condition2_Inverse}, \eqref{align_SelfLoopTime_Condition2} or \eqref{align_SelfLoopTime_Condition2_Inverse}\footnote{
        Note that we have four possible combinations:
        If $\delta t_{2uv}$ is a pass-through time, we work with condition \eqref{align_PassThroughTime_Condition2}.
        If $\delta t_{2uv}$ is an inverse pass-through time, we work with condition \eqref{align_PassThroughTime_Condition2_Inverse}.
        If $\delta t_{2uv}$ is a self-loop time, we work with condition \eqref{align_SelfLoopTime_Condition2}.
        If $\delta t_{2uv}$ is an inverse self-loop time, we work with condition \eqref{align_SelfLoopTime_Condition2_Inverse}.
    }.
\end{proof}

\subsection{Proof of Lemma \ref{lemma_TransportTimes}}

\begin{lemmaNoNb}[Lemma \ref{lemma_TransportTimes}]
    Let $e_{uv} \in E$ and $t \in \R$ hold.
    In the setting of Theorem \ref{theorem_ExistenceAndUniquenessOfPositiveTransportTimes},
    the following property holds:

    \begin{enumerate}
        \item $\delta t_{uv}$ is  (always \gls*{wrt} $e_{uv}$ and $t$) ...
        \begin{align*}
            & \text{... a pass-through time}
            && \iff
            \zcurl(\delta t_{uv}) \in \{-l_{uv},l_{uv}\}
            \text{ and }
            q_{uv}(t) > 0,
            \\
            & \text{... an inverse pass-through time}
            && \iff
            \zcurl(\delta t_{uv}) \in \{-l_{uv},l_{uv}\}
            \text{ and }
            q_{uv}(t) < 0,
            \\
            & \text{... a self-loop time} 
            && \iff
            \zcurl(\delta t_{uv}) = 0
            \text{ and }
            q_{uv}(t) < 0,
            \\
            &\text{... an inverse self-loop time} 
            && \iff
            \zcurl(\delta t_{uv}) = 0
            \text{ and }
            q_{uv}(t) > 0.
        \end{align*}
        
    \end{enumerate}
\end{lemmaNoNb}

\begin{proof}
    \enquote{$\Rightarrow$}
    If $\delta t_{uv}$ is a pass-through time, 
    similarly to \textit{Step 2.5} in the proof of Theorem \ref{theorem_ExistenceAndUniquenessOfPositiveTransportTimes}.2,
    by 
    (1) definition of $\zcurl$ and
    (2) condition \eqref{align_PassThroughTime_Condition1}, 
    we obtain 
    \begin{align*}
        \zcurl(\delta t_{uv})
        \overset{\text{(1)}}{=}
        \int_{t - \delta t_{uv}}^{t}
        \nu_{uv}(s) ~ds
        \overset{\text{(2)}}{=}
        l_{uv}
        \in 
        \{-l_{uv},l_{uv}\}.
    \end{align*}

    \noindent Moreover, by Lemma \ref{lemma_PassThroughTime} and the choice of $\delta t_{uv} > 0$ , $q_{uv}(t) \geq 0$ holds.
    Since $q_{uv}(t) \neq 0$ holds by assumption, $q_{uv}(t) > 0$ follows.

    Similarly, 
    if $\delta t_{uv}$ is an inverse pass-through time (\gls*{wrt} $e_{uv}$ and $t$),
    by
    (1) definition of $\zcurl$ and
    (2) condition \eqref{align_PassThroughTime_Condition1_Inverse},
    we obtain
    \begin{align*}
        \zcurl(\delta t_{uv})
        \overset{\text{(1)}}{=}
        \int_{t - \delta t_{uv}}^{t}
        \nu_{uv}(s) ~ds
        \overset{\text{(2)}}{=}
        - l_{uv}
        \in 
        \{-l_{uv},l_{uv}\}.
    \end{align*}

    \noindent Additionally, we use Lemma \ref{lemma_InversePassThroughTime} to conclude that 
    $\delta t_{uv}$ is a pass-through time 
    \gls*{wrt} $e_{vu}$ and $t$.
    Therefore, by
    (1) the conservation of flows (Eq. \eqref{align_ConservationOfFlows}),
    (2) Lemma \ref{lemma_PassThroughTime} and the choice of $\delta t_{uv} > 0$ considered as a pass-through time \gls*{wrt} $e_{vu}$ and $t$,
    we obtain
    \begin{align*}
        q_{uv}(t)
        \overset{\text{(1)}}{=}
        - 
        \underbrace{
        q_{vu}(t)
        }_{
        \overset{\text{(2)}}{\geq}
        0
        }
        \leq
        0.
    \end{align*}

    \noindent  Since $q_{uv}(t) \neq 0$ holds by assumption, $q_{uv}(t) < 0$ follows.

    Analogously, 
    if $\delta t_{uv}$ is a self-loop time,
    we use condition \eqref{align_SelfLoopTime_Condition1} to conclude that $\zcurl(\delta t_{uv}) = 0$ holds and
    Lemma \ref{lemma_SelfLoopTime} to conclude that $q_{uv}(t) \leq 0$ and thus, $q_{uv}(t) < 0$ holds.
    If $\delta t_{uv}$ is an inverse self-loop time,
    we use condition \eqref{align_SelfLoopTime_Condition1_Inverse} to conclude that $\zcurl(\delta t_{uv}) = 0$ holds and
    Lemma \ref{lemma_InverseSelfLoopTime}, the conservation of flows and Lemma \ref{lemma_SelfLoopTime} to conclude that $q_{uv}(t) = - q_{vu}(t) \geq 0$ and thus, $q_{uv}(t) > 0$ holds.
    \\
    \\
    \enquote{$\Leftarrow$}
    If $\zcurl(\delta t_{uv}) \in \{-l_{uv},l_{uv}\}$ holds, 
    we know by Theorem \ref{theorem_ExistenceAndUniquenessOfPositiveTransportTimes}.2 that 
    $\delta t_{uv}$ is either a pass-through time or an inverse pass-through time.

    If additionally, 
    $q_{uv}(t) > 0$ holds,
    let us assume that $\delta t_{uv}$ is an inverse pass-through time.
    Consequently, by the same arguments as in the inclusion above,
    we can conclude that in this case,
    $q_{uv}(t) = - q_{vu}(t) \leq 0$ holds
    -- a contradiction.

    If instead, additionally,
    $q_{uv}(t) < 0$ holds,
    let us assume that $\delta t_{uv}$ is a pass-through time.
    Consequently, by the same arguments as in the inclusion above, 
    we can conclude that in this case,
    $q_{uv}(t) \geq 0$ holds
    -- a contradiction.
    \\
    \\
    If $\zcurl(\delta t_{uv}) = 0$ holds, 
    we know by Theorem \ref{theorem_ExistenceAndUniquenessOfPositiveTransportTimes}.2 that 
    $\delta t_{uv}$ is either a self-loop time or an inverse self-loop time.

    If additionally, 
    $q_{uv}(t) < 0$ holds,
    let us assume that $\delta t_{uv}$ an inverse self-loop time.
    Consequently, by the same arguments as in the inclusion above, 
    we can conclude that in this case,
    $q_{uv}(t) \geq 0$ holds
    -- a contradiction.

    If instead, additionally,
    $q_{uv}(t) > 0$ holds,
    let us assume that $\delta t_{uv}$ is a self-loop time.
    Consequently, by the same arguments as in the inclusion above, 
    we can conclude that in this case,
    $q_{uv}(t) \leq 0$ holds
    -- a contradiction.
\end{proof}

\subsection{Proof of Theorem \ref{theorem_PhysicalAlignment} and \ref{theorem_PhysicalAlignment_fullVersion}}

\begin{theoremNoNb}[Theorem \ref{theorem_PhysicalAlignment} and \ref{theorem_PhysicalAlignment_fullVersion}]
    Let $v \in V$ and $t \in \R$ hold.
    We assume that Assumption \ref{assumption_SmoothnessAssumptions} holds and that for each $u \in \Nbh(v)$, 
    $A_{uv}(t)$ is not empty 
    (that is, the \gls*{OP} \eqref{align_TransportTime_OP} has a solution).
    If the function $c_{v}$ obeys the mixing at nodes (Eq. \eqref{align_CouplingCondition_MixingAtNodes}) and 
    if for each $u \in \Nbh(v)$, the function $c_{uv}$ obeys the advection edge dynamics (Eq. \eqref{align_EdgeDynamics_Advection}) with inflow continuity condition (Eq. \eqref{align_CouplingCondition_InflowContinuity}), 
    then
    \begin{align*}
        c_v(t) 
        =~
        \sum_{ 
        \substack{
        u \in \Nbh_-(v,t) \\ 
        \zcurl(\delta t_{uv}) = \pm l_{uv}
        }}
        w_{uv}(t) \cdot c_u(t - \delta t_{uv})
        +
        \sum_{ 
        \substack{
        u \in \Nbh_-(v,t) \\ 
        \zcurl(\delta t_{uv}) = 0
        }}
        w_{uv}(t) \cdot c_v(t - \delta t_{uv})
        =~
        \sum_{u \in \Nbh(v)}
        w_{uv}(t) \cdot \varphi(t,e_{uv},\cm,\num)
    \end{align*}

    \noindent
    holds, where $\delta t_{uv} = \min A_{uv}(t)$ is the transport time as defined in Theorem \ref{theorem_ExistenceAndUniquenessOfPositiveTransportTimes}.
    In other words, Equation \eqref{align_MessagePassing_Advection} with $\phi$ as defined in Equation \eqref{align_MessageFunction_Advection} holds.
\end{theoremNoNb}

\begin{proof}
    \textbf{Well-definedness:}
    Since by assumption, for each $u \in \Nbh(v)$, the set $A_{uv}(t)$ as defined in Theorem \ref{theorem_ExistenceAndUniquenessOfPositiveTransportTimes} is not empty, 
    according to Theorem \ref{theorem_ExistenceAndUniquenessOfPositiveTransportTimes}.1, a transport time $\delta t_{uv} = \min A_{uv}(t) \in (0,\infty)$ exists for each $u \in \Nbh(v)$.
    Even more, by Theorem \ref{theorem_ExistenceAndUniquenessOfPositiveTransportTimes}.3, is it uniquely determined by the condition that it is positive, allowing to define the map $\Nbh(v) \rightarrow (0,\infty), u \mapsto \delta t_{uv}$ that associates a unique transport time with each neighbor $u \in \Nbh(v)$.
    Especially to mention, for each $u \in \Nbh(v)$ and by definition of the transport time $\delta t_{uv}$ (Theorem \ref{theorem_ExistenceAndUniquenessOfPositiveTransportTimes}), $\zcurl(\delta t_{uv}) \in \{-l_{uv},0,l_{uv}\}$ holds.
    \\
    \\
    \textbf{Equality:}
    By
    (1) definition of the in-, out- and no-flow neighborhoods and
    (2) the conservation of flows (Eq. \eqref{align_ConservationOfFlows}),
    we obtain
    \begin{align*}
        \Nbh_\pm(v,t) 
        :\overset{\text{(1)}}{=}&~
        \{ u \in \Nbh(v) ~|~ \sgn(q_{vu}(t)) = \pm 1 \}
        \overset{\text{(2)}}{=}
        \{ u \in \Nbh(v) ~|~ \sgn(q_{uv}(t)) = \mp 1 \}
        \text{ and}
        \\
        \Nbh_0(v) 
        :\overset{\text{(1)}}{=}&~
        \{ u \in \Nbh(v) ~|~ \sgn(q_{vu}(t)) = {\color{white} \pm} 0 \}
        \overset{\text{(2)}}{=}
        \{ u \in \Nbh(v) ~|~ \sgn(q_{uv})(t) = {\color{white} \pm} 0 \}
        .
    \end{align*}

    \noindent Clearly, these three sub-neighborhoods define a partition of the neighborhood $\Nbh(v)$, that is, they are pairwise disjoint and $\Nbh(v) = \Nbh_-(v,t) \sqcup \Nbh_+(v,t) \sqcup \Nbh_0(v,t)$ holds. 
    Together with the definition of the $\relu$ function $\relu(q) = \max\{0,q\}$ for all $q \in \R$, we observe that for each $u \in \Nbh(v)$,
    \begin{align}
    \label{align_InProof_FlowInDependenceOfNeighborhoods}
        \relu(q_{uv}(t))
        =
        \begin{cases}
            q_{uv}(t)
            & \text{ if } u \in \Nbh_-(v,t)
            \\
            0
            & \text{ if } u \in \Nbh_+(v,t)
            \\
            0
            & \text{ if } u \in \Nbh_0(v,t)
        \end{cases}
    \end{align}

    \noindent 
    holds.
    Consequently,
    by
    (1) Equation \eqref{align_CouplingCondition_MixingAtNodes},
    (2) Equation \eqref{align_InProof_FlowInDependenceOfNeighborhoods},
    (3) the definition of the weights $w_{uv}(t)$ (Eq. \eqref{align_CouplingCondition_MixingAtNodes_MessagePassingFormulation}),
    (4) the well-definedness investigations from above and
    (5) basic transformations,
    we obtain
    \begin{align*}
        c_v(t)
        \overset{\text(1)}{=}&~
        \frac{
        \sum_{u \in \Nbh_-(v,t)}
        q_{uv}(t) \cdot c_{uv}(t,l_{uv})
        }{
        \sum_{u \in \Nbh_-(v,t)} 
        q_{uv}(t)
        }
        \\
        \overset{\text(2)}{=}&~
        \frac{
        \sum_{u \in \Nbh(v)}
        \relu(q_{uv}(t)) \cdot c_{uv}(t,l_{uv})
        }{
        \sum_{u \in \Nbh(v)} 
        \relu(q_{uv}(t))
        }
        \\
        \overset{\text(3)}{=}&~
        \sum_{u \in \Nbh(v)}
        w_{uv}(t) \cdot c_{uv}(t,l_{uv})
        \\
        \overset{\text(3)}{=}&~
        \sum_{u \in \Nbh_-(v,t)}
        w_{uv}(t) \cdot c_{uv}(t,l_{uv})
        \\
        \overset{\text(4,5)}{=}&~
        \sum_{ \substack{
        u \in \Nbh_-(v,t) \\
        \zcurl(\delta t_{uv}) = \pm l_{uv}
        }}
        w_{uv}(t) \cdot c_{uv}(t,l_{uv})
        +
        \sum_{ \substack{
        u \in \Nbh_-(v,t) \\
        \zcurl(\delta t_{uv}) = 0
        }}
        w_{uv}(t) \cdot c_{uv}(t,l_{uv}).
    \end{align*}

    \noindent
    This proves the first equality.
    Since the definition of $\Nbh_-(v,t)$ indicates that $q_{uv}(t) > 0$ is positive, 
    by Lemma \ref{lemma_TransportTimes}.1, 
    the transport times $\delta t_{uv} > 0$ in the first sum are pass-through times, while
    the transport times $\delta t_{uv} > 0$ in the second sum are inverse self-loop times.
    Since again by Lemma \ref{lemma_TransportTimes}.1, 
    the other direction also holds, we can conclude the calculations as follows:
    By
    (1) applying Theorem \ref{theorem_PassThroughTime} to $c_{uv}(t,l_{uv})$ in the first sum and Lemma \ref{lemma_InversePassThroughTime} and Theorem \ref{theorem_SelfLoopTime} to $c_{uv}(t,l_{uv})$ the second sum,
    (2) the definition of $\phi$ (Equation \eqref{align_MessageFunction_Advection}),
    (3) basic transformations,
    (4) the definition of the transport time $\delta t_{uv}$ (Theorem \ref{theorem_ExistenceAndUniquenessOfPositiveTransportTimes}) and again, 
    (5) the definition of the weights $w_{uv}(t)$ (Eq. \eqref{align_CouplingCondition_MixingAtNodes_MessagePassingFormulation}),
     we obtain
    \begin{align*}
        c_v(t)
        \overset{\text(1)}{=}&~
        \sum_{ \substack{
        u \in \Nbh_-(v,t) \\
        \zcurl(\delta t_{uv}) = \pm l_{uv}
        }}
        w_{uv}(t) \cdot c_u(t-\delta t_{uv})
        +
        \sum_{ \substack{
        u \in \Nbh_-(v,t) \\
        \zcurl(\delta t_{uv}) = 0
        }}
        w_{uv}(t) \cdot c_v(t-\delta t_{uv})
        \\
        \overset{\text(2)}{=}&~
        \sum_{ \substack{
        u \in \Nbh_-(v,t) \\
        \zcurl(\delta t_{uv}) = \pm l_{uv}
        }}
        w_{uv}(t) \cdot \phi(t,e_{uv},\cm,\num)
        +
        \sum_{ \substack{
        u \in \Nbh_-(v,t) \\
        \zcurl(\delta t_{uv}) = 0
        }}
        w_{uv}(t) \cdot \phi(t,e_{uv},\cm,\num)
        \\
        \overset{\text(3)}{=}&~
        \sum_{ \substack{
        u \in \Nbh_-(v,t) \\
        \zcurl(\delta t_{uv}) \in \{-l_{uv},0,l_{uv}\}
        }}
        w_{uv}(t) \cdot \phi(t,e_{uv},\cm,\num)
        \\
        \overset{\text(4)}{=}&~
        \sum_{u \in \Nbh_-(v,t)}
        w_{uv}(t) \cdot \phi(t,e_{uv},\cm,\num)
        \\
        \overset{\text(5)}{=}&~
        \sum_{u \in \Nbh(v)}
        w_{uv}(t) \cdot \phi(t,e_{uv},\cm,\num).
    \end{align*}

    \noindent
    This proves the second equality.
\end{proof}

\subsection{Proof of Theorem \ref{theorem_TreesWithChildrenAggregationScheme}}

\begin{theoremNoNb}[Theorem \ref{theorem_TreesWithChildrenAggregationScheme}]
    Let $\T = (A,E)$ be a finite tree with root node $a_0 \in A$, \textbf{int}ernal nodes $\An{int}$ and \textbf{leaf} nodes $\An{leaf}$ (that is, $A = \{a_0\} \sqcup \An{int} \sqcup \An{leaf}$ holds).
    If each node $a \in A$ is associated with a magnitude $x_a \in \R$ that satisfies
    \begin{align}
    \label{align_ChildrenAggregationScheme}
            x_a
            =
            \sum_{b \in \Children(a)}
            \omega_{ab} \cdot x_b
            \text{ for all } 
            a \in \{a_0\} \sqcup \An{int},
    \end{align}

    \noindent 
    where the set $\Children(a) \subset A$ consists of the children of $a \in A$ and $w_{ab}$ is a weight of the edge $e_{ab} \in E$ that depends on both the parent $a \in A$ and its child $b \in \Children(a)$, 
    then the magnitude $x_0 := x_{a_0}$ of the root node $a_0 \in A$ is determined by the magnitudes $x_{l_p} := x_{a_{l_p}}$ of the leaf nodes $a_{l_p} \in \An{leaf}$ by
    \begin{align}
    \label{align_ChildrenAggregationScheme_closedForm}
        x_0
        =
        \sum_{p = (a_0,...,a_{l_p}) \in \Pa(a_0,\An{leaf})}
        \omega(a_0,...,a_{l_p})
        \cdot
        x_{l_p},
    \end{align}

    \noindent 
    where 
    $\Pa(a_0,\An{leaf})$ be the set of paths connecting the root node $a_0$ to one of its leaf nodes $\An{leaf}$ 
    and
    \begin{align}
    \label{align_PathWeights}
        \omega(a_0,...,a_{l_p})
        :=
        \prod_{l = 0}^{l_p-1}
        \omega_{a_la_{l+1}}
    \end{align}

    \noindent 
    is the product of weights along a(n) (arbitrary) path $p = (a_0,...,a_{l_p}) \in \Pa$ with length $l_p \in \N$.
\end{theoremNoNb}

\begin{proof}
    For the sake of brevity, in this proof, we use the abbreviation $\Pa = \Pa(a_0,\An{leaf})$.
    \\
    \\
    \textit{Step 1:} 
    \textit{For each $l \in \N$, we obtain}
    \begin{align}
    \label{align_ChildrenAggregationScheme_closedForm_fixedPathLength}
        x_{a_0}
        =
        \sum_{ \substack{p = (a_0,...,a_{l_p}) \in \Pa: \\ l_p \leq l} }
        \omega(a_0,...,a_{l_p})
        \cdot
        x_{a_{l_p}}
        +
        \sum_{ \substack{p = (a_0,...,a_{l_p}) \in \Pa: \\ l_p > l} }
        \omega(a_0,...,a_l)
        \cdot
        x_{a_l}.
    \end{align}

    \noindent 
    We prove the claim by induction to $l \in \N$.
    \\
    \textit{Induction base:}
    If $l = 1$, 
    by
    (1) Equation \eqref{align_ChildrenAggregationScheme} applied to $x_{a_0}$,
    (2) a change in notation,
    (3) basic transformations,
    (4) the definition of $\omega$ (cf. \eqref{align_PathWeights}) and
    (5) the choice of $l = 1$,
    we obtain
    \begin{align*}
        x_{a_0}
        \overset{\text{(1)}}{=}&~
        \sum_{b \in \Children(a_0)}
        \omega_{a_0b} \cdot x_b
        \\
        \overset{\text{(2,3)}}{=}&~
        \sum_{a_1 \in \Children(a_0) \cap \An{leaf}}
        \omega_{a_0a_1} \cdot x_{a_1}
        +
        \sum_{a_1 \in \Children(a_0) \cap \An{int}}
        \omega_{a_0a_1} \cdot x_{a_1}
        \\
        \overset{\text{(3)}}{=}&~
        \sum_{ \substack{p = (a_0,a_1) \in \Pa: \\ l_p = 1} }
        \omega_{a_0a_1} \cdot x_{a_1}
        +
        \sum_{ \substack{p = (a_0,a_1,...,a_{l_p}) \in \Pa: \\ l_p > 1} }
        w_{a_0a_1} \cdot x_{a_1}
        \\
        \overset{\text{(4)}}{=}&~
        \sum_{ \substack{p = (a_0,a_1) \in \Pa: \\ l_p = 1} }
        \omega(a_0,a_1) \cdot x_{a_1}
        +
        \sum_{ \substack{p = (a_0,...,a_{l_p}) \in \Pa: \\ l_p > 1} }
        \omega(a_0,a_1) \cdot x_{a_l}
        \\
        \overset{\text{(3,5)}}{=}&~
        \sum_{ \substack{p = (a_0,...,a_{l_p}) \in \Pa: \\ l_p \leq l} }
        \omega(a_0,...,a_{l_p}) \cdot x_{a_{l_p}}
        +
        \sum_{ \substack{p = (a_0,...,a_{l_p}) \in \Pa: \\ l_p > l} }
        \omega(a_0,...,a_l) \cdot x_{a_l}.
    \end{align*}

    \noindent
    \textit{Induction hypothesis:}
    Equation \eqref{align_ChildrenAggregationScheme_closedForm_fixedPathLength} holds for all $\tilde{l} \leq l \in \N$.
    
    \textit{Induction step:}
    By
    (1) the \textit{induction hypothesis},
    we obtain
    \begin{align*}
        x_{a_0}
        \overset{\text{(1)}}{=}&~
        \sum_{ \substack{p = (a_0,...,a_{l_p}) \in \Pa: \\ l_p \leq l} }
        \omega(a_0,...,a_{l_p})
        \cdot
        x_{a_{l_p}}
        +
        \sum_{ \substack{p = (a_0,...,a_{l_p}) \in \Pa: \\ l_p > l} }
        \omega(a_0,...,a_l)
        \cdot
        x_{a_l}.
    \end{align*}

    \noindent
    Similar to the induction base, 
    by
    (1) Equation \eqref{align_ChildrenAggregationScheme} applied to $x_{a_l}$, (note that since $l < l_p$, $a_l \in \An{int}$ must hold),
    (2) a change in notation,
    (3) basic transformations and
    (4) the definition of $\omega$ (cf. \eqref{align_PathWeights}),
    we transform the second summand to
    \begin{align*}
        &~
        \sum_{ \substack{p = (a_0,...,a_{l_p}) \in \Pa: \\ l_p > l} }
        \omega(a_0,...,a_l)
        \cdot
        x_{a_l}
        \\
        \overset{\text{(1)}}{=}&~
        \sum_{ \substack{p = (a_0,...,a_{l_p}) \in \Pa: \\ l_p > l} }
        \omega(a_0,...,a_l)
        \cdot
        \left(
        \sum_{b \in \Children(a_l)}
        \omega_{a_lb} \cdot x_b
        \right)
        \\
        \overset{\text{(2,3)}}{=}&~
        \sum_{ \substack{p = (a_0,...,a_{l_p}) \in \Pa: \\ l_p \geq l+1} }
        \omega(a_0,...,a_l)
        \cdot
        \left(
        \sum_{a_{l+1} \in \Children(a_l) \cap \An{leaf}}
        \omega_{a_la_{l+1}} \cdot x_{a_{l+1}}
        +
        \sum_{a_{l+1} \in \Children(a_l) \cap \An{in}}
        \omega_{a_la_{l+1}} \cdot x_{a_{l+1}}
        \right)
        \\
        \overset{\text{(3)}}{=}&~
        \sum_{ \substack{p = (a_0,...,a_{l_p}) \in \Pa: \\ l_p \geq l+1} } ~~~
        \sum_{a_{l+1} \in \Children(a_l) \cap \An{leaf}}
        \omega(a_0,...,a_l)
        \cdot
        \omega_{a_la_{l+1}} 
        \cdot 
        x_{a_{l+1}}
        \\
        +&~
        \sum_{ \substack{p = (a_0,...,a_{l_p}) \in \Pa: \\ l_p \geq l+1} } ~~~
        \sum_{a_{l+1} \in \Children(a_l) \cap \An{in}}
        \omega(a_0,...,a_l)
        \cdot
        \omega_{a_la_{l+1}}
        \cdot
        x_{a_{l+1}}
        \\
        \overset{\text{(3)}}{=}&~
        \sum_{ \substack{p = (a_0,...,a_{l_p}) \in \Pa: \\ l_p = l+1} } ~~~
        \omega(a_0,...,a_l)
        \cdot
        \omega_{a_la_{l+1}} 
        \cdot 
        x_{a_{l+1}}
        \\
        +&~
        \sum_{ \substack{p = (a_0,...,a_{l_p}) \in \Pa: \\ l_p > l+1} } ~~~
        \omega(a_0,...,a_l)
        \cdot
        \omega_{a_la_{l+1}}
        \cdot
        x_{a_{l+1}}
        \\
        \overset{\text{(4)}}{=}&~
        \sum_{ \substack{p = (a_0,...,a_{l_p}) \in \Pa: \\ l_p = l+1} } ~~~
        \omega(a_0,...,a_l,a_{l+1})
        \cdot 
        x_{a_{l+1}}
        ~+~
        \sum_{ \substack{p = (a_0,...,a_{l_p}) \in \Pa: \\ l_p > l+1} } ~~~
        \omega(a_0,...,a_l,a_{l+1})
        \cdot
        x_{a_{l+1}}.
    \end{align*}

    \noindent
    Bringing it all together, again by
    (1) basic transformations,
    we obtain
    \begin{align*}
        x_{a_0}
        =&~
        \sum_{ \substack{p = (a_0,...,a_{l_p}) \in \Pa: \\ l_p \leq l} }
        \omega(a_0,...,a_{l_p})
        \cdot
        x_{a_{l_p}}
        +
        \sum_{ \substack{p = (a_0,...,a_{l_p}) \in \Pa: \\ l_p > l} }
        \omega(a_0,...,a_l)
        \cdot
        x_{a_l}
        \\
        =&~
        \sum_{ \substack{p = (a_0,...,a_{l_p}) \in \Pa: \\ l_p \leq l} }
        \omega(a_0,...,a_{l_p})
        \cdot
        x_{a_{l_p}}
        \\
        +&~
        \sum_{ \substack{p = (a_0,...,a_{l_p}) \in \Pa: \\ l_p = l+1} } ~~~
        \omega(a_0,...,a_l,a_{l+1})
        \cdot 
        x_{a_{l+1}}
        \\
        +&~
        \sum_{ \substack{p = (a_0,...,a_{l_p}) \in \Pa: \\ l_p > l+1} } ~~~
        \omega(a_0,...,a_l,a_{l+1})
        \cdot
        x_{a_{l+1}}
        \\
        \overset{\text{(1)}}{=}&~
        \sum_{ \substack{p = (a_0,...,a_{l_p}) \in \Pa: \\ l_p < l+1} }
        \omega(a_0,...,a_{l_p})
        \cdot
        x_{a_{l_p}}
        \\
        +&~
        \sum_{ \substack{p = (a_0,...,a_{l_p}) \in \Pa: \\ l_p = l+1} } ~~~
        \omega(a_0,...,a_{l_p})
        \cdot 
        x_{a_{l_p}}
        \\
        +&~
        \sum_{ \substack{p = (a_0,...,a_{l_p}) \in \Pa: \\ l_p > l+1} } ~~~
        \omega(a_0,...,a_{l+1})
        \cdot
        x_{a_{l+1}}
        \\
        \overset{\text{(1)}}{=}&~
        \sum_{ \substack{p = (a_0,...,a_{l_p}) \in \Pa: \\ l_p \leq l+1} } ~~~
        \omega(a_0,...,a_{l_p})
        \cdot 
        x_{a_{l_p}}
        ~~+~~
        \sum_{ \substack{p = (a_0,...,a_{l_p}) \in \Pa: \\ l_p > l+1} } ~~~
        \omega(a_0,...,a_{l+1})
        \cdot
        x_{a_{l+1}},
    \end{align*}

    \noindent
    which concludes the induction step.
    \\
    \\
    \textit{Step 2:}
    \textit{Equation \eqref{align_ChildrenAggregationScheme_closedForm} holds.}

    Since $\T$ is finite, that is, $\T$ an acyclic graph with a finite number of nodes $|A| < \infty$, also each path $p \in \Pa$ has to be finite. 
    Therefore, we can choose $l = \max_{p \in \Pa} l_p \in \N$. 
    Consequently, the second sum in \textit{Step 1} is empty and the condition in the first sum in \textit{Step 1} is naturally satisfied. This yields Equation \eqref{align_ChildrenAggregationScheme_closedForm}.
\end{proof}

\subsection{Proof of Theorem \ref{theorem_AdvectionMessagePassing_closedForm}}

\begin{theoremNoNb}[Theorem \ref{theorem_AdvectionMessagePassing_closedForm}]
    Let $v \in V$ and $t \in \R$ hold.
    Let $\T(v,t)$ be the inflow tree of $v$ at time $t$ as defined in Definition \ref{definition_InflowTreeAndInterpolationInflowTree}.
    If $c_v(t)$ is defined by Equation \eqref{align_MessagePassing_Advection}, then
    \begin{align*}
        c_v(t)
        =~
        \sum_{ p = ((v_0,t_0)=(v,t),(u_1,v_1,t_1),...,(u_{l_p},v_{l_p},t_{l_p})) \in \Pa(v,t) } ~
        \left(
        \prod_{l = 0}^{l_p-1}  
        w_{u_{l+1}v_l}(t_l)
        \right)
        \cdot
        c_{v_{l_p}}(t_{l_p})
    \end{align*}

    \noindent 
    with weights $w_{u_{l+1}v_l}(t_l)$ according to Equation \eqref{align_CouplingCondition_MixingAtNodes_MessagePassingFormulation} holds.
\end{theoremNoNb}

\begin{proof}
    For each layer $l \in \N_0$ of the tree, 
    we now associate each node $a_l = (u_l,v_l,t_l) \in A$ of that layer with the magnitude $x_l := x_{a_l} := c_{v_l}(t_l)$.
    \\
    \\\textit{Step 1:}
    \textit{Together with these magnitudes, 
    $\T(v,t)$ defines a tree in terms of Theorem \ref{theorem_TreesWithChildrenAggregationScheme}.}

    In order to prove the claim, we first of all investigate on the term $\phi(t_l,e_{u_{l+1}v_l},\cm,\num)$ for a given parent node $a_l = (u_l,v_l,t_l)$ in the $l$-th layer and its child node $a_{l+1} = (u_{l+1},v_{l+1},t_{l+1}) \in \Children(a_l)$.\footnote{
        Especially to mention, by definition of the inflow tree $\T(v,t)$ (Definition \ref{definition_InflowTreeAndInterpolationInflowTree}), $u_{l+1} \in \Nbh_-(v_t,t_l)$ holds, that is, the edge $e_{u_{l+1}v_l} \in E$ exists.
    } 
    Indeed, by
    (1) Equation \eqref{align_MessageFunction_Advection} and
    (2) the definition of the inflow tree $\T(v,t)$ (Definition \ref{definition_InflowTreeAndInterpolationInflowTree}),
    for any $l \in \N_0$,
    $\phi(t_l,e_{u_{l+1}v_l},\cm,\num)$ is equal to
    \begin{align}
    \label{align_InProof_InflowTree_Magnitudes}
    \begin{split}
        &~
        \phi(t_l,e_{u_{l+1}v_l},\cm,\num)
        \\
        \overset{\text{(1)}}{=}&~
        \begin{cases}
            c_{u_{l+1}}(t_l - \scurl(t_l,e_{u_{l+1}v_l},\num))
            & \text{ if } 
            \zcurl(\scurl(t_l,e_{u_{l+1}v_l},\num)) \in \{-l_{uv},l_{uv}\},
            \\
            {\color{white}_{_{+1}}}
            c_{v_l}(t_l - \scurl(t_l,e_{u_{l+1}v_l},\num)) 
            & \text{ if } 
            \zcurl(\scurl(t_l,e_{u_{l+1}v_l},\num)) = 0
        \end{cases}
        \\
        \overset{\text{(2)}}{=}&~
        \begin{cases}
            c_{u_{l+1}}(t_{l+1})
            & \text{ if } 
            \zcurl(\scurl(t_l,e_{u_{l+1}v_l},\num)) \in \{-l_{uv},l_{uv}\},
            \\
            {\color{white}_{_{+1}}}
            c_{v_l}(t_{l+1})
            & \text{ if } 
            \zcurl(\scurl(t_l,e_{u_{l+1}v_l},\num)) = 0
        \end{cases}
        \\
        \overset{\text{(2)}}{=}&~
        c_{v_{l+1}}(t_{l+1})
        \\
        \overset{\text{(2)}}{=}&~
        x_{l+1}.
    \end{split}
    \end{align}

    Consequently, again by
    (1) the definition of the inflow tree $\T(v,t)$ (Definition \ref{definition_InflowTreeAndInterpolationInflowTree}),
    (2) Equation \eqref{align_MessagePassing_Advection} and
    (3) Equation \eqref{align_InProof_InflowTree_Magnitudes},
    for any $a_l \in \{a_0\} \sqcup \An{int}$, 
    we obtain
    \begin{align*}
        x_{a_l}
        :=&~
        x_l
        \overset{\text{(1)}}{=}~
        c_{v_l}(t_l)
        \\
        \overset{\text{(2)}}{=}&~
        \sum_{u_{l+1} \in \Nbh_-(v_l,t_l)}
        w_{u_{l+1}v_l}(t_l)
        \cdot
        \phi(t_l,e_{u_{l+1}v_l},\cm,\num)
        \\
        \overset{\text{(3)}}{=}&~
        \sum_{u_{l+1} \in \Nbh_-(v_l,t_l)}
        w_{u_{l+1}v_l}(t_l)
        \cdot
        x_{{l+1}}
        \\
        \overset{\text{(1)}}{=}&~
        \sum_{a_{l+1} \in \Children(a_l)}
        \omega_{a_la_{l+1}} 
        \cdot
        x_{a_{l+1}}
    \end{align*}

    \noindent 
    with weights 
    \begin{align}
    \label{align_InProof_InflowTree_Weights}
        \omega_{a_la_{l+1}} 
        :=
        w_{u_{l+1}v_l}(t_l).
    \end{align}

    \noindent
    Consequently, $\T(v,t)$ is a tree in terms of Theorem \ref{theorem_TreesWithChildrenAggregationScheme}.
    \\
    \\\textit{Step 2:}
    \textit{Equation \eqref{align_MessagePassing_Advection_closedForm} holds.}
    
    Again by
    (1) the definition of the inflow tree $\T(v,t)$ (Definition \ref{definition_InflowTreeAndInterpolationInflowTree}),
    (2) Theorem \ref{theorem_TreesWithChildrenAggregationScheme} (which we can apply according to \textit{Step 1}) and
    (3) Equation \eqref{align_InProof_InflowTree_Weights},
    we obtain
    \begin{align*}
        c_v(t)
        \overset{\text{(1)}}{=}&~
        c_{v_0}(t_0)
        \overset{\text{(1)}}{=}~
        x_0
        \\
        \overset{\text{(2)}}{=}&~
        \sum_{p = (a_0,...,a_{l_p}) \in \Pa(a_0,\An{leaf})}
        \omega(a_0,...,a_{l_p})
        \cdot
        x_{l_p}
        \\
        \overset{\text{(1,3)}}{=}&~
        \sum_{ p = ((v_0,t_0)=(v,t),(u_1,v_1,t_1),...,(u_{l_p},v_{l_p},t_{l_p})) \in \Pa(v,t) } ~
        \left(
        \prod_{l=0}^{l_p-1}
        w_{u_{l+1}v_l}(t_l)
        \right)
        \cdot
        c_{v_{l_p}}(t_{l_p}),
    \end{align*}

    \noindent
    which was to show. 
\end{proof}

\subsection{Proof of Theorem \ref{theorem_AdvectionMessagePassingIterations_closedForm}}

\begin{theoremNoNb}[Theorem \ref{theorem_AdvectionMessagePassingIterations_closedForm}]
    Let $v \in V$ and $t \in T = \Tn{hi} \cup \Tn{fu}$ hold.
    Let $\That(v,t)$ be the interpolation inflow tree of $v$ at time $t$ as defined in Definition \ref{definition_InflowTreeAndInterpolationInflowTree}.
    Let $k \geq \max_{p \in \Pahat(v,t)} l_p$ hold.
    If $\chat_v^{(k)}(t)$ is defined by Equation \eqref{align_MessagePassing_AdvectionIterations}, then
    \small
    \begin{align*}
        \chat_v^{(k)}(t)
        =&~
        \sum_{ p = ((v_0,t_0)=(v,t),(u_1,v_1,t_1),...,(u_{l_p},v_{l_p},t_{l_p})) \in \Pahat(v,t) } ~
        \left(
        \prod_{l = 0}^{l_p-1}  
        w_{u_{l+1}v_l}(t_l)
        \cdot
        a(t_{l+1})
        \right)
        \cdot
        \chat_{v_{l_p}}^{(k-l_p)}(t_{l_p})
    \end{align*}
    \normalsize

    \noindent 
    with weights $w_{u_{l+1}v_l}(t_l)$ and $a(t_{l+1})$ according to Equation \eqref{align_CouplingCondition_MixingAtNodes_MessagePassingFormulation} and \eqref{align_LinearInterpolation}, respectively, holds.
\end{theoremNoNb}

\begin{proof}
    Similar to the proof of Theorem \ref{theorem_AdvectionMessagePassing_closedForm},
    for each layer $l \in \N_0$ of the tree, 
    we now associate each node $a_l = (u_l,v_l,t_l) \in A$ of that layer with the magnitude $x_l := x_{a_l} := \chat_{v_l}^{(k-l)}(t_l)$.
    \\
    \\\textit{Step 1:}
    \textit{Together with these magnitudes, 
    $\That(v,t)$ defines a tree in terms of Theorem \ref{theorem_TreesWithChildrenAggregationScheme}.}

    In order to prove the claim, we first of all investigate on the term $\varphi(t_l,e_{u_{l+1}v_l},\cmhat^{(k-(l+1))},\num)$ for a given parent node $a_l = (u_l,v_l,t_l)$ in the $l$-th layer and its child node $a_{l+1} = (u_{l+1},v_{l+1},t_{l+1}) \in \Children(a_l)$.
    Indeed, by
    (1) Equation \eqref{align_MessageFunction_Advection_LinearInterpolation},
    for any $l \in \N_0$,
    $\varphi(t_l,e_{u_{l+1}v_l},\cmhat^{(k-(l+1))},\num)$ is equal to
    \begin{align*}
    \begin{split}
        &~
        \varphi(t_l,e_{u_{l+1}v_l},\cmhat^{(k-(l+1))},\num)
        \\
        \overset{\text{(1)}}{=}&~
        \begin{cases}
            a(t_j) \cdot \chat_{u_{l+1}}^{(k-(l+1))}(t_j)
            +
            a(t_{j+1}) \cdot \chat_{u_{l+1}}^{(k-(l+1))}(t_{j+1})
            & \text{ if } 
            \zcurl(\scurl(t_l,e_{u_{l+1}v_l},\num)) \in \{-l_{uv},l_{uv}\},
            \\
            a(t_j) \cdot \chat_{v_l}^{(k-(l+1))}(t_j)
            +
            a(t_{j+1}) \cdot \chat_{v_l}^{(k-(l+1))}(t_{j+1})
            & \text{ if } 
            \zcurl(\scurl(t_l,e_{u_{l+1}v_l},\num)) = 0
        \end{cases}
    \end{split}
    \end{align*}

    \noindent 
    for the $j \in I$ for which $t_l - \scurl(t_l,e_{u_{l+1}v_l},\num) \in [t_j,t_{j+1})$ holds.
    Following 
    (1) the definition of the interpolation inflow tree $\That(v,t)$ (Definition \ref{definition_InflowTreeAndInterpolationInflowTree}),
    we obtain
    \small
    \begin{align}
    \label{align_InProof_InterpolationInflowTree_Magnitudes}
    \begin{split}
        &~
        \varphi(t_l,e_{u_{l+1}v_l},\cmhat^{(k-(l+1))},\num)
        \\
        \overset{\text{(1)}}{=}&~
        \begin{cases}
            \sum_{ \substack{ t_{l+1} \in \{t_j,t_{j+1}\}: \\ t_l - \scurl(t_l,e_{u_{l+1}v_l},\num) \in [t_j,t_{j+1}) } }
            a(t_{l+1}) 
            \cdot 
            \chat_{u_{l+1}}^{(k-(l+1))}(t_{l+1})
            & \text{ if } 
            \zcurl(\scurl(t_l,e_{u_{l+1}v_l},\num)) \in \{-l_{uv},l_{uv}\},
            \\
            \sum_{ \substack{ t_{l+1} \in \{t_j,t_{j+1}\}: \\ t_l - \scurl(t_l,e_{u_{l+1}v_l},\num) \in [t_j,t_{j+1}) } }
            a(t_{l+1}) 
            \cdot 
            \chat_{v_l}^{(k-(l+1))}(t_{l+1})
            & \text{ if } 
            \zcurl(\scurl(t_l,ee_{u_{l+1}v_l},\num)) = 0
        \end{cases}
        \\
        \overset{\text{(1)}}{=}&~
        \sum_{ \substack{ t_{l+1} \in \{t_j,t_{j+1}\}: \\ t_l - \scurl(t_l,e_{u_{l+1}v_l},\num) \in [t_j,t_{j+1}) } }
        a(t_{l+1}) 
        \cdot 
        \chat_{v_{l+1}}^{(k-(l+1))}(t_{l+1})
        \\
        \overset{\text{(1)}}{=}&~
        \sum_{ \substack{ t_{l+1} \in \{t_j,t_{j+1}\}: \\ t_l - \scurl(t_l,e_{u_{l+1}v_l},\num) \in [t_j,t_{j+1}) } }
        a(t_{l+1}) 
        \cdot 
        x_{l+1}.
    \end{split}
    \end{align}
    \normalsize

    Consequently, again by
    (1) the definition of the interpolation inflow tree $\T(v,t)$ (Definition \ref{definition_InflowTreeAndInterpolationInflowTree}),
    (2) Equation \eqref{align_MessagePassing_AdvectionIterations},
    (3) Equation \eqref{align_InProof_InterpolationInflowTree_Magnitudes} and
    (4) basic transformations,
    for any $a_l \in \{a_0\} \sqcup \An{int}$, 
    we obtain
    \begin{align*}
        x_{a_l}
        :=&~
        x_l
        \overset{\text{(1)}}{=}~
        \chat_{v_l}^{(k-l)}(t_l)
        \\
        \overset{\text{(2)}}{=}&~
        \sum_{u_{l+1} \in \Nbh_-(v_l,t_l)}
        w_{u_{l+1}v_l}(t_l)
        \cdot
        \varphi(t_l,e_{v_{l+1}v_l},\cmhat^{(k-(l+1))},\num)
        \\
        \overset{\text{(3)}}{=}&~
        \sum_{u_{l+1} \in \Nbh_-(v_l,t_l)}
        w_{u_{l+1}v_l}(t_l)
        \cdot
        \Bigg(
        \sum_{ \substack{ t_{l+1} \in \{t_j,t_{j+1}\}: \\ t_l - \scurl(t_l,e_{u_{l+1}v_l},\num) \in [t_j,t_{j+1}) } }
        a(t_{l+1}) 
        \cdot 
        x_{l+1}
        \Bigg)
        \\
        \overset{\text{(4)}}{=}&~
        \sum_{u_{l+1} \in \Nbh_-(v_l,t_l)}
        \sum_{ \substack{ t_{l+1} \in \{t_j,t_{j+1}\}: \\ t_l - \scurl(t_l,e_{(l+1)l},\num) \in [t_j,t_{j+1}) } }
        w_{u_{l+1}v_l}(t_l)
        \cdot
        a(t_{l+1}) 
        \cdot 
        x_{l+1}
        \\
        \overset{\text{(1)}}{=}&~
        \sum_{a_{l+1} \in \Children(a_l)}
        \omega_{a_la_{l+1}}
        \cdot
        x_{a_{l+1}}
    \end{align*}

    \noindent 
    with weights 
    \begin{align}
    \label{align_InProof_InterpolationInflowTree_Weights}
        \omega_{a_la_{l+1}} 
        :=
        w_{u_{l+1}v_l}(t_l)
        \cdot
        a(t_{l+1}).
    \end{align}

    \noindent
    Consequently, $\That(v,t)$ is a tree in terms of Theorem \ref{theorem_TreesWithChildrenAggregationScheme}.
    \\
    \\\textit{Step 2:}
    \textit{Equation \eqref{align_MessagePassing_AdvectionIterations_closedForm} holds.}

    Again by
    (1) the definition of the interpolation inflow tree $\That(v,t)$ (Definition \ref{definition_InflowTreeAndInterpolationInflowTree}),
    (2) Theorem \ref{theorem_TreesWithChildrenAggregationScheme} (which we can apply according to \textit{Step 1}) and
    (3) Equation \eqref{align_InProof_InterpolationInflowTree_Weights},
    we obtain
    
    \begin{align*}
        \chat_v^{(k)}(t)
        \overset{\text{(1)}}{=}&~
        \chat_{v_0}^{(k-0)}(t_0)
        \overset{\text{(1)}}{=}~
        x_{a_0}
        \\
        \overset{\text{(2)}}{=}&~
        \sum_{p = (a_0,...,a_{l_p}) \in \Pa(a_0,\An{leaf})}
        \omega(a_0,...,a_{l_p})
        \cdot
        x_{l_p}
        \\
        \overset{\text{(1,3)}}{=}&~
        \sum_{p = ((v_0,t_0)=(v,t),(u_1,v_1,t_1),...,(u_{l_p},v_{l_p},t_{l_p})) \in \Pahat(v,t)}
        \left(
        \prod_{l=0}^{l_p-1}
        w_{u_{l+1}v_l}(t_l)
        \cdot 
        a(t_{l+1})
        \right)
        \cdot
        \chat_{v_{l_p}}^{(k-l_p)}(t_{l_p}),
    \end{align*}

    \noindent
    which was to show. 
\end{proof}

\subsection{Proof of Theorem \ref{theorem_ConvergenceOfModelIterations}}

\begin{theoremNoNb}[Theorem \ref{theorem_ConvergenceOfModelIterations}]
    There exists a $\kappa \in \N$, such that $\chat_v^{(k)}(t)$ as defined in Equation \eqref{align_MessagePassing_AdvectionIterations} is equal to zero for all $k > \kappa$, $v \in V$ and $t \in T = \Tn{hi} \cup \Tn{fu}$.
\end{theoremNoNb}

\begin{proof}
    Let $v \in V = \Vb \sqcup (V \setminus \Vb)$ and $t \in \Tn{hi} \sqcup \Tn{fu}$ hold.
    We first identify a node- and time-dependent $\kappa(v,t)$.
    \\
    \\\textit{Case 1:}
    If $v \in \Vb$ or $t \in \Tn{hi}$ holds, 
    by Equation \eqref{align_MessagePassing_AdvectionIterations},
    $\chat_v^{(k)}(t)$ is 
    equal to zero if $k > 0$ and
    equal to $c_v^{(k)}(t)$ if $k = 0$.
    Therefore, in this case, $\kappa(v,t) = 0$.
    \\
    \\\textit{Case 2.1:}
    If $v \in V \setminus \Vb$, $t \in \Tn{fu}$ and $k \geq \max_{p \in \Pahat(v,t)} l_p$ holds, 
    by Theorem \ref{theorem_AdvectionMessagePassingIterations_closedForm}, 
    we can express $\chat_v^{(k)}(t)$ as a function of values
    $\chat_{v_{l_p}}^{(k-l_p)}(t_{l_p})$ for paths $p = ((v_0,t_0)=(v,t),(u_1,v_1,t_1),...,(u_{l_p},v_{l_p},t_{l_p})) \in \Pahat(v,t)$.
    By definition of $\Pahat(v,t)$ (Definition \ref{definition_InflowTreeAndInterpolationInflowTree}), 
    $v_{l_p} \in \Vb$ or $t_{l_p} \in \Tn{hi}$ holds
    (and even more, $v_l \in V \setminus \Vb$ and $t_l \in \Tn{fu}$ holds for all $l = 0,...,l_p-1$).
    Therefore, again
    by Equation \eqref{align_MessagePassing_AdvectionIterations},
    for each path $p \in \Pahat(v,t)$ with length $l_p \leq k$, 
    $\chat_{v_{l_p}}^{(k-l_p)}(t_{l_p})$ is 
    equal to zero if $k > l_p$ and 
    equal to $c_{v_{l_p}}(t_{l_p})$ if $k = l_p$.
    \\
    \\\textit{Case 2.2:}
    If $v \in V \setminus \Vb$, $t \in \Tn{fu}$ and $k < \max_{p \in \Pahat(v,t)} l_p$ holds, 
    we can express $\chat_v^{(k)}(t)$ as a function of values
    $\chat_{v_{l_p}}^{(k-l_p)}(t_{l_p})$ for paths $p = ((v_0,t_0)=(v,t),(u_1,v_1,t_1),...,(u_{l_p},v_{l_p},t_{l_p})) \in \Pahat(v,t)$ 
    which satisfy $l_p \leq k$, and values
    $\chat_{v_l}^{(k-l)}(t_l) = \chat_{v_{l}}^{(0)}(t_l)$ for $l = k$ for paths $p = ((v_0,t_0)=(v,t),(u_1,v_1,t_1),...,(u_l,v_l,t_l),...,(u_{l_p},v_{l_p},t_{l_p})) \in \Pahat(v,t)$
    which satisfy $k < l_p$.
    For the former paths, 
    the same findings as in \textit{Case 2.1} hold.
    For the latter paths, 
    by definition of $\Pahat(v,t)$ (Definition \ref{definition_InflowTreeAndInterpolationInflowTree}), 
    $v_l \in V \setminus \Vb$ and $t_l \in \Tn{fu}$ holds.
    Therefore, again
    by Equation \eqref{align_MessagePassing_AdvectionIterations},
    for each path $p \in \Pahat(v,t)$ with length $l_p > k$,
    $\chat_{v_l}^{(k-l)}(t_l) = \chat_{v_{l}}^{(0)}(t_l)$ is 
    equal to zero.
    \\
    \\\textit{Case 2:}
    In summary, 
    if $v \in V \setminus \Vb$ and $t \in \Tn{fu}$ holds, 
    we can express $\chat_v^{(k)}(t)$ as a function of values 
    $\chat_{v_{l_p}}^{(k-l_p)}(t_{l_p})$ which are only non-zero for $k \in \{l_p ~|~ p \in \Pahat(v,t)\}$.
    Therefore, in this case, $\kappa(v,t) = \max_{p \in \Pahat(v,t)} l_p$.
    \\
    \\
    Finally, we choose $\kappa = \max_{v \in V} \max_{t \in \Tn{hi} \sqcup \Tn{fu}} \kappa(v,t) = \max_{v \in V} \max_{t \in \Tn{hi} \sqcup \Tn{fu}} \max_{p \in \Pahat(v,t)} l_p$. 
    Then the claim follows as long as $l_p \in \Pahat(v,t)$ is finite for all $v \in V$ and $t \in \Tn{hi} \sqcup \Tn{fu}$. 
    We have discussed that this is the case in Remark \ref{remark_FiniteDepthOfInterpolationInflowTrees}.
\end{proof}

\subsection{Proof of Theorem \ref{theorem_InterpolationError} and Theorem \ref{theorem_InterpolationError_fullVersion}}

\begin{theoremNoNb}[Theorem \ref{theorem_InterpolationError} and Theorem \ref{theorem_InterpolationError_fullVersion}]
    Let $v \in V$ and $t \in T = \Tn{hi} \cup \Tn{fu}$ hold.
    Let $\T(v,t)$ and $\That(v,t)$ be the inflow tree and the interpolation inflow tree of $v$ at time $t$ as defined in Definition \ref{definition_InflowTreeAndInterpolationInflowTree}, respectively.
    Let Assumption \ref{assumption_IdentifiablePaths} and \ref{assumption_EqualtTimeForJumpingIntoHistory} hold.
    Then the prediction error in-between $c_v(t)$ and $\chat_v(t)$ is given by
    \begin{align*}
        &~~~
        |c_v(t) - \chat_v(t)| 
        \\
        \leq&~
        \sum_{ \ptilde \in \Pahat(v,t) } ~
        a(\that_1,...,\that_{l_p})
        \cdot
        \left(
        \underbrace{
        |c_{v_{l_p}}(t_{l_p}) - c_{v_{l_p}}(\that_{l_p})|
        }_{
        \text{Temporal interpolation error}
        }
        +
        |c_{v_{l_p}}(t_{l_p})|
        \cdot
        \underbrace{
        |w(t_0,...,t_{l_p-1}) - w(\that_0,...,\that_{l_p-1})|
        }_{
        \text{Accumulated weighting error}
        }
        \right)
    \end{align*}

    \noindent
    with joint paths $\ptilde = ((v_0,t_0)=(v,t),(u_1,v_1,t_1,\that_1),...,(u_{l_p},v_{l_p},t_{l_p},\that_{l_p})) \in \Pahat(v,t)$ according to Definition \ref{definition_JointPaths},
    \begin{align*}
        a(\that_1,...,\that_{l_p})
        :=
        \prod_{l = 0}^{l_p-1}
        a(\that_{l+1})
        \quad \text{ and } \quad
        w(\that_0,...,\that_{l_p-1})
        :=
        \prod_{l = 0}^{l_p-1}  
        w_{u_{l+1}v_l}(\that_l)
    \end{align*}

    \noindent 
    with weights $w_{u_{l+1}v_l}(t_l)$ and $a(t_{l+1})$ according to Equation \eqref{align_CouplingCondition_MixingAtNodes_MessagePassingFormulation} and \eqref{align_LinearInterpolation}, respectively.
    If the temporal interpolation error is zero, we obtain $|c_v(t) - \chat_v(t)| = 0$.
\end{theoremNoNb}

\begin{proof}
    By Theorem \ref{theorem_AdvectionMessagePassing_closedForm} and \ref{theorem_AdvectionMessagePassingIterations_closedForm}, 
    \small
    \begin{align*}
        c_v(t)
        =&~
        \sum_{ p = ((v_0,t_0)=(v,t),(u_1,v_1,t_1),...,(u_{l_p},v_{l_p},t_{l_p})) \in \Pa(v,t) } ~
        \left(
        \prod_{l = 0}^{l_p-1}  
        w_{u_{l+1}v_l}(t_l)
        \right)
        \cdot
        c_{v_{l_p}}(t_{l_p})
        \text{ and }
        \\
        \chat_v^{(k)}(t)
        =&~
        \sum_{ p = ((v_0,t_0)=(v,t),(u_1,v_1,t_1),...,(u_{l_p},v_{l_p},t_{l_p})) \in \Pahat(v,t) } ~
        \left(
        \prod_{l = 0}^{l_p-1}  
        w_{u_{l+1}v_l}(t_l)
        \cdot
        a(t_{l+1})
        \right)
        \cdot
        \chat_{v_{l_p}}^{(k-l_p)}(t_{l_p})
    \end{align*}
    \normalsize

    \noindent 
    holds for all $k \geq \max_{p \in \Pahat(v,t)} l_p$.
    As already discussed in the proof of Theorem \ref{theorem_ConvergenceOfModelIterations}, for a $k < \max_{p \in \Pahat(v,t)} l_p$, paths $p \in \Pahat(v,t)$ with $l_p \leq k$ contribute with the same contribution $\chat_{v_{l_p}}^{(k-l_p)}(t_{l_p}) = c_{v_{l_p}}(t_{l_p})$ to $\chat_v^{(k)}(t)$ as they would for a larger $k$, while the ones with $l_p > k$ contribute with $\chat_{v_l}^{(k-l)}(t_l) = 0$ (for $l = k < l_p)$. 
    More specifically, considering a flexible $k \in \N$, each path contributes only once to $\chat_v^{(k)}(t)$, namely for $k = l_p$.
    Therefore, $\chat_{v_{l_p}}^{(k-l_p)}(t_{l_p}) = \delta_{kl_p} \cdot c_{v_{l_p}}(t_{l_p})$ holds for all $k \geq l_p$ and also gives a meaning for cases $k < l_p$.
    Consequently, by
    (1) Equation \eqref{align_MessagePassing_AdvectionIterations} and
    (2) the findings above,
    we can express $\chat_v(t)$ by
    \small
    \begin{align*}
        \chat_v(t)
        \overset{\text{(1)}}{=}&~
        \sum_{k=0}^{\infty} 
        \chat_v^{(k)}(t)
        \\
        \overset{\text{(2)}}{=}&~
        \sum_{k=0}^{\infty} ~
        \sum_{ p = ((v_0,t_0)=(v,t),(u_1,v_1,t_1),...,(u_{l_p},v_{l_p},t_{l_p})) \in \Pahat(v,t) } ~
        \left(
        \prod_{l = 0}^{l_p-1}  
        w_{u_{l+1}v_l}(t_l)
        \cdot
        a(t_{l+1})
        \right)
        \cdot
        \chat_{v_{l_p}}^{(k-l_p)}(t_{l_p})
        \\
        \\
        \overset{\text{(2)}}{=}&~
        \sum_{ p = ((v_0,t_0)=(v,t),(u_1,v_1,t_1),...,(u_{l_p},v_{l_p},t_{l_p})) \in \Pahat(v,t) } ~
        \left(
        \prod_{l = 0}^{l_p-1}  
        w_{u_{l+1}v_l}(t_l)
        \cdot
        a(t_{l+1})
        \right)
        \cdot
        c_{v_{l_p}}(t_{l_p}).
    \end{align*}
    \normalsize

    \noindent 

    Now we want to use the shared paths $\ptilde$ from Equation \eqref{align_JointPaths} to express $c_v(t)$ and $\chat_v(t)$ using the same summands. To achieve this, it is important to note that as an artifact of the interpolation, there are multiple paths $\phat \in \Pahat(v,t)$ that are identifiable with a single path $p \in \Pa(v,t)$. 
    However, due to the fact that certain interpolation weights $a(t_{l+1})$ add up to one, a single contribution $c_{v_{l_p}}(t_{l_p})$ of a path $p \in \Pa(v,t)$ to $c_v(t)$, weighted by the weight
    \begin{align*}
        \left(
        \prod_{l = 0}^{l_p-1}  
        w_{u_{l+1}v_l}(t_l)
        \right),
    \end{align*}

    \noindent
    can be distributed to the multiple contributions $c_{v_{l_p}}(t_{l_p})$ of a path $p \in \Pa(v,t)$ to $c_v(t)$, weighted by the weights
    \begin{align*}
        \left(
        \prod_{l = 0}^{l_p-1}  
        w_{u_{l+1}v_l}(t_l)
        \cdot
        a(t_{l+1})
        \right)
    \end{align*}

    \noindent 
    associated with corresponding paths $\phat \in \Pahat(v,t)$ that each are identifiable with the path $p \in \Pa(v,t)$.    
    Using the shared paths $\ptilde$ notation from Equation \eqref{align_JointPaths}, this allows to express both $c_v(t)$ and $\chat_v(t)$ as 
    \small
    \begin{align*}
        c_v(t)
        =&~
        \sum_{ \ptilde = ((v_0,t_0)=(v,t),(u_1,v_1,t_1,\that_1),...,(u_{l_p},v_{l_p},t_{l_p},\that_{l_p})) \in \Pahat(v,t) } ~
        \left(
        \prod_{l = 0}^{l_p-1}  
        w_{u_{l+1}v_l}(t_l)
        \cdot
        a(\that_{l+1})
        \right)
        \cdot
        c_{v_{l_p}}(t_{l_p})
        \text{ and }
        \\
        \chat_v^{(k)}(t)
        =&~
        \sum_{ \ptilde = ((v_0,t_0)=(v,t),(u_1,v_1,t_1,\that_1),...,(u_{l_p},v_{l_p},t_{l_p},\that_{l_p})) \in \Pahat(v,t) } ~
        \left(
        \prod_{l = 0}^{l_p-1}  
        w_{u_{l+1}v_l}(\that_l)
        \cdot
        a(\that_{l+1})
        \right)
        \cdot
        c_{v_{l_p}}(\that_{l_p}),
    \end{align*}
    \normalsize

    \noindent
    making the sums perfectly comparable.
    Therefore, by
    (1) basic transformations and triangle-inequality and
    (2) the fact that 
    $|w_1 c_1 - w_2 c_2| 
    = 
    |w_1 c_1 + w_2 c_1 - w_2 c_1 - w_2 c_2| 
    = 
    |w_2 c_1 - w_2 c_2 + c_1 w_1 - c_1 w_2| 
    \leq 
    |c_1 - c_2| + |c_1| |w_1 - w_2|$ 
    holds for any $w_1, w_2 \in [0,1]$ and $c_1, c_2 \in \R$,
    we obtain
    \begin{align*}
        & \quad
        |c_v(t) - \chat_v(t)|
        \\
        \overset{\text{(1)}}{\leq}&~
        \sum_{ \ptilde \in \Pahat(v,t) } ~
        \left(
        \prod_{l = 0}^{l_p-1}
        a(\that_{l+1})
        \right)
        \cdot
        \left|
        \left(
        \prod_{l = 0}^{l_p-1}  
        \underbrace{
        w_{u_{l+1}v_l}(t_l)
        }_{
        \in [0,1]
        }
        \right)
        \cdot
        c_{v_{l_p}}(t_{l_p})
        -
        \left(
        \prod_{l = 0}^{l_p-1}  
        \underbrace{
        w_{u_{l+1}v_l}(\that_l)
        }_{
        \in [0,1]
        }
        \right)
        \cdot
        c_{v_{l_p}}(\that_{l_p})
        \right|
        \\
        \overset{\text{(2)}}{\leq}&~
        \sum_{ \ptilde \in \Pahat(v,t) } ~
        \left(
        \prod_{l = 0}^{l_p-1}
        a(\that_{l+1})
        \right)
        \cdot
        \Bigg(
        \quad \quad
        |c_{v_{l_p}}(t_{l_p}) - c_{v_{l_p}}(\that_{l_p})|
        \\
        &
        \quad \quad \quad \quad \quad \quad \quad \quad \quad \quad \quad \quad 
        +~~~~
        |c_{v_{l_p}}(t_{l_p})|
        \cdot
        \left|
        \left(
        \prod_{l = 0}^{l_p-1}  
        w_{u_{l+1}v_l}(t_l)
        \right)
        -
        \left(
        \prod_{l = 0}^{l_p-1}  
        w_{u_{l+1}v_l}(\that_l)
        \right)
        \right|
        \Bigg).
    \end{align*}
    
    If the temporal interpolation error is zero, that is, if $\phi = \varphi$ holds, all interpolation times $\that_l$ corresponding to left interpolation bound will be equal to $t_l$ hand have interpolation weights $a(\that_l) = 1$. Interpolation times $\that_l$ corresponding to right interpolation bounds have interpolation weights $a(\that_l) = 0$ (cf. Eq. \eqref{align_LinearInterpolation}.
    Therefore, in this case, the only non-zero summands in the error bound are the ones where $t_l = \that_l$ holds for all $l = 0,...,l_p$, in which case $|c_v(t) - \chat_v(t)| = 0$ holds. 
    This proves the claim.
    

\end{proof}

\end{document}